\documentclass[10pt]{article}

\usepackage{amsmath,amsthm,verbatim,amssymb,amsfonts,amscd, graphicx, enumitem}
\usepackage{mathrsfs}
\usepackage{graphics}
\usepackage{centernot}
\usepackage{authblk}
\usepackage[T1]{fontenc}
\theoremstyle{plain}
\newtheorem{theorem}{Theorem}
\newtheorem{corollary}{Corollary}
\newtheorem{lemma}{Lemma}
\newtheorem*{remark}{Remark}
\newtheorem*{example}{Example}

\newtheorem{assumption}{Assumption}

\usepackage[colorlinks,linkcolor=blue,citecolor=blue]{hyperref}

\usepackage[bottom]{footmisc}
\usepackage{caption}
\usepackage{subcaption}
\usepackage{helvet}  
\usepackage{courier}  
\usepackage{url}  
\usepackage{graphicx}  
\usepackage{multirow}
\usepackage{amsthm}
\usepackage{color}
\usepackage{MnSymbol}
\usepackage{makecell}
\usepackage{arydshln}
\usepackage{amsmath}
\usepackage[dvipsnames]{xcolor}
\usepackage{caption} 
\usepackage{natbib}
\usepackage{bbm}

\usepackage{textcomp}
\usepackage{wrapfig}
\usepackage{algorithm}
\usepackage{algorithmic}

\usepackage{csquotes}

\newcommand{\Tr}{{\rm Tr}}

\title{High-Dimensional Learning Dynamics of \\ Attention-Indexed Models}

\author[1,2]{Yizhou Xu}
\author[2]{Margarita Sagitova}
\author[2]{Lenka Zdeborová}
\author[1]{Florent Krzakala}
\affil[1]{\small Information, Learning and Physics Laboratory, \'Ecole Polytechnique F\'ed\'erale de Lausanne (EPFL)}
\affil[2]{\small Statistical Physics of Computation Laboratory, \'Ecole Polytechnique F\'ed\'erale de Lausanne (EPFL)}
\date{}

\begin{document}
\maketitle
\begin{abstract}
Attention mechanisms are central to modern foundation models, yet their training dynamics remain poorly understood, especially when the attention matrices have extensive rank. In this work, we study attention-indexed models, a broad framework that can represent multi-layer and multi-head attention architectures. First, we show that, in a suitable high-dimensional limit, the population-loss landscape is characterized by a finite set of trace order parameters. In contrast, online stochastic gradient descent (SGD) is governed by an infinite hierarchy of matrix moments, which we show can be exponentially well-approximated by a finite truncated system. Second, this framework reveals that attention parameterization itself can act as an architectural implicit bias. Direct optimization of an attention matrix \(S\in\mathbb{R}^{d\times d}\) can remain trapped in an uninformative state. Tied attention (\(S=WW^\top\)) induces an automatic symmetry-breaking mechanism and yields weak recovery in \(\Theta(d^2\log d)\) samples. For untied attention, \(S=UV^\top\), we uncover a fast-slow mechanism: the pre-activation mean first evolves on a fast timescale, while the overlaps evolve on a slower one. Weak recovery on the \(\Theta(d^2\log d)\) scale occurs when the state selected by the fast dynamics breaks the initial symmetry.
\end{abstract}

\section{Introduction}
Attention mechanisms \cite{vaswani2017attention} are the central building blocks of modern foundation models, but their training dynamics remain poorly understood. This is particularly true in the high-dimensional feature-learning regime with extensive-rank attention matrices, whose ranks diverge as the embedding dimension grows. Understanding this regime raises two related but distinct questions. First, what is the correct macroscopic description of learning when the attention matrices have extensive rank? Second, once such a description is available, how does the parameterization of an attention matrix affect whether features can be learned?

Existing theoretical work has made substantial progress on several aspects of attention. One line studies the mechanisms of in-context learning \cite{kim2024transformers,chen2024training,nishikawa2025nonlinear}, while another analyzes training from scratch in simplified transformer models \cite{yang2024training,huang2025transformers,kunin2025alternating}. For feature-learning dynamics, however, most existing analyses either do not take a high-dimensional limit \cite{song2024unraveling,tian2023scan,makkuva2024local,nichani2024transformers,yang2024training} or focus on finite-rank attention matrices, corresponding to sequence-indexed models \cite{cui2024phase,cui2025high,troiani2025fundamental,duranthon2025statistical,marion2024attention,arnaboldi2025asymptotics}. Such low-rank descriptions do not capture the extensive-rank regime naturally associated with attention matrices in practice, whose rank scales with the embedding dimension.

In this work, we study this regime through attention-indexed models, building on the formulation introduced in \cite{boncoragliobayes,boncoragliosingle}. The model class describes losses depending on collections of quadratic token interaction $\frac{1}{\sqrt{d}}x_i^TS_{ij}x_k$ with the attention matrix $S_{ij}\in\mathbb{R}^{d\times d}$ of extensive rank. The formulation is broad: multi-layer and multi-head attention-only architectures can be expressed in terms of finitely many such quadratic forms; we give the construction in Section \ref{sec:models} and Appendix \ref{app:examples}. Previous analyses of attention-indexed models focused on static Bayes-optimal or empirical-risk-minimization \cite{boncoragliobayes,boncoragliosingle}, with tractable analysis restricted to a single-layer, single-head tied setting. Dynamical analyses of extensive-rank matrices have so far been available mainly for substantially simpler quadratic models \cite{martin2024impact,martin2026high,ziyin2026neural}.

A first difficulty is that the static and dynamical descriptions have fundamentally different dimensionality. We prove that, as \(d\to\infty\), the population loss is determined by finitely many order parameters. At the level of the loss landscape, the original \(\Theta(d^2)\)-dimensional optimization problem therefore admits a finite-dimensional description. The training dynamics are richer. Unlike multi-indexed and sequence-indexed models whose dynamics can be tracked by a finite set of order parameters \cite{goldt2019dynamics,arnaboldi2025asymptotics,bocchi2026escape}, for the tied and untied factorizations analyzed in this paper, online SGD generates an infinite hierarchy of joint matrix moments, and no fixed finite collection of moments is sufficient in general to close the dynamics. We nevertheless show that this hierarchy is tractable: its empirical trajectory converges to a deterministic infinite-dimensional ODE, and finite-order truncations approximate the limiting dynamics with an error that decays exponentially with the truncation degree.

The second question is what this dynamical description reveals about parameterization. We compare three ways of representing an effective attention matrix: direct optimization of \(S\), the tied factorization  $S=WW^\top$, and the untied factorization $ S=UV^\top$. Although these parameterizations can describe closely related predictors, they induce qualitatively different optimization geometries. This distinction becomes especially sharp when learning starts from an uninformative initialization. In high-dimensional inference, direct optimization can be bottlenecked by the information exponent \cite{arous2021online}: when the first-order correlation vanishes, SGD over \(S\) can remain trapped on the \(d^2\log d\) sample scale. Our analysis shows that factorized attention can alter this conclusion, but through different mechanisms in the tied and untied cases.

For tied attention, \(S=WW^\top\), the positive-semidefinite (PSD) factorization provides an automatic symmetry-breaking mechanism. Under a mild non-degeneracy condition, this mechanism yields weak recovery in $\Theta(d^2\log d)$ samples. The untied factorization \(S=UV^\top\) behaves differently. We find a fast–slow learning mechanism. The pre-activation means $\frac{1}{\sqrt d}\operatorname{Tr}(UV^\top)$ evolve on a fast timescale, while the matrix moments evolve on a slower timescale. The subsequent weak-recovery behavior is therefore controlled by the state selected during the fast boundary layer. If this shift breaks the relevant symmetry, weak recovery occurs in \(\Theta(d^2\log d)\) samples; if the symmetry survives, the dynamics remains uninformative on this sample scale.

These results give two complementary messages. At the methodological level, extensive-rank attention requires a new macroscopic dynamical description: a finite-dimensional loss landscape coexists with an intrinsically infinite-dimensional training dynamics. At the learning level, the same framework shows that attention parameterization acts as an architectural implicit bias: tied and untied factorizations can overcome—or fail to overcome—uninformative symmetries through mechanisms that are absent from direct optimization of the attention matrix.

Our main contributions are therefore:
\begin{itemize}
\item A high-dimensional dynamical framework for extensive-rank attention. We formulate a broad class of attention-indexed models and show that their population loss admits a finite-dimensional characterization. For the tied and untied parameterizations, we derive high-dimensional limits for online SGD. The resulting dynamics involve an infinite hierarchy of matrix moments, for which we establish well-posedness, convergence, and exponentially accurate truncations.
\item Tied attention induces automatic symmetry breaking and achieves weak recovery in \(\Theta(d^2\log d)\) samples. Untied attention learns through a fast–slow mechanism, which leads to a weak-recovery dichotomy: recovery on the \(\Theta(d^2\log d)\) scale occurs when the fast process breaks the relevant symmetry, whereas the dynamics stalls on this scale when it does not.
\end{itemize}
The remainder of the paper is organized as follows. Section \ref{sec:models} introduces attention-indexed models and establishes the finite-dimensional characterization of their population-loss landscape. Section \ref{sec:tied} studies tied attention, derives its infinite-dimensional macroscopic dynamics and truncation theory, and identifies the symmetry-breaking mechanism responsible for weak recovery. Section \ref{sec:untied} turns to untied attention and derives its two-timescale dynamics and associated weak-recovery dichotomy. The appendices contain the proofs, additional examples of attention-indexed models, and the analysis of direct optimization over the attention matrix $S$.

\section{Attention-indexed models and their population loss landscape}
\label{sec:models}

The first goal of this paper is to identify a macroscopic description that remains tractable when the ranks of the attention matrices diverge with the embedding dimension. Before turning to training dynamics, we first ask a simpler static question: which quantities determine the population loss in the high-dimensional limit? We show in this section that, despite the high dimensionality of the underlying matrices, the answer is finite-dimensional. This finite-dimensional loss landscape will provide the potential that drives the dynamical equations in Sections \ref{sec:tied} and \ref{sec:untied}.

Let $L=\Theta(1)$ represent the sequence length (number of tokens), and $K=\Theta(1)$ denote the number of indices, such as attention heads or layers ($K=2$ in \cite{boncoragliosingle} with one student and one teacher index). Let the data vectors $x_1, \dots, x_L \in \mathbb{R}^d$ be jointly Gaussian with zero mean and covariance given by $\mathbb{E}[x_{ia}x_{jb}] = \mathcal{C}_{ij}\delta_{ab}$, where $C \in \mathbb{R}^{L \times L}$ is a PSD matrix. Let $S^k_{ij} \in \mathbb{R}^{d \times d}$ for $i,j = 1, \dots, L$ and $k=1,\dots,K$ be a set of deterministic weight matrices. The assumption $\mathbb{E}[x_{ia}x_{jb}] = \mathcal{C}_{ij}\delta_{ab}$ is made without loss of generality, as any arbitrary covariance structure can be absorbed into the weight matrix $S$. 

We consider losses of the form
\begin{equation*}
\mathcal{L}\left(\left\{\frac{x_i^TS_{ij}^kx_j}{\sqrt{d}}\right\}_{i,j,k=1}^{L,L,K}\right).
\end{equation*}
under the following assumption.
\begin{assumption}
In the limit as $d \to \infty$ with $K,L = \Theta(1)$, the following conditions hold for all $i, j, k, i', j', k'$:

\begin{itemize}
\item[1.] The matrices satisfy $\lim_{d\to\infty} \frac{||S^k_{ij}||_{op}}{||S^k_{ij}||_F} = 0$.
\item[2.] The scaled traces have finite limits: $\lim_{d\to\infty} \frac{1}{\sqrt{d}}\Tr(S^k_{ij}) = \mu_{ijk}$, $\lim_{d\to\infty} \frac{1}{d}\Tr(S^k_{ij}(S^{k'}_{i'j'})^T) = \Omega_{ijk,i'j'k'}$ and $\lim_{d\to\infty} \frac{1}{d}\Tr(S^k_{ij}S^{k'}_{i'j'}) = \Psi_{ijk,i'j'k'}$.
\item[3.] The loss $\mathcal{L}: \mathbb{R}^{L \times L\times K} \to \mathbb{R}$ is continuous with at most polynomial growth.
\end{itemize}
\label{assum:attention-index}
\end{assumption}
\begin{remark} While Assumptions \ref{assum:attention-index}.2 and \ref{assum:attention-index}.3 are primarily regularity conditions, Assumption \ref{assum:attention-index}.1 identifies the extensive-rank regime of interest here. We argue that this assumption is highly relevant in practice. In standard multi-head attention, for example, the rank of an individual attention matrix can grow with the embedding dimension when the number of heads remains fixed. This marks a critical difference from existing literature on low-rank attention (sequence-indexed models) \cite{cui2024phase,troiani2025fundamental,arnaboldi2025asymptotics,duranthon2025statistical}.
\end{remark}

In this section, we do not impose any structural assumptions on the attention matrices $S_{ij}^k$. We will specifically analyze both tied and untied attention mechanisms in subsequent sections. Remarkably, the attention-indexed model inherently covers standard multi-layer, multi-head attention networks (i.e., attention-only transformers).
\begin{example}[Multi-layer multi-head attention]
Consider an $M$-layer, $H$-head attention network:
\begin{equation}
X_m = X_{m-1} + \sum_{h=1}^H \text{softmax}\left(A_m^{(h)}(X_{m-1})\right) X_{m-1} V_{m-1}^{(h)} \in \mathbb{R}^{L \times d}, \quad m=1,\dots,M
\end{equation}
where the pre-softmax attention matrix is given by:
\begin{equation}
A_m^{(h)}(X_{m-1}) = \frac{1}{\sqrt{d}} X_{m-1} K_{m-1}^{(h)}(Q_{m-1}^{(h)})^T X_{m-1}^T \in \mathbb{R}^{L \times L}.
\end{equation}
The matrices $K_m^{(h)}, Q_m^{(h)}, \text{ and } V_m^{(h)}$ represent the key, query, and value weights of the $m-$th layer, respectively. The final output is $y=\sigma\left(\frac{1}{\sqrt{d}}X_MK_MQ_M^TX_M^T\right)$ for some activation function $\sigma$. 

Consider a teacher-student setup with input data $X_0 \sim \mathcal{N}(0, \mathcal{C} \otimes I_d)$ and a loss function $\mathcal{L}(y,\hat{y})$, where both teacher and student networks employ multi-layer multi-head architecture.
This setup naturally falls within our general framework (see Appendix \ref{app:examples} for detailed derivations). Notably, we do not require the teacher weights to be Gaussian, rendering the existence of such a teacher model a relatively plausible assumption.
\end{example}
In Appendix \ref{app:examples}, we introduce several additional models that fit this framework, including auto-regressive sequence models, multiple-location regression models, multiple-step reasoning models and matrix denoising problems.

Although the matrices \(S_{ij}^k\) contain \(\mathcal O(d^2)\) degrees of freedom, the loss does not retain this full complexity in the high-dimensional limit. The reason is that, under the extensive-rank condition, the finite collection of quadratic forms entering the loss becomes jointly Gaussian. Its limiting distribution is therefore completely specified by the first- and second-order trace statistics in Assumption \ref{assum:attention-index}.

The following theorem makes this reduction precise.
\begin{theorem}
\label{theo:general}
Under Assumption \ref{assum:attention-index}, the joint distribution of the quadratic forms converges weakly to a multivariate Gaussian, and expectations converge as follows:
\begin{equation}
\lim_{d\to\infty} \mathbb{E}_x\left[\mathcal{L}\left(\left\{\frac{x_i^TS^k_{ij}x_j}{\sqrt{d}}\right\}_{i,j,k=1}^{L,L,K}\right)\right] = \mathbb{E}\left[\mathcal{L}(\{G_{ijk}\}_{i,j,k=1}^{L,L,K})\right],    
\end{equation}
where $\{G_{ijk}\}_{i,j,k=1}^{L,L,K}$ is a multivariate Gaussian vector whose mean and covariance are given by:
\begin{equation}
\mathbb{E}[G_{ijk}] = \mathcal{C}_{ij}\mu_{ijk},\
\text{Cov}(G_{ijk}, G_{i'j'k'}) = \mathcal{C}_{ii'}\mathcal{C}_{jj'}\Omega_{ijk,i'j'k'} + \mathcal{C}_{ij'}\mathcal{C}_{ji'}\Psi_{ijk,i'j'k'}.
\end{equation}
\end{theorem}

Theorem \ref{theo:general} gives the basic reduction underlying the rest of the paper. In the high-dimensional limit, the population loss depends only on the finite collection of order parameters \(\{\mu,\Omega,\Psi\}\). Thus, although the parameter space grows with \(d\), the limiting loss landscape is finite-dimensional.

The Gaussian limit in Theorem \ref{theo:general} can be viewed as a consequence of the fourth-moment theorem for Wiener chaos \cite{nualart2005central}; for completeness, we provide a direct proof in Appendix \ref{app:proof_theo_general}. A centered version needed for tied attention is given in Appendix \ref{app:cor-symmetric}.

Theorem \ref{theo:general} characterizes the loss at any fixed state. It also implies a stronger optimization-level statement: subject to an appropriate realizability condition, optimizing over the high-dimensional matrices is asymptotically equivalent to optimizing over their finite set of limiting order parameters.
\begin{corollary}
\label{cor:minimum}
Let $\mathcal{K} \subseteq \{1, 2, \dots, K\}$ denote the index set of the learnable weights, while the remaining weights $\{S^k_{ij}\}_{k \notin \mathcal{K}}$ are fixed and satisfy the condition $\lim_{d\to\infty} \|S^k_{ij}\|_{\text{op}} / \|S^k_{ij}\|_F = 0$. Define the order parameter set $Q$ as
\begin{equation}
Q:=\left\{\left\{\frac{1}{\sqrt{d}}\Tr[S^k_{ij}]\right\}_{i,j,k},\left\{\frac{1}{d}\Tr[S^k_{ij}S^{k'}_{i'j'}], \frac{1}{d}\Tr[S^k_{ij}(S^{k'}_{i'j'})^T]\right\}_{i,j,k,i',j',k'}\right\}.
\end{equation}
Let $\mathcal{Q}$ denote the feasible domain of $Q$ realizable by matrix sequences satisfying Assumption \ref{assum:attention-index} in the limit $d \to \infty$, subject to the constraints imposed by the fixed weights $\{S^k_{ij}\}_{k \notin \mathcal{K}}$. Let $\mathcal{R}(Q) := \mathbb{E}_{G}[\mathcal{L}(\{G_{ijk}\}_{i,j,k=1}^{L,L,K})]$ denote the expected population loss in Theorem \ref{theo:general} specified by $Q$ and the covariance $\mathcal{C}$.

Under the regularity assumption specified in Appendix \ref{app:proof-cor-minimum},
\begin{equation}
\lim_{d\to\infty} \inf_{\{S^k_{ij}\}_{k\in\mathcal{K}}} \mathbb{E}_x\left[\mathcal{L}\left(\left\{\frac{x_i^TS^k_{ij}x_j}{\sqrt{d}}\right\}_{i,j,k=1}^{L,L,K}\right)\right] = \inf_{Q \in \mathcal{Q}} \mathcal{R}(Q).
\end{equation}
\end{corollary}
Corollary \ref{cor:minimum}, proven in Appendix \ref{app:proof-cor-minimum}, makes the finite-dimensional nature of the limiting landscape explicit: The optimal asymptotic population loss is determined by an optimization over the finite-dimensional feasible set \(\mathcal Q\).

Importantly, this does not imply that training can be described by gradient descent directly on \(Q\). Whether the order parameters form a closed dynamical system depends on how the attention matrices are parameterized. Direct optimization over \(S\) does close at the level of finitely many first- and second-order quantities (Appendix \ref{app:sgd-over-S}), whereas the tied and untied factorizations studied in Sections \ref{sec:tied} and \ref{sec:untied} generate additional matrix moments. This distinction has important consequences for weak recovery.

\section{Online SGD of tied attention}
\label{sec:tied}
We now turn to a question that cannot be answered from the loss landscape alone: does the parameterization of the attention matrix change the way SGD learns features?

In this section, we first derive the macroscopic dynamics of tied attention and show that, despite the finite-dimensional population loss, the factorized flow requires an infinite hierarchy of matrix moments. We then use these dynamics to characterize weak recovery and contrast it with direct optimization over \(S\).

\subsection{Macroscopic dynamics}
For tied attention, the trace of $S_k=W_kW_k^\top $ is typically of order \(d\). The pre-activation \(d^{-1/2}x_i^\top S_kx_j\) therefore contains a diverging contribution. Following the centered formulation of Appendix \ref{app:cor-symmetric}, we consider the regularized loss
\begin{equation*}
\mathcal{L}\left(\left\{\frac{\Tr[W_kW_k^T(x_jx_i^T-\mathcal{C}_{ij}I_d)]}{\sqrt{d}}\right\}_{i,j,k=1}^{L,L,K}\right)+\frac{\gamma}{2d}\sum_{k=1}^K||W_k||_F^2,
\end{equation*}
where the input data $x_1, \dots, x_L \sim \mathcal{N}(0, \mathcal{C} \otimes I_d)$ and $\gamma=\Theta(1)>0$ denotes the weight decay strength. 

For \(k=1,\ldots,K\), let $W_k\in\mathbb R^{d\times d_k}$ and $S_k=W_kW_k^\top\succeq0$. As before, let \(\mathcal K\subseteq\{1,\ldots,K\}\) denote the trainable indices. For \(k\in\mathcal K\), \(W_k\) evolves during training, whereas for \(k\notin\mathcal K\) it remains fixed and may, for example, represent a teacher matrix.

At iteration \(n\), online SGD uses a fresh sample $x^{(n)}\sim\mathcal N(0,C\otimes I_d)$ and updates, for \(k\in\mathcal K\),
\begin{equation}
W_k^{(n+1)} = W_k^{(n)} - \alpha_d \nabla_{W_k}\mathcal{L}(G^{(n)}) - \frac{\alpha_d \gamma}{d} W_k^{(n)},
\label{eq:SGD}
\end{equation}
where we denote $G^{(n)}_{ijk} := \frac{1}{\sqrt{d}} \Tr[W_k W_k^T (x^{(n)}_j (x^{(n)}_i)^T - \mathcal{C}_{ij}I_d)]$ and $\alpha_d$ is the step size.

Define $S_k^{(n)} = W_k^{(n)}(W_k^{(n)})^\top$ and, for any multi-index \(\alpha=(k_1,\ldots,k_w)\),
\begin{equation}
\mu_\alpha^{(n)} = \frac1d \operatorname{Tr} \left[ S_{k_1}^{(n)} S_{k_2}^{(n)} \cdots S_{k_w}^{(n)} \right].
\end{equation}
In particular, $q_{kl}=\mu_{(k,l)}$ denotes the second-order moment matrix. We introduce the continuous time $t=n\tau_d$, $\tau_d=\frac{4\alpha_d}{d}$ and let \(\tilde\mu^{(d)}(t)\) denote the piecewise-constant interpolation of these moments:
\begin{equation}
\tilde{\mu}^{(d)}(t) = \mu^{(n)}, \quad \text{for } t \in [n \tau_d, (n+1)\tau_d).
\label{eq:interpolation-tied}
\end{equation}
Define the potential as the limiting population loss in Theorem \ref{theo:general}:
\begin{equation}
\Phi(q):=\mathbb E_{G}\!\left[\mathcal L(G)\right],\qquad q\in\mathbb S_+^K,
\label{eq:potential-tied}
\end{equation}
where $G\in\mathbb R^{L\times L\times K}$ is a zero-mean Gaussian tensor with covariance $\operatorname{Cov}\left((G)_{ijk},(G)_{i'j'l}\right)=(\mathcal C_{ii'}\mathcal C_{jj'}+\mathcal C_{ij'}\mathcal C_{ji'})q_{kl}$. We define its matrix derivative as $\nabla\Phi(q)$\footnote{More rigorously speaking, it should be defined as $\lim_{\epsilon\to0^+}\nabla\Phi(q+\epsilon I)$.} for $q\succeq0$.

\begin{assumption}
\label{assum:tied}
\begin{itemize}
\item[1.] The loss function $\mathcal{L}: \mathbb{R}^{L\times L\times K} \to \mathbb{R}$ is four times continuously differentiable. $\mathcal{L}$ and its derivatives up to the fourth order have at most polynomial growth. 
\item[2.] At initialization, the weights $W_k^{(0)}$ satisfy $\|W_k^{(0)}\|_F = \Theta(\sqrt{d})$.
Additionally, there exists a constant $C_0 > 0$ independent of $d$ such that $\sup_{1\leq k\leq K} \|W_k^{(0)}\|_{\text{op}} \le C_0$ almost surely.
\item[3.] $\lim_{d \to \infty} |\tilde\mu_\alpha^{(d)}(0) - \bar{\mu}_\alpha(0)| = 0$ in probability for every multi-index $\alpha$. 
\item[4.] The regularization is sufficiently large, i.e., $\gamma>\gamma_*(C_0,\mathcal{C},\mathcal{L},K,L)$ for some $\gamma_*$ independent of $d$.
\end{itemize}
\end{assumption}
\begin{remark}
Assumption \ref{assum:tied}.4 is used only to guarantee uniform spectral control and prevent finite-time blow-up. Equivalently, one may impose an a priori uniform bound on the operator norms of the weights; numerically, our trajectories remain bounded even without weight decay.
\end{remark}
The key point is that the dimensionality of the loss and that of the dynamics are different. The potential \(\Phi\) depends only on the finite-dimensional matrix \(q\). However, the evolution of \(q\) is not closed: second moments depend on third moments, third moments on fourth moments, and so on. The tied factorization therefore converts a finite-dimensional loss landscape into an infinite-dimensional dynamical hierarchy. The following theorem shows that this hierarchy nevertheless has a deterministic and well-posed high-dimensional limit.

\begin{theorem}
\label{theo:tied-SGD}
Under Assumption \ref{assum:tied}, the online SGD trajectory is governed by an infinite-dimensional ODE driven by the gradients of $\Phi$:
\begin{equation}
\! \! \frac{d \bar\mu_{(k_1, \dots, k_w)}}{dt} \! = \!- \!\sum_{\substack{i=1 \\ k_i\in\mathcal{K}}}^w \! \sum_l [\nabla\Phi(\bar{q})]_{k_i l}
\left( \bar\mu_{(k_1, \dots, k_{i-1}, l, k_i, \dots, k_w)}\! + \bar\mu_{(k_1, \dots, k_i, l, k_{i+1}, \dots, k_w)} \right)\! -\! \frac{\gamma}{2} |\alpha|_{\mathcal{K}} \bar\mu_{(k_1, \dots, k_w)},
\label{eq:tied_macroscopic_equation}
\end{equation}
where $|\alpha|_{\mathcal{K}} \le w$ is the number of trainable indices in the multi-index $\alpha$ of length $w$. \eqref{eq:tied_macroscopic_equation} has a unique admissible solution and the empirical trajectory uniformly converges to it in probability.

Specifically, if the width scales such that $\max_{1\le k\le K} d_k\le d^{c_w}$ for some fixed $c_w<+\infty$ and the learning rate scales such that $\alpha_d \leq c_0(d\log d)^{-1}$ and $\alpha_d = \Omega(d^{-\iota})$ for some constants $c_0>0$ small enough and $\iota > 0$, then the empirical SGD trajectory satisfies:
\begin{equation}
\lim_{d \to \infty} \mathbb{P} \left( \sup_{t \in [0, T]} \left| \tilde\mu^{(d)}_\alpha(t) - \bar{\mu}_\alpha(t) \right| > \epsilon \right) = 0,\ \forall\epsilon>0
\end{equation}
for every multi-index $\alpha$ and for any fixed $T>0$.
\end{theorem}
\begin{remark}
We regard \eqref{eq:tied_macroscopic_equation} as an ODE on the moment space $\bigcup_{s>0}X_s^0$, where $X_s^0:=\left\{\mu:\lim_{w\to\infty}
\sup_{|\alpha|=w}\frac{|\mu_\alpha|}{s^w}=0
\right\}$.
A trajectory $\mu$ on $[0,T]$ is called an admissible solution if $q(\mu(t))\succeq0$ for every $t\in[0,T]$ and there exist $0<s_T<s_T'<\infty$ such that $\mu\in C([0,T];X_{s_T}^0)$ and $\mu(t)=\mu(0)+\int_0^tV(\mu(r))\,dr$
holds as an identity in $X_{s_T'}^0$, where $V$ refers to the right side of \eqref{eq:tied_macroscopic_equation}.
\end{remark}

Theorem \ref{theo:tied-SGD} gives the macroscopic dynamical counterpart of the static reduction in Section \ref{sec:models}. The same finite-dimensional potential \(\Phi(q)\) determines the coefficients of the flow, but it does not determine a closed ODE for \(q\). Instead, the geometry of the factorization \(S=WW^\top\) propagates information through the entire hierarchy of joint moments.

This infinite-dimensionality is not merely an artifact of our choice of coordinates. Appendix \ref{app:solvable} gives an exactly solvable quadratic example. In that example, no finite-dimensional ODE with a continuous initialization map can reproduce all possible trajectories.

The proof, given in Appendix \ref{app:proof-tied-SGD}, combines four ingredients: local well-posedness in a scale of weighted moment spaces, uniform spectral bounds obtained through a stopping-time argument and matrix Freedman inequalities, tightness and convergence of the empirical moment process, and global well-posedness of the resulting admissible infinite-dimensional trajectory.

The infinite hierarchy is required for an exact description, but this does not make the theory numerically intractable. Higher-order moments can be truncated with exponentially small error. For $M\ge2$, we define the degree-$M$ truncated trajectory $\mu^{(M)}(t)$ by
\begin{equation}
\frac{d}{dt}\mu_\alpha^{(M)}(t)=[V(\mu^{(M)}(t))]_\alpha,
\qquad |\alpha|\le M,
\label{eq:truncated-tied-system}
\end{equation}
with the initial condition $\mu_\alpha^{(M)}(0)=\bar\mu_\alpha(0)$ for $|\alpha|\le M$ and the boundary condition $\mu_\alpha^{(M)}(t)\equiv0$ for $|\alpha|>M$.
Here $V$ denotes the right side of \eqref{eq:tied_macroscopic_equation}\footnote{More specifically, $V$ refers to the extended vector field defined in
\eqref{eq:V-definition} because the truncated
moment sequence might not satisfy $q\succeq0$.}. The following corollary guarantees that $\mu^{(M)}(t)$ can approximate the true dynamics $\bar{\mu}(t)$ to arbitrary precision.

\begin{corollary}
\label{cor:truncation-tied}
Under the conditions of Theorem \ref{theo:tied-SGD}, there exist constants $s_0>C_*^2$ and $c<\infty$ independent of $M$, $\alpha$ and $T$ such that, for all sufficiently
large $M$, the truncated ODE \eqref{eq:truncated-tied-system}
has a unique solution and for every multi-index
$\alpha$,
\begin{equation}
\sup_{t\in[0,T]}\left|\mu_\alpha^{(M)}(t)-\bar\mu_\alpha(t)\right|\le c\left(\frac{C_*^2}{s_0}\right)^{M+1}
s_0^{|\alpha|}.
\label{eq:truncation-coordinate-bound}
\end{equation}
\end{corollary}

Corollary \ref{cor:truncation-tied} suggests that while the limiting trajectory is intrinsically infinite-dimensional, any fixed collection of low-order observables can be approximated to arbitrary accuracy by increasing the truncation degree. In the exactly solvable quadratic setting (quadratic networks \cite{bodin2023gradient,martin2024impact,martin2026high}) of Appendix \ref{app:solvable}, these finite-order systems coincide with the Taylor expansion of the exact matrix flow.

Corollary \ref{cor:truncation-tied} is proven in Appendix \ref{app:proof-cor-truncation-tied}. As a side remark, Assumption \ref{assum:tied}.4 is used in its proof to obtain a uniform-in-time truncation error bound; without this condition, the constants in Corollary \ref{cor:truncation-tied} would generally depend on $T$.

Figure \ref{fig:tied-theory} verifies the same phenomenon beyond the exactly solvable setting. For \(L=2\), softmax activation, and MSE loss, increasing the truncation degree systematically improves agreement between the deterministic theory and online SGD.

\begin{figure}
\centering
\includegraphics[width=0.8\linewidth]{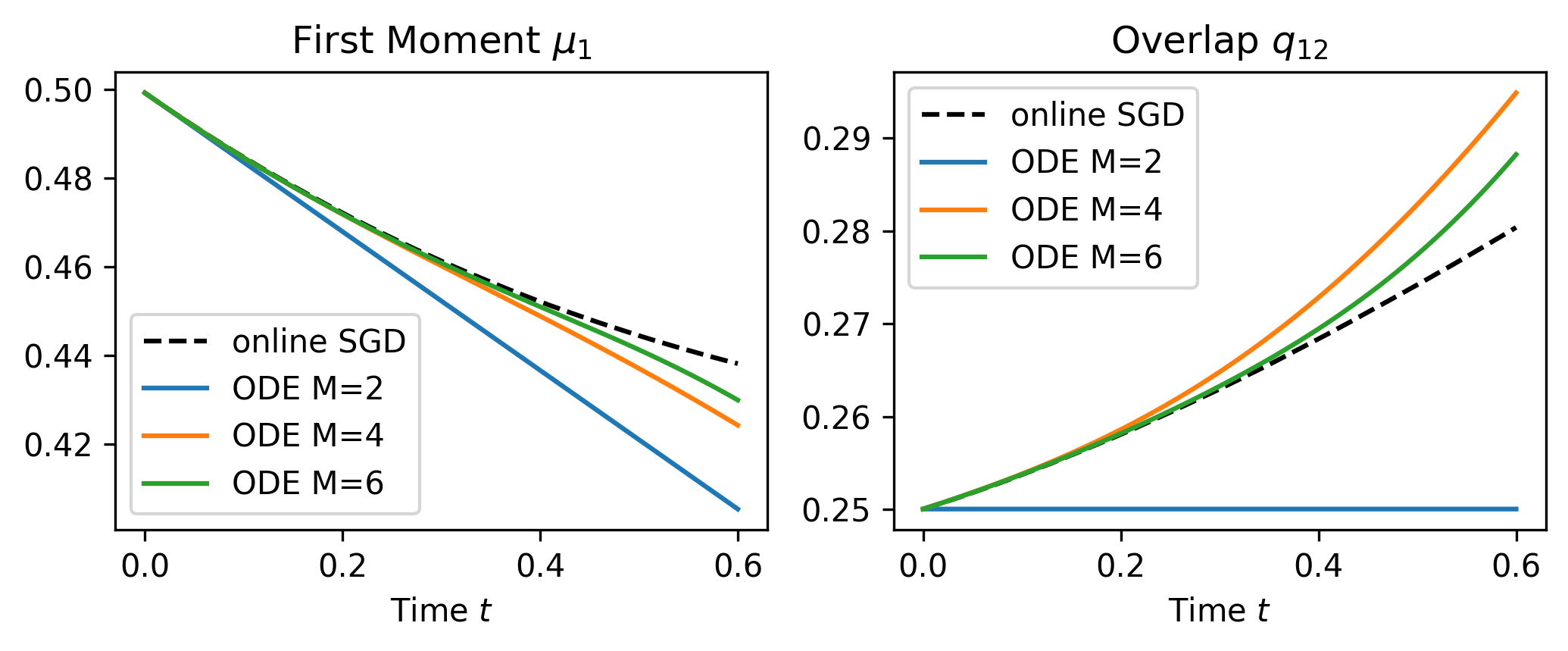}
\vspace{-4mm}
\caption{Finite truncations accurately approximate the infinite moment dynamics. Comparison of online SGD with degree-\(M\) truncations of \eqref{eq:truncated-tied-system} for \(L=2\), softmax activation, and MSE loss. Increasing \(M\) improves the agreement with the empirical trajectory. We use \(d=400\), learning rate \(0.01\), batch size \(512\), and no weight decay.}
\label{fig:tied-theory}
\vspace{-4mm}
\end{figure}

\subsection{Weak recovery and its sample complexity}
\label{sec:weak_recovery}
We now use the macroscopic dynamics to address the main learning question of this section: can the tied factorization escape an initially uninformative state, and does its behavior differ from direct optimization over \(S\)?

The relevant notion of information is not the raw overlap \(\frac1d\operatorname{Tr}(S_1S_2)\). Because tied matrices have a nonzero trace, this quantity can be \(\mathcal O(1)\) even when the student contains no information about the teacher. We therefore remove the isotropic components and measure learning through the overlap of the traceless parts.

For clarity, consider one trainable student \(S_1=W_1W_1^T\) and one fixed teacher \(S_2\). Define $t_k^{(d)}:=\frac1d\Tr[S_k]$, $\mathring S_k:=S_k-t_k^{(d)}I_d$ and the structural overlaps
\begin{equation}
p_{kl}^{(d)}:=\frac1d\Tr[\mathring S_k\mathring S_l]=q_{kl}^{(d)}-t_k^{(d)}t_l^{(d)}.
\label{eq:structural-overlap}
\end{equation}
Their deterministic limits are denoted by
\(\bar t_k,\bar q_{kl}\), and \(\bar p_{kl}\). Thus \(\bar p_{12}\), rather than \(\bar q_{12}\), measures alignment between the student and teacher.

We impose the following uninformative initialization.
\begin{assumption}
\label{assum:tied_initialization}
As \(d\to\infty\), the initial student \(S_1(0)\) and the teacher \(S_2\) satisfy
\begin{equation}
\bar p_{12}(0)=0,\qquad
\lim_{d\to\infty}\frac1d\Tr[\mathring S_1(0)^2\mathring S_2]=0,
\qquad
\lim_{d\to\infty}\frac1d\Tr[\mathring S_1(0)\mathring S_2^2]=0.
\label{eq:free-centered-moments}
\end{equation}
Moreover, we assume that $\bar t_1(0)>0$ and $\bar p_{22}>0$.
\end{assumption}
Assumption \ref{assum:tied_initialization} places the student on an uninformative manifold: its centered overlap with the teacher, as well as the mixed third-order centered moments vanish asymptotically at initialization. An independently initialized Gaussian \(W_1(0)\) satisfies these conditions.

The tied parameterization nevertheless contains one additional macroscopic quantity that does not vanish: $\bar t_1(0)>0$ as a direct consequence of \(S_1=W_1W_1^\top\succeq0\). Evaluating the moment dynamics \eqref{eq:tied_macroscopic_equation} at the uninformative initialization gives
\begin{equation}
\left. \frac{d\bar p_{12}}{dt} \right|_{t=0} = -2[\nabla\Phi(\bar q(0))]_{12} \bar t_1(0)\bar p_{22}.
\end{equation}
Hence, whenever \([\nabla\Phi(\bar q(0))]_{12}\) is nonzero, the uninformative state has a nonzero velocity in the informative direction. This is the symmetry-breaking effect of the tied factorization.

Corollary \ref{cor:tied_weak_recovery_sgd} turns this heuristic into a sample-complexity statement, proven in Appendix \ref{app:proof_tied_weak_recovery_sgd}. For a fixed recovery threshold \(c>0\), define the weak-recovery sample complexity by
\begin{equation}
N_{\mathrm{wr}}^{(d)}(c):=\inf\left\{n\ge0:|p_{12}^{(d,n)}|\ge c
\right\},
\label{eq:weak-recovery-hitting-time}
\end{equation}
where \(p_{12}^{(d,n)}\) denotes the structural overlap after \(n\) steps of online SGD \eqref{eq:SGD}.
\begin{corollary}
\label{cor:tied_weak_recovery_sgd}
Under Assumption \ref{assum:tied_initialization} and the conditions of Theorem \ref{theo:tied-SGD}, suppose additionally that
\begin{equation}
[\nabla\Phi(\bar q(0))]_{12}\neq0,
\label{eq:non-degeneracy}
\end{equation}
Then there exist constants \(c>0\) and \(0<T_-<T_+<\infty\) independent of \(d\), such that
\begin{equation}
\lim_{d\to\infty}\mathbb P\left(
\left\lfloor\frac{T_-d}{4\alpha_d}\right\rfloor<N_{\mathrm{wr}}^{(d)}(c)
\le\left\lceil\frac{T_+d}{4\alpha_d}\right\rceil
\right)=1.
\label{eq:weak-recovery-hitting-bound}
\end{equation}
In particular, choosing the largest learning-rate scaling $\alpha_d=\Theta\left(\frac1{d\log d}\right)$ covered by Theorem \ref{theo:tied-SGD} gives the weak recovery sample complexity
\begin{equation}
N_{\mathrm{wr}}^{(d)}(c)=\Theta_{\mathbb P}(d^2\log d).
\label{eq:weak-recovery-d2logd}
\end{equation}
\end{corollary}

\begin{remark}
The condition \eqref{eq:non-degeneracy} is mild in a broad class of teacher-student losses. To see this, consider the MSE loss
$\mathcal L(G_1,G_2) = \frac12\|F(G_1)-F(G_2)\|^2$ with the same activation \(F\) for student and teacher. Then
\begin{equation}
[\nabla\Phi(\bar q)]_{12}=-\frac12\mathbb E\left[\left\langle \nabla F(G_1)B^{1/2},\nabla F(G_2)B^{1/2}\right\rangle_F\right],
\end{equation}
where $B_{(ij),(i'j')}=\mathcal C_{ii'}\mathcal C_{jj'}+\mathcal C_{ij'}\mathcal C_{ji'}$. Suppose additionally that the initialization scale is chosen such that
$q_{11}(0)=q_{22}>0$. A Gaussian Hermite expansion then gives
\begin{equation}
-[\nabla\Phi(\bar q(0))]_{12}=\frac12\sum_\nu\rho^{|\nu|}\nu!\|H_\nu\|_F^2,
\end{equation}
where $\{H_\nu\}$ are the Hermite coefficients of
$\nabla F(\sqrt{q_{22}}B^{1/2}z)B^{1/2}$ and $\rho=\frac{\bar t_1(0)\bar t_2}{\sqrt{\bar q_{11}(0)\bar q_{22}}}>0$. Hence $[\nabla\Phi(\bar q(0))]_{12}<0$ whenever $\nabla F(\sqrt{q_{22}}B^{1/2}Z)B^{1/2}$ is not constantly zero.
\end{remark}

The significance of Corollary \ref{cor:tied_weak_recovery_sgd} is not only the \(d^2\log d\) sample scale, but the mechanism that produces it. The student begins with zero structural overlap with the teacher, yet the tied factorization makes this uninformative state non-invariant: the nonzero isotropic component \(\bar t_1(0)\) couples to the teacher variance \(\bar p_{22}\) through the cross-gradient of the population loss and immediately generates informative overlap.

This behavior should be contrasted with direct optimization over \(S\). As shown in Appendix \ref{app:sgd-over-S}, the direct \(S\)-flow closes at the level of finitely many first- and second-order parameters. More importantly, when the relevant cross-gradients vanish throughout the uninformative manifold, that manifold is invariant under the limiting \(S\)-flow, and weak recovery does not occur on the \(d^2\log d\) sample scale. The tied factorization therefore acts as an architectural implicit bias: it changes the geometry of the gradient flow so that an initialization that is structurally uninformative can nevertheless have a nonzero escape direction. 

Figure \ref{fig:tied} illustrates this distinction. For a linear activation, both parameterizations can develop overlap. When we do online SGD directly over \(S\) on higher-order activations (\(h_2,h_3\)), however, it remains uninformative on the displayed sample scale, whereas the tied flow escapes and learns the teacher.

Corollary \ref{cor:tied_weak_recovery_sgd} identifies the mechanism governing weak recovery. Under an additional Polyak–Łojasiewicz condition, Appendix \ref{app:strong-recovery} also gives a convergence rate for the subsequent strong-recovery phase.

\begin{figure}
\centering
\includegraphics[width=0.8\linewidth]{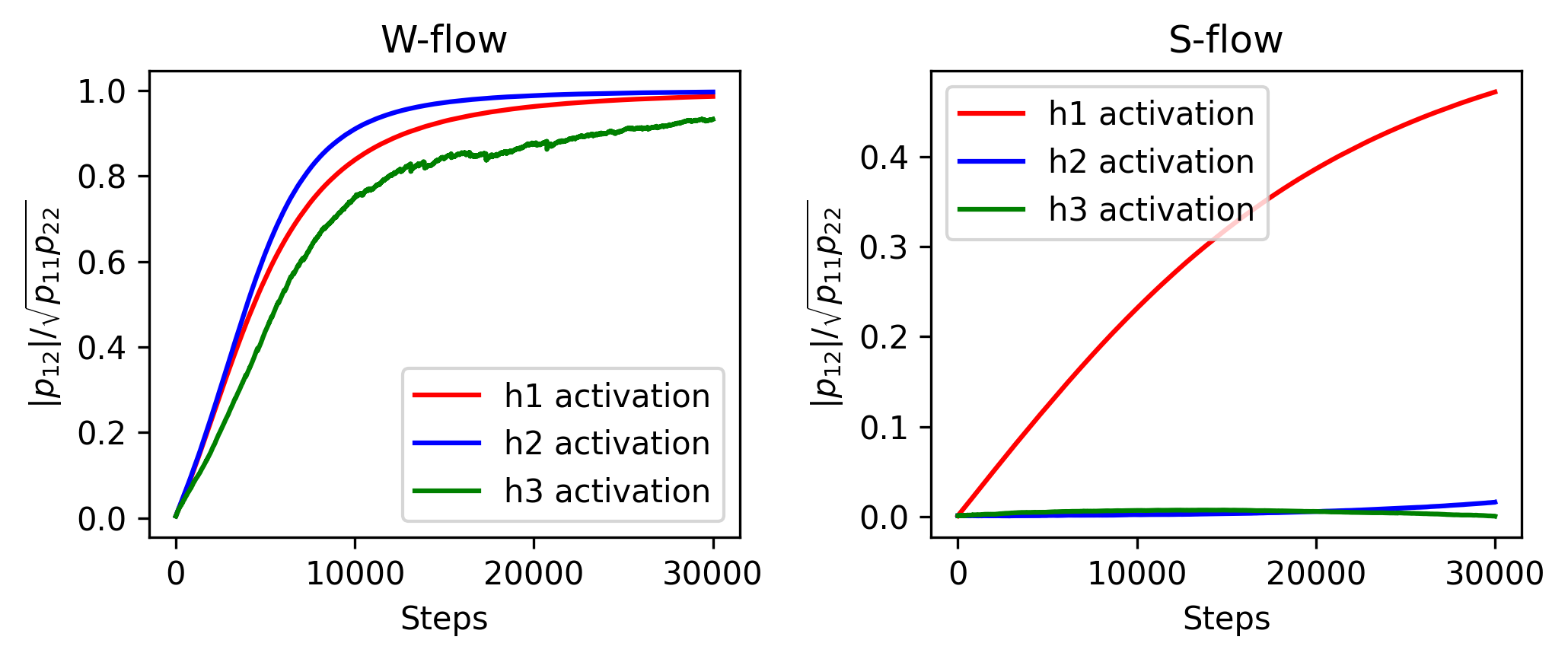}
\vspace{-1em}
\caption{The parameterization of the attention matrix qualitatively changes weak recovery. We compare online SGD on the tied factor \(W\), with \(S=WW^\top\) (\textbf{Left}), and direct online SGD on \(S\) (\textbf{Right}), starting from uninformative initializations in the same teacher–student setting. For the linear activation both parameterizations recover the teacher on the displayed scale, whereas for the higher-order activations the direct \(S\)-flow remains close to the uninformative state while the tied \(W\)-flow develops substantial overlap. We use \(d=400\), learning rate \(0.005\), batch size \(512\), and no weight decay. The activations are chosen from: $h_1(x)=x$, $h_2(x)=\frac{x^2-1}{\sqrt{2}}$ and $h_3(x)=\frac{x^3-3x}{\sqrt{6}}$, respectively ($L=1$). The input data is $\mathcal{N}(0, I_d)$.} 
    \label{fig:tied}
\vspace{-1em}
\end{figure}

\section{Online SGD of untied attention}
The symmetry breaking described in Section \ref{sec:tied} relies crucially on the PSD geometry \(S=WW^\top\). This raises the natural next question: what remains of the mechanism when $S$ is not PSD? This section answers this question for the untied parameterization $S=UV^T$.

The mechanism is qualitatively different from the tied case. For untied attention, the pre-activation has a nontrivial mean $m_k=\frac{1}{\sqrt d}\operatorname{Tr}(U_kV_k^\top)$, and this mean evolves faster than the matrix moments. As a result, training separates into two stages. On the fast timescale, the mean variables evolve while the moments remain asymptotically frozen. The fast flow drives the trainable means toward a critical manifold of the population loss. Only afterwards, on a slower timescale, do the matrix moments undergo \(\mathcal O(1)\) evolution, with the means remaining asymptotically slaved to this manifold.

This separation has a direct consequence for feature learning. The fast phase does not itself create an \(\mathcal O(1)\) teacher overlap, since the overlaps remain frozen on that timescale. Instead, it selects the effective state from which slow feature learning begins. Weak recovery is therefore controlled by the loss gradients evaluated after the fast relaxation. The fast flow can break a student-side symmetry by shifting the pre-activation mean and then the slow dynamics escapes the uninformative manifold. If the relevant symmetry survives the fast relaxation, this recovery mechanism is absent.

The remainder of this section makes this fast–slow picture rigorous. We first derive the two-timescale macroscopic limit and then use it to characterize a corresponding dichotomy for weak recovery.

\label{sec:untied}
\subsection{Macroscopic dynamics}
For each index \(k=1,\ldots,K\), let $U_k,V_k\in\mathbb R^{d\times d_k}$ and define the associated matrices
$M_k:=U_kV_k^\top$, $A_k:=U_kU_k^\top$ and $B_k:=V_kV_k^\top$. We consider the regularized loss
\begin{equation}
\mathcal{L}\left(\left\{\frac{1}{\sqrt d}x_i^TU_kV_k^Tx_j\right\}_{i,j,k=1}^{L,L,K}\right)+\frac{\gamma}{2d}
\sum_{k\in\mathcal K}\left(\|U_k\|_F^2+\|V_k\|_F^2\right),
\label{eq:untied-objective}
\end{equation}
where $x_1,\ldots,x_L\sim\mathcal N(0,\mathcal C\otimes I_d)$, and $\mathcal K\subseteq\{1,\ldots,K\}$ denotes the set of trainable indices. 

At iteration \(n\), a fresh sample \(x^{(n)}\sim\mathcal N(0,\mathcal C\otimes I_d)\) is drawn, and for \(k\in\mathcal K\) the online SGD updates are
\begin{equation}
\begin{aligned}
U_k^{(n+1)}&=U_k^{(n)}-\alpha_d\nabla_{U_k}\mathcal L(G^{(n)})-\frac{\alpha_d\gamma}{d}U_k^{(n)},\\
V_k^{(n+1)}&=V_k^{(n)}-\alpha_d\nabla_{V_k}\mathcal L(G^{(n)})-\frac{\alpha_d\gamma}{d}V_k^{(n)},
\end{aligned}
\label{eq:untied-SGD}
\end{equation}
where $G_{ijk}^{(n)}:=\frac{1}{\sqrt d}(x_i^{(n)})^TU_k^{(n)}(V_k^{(n)})^Tx_j^{(n)}$ and $\alpha_d$ is the learning rate.

\paragraph{Order parameters.}
Unlike the centered tied model of Section \ref{sec:tied}, the untied pre-activations can have an \(\mathcal O(1)\) limiting mean. We therefore need to track separately $m_k^{(d)}:=\frac{1}{\sqrt d}\Tr(M_k)$.

The remaining macroscopic state is described by normalized joint moments of the matrices \(M_k,M_k^\top,A_k,B_k\). Introduce the alphabet $\mathscr Z:=\{M_k,M_k^T,A_k,B_k:1\le k\le K\}$ and for a word \(\alpha=(Z_1,\ldots,Z_w)\in\mathscr Z^w\), define
\begin{equation}
\mu_\alpha^{(d)}:=\frac1d\Tr[Z_1\cdots Z_w].
\end{equation}
In particular, we define the traces $t_k^A:=\mu_{(A_k)}$, $t_k^B:=\mu_{(B_k)}$ and the overlaps $q_{kl}^{(1)}:=\mu_{(M_k,M_l^T)}$, $q_{kl}^{(2)}:=\mu_{(M_k,M_l)}$.

The potential is defined as the limiting population loss in Theorem \ref{theo:general}: $\Phi(m,q^{(1)},q^{(2)}):=\mathbb E\bigl[\mathcal L(G)\bigr]$, where \(G\in\mathbb R^{L\times L\times K}\) is Gaussian with $\mathbb E[G_{ijk}]=\mathcal C_{ij}m_k$ and $\operatorname{Cov}(G_{ijk},G_{i'j'l})=\mathcal C_{ii'}\mathcal C_{jj'}q_{kl}^{(1)}+\mathcal C_{ij'}\mathcal C_{ji'}q_{kl}^{(2)}$. We denote by $\nabla_m\Phi$ and $\nabla_{q^{(r)}}\Phi$, $r=1,2$, the vector and matrix gradients of the potential, respectively.

For later use, let $M_l^{(1)}:=M_l$ and $M_l^{(2)}:=M_l^T$ for $\bar 1:=2,\quad \bar 2:=1$.
Define
\begin{equation}
\psi_k(m,\mu):=2\sum_{l=1}^K\sum_{r=1}^2\left[\nabla_{q^{(r)}}\Phi(m,q^{(1)},q^{(2)})\right]_{kl}\left[\mu_{(M_l^{(r)},B_k)}+\mu_{(A_k,M_l^{(r)})}\right].
\label{eq:untied-psi}
\end{equation}

The crucial new point is that \(m\) and \(\mu\) evolve on different scales. A single SGD step changes \(m\) at order \(\alpha_d\), whereas the normalized moments change only at order \(\alpha_d/d\). This produces the two-timescale limit below.

\paragraph{Fast learning phase.}
On the fast timescale $t=n\alpha_d$, the normalized moments remain asymptotically frozen, while the mean vector evolves at order one. Denoting the limiting fast trajectory by \(\widetilde m(t)\), the trainable components satisfy
\begin{equation}
\frac{d\tilde m_k}{dt}=-\left(\bar t_k^A(0)+\bar t_k^B(0)\right)\frac{\partial\Phi}{\partial m_k}\left(\tilde m(t),\bar q^{(1)}(0),\bar q^{(2)}(0)\right),\qquad k\in\mathcal K.
\label{eq:m-fast}
\end{equation}
The non-trainable components remain fixed.

Thus, to leading order, the fast learning phase performs gradient flow in the mean variables while seeing the covariance structure as frozen. Under the local strong-convexity condition introduced below, this flow converges toward the unique critical point $m^\star_{\mathcal K} = m^\star_{\mathcal K}(q^{(1)},q^{(2)})$ satisfying $\nabla_{m_{\mathcal K}} \Phi \left( m^\star(q^{(1)},q^{(2)}), q^{(1)},q^{(2)} \right) = 0$.
We denote the corresponding critical manifold by $
\mathcal M_0:=\left\{(m,\mu):\nabla_{m_{\mathcal K}}\Phi(m,q^{(1)},q^{(2)})=0\right\}$.

\paragraph{Slow learning phase.}
We now introduce the slow timescale $\tau=\frac{t}{d}=\frac{n\alpha_d}{d}$. After the initial boundary layer, the trainable means are asymptotically slaved to the critical manifold:
\begin{equation}
\bar m_{\mathcal K}(\tau)=m_{\mathcal K}^\star\left(\bar q^{(1)}(\tau),\bar q^{(2)}(\tau)\right),\qquad(\bar m(\tau),\bar\mu(\tau))\in\mathcal M_0.
\label{eq:untied-slaving}
\end{equation}
The small residual of the fast variables contributes during the slow timescale. This produces the effective bias
\begin{equation}
\bar\beta_k(\bar m,\bar\mu):=-\frac{\psi_k(\bar m,\bar\mu)}{\bar t_k^A+\bar t_k^B}.
\label{eq:untied-effective-bias}
\end{equation}
with \(\psi_k\) defined in \eqref{eq:untied-psi}.

For a word \(\alpha=(Z_1,\ldots,Z_w)\), let \(k_i\) denote the index carried by the letter \(Z_i\). For each trainable position \(i\), define the substitution operators
\begin{equation}
\mathcal T_i^{\mathrm{bias}}(\mu)=
\begin{cases}
\mu_{(\ldots,A_k+B_k,\ldots)},&Z_i\in\{M_k,M_k^T\},
\\[1mm]
\mu_{(\ldots,M_k+M_k^T,\ldots)},&Z_i\in\{A_k,B_k\},
\end{cases}\qquad
\mathcal T_{i,l,r}^{\mathrm{diff}}(\mu)=
\begin{cases}
\mu_{(\ldots,M_l^{(r)}B_k+A_kM_l^{(r)},\ldots)},&Z_i=M_k,
\\[1mm]
\mu_{(\ldots,B_kM_l^{(\bar r)}+M_l^{(\bar r)}A_k,\ldots)},&Z_i=M_k^T,
\\[1mm]
\mu_{(\ldots,M_l^{(r)}M_k^T+M_kM_l^{(\bar r)},\ldots)},&Z_i=A_k,
\\[1mm]
\mu_{(\ldots,M_l^{(\bar r)}M_k+M_k^TM_l^{(r)},\ldots)},&Z_i=B_k.
\end{cases}
\label{eq:untied-substitution}
\end{equation}
The notation is understood linearly; for example, $\mu_{(\ldots,A_k+B_k,\ldots)}=\mu_{(\ldots,A_k,\ldots)}+\mu_{(\ldots,B_k,\ldots)}$.

With these operators, the slow evolution of all normalized matrix moments closes as an infinite-dimensional hierarchy
\begin{align}
\frac{d\bar\mu_\alpha}{d\tau}={}&-\sum_{\substack{i=1\\ k_i\in\mathcal K}}^w
\bar\beta_{k_i}(\bar m,\bar\mu)\mathcal T_i^{\mathrm{bias}}(\bar\mu)-2\sum_{\substack{i=1\\ k_i\in\mathcal K}}^w
\sum_{l=1}^K\sum_{r=1}^2\frac{\partial\Phi}{\partial q_{k_il}^{(r)}}\left(\bar m,\bar q^{(1)},\bar q^{(2)}\right)\mathcal T_{i,l,r}^{\mathrm{diff}}(\bar\mu)-2\gamma|\alpha|_{\mathcal K}\bar\mu_\alpha
\label{eq:slow_tau_mu}
\end{align}
for $0\leq\tau\leq T$. Together with \eqref{eq:untied-slaving}, this determines the limiting trajectory in the slow learning phase.

Equations \eqref{eq:m-fast}–\eqref{eq:slow_tau_mu} summarize the fast–slow structure. The mean variables are no longer independent coordinates on the slow scale: they are determined by the current second moments through \(m^\star(q)\). The remaining matrix moments evolve through an infinite hierarchy, as in the tied case, but with the additional effective bias inherited from the fast learning phase.

We now state conditions ensuring that this formal separation of timescales is valid and that both phases remain well posed.
\begin{assumption}
\label{assum:untied}
The following conditions hold.
\begin{itemize}
\item[1.]
The loss \(\mathcal L:\mathbb R^{L\times L\times K}\to\mathbb R\) is six times continuously differentiable, and \(\mathcal L\) together with its derivatives up to sixth order has at most polynomial growth.
\item[2.]
For every \(k\), $\|U_k^{(0)}\|_F+\|V_k^{(0)}\|_F=\Theta(\sqrt d)$. There exists a constant \(C_0>0\), independent of \(d\), such that $\max_{k}\left\{\|U_k^{(0)}\|_{\mathrm{op}},\|V_k^{(0)}\|_{\mathrm{op}}\right\}\le C_0$
almost surely.
\item[3.]
There exist deterministic limits \(\tilde m_k(0)\) and \(\bar\mu_\alpha(0)\) such that $m_k^{(d)}(0)\to\tilde m_k(0)$ and $\mu_\alpha^{(d)}(0)\to\bar\mu_\alpha(0)$ in probability for every \(k\) and every \(\alpha\).
\item[4.]
For every trainable index \(k\in\mathcal K\), $\bar t_k^A(0)>0$, $\bar t_k^B(0)>0$ and $\left|\bar t_k^A(0)-\bar t_k^B(0)\right|\ge c_{\mathrm{bal}}$ for some \(c_{\mathrm{bal}}>0\).
\item[5.] The regularization is sufficiently large, i.e., $\gamma>\gamma_*(C_0,\mathcal{C},\mathcal{L},K,L,T)$ for some $\gamma_*$ independent of $d$.
\end{itemize}
\end{assumption}
Assumption \ref{assum:untied}.4 is a technical non-degeneracy condition used to control the denominator \(t_k^A+t_k^B\) of \eqref{eq:untied-effective-bias} and obtain uniform bounds along the slow trajectory. It is satisfied, for example, by independent Gaussian initializations with slightly different variances for \(U_k\) and \(V_k\). Assumption \ref{assum:untied}.5 plays the same spectral-control role as in the tied case.

\begin{assumption}[Local strong convexity]
\label{assum:untied_convex}
Fix a time horizon \(T>0\). Let \(m_{\mathcal K}\) denote the trainable mean components, while \(m_{\setminus\mathcal K}\) is fixed at \(m_{\setminus\mathcal K}(0)\). Assume that there exist an open convex set \(\mathcal M\subset\mathbb R^{|\mathcal K|}\), an open set \(\mathcal Q\) of covariance parameters \(q=(q^{(1)},q^{(2)})\), and a constant \(\lambda_0>0\) such that:

\begin{itemize}
\item[1.]
For every \((m_{\mathcal K},q)\in\mathcal M\times\mathcal Q\), $\nabla_{m_{\mathcal K}}^2\Phi(m,q)\succeq\lambda_0 I$.

\item[2.]
Writing \(q_0=(\bar q^{(1)}(0),\bar q^{(2)}(0))\), we have \(\tilde m_{\mathcal K}(0)\in\mathcal M\), \(q_0\in\mathcal Q\), and
\begin{equation}
\left\{m_{\mathcal K}\in\mathcal M:\Phi(m_{\mathcal K},m_{\setminus\mathcal K}(0),q_0)\le\Phi(\tilde m(0),q_0)\right\}\Subset\mathcal M.
\end{equation}

\item[3.]
For every \(q\in\mathcal Q\), the equation $\nabla_{m_{\mathcal K}}\Phi(m,q)=0$ has a solution \(m_{\mathcal K}^\star(q)\in\mathcal M\).
The solution $ \bigl(m_{\mathcal K}^\star(\bar q(\tau)),\bar q(\tau)\bigr)$ of \eqref{eq:untied-slaving}-\eqref{eq:slow_tau_mu} remains in a compact subset of \(\mathcal M\times\mathcal Q\) for $0\le\tau\le T$\footnote{More precisely, on every interval of local existence contained in $[0,T]$.} and sufficiently large $\gamma$.
\end{itemize}
\end{assumption}
Assumption \ref{assum:untied_convex} ensures that, for each nearby covariance state, the fast dynamics has a locally unique attracting critical point \(m^\star(q)\). This property allows the fast mean variables to be eliminated on the slow timescale.

We now define the empirical interpolations. On the fast timescale,
\begin{equation}
\tilde m^{(d)}(t):=m^{(n)},\qquad t\in[n\alpha_d,(n+1)\alpha_d),
\label{eq:untied-fast-interpolation}
\end{equation}
while on the slow timescale,
\begin{equation}
\tilde\mu^{(d)}(\tau)=\mu^{(n)},\qquad\tau\in\left[\frac{n\alpha_d}{d},\frac{(n+1)\alpha_d}{d}\right).
\label{eq:untied-slow-interpolation}
\end{equation}

The following theorem rigorously justifies the separation into a fast boundary layer and a slow infinite-dimensional flow.
\begin{theorem}
\label{theo:untied-SGD}
Suppose that Assumptions \ref{assum:untied} and \ref{assum:untied_convex} hold for a fixed slow-time horizon \(T>0\). Assume furthermore that $\max_{1\le k\le K}d_k\le d^{c_w}$ for some fixed \(c_w<\infty\), and that the learning rate satisfies $\alpha_d\le\frac{c_0}{d\log d}$ and $\alpha_d=\Omega(d^{-\iota})$ for some constants \(\iota>0\) and sufficiently small \(c_0>0\) independent of \(d\).

Then the online SGD trajectory exhibits the following two-timescale limit. Fix any $\epsilon>0$. For every finite fast-time horizon \(\widetilde T>0\),
\begin{equation}
\lim_{d\to\infty}\mathbb P\left(\sup_{t\in[0,\widetilde T]}\left\|\tilde m^{(d)}(t)-\tilde m(t)\right\|_\infty>\epsilon\right)=0,
\label{eq:untied-fast-convergence}
\end{equation}
where \(\tilde m\) is the unique solution of \eqref{eq:m-fast}.

On the slow timescale, let \((\bar m(\tau),\bar\mu(\tau))\) denote the unique admissible solution
of \eqref{eq:untied-slaving}--\eqref{eq:slow_tau_mu}. Then for every \(0<T_1<T_2\le T\),
\begin{equation}
\lim_{d\to\infty}\mathbb P\left(\sup_{\tau\in[T_1,T_2]}\left\|\tilde m^{(d)}(\tau d)-\bar m(\tau)\right\|_\infty>\epsilon\right)=0
\label{eq:untied-slow-m-convergence}
\end{equation}
and, for every word \(\alpha\),
\begin{equation}
\lim_{d\to\infty}\mathbb P\left(\sup_{\tau\in[0,T_2]}\left|\tilde\mu_\alpha^{(d)}(\tau)-\bar\mu_\alpha(\tau)\right|>\epsilon\right)=0.
\label{eq:untied-slow-moment-convergence}
\end{equation}
\end{theorem}
Theorem \ref{theo:untied-SGD} shows that the untied factorization introduces a dynamical structure that has no analogue in the tied model. On the fast scale \(t=n\alpha_d\), only the pre-activation means undergo \(\mathcal O(1)\) evolution. On the slow scale \(\tau=n\alpha_d/d\), these means have already equilibrated and follow the evolving covariance quasi-statically, while the matrix moments undergo their \(\mathcal O(1)\) evolution. This distinction will be essential for weak recovery.

The proof is given in Appendix \ref{app:untied}. Its main additional difficulty relative to the tied case is the singular separation of timescales. We first establish convergence of the fast mean dynamics while the normalized moments remain frozen. We then obtain a refined \(\mathcal O(d^{-1/2})\) control of the residual mean gradient after the boundary layer; this residual accumulates into the effective bias \(\beta_k\) on the slow scale. Uniform spectral bounds, tightness of the moment process, and uniqueness then yield the limiting slow trajectory.

As in the tied case, the slow infinite-dimensional hierarchy admits finite-order approximations. Appendix \ref{app:untied-truncation} proves an exponentially decaying truncation error.

\subsection{Weak recovery and its sample complexity}
\label{sec:untied-weak-recovery}
We now ask what the fast–slow structure implies for feature learning. To make this mechanism explicit, consider one trainable student \(M_1=U_1V_1^\top\) and one fixed teacher \(M_2\). Since
$\frac1d\operatorname{Tr}(M_k) = \frac{m_k^{(d)}}{\sqrt d} \to0$, feature learning is naturally measured through the matrix overlaps
\begin{equation}
q_{12}^{(1,d)}:=\frac1d\Tr[M_1M_2^T],
\qquad
q_{12}^{(2,d)}:=\frac1d\Tr[M_1M_2].
\label{eq:untied-structural-overlaps}
\end{equation}
We write
\begin{equation}
\mathbf q_{12}^{(d)}:=\begin{pmatrix}
q_{12}^{(1,d)}\\
q_{12}^{(2,d)}
\end{pmatrix},
\qquad
\bar{\mathbf q}_{12}:=
\begin{pmatrix}
\bar q_{12}^{(1)}\\
\bar q_{12}^{(2)}
\end{pmatrix}.
\end{equation}
We impose the following uninformative initialization.

\begin{assumption}
\label{assum:untied_initialization}
At the initial time, the matrix algebra generated by
$\{M_1,M_1^T,A_1,B_1\}$ is asymptotically free from the algebra generated by $\{M_2,M_2^T,A_2,B_2\}$.
\end{assumption}
Assumption \ref{assum:untied_initialization} is the untied analogue of an uninformative initialization. Asymptotic freeness removes not only the pairwise overlaps \(q_{12}^{(1)}\) and \(q_{12}^{(2)}\), but also the mixed student–teacher moments.

Define the teacher covariance as $
\mathsf Q_2:=
\begin{pmatrix}
\bar q_{22}^{(1)} & \bar q_{22}^{(2)}
\\
\bar q_{22}^{(2)} & \bar q_{22}^{(1)}
\end{pmatrix}$. For \(c>0\), define the weak-recovery sample complexity by
\begin{equation}
N_{\rm wr}^{(d)}(c):=\inf\left\{
n\ge0:\left\|\mathbf q_{12}^{(d,n)}
\right\|_\infty\ge c\right\}.
\label{eq:untied-weak-recovery-hitting-time}
\end{equation}
Let $D_1(0):=\bar t_1^A(0)+\bar t_1^B(0)>0$,
and recall that the fast dynamics drive the trainable mean toward $\bar m_1^\star:=m_1^\star\left(\bar q^{(1)}(0),\bar q^{(2)}(0)\right)$.
Define the gradient at the slow-time initial state by $g_{12}^\star:=
\begin{pmatrix}
\frac{\partial\Phi}{\partial q_{12}^{(1)}},
\frac{\partial\Phi}{\partial q_{12}^{(2)}}
\end{pmatrix}^T$ evaluated at $\left(\bar m^\star,\bar q^{(1)}(0),\bar q^{(2)}(0)\right)$.

The role of the fast phase becomes transparent by evaluating the overlap dynamics at its initial state according to \eqref{eq:slow_tau_mu}:
\begin{equation}
\left. \frac{d\bar q_{12}}{d\tau} \right|_{\tau=0} = -2D_1(0)Q_2g_{12}^\star.
\end{equation}
Thus the fast dynamics affects feature learning by changing the point at which the gradient $g_{12}^\star$ is evaluated. If \(Q_2g_{12}^\star\neq0\), the slow trajectory leaves the uninformative state with nonzero velocity. The following corollary formalizes this intuition, proven in Appendix \ref{app:proof_untied_weak_recovery}, and Appendix \ref{app:untied-strong} gives a convergence rate for the subsequent strong-recovery phase under an additional Polyak–Łojasiewicz condition.

\begin{corollary}
\label{cor:untied_weak_recovery}
Suppose that Assumption \ref{assum:untied_initialization} and the
conditions of Theorem \ref{theo:untied-SGD} hold. 
\begin{itemize}
\item[\textnormal{(i)}] If
\begin{equation}
\mathsf Q_2g_{12}^\star\neq0,
\label{eq:untied-weak-recovery-condition}
\end{equation}
then there exist constants \(c>0\) and \(0<T_-<T_+<\infty\), independent of \(d\), such that
\begin{equation}
\lim_{d\to\infty}
\mathbb P\left(\left\lfloor\frac{T_-d}{\alpha_d}\right\rfloor<N_{\rm wr}^{(d)}(c)\le
\left\lceil\frac{T_+d}{\alpha_d}\right\rceil\right)=1.
\label{eq:untied-weak-recovery-hitting-bound}
\end{equation}
Consequently, choosing the largest learning-rate scaling $\alpha_d=\Theta\left(\frac1{d\log d}\right)$ covered by Theorem \ref{theo:untied-SGD}
gives
\begin{equation}
N_{\rm wr}^{(d)}(c)=\Theta_{\mathbb P}(d^2\log d).
\label{eq:untied-weak-recovery-d2logd}
\end{equation}

\item[\textnormal{(ii)}] Let \(\mathcal M_{\mathrm{free}}\) denote the set of moment sequences for which the algebra generated by \(\{M_1,M_1^T,A_1,B_1\}\) is free from the algebra generated by \(\{M_2,M_2^T,A_2,B_2\}\). Suppose that throughout \(\mathcal M_{\mathrm{free}}\),
\begin{equation}
\left[\nabla_{q^{(r)}}\Phi\left(m^\star(q),q^{(1)},q^{(2)}\right)\right]_{12}=0,
\qquad r=1,2.
\label{eq:untied-free-trap-condition}
\end{equation}
Then for every \(c>0\) and every fixed \(T>0\),
\begin{equation}
\lim_{d\to\infty}\mathbb P\left(N_{\mathrm{wr}}^{(d)}(c)\le\left\lfloor\frac{Td}{\alpha_d}\right\rfloor\right)=0.
\label{eq:untied-no-recovery-fixed-T}
\end{equation}
Thus, for \(\alpha_d=\Theta((d\log d)^{-1})\), weak recovery does not occur on the \(d^2\log d\) sample scale.
\end{itemize}
\end{corollary}

Corollary \ref{cor:untied_weak_recovery} identifies a dichotomy. Part (i) states that if the state selected by the fast dynamics exposes a nonzero gradient, weak recovery occurs after \(\Theta(d^2\log d)\) samples. Part (ii) gives the complementary invariant-manifold mechanism: if the relevant cross-gradients vanish throughout the uninformative manifold, then weak recovery does not occur on the same sample scale.

The contrast with tied attention is therefore sharp. In the tied case, the PSD geometry itself generates a nonzero escape direction from the original initialization. In the untied case, the fast phase first selects a new mean state, and it is the geometry of the loss at this selected state that determines whether the subsequent slow dynamics can escape.

The mechanism is particularly transparent for \(L=1\) and the MSE loss $\mathcal L(z_1,z_2) = \frac12\big(\sigma(z_1)-\sigma_\star(z_2)\big)^2$. At an uninformative state \(q_{12}=0\), the student and teacher pre-activations
$Z_1\sim \mathcal N \left( \bar m_1^\star, \bar q_{11}^{(1)}+\bar q_{11}^{(2)} \right)$ and $Z_2\sim \mathcal N \left( \bar m_2, \bar q_{22}^{(1)}+\bar q_{22}^{(2)} \right)$ are independent. Price's theorem gives
\begin{align}
\frac{\partial\Phi}{\partial q_{12}^{(1)}}=\frac{\partial\Phi}{\partial q_{12}^{(2)}}=\mathbb E\left[\frac{\partial^2\mathcal L}{\partial z_1\partial z_2}\right]=-\mathbb E[\sigma'(Z_1)]\mathbb E[\sigma_*'(Z_2)].
\label{eq:untied-weak-MSE-gradient}
\end{align}
Hence a sufficient condition for weak recovery is
\begin{equation}
\mathbb E[\sigma'(Z_1)]\mathbb E[\sigma_*'(Z_2)]\neq0.
\label{eq:untied-weak-activation-condition}
\end{equation}
The two factors play different roles. The teacher-side term
$\mathbb E[\sigma_\star'(Z_2)]$ determines whether a first-order teacher signal is available at all. The fast dynamics cannot change this factor because the teacher is fixed. By contrast, the student-side term $\mathbb E[\sigma'(Z_1)]$ is evaluated at the shifted mean \(\bar m_1^\star\) selected by the fast phase. The fast learning phase can therefore repair a student-side symmetry: even if \(\mathbb E[\sigma'(Z_1)]\) vanishes initially, it may become nonzero after the mean shifts.
When this occurs and the teacher-side first-order signal is nonzero, weak recovery takes \(\Theta(d/\alpha_d)\) samples. If $\bar m_1^*$ fails to break this symmetry, the network stalls.

\begin{wrapfigure}{r}{0.45\textwidth}
\vspace{-2mm}
    \centering
    \includegraphics[width=\linewidth]{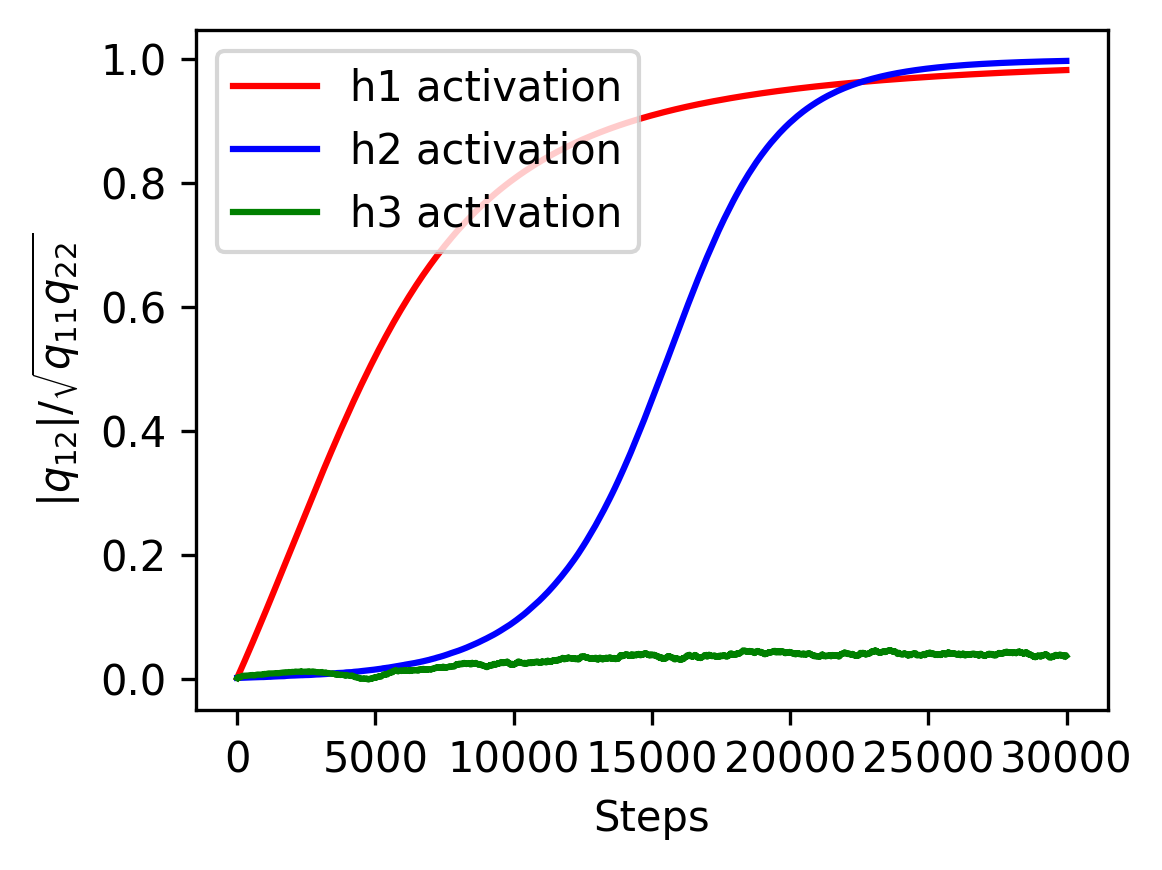}
    \caption{Weak recovery becomes delayed for higher-order activations. Evolution of the symmetric student–teacher overlap (the overlap between $M_1+M_1^T$ and $M_2+M_2^T$) for untied attention in a single-token teacher–student model with MSE loss. The linear activation \(h_1(x)=x\) develops overlap rapidly, whereas the higher-order activations \(h_2(x)=\frac{x^2-1}{\sqrt{2}}\) and \(h_3(x)=\frac{x^3-3x}{\sqrt{6}}\) exhibit increasingly slow escape from the uninformative state, whose precise sample complexity is left for future work. We use \(d=800\), learning rate \(0.01\), batch size \(4096\), standard Gaussian initialization, and no weight decay. The input data is $\mathcal{N}(0, I_d)$.}
    \label{fig:SGD_untied}
\vspace{-10mm}
\end{wrapfigure}

A numerical illustration is presented in Figure \ref{fig:SGD_untied}, which verifies that when $\mathbb{E}_{z_2}\big[ \sigma_*'(Z_2) \big]=0$, weak recovery takes substantially more samples. Figure \ref{fig:SGD_untied} also suggests that the weak recovery time depends on the information exponent of the activation functions, which is left for the future work.

\section{Conclusion and future work}
\vspace{-0.5em}
This work has two outcomes. First, we develop a high-dimensional framework for attention-indexed models with extensive-rank matrices. In the large-\(d\) limit, the population-loss landscape is characterized by a finite collection of trace order parameters. The corresponding training dynamics are substantially richer: online SGD generates an infinite hierarchy of matrix moments. We derive deterministic limits for these dynamics and show that the infinite hierarchy can nevertheless be approximated by finite-order truncations with exponentially decaying error.

Second, this framework reveals that the parameterization of the attention matrix can qualitatively change feature learning. For tied attention $S=WW^\top$, the PSD factorization provides an automatic symmetry-breaking mechanism, yielding weak recovery in \(\Theta(d^2\log d)\) samples under a mild condition. This behavior contrasts with direct optimization over \(S\), for which the uninformative manifold can remain invariant. For untied attention $S=UV^\top$, the mechanism is different. Training separates into a fast and a slower learning phase. The fast phase selects the state from which feature learning begins, and weak recovery on the \(\Theta(d^2\log d)\) scale depends on whether this selected state breaks the relevant symmetry. Thus, tied and untied factorizations do not merely provide different coordinates for the same optimization problem; they induce different routes by which SGD can escape an uninformative state.

While the attention-indexed model class is broad enough to represent attention-only architectures, the detailed training dynamics established here, however, only concern tied and untied parameterizations. We do not derive the full dynamics of all query, key, and value matrices in a general deep transformer. Extending the present approach to such parameter dynamics is a natural next step.

A second open direction concerns learning on a longer timescale. Corollary \ref{cor:untied_weak_recovery} characterizes the \(\Theta(d^2\log d)\) sample scale. However, the numerical results suggest that recovery can occur on longer scales controlled by the information exponents. Analyzing these regimes requires tracking finite-\(d\) corrections over polynomially long training horizons, since higher-order terms that vanish at $\Theta(1)$ time can accumulate over such scales. Developing such a theory would provide a more complete characterization of the parameterization-dependent sample complexity.

Several other extensions remain open. The present work studies online SGD, whereas empirical risk minimization, empirical gradient flow, and other finite-sample optimization procedures can exhibit additional correlations and require different dynamical tools; existing results cover only special cases of attention-indexed models \cite{erba2025nuclear,boncoragliosingle,martin2026high}. It would also be interesting to seek alternative representations of the infinite moment hierarchy. In particular, a resolvent-based formulation \cite{paquette20244+} may provide a more direct connection between the training dynamics and the evolving spectral distributions \cite{martin2021implicit,defilippis2025scaling} of the attention matrices.

More broadly, the results suggest that the parameterization of attention should be regarded as part of the learning mechanism itself. The same loss landscape can support qualitatively different training trajectories depending on the parameterization. Understanding this interaction between architecture, parameterization, and high-dimensional learning dynamics is, in our view, an important step toward a theory of more realistic attention models.

\section*{Acknowledgements}

We acknowledge funding from the Swiss National Science Foundation grants SNSF SMArtNet (grant number 212049), OperaGOST (grant number 200021 200390), and DSGIANGO (grant number 225837). This work was supported by the Simons Collaboration on the Physics of Learning and Neural Computation via the Simons Foundation grant (\#1257412 (FK) and \#1257413 (LZ)). This work was also supported by the EPFL AI Center PhD Fellowship Program 2026 (YX).

\bibliographystyle{plain}
\bibliography{main}
\appendix

\newpage

\tableofcontents

\section{Notation}

We write $[K]:=\{1,\ldots,K\}$ and let $\mathcal K\subseteq[K]$ denote the set of trainable indices. For a multi-index or word $\alpha$, $|\alpha|$ denotes its length,
while $|\alpha|_{\mathcal K}$ denotes the number of positions whose associated index belongs to $\mathcal K$. The underlying finite alphabet of $\alpha$ will be clear from context.

For matrices, $\|\cdot\|_{\rm op}$ and $\|\cdot\|_{\rm F}$ denote the
operator and Frobenius norms, respectively, and $\langle A,B\rangle_F:=\Tr(A^\top B)$.
We write $\mathbb S^K$ for the space of real symmetric $K\times K$
matrices and $\mathbb S_+^K$ for the space of positive-semidefinite matrices. For symmetric matrices, $A\succeq B$ means that $A-B$ is
positive semidefinite. For finite-dimensional vectors, $\|\cdot\|_\infty$ and $\|\cdot\|$ denote the $\ell_\infty$ and
Euclidean norms.

We write $\mathbf 1_E$ for the indicator of an event $E$, $a\wedge b:=\min\{a,b\}$, and $[a]_+:=\max\{a,0\}$. For sets $A\subseteq B$, the notation $A\Subset B$ means that $\overline A$ is a compact subset of $B$.

We use the standard asymptotic notation $O(\cdot)$, $o(\cdot)$,
$\Omega(\cdot)$, and $\Theta(\cdot)$ as $d\to\infty$.
For random variables, $X_d=O_{\mathbb P}(a_d)$ means that
$X_d/a_d$ is bounded in probability,
$X_d=o_{\mathbb P}(a_d)$ means that
$X_d/a_d\xrightarrow{\mathbb P}0$, and
$X_d=\Theta_{\mathbb P}(a_d)$ means that both
$X_d=O_{\mathbb P}(a_d)$ and $a_d=O_{\mathbb P}(X_d)$.
We write $X_d\xrightarrow{\mathbb P}X$ for convergence in probability.

For $s>0$, let $X_s$ denote the Banach space of moment sequences $\mu=(\mu_\alpha)_{|\alpha|\ge1}$ equipped with the norm $\|\mu\|_s:=\sup_{w\ge1}\sup_{|\alpha|=w}\frac{|\mu_\alpha|}{s^w}$ and define $X_s^0:=\left\{\mu\in X_s:\lim_{w\to\infty}\sup_{|\alpha|=w}\frac{|\mu_\alpha|}{s^w}=0\right\}$. Equipped with the norm inherited from $X_s$, $X_s^0$ is a separable
Banach space. $C([0, T]; X_s^0)$ denotes the space of continuous functions from $[0, T]$ to $X_s^0$. $D([0, T]; X_s^0)$ represents the space of all right-continuous functions with left limits mapping from $[0, T]$ into $X_s^0$, equipped with the Skorokhod $J_1$ topology.

Throughout the paper, a superscript $(d)$ denotes a finite-dimensional quantity, a bar denotes its deterministic high-dimensional limit, and a tilde denotes a piecewise-constant time interpolation, unless specified otherwise.

\section{Examples of attention-indexed models}
\label{app:examples}
\subsection{Example 1: Multi-layer multi-head attention}
In this example, we demonstrate that the standard multi-layer multi-head attention network (attention-only transformer) falls within our general theoretical framework.

Let the input to the model be Gaussian $X_0 \sim \mathcal{N}(0, \mathcal{C} \otimes I_d)$. For layer $m = 1, \dots, M$ with $H$ attention heads, the pre-softmax attention matrix for the $h$-th head is given by
\begin{equation}
A_m^{(h)}(X_{m-1}) = \frac{1}{\sqrt{d}} X_{m-1} S_{m-1}^{(h)} X_{m-1}^T \in \mathbb{R}^{L \times L}    
\end{equation}
where $S_{m-1}^{(h)} := K_{m-1}^{(h)}(Q_{m-1}^{(h)})^T \in \mathbb{R}^{d \times d}$ is the combined key-query weight matrix. The intermediate representations follow the standard multi-head attention mechanism:
\begin{equation}
X_m = X_{m-1} + \sum_{h=1}^H \text{softmax}\left(A_m^{(h)}(X_{m-1})\right) X_{m-1} V_{m-1}^{(h)} \in \mathbb{R}^{L \times d},
\end{equation}
where $V_{m-1}^{(h)} \in \mathbb{R}^{d \times d}$ represents the combined value and output projection matrix for head $h$. Unrolling the recurrence relation reveals that this network can be expressed as a composite function of the quadratic forms $G_k := \frac{1}{\sqrt{d}}X_0 \tilde{S}_k X_0^T$.

For the first layer ($m=1$), the pre-softmax matrix constitutes a bilinear form: 
\begin{equation}
A_1^{(h)}(X_0) = \frac{1}{\sqrt{d}} X_0 S_0^{(h)} X_0^T := G_0^{(h)}.    
\end{equation}
The output of the first layer is thus $X_1 = X_0 + \sum_{h=1}^H \text{softmax}(G_0^{(h)}) X_0 V_0^{(h)}$.

For the second layer ($m=2$), substituting $X_1$ into the pre-softmax attention matrix $A_2^{(h)}(X_1) = \frac{1}{\sqrt{d}} X_1 S_1^{(h)} X_1^T$ yields a sum of four types of cross terms:
\begin{equation}
\begin{aligned}
A_2^{(h)}(X_1) =& \frac{1}{\sqrt{d}} X_0 S_1^{(h)} X_0^T + \sum_{i=1}^H \text{softmax}(G_0^{(i)}) \left( \frac{1}{\sqrt{d}} X_0 \left( V_0^{(i)} S_1^{(h)} \right) X_0^T \right) \\
&+ \sum_{j=1}^H \left( \frac{1}{\sqrt{d}} X_0 \left( S_1^{(h)} V_0^{(j)T} \right) X_0^T \right) \text{softmax}(G_0^{(j)})^T\\
&+ \sum_{i=1}^H \sum_{j=1}^H \text{softmax}(G_0^{(i)}) \left( \frac{1}{\sqrt{d}} X_0 \left( V_0^{(i)} S_1^{(h)} V_0^{(j)T} \right) X_0^T \right) \sigma(G_0^{(j)})^T.
\end{aligned}    
\end{equation}
Every inner term encapsulated by parentheses in the expansion above takes the form $\frac{1}{\sqrt{d}}X_0 \tilde{S} X_0^T$. For instance, the attention outputs cross term forms a new bilinear form defined by the effective weight matrix $\tilde{S} := V_0^{(i)} S_1^{(h)} V_0^{(j)T}$.

By induction, any attention matrix $A_m$, and consequently the final output representation, can be written as a finite composite function of a set of base bilinear forms $\{G_k\}_{k=1}^K$ generated by equivalent effective weight matrices $\tilde{S}_k$. These effective matrices $\tilde{S}_k$ are finite products of the original network weights ($S$, $V$), which falls within the scope of attention-indexed models.

Furthermore, because the softmax function is continuous and bounded, the composite function remains continuous and exhibits at most polynomial growth. Therefore, when we assume that both the teacher and the student models are the multi-layer multi-head attention architecture, the final loss satisfies Assumption \ref{assum:attention-index}. Note that we do not require the teacher weights to be Gaussian, making the existence of such a teacher model a weak assumption.

\subsection{Example 2: Autoregressive sequence model}
This example considers the most fundamental pretraining task: next-token prediction. We aim to predict the next token $x_L$ using the context $x_1, \dots, x_{L-1}$ via a simplified single-layer multi-head attention model.

The student is a standard attention block, defined as:
\begin{equation}
\hat{x}_L = \sum_{h=1}^H \sum_{j=1}^{L-1} \alpha_{j,h} V_h x_j, \quad \text{where } \alpha_{j,h} = \text{Softmax}_j\left(\{ \frac{1}{\sqrt{d}} x_{L-1}^T W^h x_k \}_{k=1}^{L-1} \right).
\end{equation}
Here, we absorb the query-key matrices into a single weight matrix $W^h$.

We assume the data is generated by a sparse contextual dependency process. The next token $x_L$ is a composition of specific past tokens that satisfy a relevance criterion
\begin{equation}
x_L = \sum_{h=1}^{H} \sum_{j=1}^{L-1} a_{j,h} V_h^* x_j,
\end{equation}
where $a_{j,h} \in \{0, 1\}$ is a gating variable determining if token $j$ is relevant for generating $L$ under head $h$:
\begin{equation}
a_{j,h} \sim \text{Bernoulli}\left( g\left( \left\{ \frac{1}{\sqrt{d}} x_i^T M^h x_j \right\}_{i,j=1}^{L-1} \right) \right).
\end{equation}
This models a scenario where the ground truth dependency structure is latent and dynamic, determined by the interaction between tokens.

The MSE loss can be expanded as:
\begin{equation}
\mathcal{L} := \| \hat{x}_L - x_L \|^2 = \sum_{j,j'} x_{j}^T \Omega_{jj'} x_{j'},
\end{equation}
where $\Omega_{jj'}:=\sum_{h,h'}(a_{j,h} V_h^* - \alpha_{j,h} V_h)^T (a_{j',h'} V_{h'}^* - \alpha_{j',h'} V_{h'})$. It is thus an attention-indexed model.

We can similarly analyze the standard next-token prediction MSE loss
\begin{equation}
\mathcal{L} :=\mathbb{E}_{L,x_1,\cdots,x_L}\| \hat{x}_L - x_L \|^2=\mathbb{E}_L\mathbb{E}_{\mathcal{C}|L}\mathbb{E}_{(x_1,\cdots,x_{L-1})\sim\mathcal N(0,\mathcal{C})}\sum_{j,j'}^{L-1} x_{j}^T \Omega_{jj'} x_{j'},
\end{equation}
which fits the general framework.

\subsection{Example 3: Multiple-location regression}
This example considers a classification task (e.g., sentiment analysis), where the goal is to predict the label $y \in \{+1, -1\}$ based on the interaction between tokens (e.g., detecting negation or intensification). We model the student predictor as a one-layer attention:
\begin{equation}
\hat{y}=\text{sign}\left(\sum_{i,j} \alpha_{ij} \sigma\left(\{x_i^T W^h x_j\}_{ij}\right)\right),
\end{equation}
where the model learns $\alpha_{ij}$ to weight different token pairs. Moreover, we assume that the true label is generated by a "bag-of-interactions" model:
\begin{equation}y = \text{sign}\left( \sum_{i,j} S_{ij} \cdot \epsilon_{ij} \right),
\end{equation}
where $S_{ij} \in \{+1, -1\}$ represents the inherent sentiment polarity of the word pair $(i, j)$ (e.g., $+1$ for "very good", $-1$ for "not good"), and $\epsilon_{ij}$ indicates whether a meaningful grammatical connection exists between them:\begin{equation}\epsilon_{ij} = \begin{cases}1, & \text{w.p. } g(\{x_i^T W^* x_j\}_{ij}) \\
0, & \text{otherwise}.
\end{cases}
\end{equation}
In this setting, $\epsilon_{ij}$ can be interpreted as the ground-truth dependency graph (e.g., $i$ modifies $j$), which is activated probabilistically by the underlying semantic matching $W^*$.

\subsection{Example 4: Multi-step reasoning}
This example models complex tasks requiring multi-step reasoning, such as in-context retrieval and composition (e.g., "Context: A:1, B:2, C:3. Query: Calculate A+B").

We model the ground truth $y$ as a hierarchical interaction between two latent tokens. The teacher first identifies two relevant tokens from the context based on the query $x_L$, and then computes their quadratic interaction:
\begin{equation}
y = x_{\epsilon_1}^T U^* x_{\epsilon_2},
\end{equation}
where the indices $\epsilon_1, \epsilon_2 \in \{1, \dots, L-1\}$ are latent pointers. These pointers follow a probabilistic selection mechanism (e.g., a "ground truth" attention):
\begin{equation}
P(\epsilon_h=k) \propto \exp\left(\frac{1}{\sqrt{d}}x_L^T W_h^* x_k\right), \quad \text{for } h \in \{1, 2\}.
\end{equation}
This represents the "reasoning logic" where $W_h^*$ retrieves the necessary operands and $U^*$ defines the operation performed on them.

To solve this, the student must implement a two-stage process: first filtering/aggregating the context into latent representations, and then combining them. We model the student as a two-head attention structure followed by a bilinear composition layer. The student first generates two latent vectors $z_1$ and $z_2$ by aggregating the context tokens $\{x_k\}$:
\begin{equation}
z_h = \sum_{j=1}^{L-1} \alpha_{j,h} x_j, \quad \text{where } \alpha_{j,h} = \text{Softmax}_j \left( \left\{ \frac{1}{\sqrt{d}} x_L^T W^h x_k \right\}_{k=1}^{L-1} \right).\end{equation}
Here, $z_h$ serves as the "retrieved operand" for head $h$. The final prediction is computed as the interaction between these aggregated representations:
\begin{equation}
\hat{y} =z_1^T \tilde{W} z_2,
\end{equation}
which takes the form of an attention-indexed model.

\subsection{Example 5: Extensive-rank matrix denoising}
Consider the observation model:
\begin{equation}
Y = S^* + Z \in \mathbb{R}^{d \times d},
\end{equation}
where $S^*$ is an unknown deterministic symmetric matrix with an extensive rank (i.e., $\lim_{d\to\infty} \|S^*\|_{\text{op}} / \|S^*\|_F = 0$), and $Z$ is the noise matrix drawn from the Gaussian Orthogonal Ensemble (GOE), satisfying $\mathbb{E}[Z_{ab}] = 0$ and $\mathbb{E}[Z_{ab}Z_{cd}] = \frac{1}{d}(\delta_{ac}\delta_{bd} + \delta_{ad}\delta_{bc})$. While prior work has extensively studied rotationally invariant priors \cite{troiani2022optimal,pourkamali2024matrix,pourkamali2025rectangular}, the general extensive-rank setting remains largely open \cite{barbier2025phase}.

A standard approach for inferring $S^*$ from $Y$ is via:
\begin{equation*}
\min_{S} \mathcal{L}\left(\frac{1}{d}\Tr[YS]\right) + \lambda f(S),
\end{equation*}
where $\mathcal{L}: \mathbb{R} \to \mathbb{R}$ is the loss function and $f(S)$ represents a regularizer.
The population loss can be expressed as an expectation over a standard scalar Gaussian variable $g \sim \mathcal{N}(0, 1)$:
\begin{equation}
\mathbb{E}_Z \mathcal{L}\left(\frac{1}{d}\Tr[YS]\right) = \mathbb{E}_{g \sim \mathcal{N}(0,1)} \left[ \mathcal{L}\left( \frac{1}{d}\Tr[S^*S] + \frac{\sqrt{2}}{d^{3/2}}\|S\|_F g \right) \right],
\end{equation}
which can be regarded as a single-token single-index attention-indexed model as a special case of Theorem \ref{theo:general} and Corollary \ref{cor:symmetric}.

\section{Theory in Section \ref{sec:models}}
\subsection{Proof of Theorem \ref{theo:general}}
\label{app:proof_theo_general}
The proof relies on the method of joint cumulants. If all joint cumulants of order $m \ge 3$ for a set of random variables vanish in the limit, the variables converge in distribution to a multivariate Gaussian (satisfying Carleman's condition).

Let $Z_{ijk} = \frac{1}{\sqrt{d}}x_i^TS^k_{ij}x_j$. We concatenate all data vectors into a single vector $X = [x_1^T, x_2^T, \dots, x_L^T]^T \in \mathbb{R}^{dL}$. Given the covariance structure $\mathbb{E}[x_{ia}x_{jb}] =\mathcal{C}_{ij}\delta_{ab}$, we have $X \sim \mathcal{N}(0, \Sigma)$, where the covariance matrix is the Kronecker product $\Sigma = \mathcal{C} \otimes I_d$. We can rewrite $Z_{ijk}$ as a quadratic form:
\begin{equation}
Z_{ijk} = X^T A^{(ijk)} X,
\end{equation}
where $A^{(ijk)}$ is a $dL \times dL$ symmetric block matrix. Its $(i, j)$-th block is $\frac{1}{2\sqrt{d}}S^k_{ij}$, its $(j, i)$-th block is $\frac{1}{2\sqrt{d}}(S^k_{ij})^T$, and all other blocks are zero.

By the standard expectation of Gaussian quadratic forms, the first moment is:
\begin{equation}
\mathbb{E}[Z_{ijk}] = \Tr(A^{(ijk)}\Sigma) = \frac{1}{\sqrt{d}}\sum_{a=1}^d \mathbb{E}[x_{ia}x_{ja}](S^k_{ij})_{aa} = \frac{\mathcal{C}_{ij}}{\sqrt{d}}\Tr(S^k_{ij}).
\end{equation}
By Assumption \ref{assum:attention-index}.2, as $d \to \infty$, this converges to $\mathbb{E}[G_{ijk}] = \mathcal{C}_{ij}\mu_{ijk}$.

To calculate the covariance, we use Wick's Theorem:
\begin{equation}
\text{Cov}(Z_{ijk}, Z_{i'j'k'}) = 2\Tr(A^{(ijk)}\Sigma A^{(i'j'k')}\Sigma).    
\end{equation}
Expanding this into block components yields:
\begin{equation}
\begin{aligned} \text{Cov}(Z_{ijk}, Z_{i'j'k'}) &= \frac{1}{d}\sum_{a,b,c,e}(S^k_{ij})_{ab}(S^{k'}_{i'j'})_{ce}\left(\mathcal{C}_{ii'}\mathcal{C}_{jj'}\delta_{ac}\delta_{be} + \mathcal{C}_{ij'}\mathcal{C}_{ji'}\delta_{ae}\delta_{bc}\right)\\ &= \mathcal{C}_{ii'}\mathcal{C}_{jj'}\frac{\Tr(S^k_{ij}(S^{k'}_{i'j'})^T)}{d} + \mathcal{C}_{ij'}\mathcal{C}_{ji'}\frac{\Tr(S^k_{ij}S^{k'}_{i'j'})}{d}. \end{aligned}    
\end{equation}
By Assumption \ref{assum:attention-index}.2, as $d \to \infty$, this converges to the finite constant $\mathcal{C}_{ii'}\mathcal{C}_{jj'}\Omega_{ijk,i'j'k'} + \mathcal{C}_{ij'}\mathcal{C}_{ji'}\Psi_{ijk,i'j'k'}$.

Now we calculate the joint cumulant $\kappa_m:=\kappa(Q_1,\cdots,Q_m)$ for $Q_r = X^T A_r X$, where we denote $A_r := A^{(i_r j_r k_r)}$ for $r \in \{1, \dots, m\}$. The joint cumulant generating function reads
\begin{equation}
K(t_1, \dots, t_m) = \log  \mathbb{E}\left[ \exp\left( \sum_{r=1}^m t_r Q_r \right) \right] = \log  \mathbb{E}\left[\exp(X^T\tilde{A}X)\right],
\end{equation}
where $\tilde{A} = \sum_{r=1}^m t_r A_r$. Using the standard Gaussian integral for quadratic forms in the neighborhood of $(t_1, \dots, t_m) = 0$, the expectation can be calculated as
\begin{equation}
K(t_1, \dots, t_m) = -\frac{1}{2} \log  \det(I - 2 \Sigma \tilde{A}).
\end{equation}
Using the identity $\log \det(I-M) = -\sum_{n=1}^\infty \frac{1}{n}\Tr(M^n)$, the right side can be expanded into a series of traces:
\begin{equation}
K(t_1, \dots, t_m) = \sum_{n=1}^\infty \frac{2^{n-1}}{n} \Tr\left( (\Sigma \tilde{A})^n \right) = \sum_{n=1}^\infty \frac{2^{n-1}}{n} \Tr\left( \left( \sum_{r=1}^m t_r \Sigma A_r \right)^n \right).    
\end{equation}
Then the joint cumulant is given by $\kappa_m := \left. \frac{\partial^m K}{\partial t_1 \cdots \partial t_m} \right|_{t_1=\dots=t_m=0}$. The only term in the series that contributes to this cross-derivative is $n=m$. Therefore, after expanding the $m$-th power, $\kappa_m$ is a linear combination of traces over all permutations $\pi$ in the symmetric group $\mathcal{S}_m$:
\begin{equation}
\kappa_m = c_m \sum_{\pi \in \mathcal{S}_m} \Tr\left(A_{\pi(1)}\Sigma A_{\pi(2)}\Sigma \dots A_{\pi(m)}\Sigma\right),    
\end{equation}
where $c_m$ is a constant depending only on $m$. Since $m$ is fixed, the sum consists of a finite number of terms. We can bound the absolute value of each individual trace in the sum using the matrix norm inequality $|\Tr(B_1B_2\dots B_m)| \le \left(\prod_{r=1}^{m-2}||B_r||_{op}\right) ||B_{m-1}||_F ||B_m||_F$.

Let $B_r = A^{(i_r j_r k_r)}\Sigma$ and let us analyze the asymptotic scales of these norms:  $||\Sigma||_{op} = \lambda_{\max}(\mathcal{C} \otimes I_d) = \lambda_{\max}(\mathcal{C}) = \Theta(1)$. $||A^{(ijk)}||_F = \frac{1}{\sqrt{2d}}||S^k_{ij}||_F = \mathcal O(1)$, based on Assumption \ref{assum:attention-index}.2. From Assumption \ref{assum:attention-index}.1, $||S^k_{ij}||_{op} = o(||S^k_{ij}||_F) = o(\sqrt{d})$. Therefore, $||A^{(ijk)}||_{op} = \frac{1}{2\sqrt{d}}||S^k_{ij}||_{op} = o(1)$.

Thus, for $B_r$, the spectral norm is bounded by $||B_r||_{op} \le ||A^{(i_r j_r k_r)}||_{op}||\Sigma||_{op} = o(1) \cdot \mathcal O(1) = o(1)$, and the Frobenius norm is $\mathcal{O}(1)$. Applying this to the trace inequality for any $m \ge 3$:
\begin{equation}
\kappa_m = O\left( [o(1)]^{m-2} \cdot \mathcal{O}(1) \cdot \mathcal{O}(1) \right) = o(1).    
\end{equation}
Therefore, all joint cumulants of order 3 and higher vanish as $d \to \infty$.

Because all higher-order cumulants vanish while the first two cumulants converge to finite limits, the vector of quadratic forms $\{Z_{ijk}\}_{i,j,k=1}^{L,L,K}$ converges weakly in distribution to the multivariate Gaussian vector $\{G_{ijk}\}_{i,j,k=1}^{L,L,K}$.

Finally, to pass the limit inside the expectation, we need to show that the sequence $\mathcal{L}(\{Z_{ijk}\}_{i,j,k=1}^{L,L,K})$ is uniformly integrable. By Assumption \ref{assum:attention-index}.3, $\mathcal{L}$ is a continuous function with at most polynomial growth, meaning there exist constants $c > 0$ and $p > 0$ such that $|\mathcal{L}(Z)| \le c(1 + ||Z||^p)$. Therefore, it suffices to show that the random tensor $Z = \{Z_{ijk}\}_{i,j,k=1}^{L,L,K}$ has bounded joint moments of all orders (up to $p+1$) as $d \to \infty$.

Recall that any joint moment of a set of random variables can be expressed as a finite polynomial of their joint cumulants. From our previous derivations, the joint cumulants $\kappa_m$ converge to finite limits (specifically, bounded constants for $m \le 2$ and $0$ for $m \ge 3$). Because all joint cumulants of $Z$ are uniformly bounded with respect to $d$, it immediately follows that all joint moments of $Z$ (up to $p+1$) are also uniformly bounded.

This uniform boundedness of all moments guarantees that $\mathcal{L}(\{Z_{ijk}\}_{i,j,k=1}^{L,L,K})$ is uniformly integrable. Thus, we can exchange the limit and the expectation, yielding:
\begin{equation}
\lim_{d\to\infty} \mathbb{E}_x[\mathcal{L}(\{Z_{ijk}\}_{i,j,k=1}^{L,L,K})] = \mathbb{E}[\mathcal{L}(\{G_{ijk}\}_{i,j,k=1}^{L,L,K})].    
\end{equation}
This completes the proof.

\subsection{Tied attention}
\label{app:cor-symmetric}
The assumption $\lim_{d\to\infty} \frac{1}{\sqrt{d}}\text{Tr}(S^k_{ij}) = \mu_{ijk}$ in Theorem \ref{theo:general} might not be satisfied for certain architectures, particularly those with tied weights. For instance, if the attention mechanism uses symmetric tied weights (e.g., $S = WW^T$), its trace scales proportionally to the dimension $d$, causing the term $\frac{1}{\sqrt{d}}\text{Tr}(S)$ to diverge as $d \to \infty$. To accommodate such cases, we introduce a centered variant of the attention-indexed model.
\begin{corollary}
Under the Assumption \ref{assum:attention-index} except that the condition $\lim_{d\to\infty}\frac{1}{\sqrt{d}}\Tr(S^k_{ij})=\mu_{ijk}$ is dropped, we have
\begin{equation}
\lim_{d\to\infty} \mathbb{E}_x\left[\mathcal{L}\left(\left\{\frac{\Tr[S^k_{ij}(x_jx_i^T-\mathcal{C}_{ij}I_d)]}{\sqrt{d}}\right\}_{i,j,k=1}^{L,L,K}\right)\right] = \mathbb{E}\left[\mathcal{L}(\{G_{ijk}\}_{i,j,k=1}^{L,L,K})\right]
\end{equation}
where $\{G_{ijk}\}_{i,j,k=1}^{L,L,K}$ is a multivariate Gaussian vector whose mean and covariance are given by:
\begin{equation}
\mathbb{E}[G_{ijk}] = 0,\
\text{Cov}(G_{ijk}, G_{i'j'k'}) = \mathcal{C}_{ii'}\mathcal{C}_{jj'}\Omega_{ijk,i'j'k'} + \mathcal{C}_{ij'}\mathcal{C}_{ji'}\Psi_{ijk,i'j'k'}.
\end{equation}
\label{cor:symmetric}
\end{corollary}
\begin{proof}
Let $Y_{ijk} = \frac{1}{\sqrt{d}}\Tr[S^k_{ij}(x_jx_i^T-\mathcal{C}_{ij}I_d)]$. We can rewrite this as a centered quadratic form:
\begin{equation}
Y_{ijk} = \frac{1}{\sqrt{d}}(x_i^TS^k_{ij}x_j - \mathcal{C}_{ij}\Tr(S^k_{ij})) = Z_{ijk} - \mathbb{E}[Z_{ijk}],
\end{equation}
where $Z_{ijk} = \frac{1}{\sqrt{d}}x_i^TS^k_{ij}x_j$ as defined in the proof of Theorem \ref{theo:general}.

By construction, the first moment of $Y_{ijk}$ is zero for any $d$: $\mathbb{E}[Y_{ijk}] = 0$. Because $Y_{ijk}$ is a deterministic shift of $Z_{ijk}$, all joint cumulants of order $m \ge 2$ are identical to those of $Z_{ijk}$. Specifically, the covariance structure remains:
\begin{equation}\lim_{d\to\infty}\text{Cov}(Y_{ijk}, Y_{i'j'k'}) = \lim_{d\to\infty}\text{Cov}(Z_{ijk}, Z_{i'j'k'}) = \mathcal{C}_{ii'}\mathcal{C}_{jj'}\Omega_{ijk,i'j'k'} + \mathcal{C}_{ij'}\mathcal{C}_{ji'}\Psi_{ijk,i'j'k'}.
\end{equation}
Furthermore, for any order $m \ge 3$, the joint cumulants of $\{Y_{ijk}\}$ match those of $\{Z_{ijk}\}$ and thus vanish as $o(1)$, following Theorem \ref{theo:general}.

Since the first two cumulants converge to finite limits and all higher-order cumulants vanish, the sequence of random vectors $\{Y_{ijk}\}_{i,j,k=1}^{L,L,K}$ converges weakly in distribution to $\{G_{ijk}\}_{i,j,k=1}^{L,L,K}$, which is a multivariate Gaussian vector with mean zero and the asymptotic covariance structure derived above.

Because the loss function $\mathcal{L}$ is continuous with at most polynomial growth, and $\{Y_{ijk}\}$ has bounded moments (due to the bounds on the spectral and Frobenius norms of $S^k_{ij}$ by Assumption \ref{assum:attention-index}), the sequence $\mathcal{L}(\{Y_{ijk}\}_{i,j,k=1}^{L,L,K})$ is uniformly integrable. This allows us to pass the limit inside the expectation:
\begin{equation}
\lim_{d\to\infty} \mathbb{E}_x\left[\mathcal{L}\left(\left\{Y_{ijk}\right\}_{i,j,k=1}^{L,L,K}\right)\right] = \mathbb{E}\left[\mathcal{L}(\{G_{ijk}\}_{i,j,k=1}^{L,L,K})\right].
\end{equation}
This completes the proof.
\end{proof}

\subsection{Proof of Corollary \ref{cor:minimum}}
\label{app:proof-cor-minimum}
Corollary \ref{cor:minimum} is proven under the following regularity assumption.
\begin{assumption}
\label{assum:cor-minimum}
For a given dimension $d$, let $\hat{S}^{(d)} = \{\hat{S}_{ij}^{k, (d)}\}_{i,j=1, k\in\mathcal{K}}^{L,L}$ be a sequence of minimizers of the population loss with respect to the learnable weights. Assume that as $d \to \infty$:
\begin{itemize}
\item[1.] The minimizers satisfy the spectral norm condition uniformly: $\lim_{d\to\infty} \|\hat{S}_{ij}^{k, (d)}\|_{\text{op}} / \|\hat{S}_{ij}^{k, (d)}\|_F = 0$ for all $k \in \mathcal{K}$.
\item[2.] The order parameters associated with both the minimizers and the fixed weights, denoted as $Q(\hat{S}^{(d)})$, converge to a valid point $\hat{Q} \in \mathcal{Q}$.
\end{itemize}
\end{assumption}

Now we prove Corollary \ref{cor:minimum}.

Let $S_{\mathcal{K}}$ denote the collection of learnable weights $\{S^k_{ij}\}_{k \in \mathcal{K}}$. Let the population loss in dimension $d$ be denoted by:
\begin{equation}
F_d(S_{\mathcal{K}}) := \mathbb{E}_x\left[\mathcal{L}\left(\left\{\frac{x_i^TS^k_{ij}x_j}{\sqrt{d}}\right\}_{i,j,k=1}^{L,L,K}\right)\right].
\end{equation}
We aim to prove that $\lim_{d\to\infty} \inf_{S_{\mathcal{K}}} F_d(S_{\mathcal{K}}) = \inf_{Q \in \mathcal{Q}} \mathcal{R}(Q)$ by establishing matching upper and lower bounds.

For any arbitrary $\epsilon > 0$, by the definition of the infimum, there exists $Q_\epsilon \in \mathcal{Q}$ such that:
\begin{equation}
\mathcal{R}(Q_\epsilon) \le \inf_{Q \in \mathcal{Q}} \mathcal{R}(Q) + \epsilon.
\end{equation}

Because $Q_\epsilon$ lies within the feasible domain $\mathcal{Q}$, there exists a  sequence of learnable matrices $\tilde{S}_{\mathcal{K}}^{(d)}$ such that, combined with the fixed matrices, the joint order parameters satisfy $\lim_{d\to\infty} Q(\tilde{S}_{\mathcal{K}}^{(d)}) = Q_\epsilon$, and the learnable matrices satisfy the norm condition $\lim_{d\to\infty} \|\tilde{S}_{ij}^{k, (d)}\|_{\text{op}} / \|\tilde{S}_{ij}^{k, (d)}\|_F = 0$.

Since the infimum over all possible learnable weights $S_{\mathcal{K}}$ must be less than or equal to the loss evaluated at this specific sequence $\tilde{S}_{\mathcal{K}}^{(d)}$, we have:
\begin{equation}
\inf_{S_{\mathcal{K}}} F_d(S_{\mathcal{K}}) \le F_d(\tilde{S}_{\mathcal{K}}^{(d)}).
\end{equation}
Taking the limit superior as $d \to \infty$ on both sides, and noting that the complete set of matrices (both $\tilde{S}_{\mathcal{K}}^{(d)}$ and the fixed matrices) satisfies all conditions of Theorem \ref{theo:general}, we can apply Theorem \ref{theo:general} to the right-hand side:
\begin{equation}
\limsup_{d\to\infty} \inf_{S_{\mathcal{K}}} F_d(S_{\mathcal{K}}) \le \lim_{d\to\infty} F_d(\tilde{S}_{\mathcal{K}}^{(d)}) = \mathcal{R}(Q_\epsilon) \le \inf_{Q \in \mathcal{Q}} \mathcal{R}(Q) + \epsilon.
\end{equation}
Taking the limit as $\epsilon \to 0$ yields the upper bound:
\begin{equation}
\limsup_{d\to\infty} \inf_{S_{\mathcal{K}}} F_d(S_{\mathcal{K}}) \le \inf_{Q \in \mathcal{Q}} \mathcal{R}(Q).
\label{eq:cor1-upper}
\end{equation}
Let $\hat{S}_{\mathcal{K}}^{(d)}$ be the sequence of minimizers such that $\inf_{S_{\mathcal{K}}} F_d(S_{\mathcal{K}}) = F_d(\hat{S}_{\mathcal{K}}^{(d)})$. According to Assumption \ref{assum:cor-minimum}.1, the sequence $\hat{S}_{\mathcal{K}}^{(d)}$ uniformly satisfies the norm ratio condition, and the combined order parameters converge to some state $\hat{Q} \in \mathcal{Q}$. We can then apply Theorem \ref{theo:general} to this joint sequence of minimizers and fixed weights:
\begin{equation}
\lim_{d\to\infty} F_d(\hat{S}_{\mathcal{K}}^{(d)}) = \mathcal{R}(\hat{Q}).
\end{equation}
Furthermore, since $\hat{Q}$ is an element of the feasible domain $\mathcal{Q}$, we have $\mathcal{R}(\hat{Q}) \ge \inf_{Q \in \mathcal{Q}} \mathcal{R}(Q)$. Therefore, taking the limit inferior, we obtain:
\begin{equation}
\liminf_{d\to\infty} \inf_{S_{\mathcal{K}}} F_d(S_{\mathcal{K}}) = \lim_{d\to\infty} F_d(\hat{S}_{\mathcal{K}}^{(d)}) = \mathcal{R}(\hat{Q}) \ge \inf_{Q \in \mathcal{Q}} \mathcal{R}(Q).
\label{eq:cor1-lower}
\end{equation}
Combining the upper bound \eqref{eq:cor1-upper} and the lower bound \eqref{eq:cor1-lower}, we get:
\begin{equation}
\lim_{d\to\infty} \inf_{S_{\mathcal{K}}} F_d(S_{\mathcal{K}}) = \inf_{Q \in \mathcal{Q}} \mathcal{R}(Q).
\end{equation}
This concludes the proof.

\section{Theory in Section \ref{sec:tied}}
\subsection{Proof of Theorem \ref{theo:tied-SGD}}
\label{app:proof-tied-SGD}
\subsubsection{Local well-posedness}
We first define a scale of the Banach spaces $X_s$ for $s > 0$, consisting of infinite-dimensional moment sequences $\mu = (\mu_\alpha)_{|\alpha| \ge 1}$ equipped with the norm $\|\mu\|_s := \sup_{n \ge 1} \sup_{|\alpha|=n} |\mu_\alpha|/s^n$.

By Price's theorem, the matrix derivative of the potential can be written as $\nabla\Phi(q)=\Gamma_{kl}(q)$ with
\begin{equation}
\Gamma_{kl}(q):=\frac12
\sum_{i,j,i',j'=1}^L(\mathcal C_{ii'}\mathcal C_{jj'}+\mathcal C_{ij'}\mathcal C_{ji'})
\,\mathbb E\left[\partial^2_{(ijk),(i'j'l)}\mathcal L(G)\right],
\qquad1\le k,l\le K.
\label{eq:Gamma-definition}
\end{equation}
See Lemma \ref{lemma:bound-gradient} for more detailed derivation.
A technical issue is that the potential \eqref{eq:potential-tied}, and hence the gradient $\Gamma$ \eqref{eq:Gamma-definition}, are well-defined only for $q\succeq0$, whereas an arbitrary element of $X_s$ need not
have a PSD second-moment matrix. To use the Ovsyannikov argument, we extend $\Gamma$ outside the PSD cone. Define
\begin{equation}
\mathscr P(q):=\Pi_{\mathbb S_+^K}
\left(\frac{q+q^T}{2}\right),
\end{equation}
where $\Pi_{\mathbb S_+^K}$ denotes the orthogonal projection onto the closed convex cone of PSD matrices, and set
\begin{equation}
\widetilde\Gamma(q):=\Gamma(\mathscr P(q)),
\qquad q\in\mathbb R^{K\times K}.
\label{eq:Gamma-extension}
\end{equation}
For $q\succeq0$, this extension agrees with the generalized gradient: $\widetilde\Gamma(q)=\Gamma(q)$.

Now we denote the flow as
\begin{equation}
\frac{d}{dt}\bar{\mu}(t) = V(\bar{\mu}(t)):= F(\bar{\mu}(t)) + G(\bar{\mu}(t)),
\label{eq:V-definition}
\end{equation}
where
\begin{equation}
[F(\mu)]_{(k_1, \dots, k_w)} := - \sum_{\substack{i=1 \\ k_i \in \mathcal{K}}}^w \sum_{l=1}^K \widetilde\Gamma(q)_{k_i l}
\left( \mu_{(k_1, \dots, k_{i-1}, l, k_i, \dots, k_w)} + \mu_{(k_1, \dots, k_i, l, k_{i+1}, \dots, k_w)} \right).
\end{equation}
and
\begin{equation}
[G(\bar\mu)]_{\alpha} := -\frac{\gamma}{2} |\alpha|_{\mathcal{K}} \bar\mu_{\alpha}
\end{equation}
are two unbounded operators in the Banach space $X_s$.

\begin{lemma}
\label{lemma:local_existence}
Given an initial condition $\bar\mu(0)\in X_{s_0}$ for some $s_0>0$, there exist $T_{\rm local}>0$ and a strictly increasing continuous function $s(t)$ with $s(0)=s_0$ such that \eqref{eq:V-definition} admits a unique solution $\bar\mu(t)\in X_{s(t)}$ for $t\in[0,T_{\rm local})$.
\end{lemma}

\begin{proof}
For every bounded subset of $\mathbb S_+^K$, $\Gamma(q)\in\mathbb S^K$ is bounded and Lipschitz continuous by definition \eqref{eq:Gamma-definition} and Assumption \ref{assum:tied}.1 (using Price's theorem). Therefore, for every $R>0$ there exist constants $C_\Gamma(R)$ and $L_\Gamma(R)$ such that
\begin{equation}
\sup_{\substack{q\succeq0,\|q\|_F\le R}}
\|\Gamma(q)\|_F\le C_\Gamma(R),
\end{equation}
and
\begin{equation}
\|\Gamma(q)-\Gamma(q')\|_F
\le L_\Gamma(R)\|q-q'\|_F
\end{equation}
whenever $q,q'\succeq0$ and $\|q\|_F,\|q'\|_F\le R$. Since the orthogonal projection onto the closed convex cone $\mathbb S_+^K$ satisfies $
\|\mathscr P(q)-\mathscr P(q')\|_F\le\|q-q'\|_F$,
$\widetilde\Gamma(q)=\Gamma(\mathscr P(q))$
is locally bounded and locally Lipschitz on $\mathbb R^{K\times K}$.

Now fix $0<s<s'$ and $R>0$, and let $B_R(X_s):=\{\mu\in X_s:\|\mu\|_s\le R\}$. For $\mu,\nu\in B_R(X_s)$,
\begin{equation}
|q_{kl}(\mu)|\le Rs^2,\qquad|q_{kl}(\mu)-q_{kl}(\nu)|\le s^2\|\mu-\nu\|_s,
\end{equation}
where we denote $q_{kl}(\mu)=\mu_{(k,l)}$ Hence, because $K$ is fixed, we have
\begin{equation}
|\widetilde\Gamma(q(\mu))|_F
\le C_{R,s},
\end{equation}
and
\begin{equation}
\|\widetilde\Gamma(q(\mu))
-\widetilde\Gamma(q(\nu))\|_F
\le
L_{R,s}\|\mu-\nu\|_s.
\end{equation}
for constants $C_{R,s}$ and $L_{R,s}$.

Now we prove that the vector field $V: X_s \to X_{s'}$ is bounded and locally Lipschitz.
For the linear term $G$, we have $|[G(\mu)]_\alpha-[G(\nu)]_\alpha| \le \frac{\gamma}{2}w|\mu_\alpha-\nu_\alpha|$, and thus
\begin{equation}
\|G(\mu)-G(\nu)\|_{s'} = \sup_{|\alpha|=w \ge 1} \frac{\gamma w |\mu_\alpha-\nu_\alpha|}{2(s')^w} \le \frac{\gamma}{2} \|\mu-\nu\|_s \sup_{w \ge 1} \left[w \left(\frac{s}{s'}\right)^w \right].
\end{equation}
Using the inequality $\sup_{w\ge 1} w x^w \le \frac{1}{e|\log  x|} \le \frac{1}{1-x}$ for $x \in (0, 1)$, we obtain:
\begin{equation}
\sup_{w\ge 1} w\left(\frac{s}{s'}\right)^w \le \frac{s'}{s' - s}.
\end{equation}
Thus, $\|G(\mu)-G(\nu)\|_{s'} \le \frac{\gamma s'}{2(s'-s)} \|\mu-\nu\|_s$.

For the non-linear term $F(\mu)$, the difference $[F(\mu)]_\alpha - [F(\nu)]_\alpha$ for a multi-index $\alpha$ of length $w$ is bounded by
\begin{align}
&|[F(\mu)]_\alpha - [F(\nu)]_\alpha| \\&\le \sum_{\substack{i=1 \\ k_i\in\mathcal{K}}}^w \sum_l \Bigg( \left|\widetilde{\Gamma}(q(\mu))_{k_il} - \widetilde{\Gamma}(q(\nu))_{k_il}\right| \big|\mu_{w+1}^{(1)} + \mu_{w+1}^{(2)}\big| + \left|\widetilde{\Gamma}(q(\nu))_{k_il}\right| \big|\big(\mu_{w+1}^{(1)} + \mu_{w+1}^{(2)}\big) - \big(\nu_{w+1}^{(1)} + \nu_{w+1}^{(2)}\big)\big| \Bigg).
\end{align}
Since the moments of length $w+1$ are bounded by $\|\mu\|_s s^{w+1} \le R s^{w+1}$ and their differences are bounded by $\|\mu - \nu\|_s s^{w+1}$, we can obtain:
\begin{equation}
|[F(\mu)]_\alpha - [F(\nu)]_\alpha| \le w K \left[ L_{R,s} \|\mu - \nu\|_s \left( 2R s^{w+1} \right) + C_{R,s}\left( 2 \|\mu - \nu\|_s s^{w+1} \right) \right].
\end{equation}
Computing the norm in $X_{s'}$, we obtain:
\begin{align}
\|F(\mu) - F(\nu)\|_{s'} &= \sup_{|\alpha|=w \ge 1} \frac{|[F(\mu)]_\alpha - [F(\nu)]_\alpha|}{(s')^w} \le 2K \left( R L_{R,s} + C_{R,s} \right) s \|\mu - \nu\|_s \sup_{w \ge 1} \left[ w \left(\frac{s}{s'}\right)^w \right].
\end{align}
Using the previously established inequality $\sup_{w\ge 1} w\left(\frac{s}{s'}\right)^w \le \frac{s'}{s' - s}$, we consequently arrive at:
\begin{equation}
\|F(\mu) - F(\nu)\|_{s'} \le \frac{2K \left( R L_{R,s} + C_{R,s} \right) s s'}{s' - s} \|\mu - \nu\|_s.
\end{equation}
Combining $F$ and $G$, the vector field $V:X_s\to X_{s'}$ is locally Lipschitz and bounded (using $V(0)=0$). Then Lemma \ref{lemma:local_existence} is proven by the Ovsyannikov theorem \cite{treves1968}.
\end{proof}

\subsubsection{Uniform spectral bounds}
Let $\mathcal{F}_n = \sigma(x^{(0)}, x^{(1)}, \dots, x^{(n-1)})$ be the filtration generated by the data samples up to step $n-1$. The weights $W_k^{(n)}$ are $\mathcal{F}_n$-measurable. In the online SGD, the update at step $n$ relies on a fresh sample $x^{(n)}$. Consider the stochastic gradient:
\begin{equation}
g_k^{(n)}(x^{(n)}, W) := \nabla_{W_k} \mathcal{L}(G^{(n)}(x^{(n)}, W)),
\end{equation}
and the expected gradient
$\bar{g}_k^{(n)}=\mathbb{E}_x \left[ \sum_{i,j=1}^L \partial_{(i,j,k)}\mathcal{L}(G^{(n)}) X_{ij} W_k\right]$ conditioned on $\mathcal{F}_n$.

\begin{lemma}
\label{lemma:bound-gradient}
Fix a constant $C_*>0$. Suppose that, at step $n$,
\begin{equation}
\max_{1\le l\le K}\|W_l^{(n)}\|_{\mathrm{op}}\le C_*.
\label{eq:bounded-spectrum}
\end{equation}
The conditional expected gradient $\bar g_k^{(n)}$ satisfies $\bar g_k^{(n)}=\frac{1}{d}H_k^{(n)}W_k^{(n)}+\mathcal E_{\mathrm{Stein},k}^{(n)}$, where
\begin{equation}
\|H_k^{(n)}\|_{\mathrm{op}}\le C_H(C_*),
\qquad
\|\mathcal E_{\mathrm{Stein},k}^{(n)}\|_{\mathrm{op}}
\le C_E(C_*)\,d^{-3/2},
\end{equation}
for constants $C_H(C_*)$ and $C_E(C_*)$ independent of $d$ and $n$.
\end{lemma}

\begin{proof}
Throughout the proof we condition on $\mathcal F_n$ and suppress the superscript $(n)$ whenever no confusion can arise. Conditional on $\mathcal F_n$, the matrices $\{W_l,S_l\}_{l=1}^K$ are deterministic, while $x=(x_1,\ldots,x_L)\sim\mathcal N(0,\mathcal C\otimes I_d).$

Recall that
\begin{equation}
G_{ijk}:=\frac{1}{\sqrt d}
\Tr\!\left[S_k(x_jx_i^T-\mathcal C_{ij}I_d)\right].
\label{eq:G-definition}
\end{equation}
Since $S_k=S_k^T$, we have $\nabla_{W_k}G_{ijk}=X_{ij}W_k$ with $X_{ij}:=\frac{1}{\sqrt d}
\left(x_jx_i^T+x_ix_j^T-2\mathcal C_{ij}I_d\right)$.
Consequently,
\begin{equation}
g_k=\sum_{i,j=1}^L
f_{ijk}(x)X_{ij}W_k,
\qquad
f_{ijk}(x):=\partial_{(ijk)}\mathcal L(G(x)),
\end{equation}
and hence
\begin{equation}
\bar g_k=\sum_{i,j=1}^L\mathbb E_x[f_{ijk}(x)X_{ij}]W_k.
\label{eq:mean-gradient-start}
\end{equation}
We evaluate $\mathbb E_x[f_{ijk}(x)X_{ij}]$ using the second-order Stein identity. For any $f\in C^2(\mathbb{R}^{dL})$ whose derivatives up to second order have at most polynomial growth\footnote{This is satisfied under Assumption \ref{assum:tied}.1.}, and for a centered Gaussian vector $x$ with covariance matrix $\Sigma$,
\begin{equation}
\mathbb E\!\left[f(x)(xx^T-\Sigma)\right]=\Sigma\,
\mathbb E[\nabla_x^2f(x)]\,\Sigma.
\end{equation}
Taking the $(j,i)$ block of this identity for
$\Sigma=\mathcal C\otimes I_d$ gives
\begin{equation}
\mathbb E_x\left[f_{ijk}(x)\left(x_jx_i^T-\mathcal C_{ji}I_d\right)\right]=\sum_{u,v=1}^L\mathcal C_{ju}\mathcal C_{vi}\,\mathbb E_x\left[\nabla_{x_u}\nabla_{x_v}f_{ijk}(x)\right].
\end{equation}
Using the definition of $X_{ij}$:
\begin{equation}
\label{eq:stein-expanded}
\mathbb E_x[f_{ijk}(x)X_{ij}]=\frac{1}{\sqrt d}
\sum_{u,v=1}^L\left(\mathcal C_{ju}\mathcal C_{vi}+\mathcal C_{iu}\mathcal C_{vj}\right)
\mathbb E_x\left[\nabla_{x_u}\nabla_{x_v}f_{ijk}(x)
\right].
\end{equation}
We next compute the Hessian of $f_{ijk}$. Introduce the shorthand
$a=(i,j,k)$ and let $b=(i',j',l)$ and $c=(i'',j'',m)$ as coordinates of the tensor $G$. By the chain rule,
\begin{equation}
\label{eq:f-hessian-chain}
\nabla_{x_u}\nabla_{x_v}f_a=\sum_b\partial^2_{a,b}\mathcal L(G)\,\nabla_{x_u}\nabla_{x_v}G_b+R_{uv,a},
\end{equation}
where
\begin{equation}
\label{eq:Ruv-definition}
R_{uv,a}:=\sum_{b,c}\partial^3_{a,b,c}\mathcal L(G)\,
\nabla_{x_u}G_b\left(\nabla_{x_v}G_c\right)^T.
\end{equation}
For $b=(i',j',l)$,
\begin{equation}
\nabla_{x_u}G_{i'j'l}=\frac{1}{\sqrt d}
\left(\delta_{ui'}S_lx_{j'}+\delta_{uj'}S_lx_{i'}
\right),
\label{eq:G-first-derivative}
\end{equation}
and
\begin{equation}
\nabla_{x_u}\nabla_{x_v}G_{i'j'l}=\frac{1}{\sqrt d}S_l\left(
\delta_{ui'}\delta_{vj'}+\delta_{uj'}\delta_{vi'}\right).
\label{eq:G-second-derivative}
\end{equation}
Substituting the first term in \eqref{eq:f-hessian-chain} into
\eqref{eq:stein-expanded}, and then using
\eqref{eq:G-second-derivative}, gives
\begin{align}
&\frac{1}{\sqrt d}
\sum_{u,v=1}^L
\left(\mathcal C_{ju}\mathcal C_{vi}+\mathcal C_{iu}\mathcal C_{vj}\right)\sum_{i',j',l}
\mathbb E_x
\left[\partial^2_{(ijk),(i'j'l)}\mathcal L(G)\right]
\nabla_{x_u}\nabla_{x_v}G_{i'j'l}
\nonumber\\
&\qquad =\frac{2}{d}
\sum_{i',j',l}
\left(\mathcal C_{ii'}\mathcal C_{jj'}+\mathcal C_{ij'}\mathcal C_{ji'}\right)
\mathbb E_x\left[\partial^2_{(ijk),(i'j'l)}\mathcal L(G)\right]S_l.
\label{eq:stein-main-term}
\end{align}
The factor $2$ comes from the two Kronecker-delta contributions in
\eqref{eq:G-second-derivative}.

It remains to control the remainder $R_{uv,a}$. Under the spectral bound $\max_l\|W_l\|_{\mathrm{op}}\le C_*$, we have
\begin{equation}
\|S_l\|_{\mathrm{op}}\le C_*^2,\qquad\|S_l\|_F\le \sqrt d\,C_*^2.
\label{eq:S-bounds}
\end{equation}
Consequently, each $G_{ijk}$ satisfies
\begin{equation}
\sup_d\max_{i,j,k}\mathbb E_x|G_{ijk}|^p\le C_p(C_*)
\label{eq:G-moment-bound}
\end{equation}
for every fixed $p<\infty$. Since $L$ and $K$ are fixed, the same is true for any polynomial function of the finite-dimensional tensor $G$. By Assumption \ref{assum:tied}.1, the third derivatives of $\mathcal L$ have at most polynomial growth. Hence, for every collection of indices,
\begin{equation}
\mathbb E_x\left|\partial^3_{a,b,c}\mathcal L(G)\right|^2\le C(C_*).
\label{eq:L3-moment-bound}
\end{equation}
Now fix arbitrary unit vectors $p,q\in\mathbb R^d$. Using
\eqref{eq:Ruv-definition} and \eqref{eq:G-first-derivative}, every term in
$p^T\mathbb E_x[R_{uv,a}]q$ is, up to constants depending only on $L,K$, and $\mathcal C$, bounded by $\frac{1}{d}\mathbb E_x\left[\left|\partial^3_{a,b,c}\mathcal L(G)\right|\left|p^TS_lx_r\right|\left|q^TS_mx_s\right|\right]$ for some $r,s,l,m$. By Hölder's inequality,
\begin{align}
\mathbb E_x\left[\left|\partial^3_{a,b,c}\mathcal L(G)
\right|\left|p^TS_lx_r\right|
\left|q^TS_mx_s\right|\right]\le\left(\mathbb E_x\left|\partial^3_{a,b,c}\mathcal L(G)\right|^2\right)^{1/2}
\left(\mathbb E_x|p^TS_lx_r|^4\right)^{1/4}\left(\mathbb E_x|q^TS_mx_s|^4\right)^{1/4}.
\end{align}
The first factor is bounded by \eqref{eq:L3-moment-bound}. Since $x_r$ and
$x_s$ are Gaussian and $\|S_l\|_{\mathrm{op}},
\|S_m\|_{\mathrm{op}}\le C_*^2$, the remaining two factors are also uniformly bounded. Therefore,
\begin{equation}
\left\|\mathbb E_x[R_{uv,a}]\right\|_{\mathrm{op}}\le\frac{C(C_*)}{d}.
\label{eq:Ruv-bound}
\end{equation}
Combining \eqref{eq:stein-expanded}, \eqref{eq:stein-main-term}, and \eqref{eq:Ruv-bound}, we obtain
\begin{equation}
\label{eq:fX-expansion}
\mathbb E_x[f_{ijk}(x)X_{ij}]=\frac{2}{d}
\sum_{i',j',l}\left(\mathcal C_{ii'}\mathcal C_{jj'}+\mathcal C_{ij'}\mathcal C_{ji'}\right)\mathbb E_x\left[\partial^2_{(ijk),(i'j'l)}
\mathcal L(G)\right]S_l+E_{ijk},\end{equation}
where, uniformly over \eqref{eq:bounded-spectrum},
\begin{equation}
\|E_{ijk}\|_{\mathrm{op}}\le C(C_*)d^{-3/2}.
\label{eq:Eijk-bound}
\end{equation}
Substituting \eqref{eq:fX-expansion} into
\eqref{eq:mean-gradient-start}, and defining $H_k$ as 
\begin{equation}
\label{eq:H-definition}
H_k^{(n)}:=2\sum_{i,j,i',j'=1}^L\sum_{l=1}^K\left(\mathcal C_{ii'}\mathcal C_{jj'}+\mathcal C_{ij'}\mathcal C_{ji'}\right)\mathbb E_x\left[\partial^2_{(ijk),(i'j'l)}\mathcal L(G^{(n)})\right]S_l^{(n)},
\end{equation}
gives
\begin{equation}
\bar g_k=\frac{1}{d}H_kW_k+\mathcal E_{\mathrm{Stein},k},
\end{equation}
where $\mathcal E_{\mathrm{Stein},k}:=\sum_{i,j=1}^L E_{ijk}W_k$. By \eqref{eq:bounded-spectrum} and
\eqref{eq:Eijk-bound} we have
\begin{equation}
\|\mathcal E_{\mathrm{Stein},k}\|_{\mathrm{op}}
\le C_E(C_*)d^{-3/2}.
\end{equation}
Finally, Assumption \ref{assum:tied}.1 and the uniform moment bound \eqref{eq:G-moment-bound} imply
\begin{equation}
\sup_{i,j,k,i',j',l}
\mathbb E_x\left|\partial^2_{(ijk),(i'j'l)}\mathcal L(G)\right|\le C(C_*).
\end{equation}
Using \eqref{eq:S-bounds}, together with the fact that $K$ and $L$ are fixed, we therefore obtain
\begin{equation}
\|H_k\|_{\mathrm{op}}\le C_H(C_*).\end{equation}
Restoring the superscript $(n)$ completes the proof.
\end{proof}

To ensure the trajectory does not blow up, we must prove that the spectral norm of the weights remains uniformly bounded across all training steps.

\begin{lemma}
\label{lemma:bound-tied-SGD}
Under the conditions of Theorem \ref{theo:tied-SGD}, there exist constants
$C_*>C_0$ independent of $d$ and $T$, such that for every fixed $T>0$,
\begin{equation}
\lim_{d\to\infty}\mathbb P\left(\sup_{0\le n\le N_T}\max_{1\le k\le K}\|W_k^{(n)}\|_{\mathrm{op}}\le C_*\right)
=1,
\label{eq:spectral-bound-lemma}
\end{equation}
where
\begin{equation}
N_T:=
\left\lfloor\frac{T}{\tau_d}\right\rfloor
=
\left\lfloor\frac{Td}{4\alpha_d}\right\rfloor .
\end{equation}
\end{lemma}
\begin{proof}
The weights with $k\notin\mathcal K$ remain fixed throughout training and satisfy
$\|W_k^{(n)}\|_{\mathrm{op}}\le C_0$. It therefore suffices to control
$k\in\mathcal K$.

Fix a constant $C_*>C_0$, to be chosen later, and define the stopping time
\begin{equation}
J:=\inf\left\{n\ge0:\max_{1\le k\le K}\|W_k^{(n)}\|_{\mathrm{op}}>C_*\right\}.
\label{eq:stopping-time-J}
\end{equation}
Since $\max_k\|W_k^{(0)}\|_{\mathrm{op}}\le C_0<C_*$ almost surely, $J\ge1$. On the event $\{n<J\}$,
\begin{equation}
\max_k\|W_k^{(n)}\|_{\mathrm{op}}\le C_*,
\qquad
\max_k\|S_k^{(n)}\|_{\mathrm{op}}\le C_*^2.
\label{eq:pre-stopping-bound}
\end{equation}

\paragraph{Step 1: Truncation of the fresh samples.} For every step $n$, define
\begin{equation}
\mathcal A_n:=\left\{\|x^{(n)}\|^2\le c_xd\right\}\cap\left\{\max_{i,j,k}|G_{ijk}^{(n)}|\le c_G\log d\right\},
\label{eq:data-truncation-event}
\end{equation}
where $c_x$ and $c_G$ are sufficiently large constants independent of $d$ and $n$.

Conditional on $\mathcal F_n$, the state $W^{(n)}$ is deterministic. On $\{n<J\}$, \eqref{eq:pre-stopping-bound} holds. Standard Gaussian concentration gives
\begin{equation}
\mathbf 1_{\{n<J\}}\mathbb P\left(\|x^{(n)}\|^2>c_xd\,\middle|\,\mathcal F_n\right)\le e^{-c_1d}.
\end{equation}
Moreover, conditional on $\mathcal F_n$, each $G_{ijk}^{(n)}$ is a centered quadratic form of $x^{(n)}$. To
make it explicit, let
$x^{(n)}=(\mathcal C^{1/2}\otimes I_d)z^{(n)}$, where
$z^{(n)}\sim\mathcal N(0,I_{Ld})$ is independent of $\mathcal F_n$. For
$i,j\in\{1,\ldots,L\}$, define $B_{ij}:=\frac12(e_ie_j^T+e_je_i^T)$,
where $\{e_i\}_{i=1}^L$ is the canonical basis of $\mathbb R^L$. Then from \eqref{eq:G-definition} we can write
\begin{equation}
G_{ijk}^{(n)}=(z^{(n)})^T A_{ijk}^{(n)}z^{(n)}-\Tr(A_{ijk}^{(n)}),
\label{eq:G-quadratic-form}
\end{equation}
where
\begin{equation}
A_{ijk}^{(n)}:=\frac{1}{\sqrt d}
(\mathcal C^{1/2}\otimes I_d)
(B_{ij}\otimes S_k^{(n)})
(\mathcal C^{1/2}\otimes I_d).
\label{eq:Aijk-definition}
\end{equation}
Notice that $A_{ijk}^{(n)}$ is an $Ld\times Ld$ symmetric matrix.
On the event $\{n<J\}$,
\begin{equation}
\|S_k^{(n)}\|_{\mathrm{op}}\le C_*^2,
\qquad
\|S_k^{(n)}\|_F
\le \sqrt d\,C_*^2.
\end{equation}
Since $L=\Theta(1)$ and
$\|B_{ij}\|_{\mathrm{op}}\le1$,
$\|B_{ij}\|_F\le1$, it follows that
\begin{align}
\|A_{ijk}^{(n)}\|_{\mathrm{op}}
&\le
\frac{\|\mathcal C\|_{\mathrm{op}}}{\sqrt d}
\|B_{ij}\|_{\mathrm{op}}
\|S_k^{(n)}\|_{\mathrm{op}}
\le
\frac{\|\mathcal C\|_{\mathrm{op}}C_*^2}{\sqrt d},
\\
\|A_{ijk}^{(n)}\|_F
&\le
\frac{\|\mathcal C\|_{\mathrm{op}}}{\sqrt d}
\|B_{ij}\otimes S_k^{(n)}\|_F=
\frac{\|\mathcal C\|_{\mathrm{op}}}{\sqrt d}
\|B_{ij}\|_F\|S_k^{(n)}\|_F
\le
\|\mathcal C\|_{\mathrm{op}}C_*^2.
\end{align}
Therefore, uniformly on $\{n<J\}$,
\begin{equation}
\|A_{ijk}^{(n)}\|_F\le C,
\qquad
\|A_{ijk}^{(n)}\|_{\mathrm{op}}\le\frac{C}{\sqrt d},
\label{eq:Aijk-norm-bounds}
\end{equation}
where $C$ depends only on $C_*$ and $\mathcal C$.

Applying the Hanson--Wright inequality to
\eqref{eq:G-quadratic-form}, we obtain
\begin{equation}
\mathbf 1_{\{n<J\}}
\mathbb P\left(|G_{ijk}^{(n)}|>t\,\middle|\,\mathcal F_n\right)
\le2\exp\left[-c\min\left\{\frac{t^2}{\|A_{ijk}^{(n)}\|_F^2},\frac{t}{\|A_{ijk}^{(n)}\|_{\mathrm{op}}}\right\}\right].
\end{equation}
Hence, taking $t=c_G\log d$ and using
\eqref{eq:Aijk-norm-bounds},
\begin{equation}
\mathbf 1_{\{n<J\}}\mathbb P\left(
|G_{ijk}^{(n)}|>c_G\log d
\,\middle|\,\mathcal F_n
\right)\le2\exp\left[-c\min\left\{
(\log d)^2,\sqrt d\,\log d\right\}\right]\le2e^{-c'(\log d)^2}.
\end{equation}
Since $K$ and $L$ are fixed, a union bound over $(i,j,k)$ gives
\begin{equation}
\mathbf 1_{\{n<J\}}\mathbb P\left(\max_{i,j,k}|G_{ijk}^{(n)}|>c_G\log d\,\middle|\,\mathcal F_n\right)\le2L^2K\,e^{-c'(\log d)^2}.
\end{equation}
Therefore there exists $p_d\le C\left(e^{-c_1d}+e^{-c_2(\log d)^2}\right)$
such that
\begin{equation}
\mathbf 1_{\{n<J\}}\mathbb P(\mathcal A_n^c\mid\mathcal F_n)\le p_d.
\label{eq:conditional-truncation-bound}
\end{equation}

Define the event that no truncation failure occurs before the stopping
time,
\begin{equation}
\mathcal A_{\mathrm{data}}:=\bigcap_{n=0}^{N_T-1}
\left(\{n\ge J\}\cup\mathcal A_n\right).
\end{equation}
Using \eqref{eq:conditional-truncation-bound},
\begin{align}
\mathbb P(\mathcal A_{\mathrm{data}}^c)\le\sum_{n=0}^{N_T-1}\mathbb E\left[\mathbf 1_{\{n<J\}}\mathbb P(\mathcal A_n^c\mid\mathcal F_n)\right]\le N_Tp_d.
\label{eq:data-event-union}
\end{align}
Since $\alpha_d=\Omega(d^{-\iota})$, $N_T\le C_Td^{1+\iota}$, and hence
\begin{equation}
\mathbb P(\mathcal A_{\mathrm{data}}^c)\to0.
\label{eq:data-event-high-prob}
\end{equation}

\paragraph{Step 2: Stopped martingale decomposition and its bounds.}
For $k\in\mathcal K$, define the stopped and truncated stochastic gradient
\begin{equation}
\widehat g_k^{(n)}:=\mathbf 1_{\{n<J\}}
g_k^{(n)}\mathbf 1_{\mathcal A_n},
\end{equation}
its conditional mean
\begin{equation}
\widehat{\bar g}_k^{(n)}:=\mathbb E[\widehat g_k^{(n)}\mid\mathcal F_n],
\end{equation}
and the martingale difference
\begin{equation}
\widehat\xi_k^{(n)}:=\widehat g_k^{(n)}-\widehat{\bar g}_k^{(n)}.
\end{equation}
Then $\mathbb E[\widehat\xi_k^{(n)}\mid\mathcal F_n]=0$.

Let
\begin{equation}
\bar g_k^{(n)}:=\mathbb E[g_k^{(n)}\mid\mathcal F_n].
\end{equation}
Since $\{n<J\}\in\mathcal F_n$,
\begin{equation}
\widehat{\bar g}_k^{(n)}=\mathbf 1_{\{n<J\}}\left(\bar g_k^{(n)}-\mathcal T_k^{(n)}\right),
\label{eq:truncated-mean}
\end{equation}
where
\begin{equation}
\mathcal T_k^{(n)}:=\mathbb E\left[
g_k^{(n)}\mathbf 1_{\mathcal A_n^c}\,\middle|\,\mathcal F_n\right].
\label{eq:tail-error-definition}
\end{equation}
We next record a uniform second-moment estimate for the stochastic gradient. On $\{n<J\}$,
\begin{equation}
\max\left\{
\left\|
\mathbb E[
g_k^{(n)}(g_k^{(n)})^T
\mid\mathcal F_n]
\right\|_{\mathrm{op}},
\left\|
\mathbb E[
(g_k^{(n)})^Tg_k^{(n)}
\mid\mathcal F_n]
\right\|_{\mathrm{op}}
\right\}
\le C_g(C_*).
\label{eq:gradient-second-moment}
\end{equation}
Indeed, for any unit vectors $u\in\mathbb R^d$ and
$v\in\mathbb R^{d_k}$, the two quadratic forms in
\eqref{eq:gradient-second-moment} are respectively bounded by finite sums of expectations of the form
\begin{equation*}
\mathbb E\left[|\partial_{(ijk)}\mathcal L(G^{(n)})|^2
\,\|(W_k^{(n)})^T(X_{ij}^{(n)})^Tu\|^2
\,\middle|\,\mathcal F_n
\right]
\end{equation*}
and
\begin{equation*}
\mathbb E\left[
|\partial_{(ijk)}\mathcal L(G^{(n)})|^2
\,
\|X_{ij}^{(n)}W_k^{(n)}v\|^2
\,\middle|\,
\mathcal F_n
\right].
\end{equation*}
Similarly to the estimates used in the proof of Lemma
\ref{lemma:bound-gradient}, Assumption \ref{assum:tied}.1 together with
the spectral bound \eqref{eq:pre-stopping-bound} implies uniform bounds on all fixed moments of $G^{(n)}$ and hence on the corresponding
polynomially growing derivatives of $\mathcal L$. In particular, by
Hölder's inequality, for every unit vector $u\in\mathbb R^d$,
\begin{align}
&\mathbb E\left[|\partial_{(ijk)}\mathcal L(G^{(n)})|^2
\,\|(W_k^{(n)})^T(X_{ij}^{(n)})^Tu\|^2\,\middle|\,\mathcal F_n\right]\nonumber\\
&\qquad\le
\left(\mathbb E\left[|\partial_{(ijk)}\mathcal L(G^{(n)})|^4
\,\middle|\,\mathcal F_n\right]\right)^{1/2}\left(\mathbb E\left[\|(W_k^{(n)})^T(X_{ij}^{(n)})^Tu\|^4
\,\middle|\,\mathcal F_n\right]\right)^{1/2}=\mathcal O(1),
\end{align}
uniformly on $\{n<J\}$. The same argument applies to the other quadratic variation in
\eqref{eq:gradient-second-moment}.

Consequently,
\begin{equation}
\max\left\{\left\|\mathbb E[
\widehat\xi_k^{(n)}(\widehat\xi_k^{(n)})^T
\mid\mathcal F_n]\right\|_{\mathrm{op}},\left\|\mathbb E[(\widehat\xi_k^{(n)})^T\widehat\xi_k^{(n)}\mid\mathcal F_n]\right\|_{\mathrm{op}}\right\}
\le C_g(C_*).
\label{eq:truncated-gradient-variance}
\end{equation}
On $\mathcal A_n\cap\{n<J\}$, $\|X_{ij}^{(n)}\|_{\mathrm{op}}\le C\sqrt d$, the first derivatives of $\mathcal L$ have polynomial growth and
$\max_{i,j,k}|G_{ijk}^{(n)}|\le c_G\log d$. Therefore, for some fixed integer $r$,
\begin{equation}
\|\widehat g_k^{(n)}\|_{\mathrm{op}}\le C\sqrt d\,(1+(\log d)^r)=:B_d.
\end{equation}
The same bound holds for
$\|\widehat{\bar g}_k^{(n)}\|_{\mathrm{op}}$, and hence
\begin{equation}
\|\widehat\xi_k^{(n)}\|_{\mathrm{op}}
\le 2B_d.
\label{eq:xi-increment-bound}
\end{equation}

\paragraph{Step 3: Uniform control of the martingale noise.}
Set
\begin{equation}
a_d:=1-\frac{\alpha_d\gamma}{d}=1-\frac{\tau_d\gamma}{4}.
\label{eq:ad-definition}
\end{equation}
For sufficiently large $d$, $a_d\in(0,1)$. Define
\begin{equation}
D_{k,n}:=-\alpha_d\widehat\xi_k^{(n)}.
\end{equation}
For every terminal time $m\ge1$, consider the discounted martingale sum
\begin{equation}
Y_{k,m}:=\sum_{n=0}^{m-1}
a_d^{m-1-n}D_{k,n}.
\label{eq:discounted-noise}
\end{equation}
For each fixed $m$, the summands in \eqref{eq:discounted-noise} form a
rectangular matrix martingale difference sequence.

By \eqref{eq:xi-increment-bound},
\begin{equation}
\|a_d^{m-1-n}D_{k,n}\|_{\mathrm{op}}
\le2\alpha_d B_d=:R_d.
\end{equation}
Moreover, by \eqref{eq:truncated-gradient-variance}, the predictable
quadratic variation in both rectangular directions is bounded by
\begin{align}
\sigma_d^2\le C_g\alpha_d^2
\sum_{r=0}^{\infty}a_d^{2r}
=\frac{C_g\alpha_d^2}{1-a_d^2}\le C\frac{d\alpha_d}{\gamma}.
\label{eq:discounted-variance}
\end{align}
Under $\alpha_d\le\frac{c_0}{d\log d}$,
we have
\begin{equation}
\sigma_d^2\le\frac{Cc_0}{\gamma\log d},
\qquad R_d=o((\log d)^{-1}).
\label{eq:freedman-scales}
\end{equation}
The Matrix Freedman inequality therefore gives for every fixed $\delta>0$,
\begin{equation}
\mathbb P(\|Y_{k,m}\|_{\mathrm{op}}\ge\delta)
\le(d+d_k)\exp\left[-\frac{\delta^2/2}
{\sigma_d^2+R_d\delta/3}\right].
\label{eq:freedman-discounted}
\end{equation}
For sufficiently large $d$, \eqref{eq:freedman-scales} implies
\begin{equation}
\mathbb P(\|Y_{k,m}\|_{\mathrm{op}}\ge\delta)
\le(d+d_k)d^{-c\delta^2/c_0}
\label{eq:freedman-polynomial}
\end{equation}
for some constant $c>0$ independent of $m$, $T$, and $d$. Choose $\delta:=\frac{C_*-C_0}{4}>0$ and set $\beta:=\max\{1,c_w\}$.
Since $\max_k d_k\le d^{c_w}$, for all sufficiently large $d$, $d+d_k\le 2d^\beta$. Moreover, since $\alpha_d=\Omega(d^{-\iota})$, we have $N_T\le C_Td^{1+\iota}$. Therefore, by \eqref{eq:freedman-polynomial} and a union bound over
$k\in\mathcal K$ and $1\le m\le N_T$,
\begin{align}
\mathbb P\left(\max_{k\in\mathcal K}\max_{1\le m\le N_T}\|Y_{k,m}\|_{\mathrm{op}}>\delta\right)\le2KC_T
d^{1+\iota+\beta-c\delta^2/c_0}.
\label{eq:noise-union-bound}
\end{align}
Choose $c_0>0$ sufficiently small so that
\begin{equation}
\frac{c\delta^2}{c_0}>1+\iota+\beta.
\label{eq:c0-condition}
\end{equation}
Then the exponent in \eqref{eq:noise-union-bound} is negative, and hence
\begin{equation}
\mathbb P\left(\max_{k\in\mathcal K}\max_{1\le m\le N_T}\|Y_{k,m}\|_{\mathrm{op}}>\delta\right)\to0.
\label{eq:noise-event-high-prob}
\end{equation}
Denote the complementary high-probability event by
$\mathcal A_{\mathrm{noise}}$.

We also control the truncation tail
\eqref{eq:tail-error-definition}. On $\{n<J\}$, the stochastic gradient $g_k^{(n)}$ has at most polynomial growth in
$x^{(n)}$. Hence, for some finite constant $q$,
\begin{equation}
\mathbb E\left[\|g_k^{(n)}\|_{\mathrm{op}}^2
\,\middle|\,\mathcal F_n\right]\le d^q.
\end{equation}
Combining Cauchy--Schwarz with
\eqref{eq:conditional-truncation-bound},
\begin{align}
\mathbf 1_{\{n<J\}}\|\mathcal T_k^{(n)}\|_{\mathrm{op}}\le\mathbf 1_{\{n<J\}}
\left(\mathbb E[\|g_k^{(n)}\|_{\mathrm{op}}^2
\mid\mathcal F_n]\right)^{1/2}\left(\mathbb P(\mathcal A_n^c\mid\mathcal F_n)\right)^{1/2}\le
r_d,
\label{eq:tail-error-bound}
\end{align}
where
\begin{equation}
r_d\le d^{q/2}p_d^{1/2}=o(d^{-M})
\end{equation}
for every fixed $M>0$.

\paragraph{Step 4: Closing the stopping-time argument.}
On $\mathcal A_{\mathrm{data}}$, for every $m\le J$ and every
$n<m$, we have $n<J$ and $\mathcal A_n$ occurs. Hence
\begin{equation}
g_k^{(n)}=\widehat g_k^{(n)}=\bar g_k^{(n)}-\mathcal T_k^{(n)}+\widehat\xi_k^{(n)}.
\end{equation}
The SGD \eqref{eq:SGD} therefore gives, for $m\le J$,
\begin{equation}
W_k^{(m)}=a_d^mW_k^{(0)}-\alpha_d
\sum_{n=0}^{m-1}
a_d^{m-1-n}\bar g_k^{(n)}+\alpha_d
\sum_{n=0}^{m-1}a_d^{m-1-n}\mathcal T_k^{(n)}+Y_{k,m}.
\label{eq:stopped-unrolled-W}
\end{equation}
By Lemma \ref{lemma:bound-gradient}, uniformly on $\{n<J\}$,
\begin{equation}
\bar g_k^{(n)}=\frac1dH_k^{(n)}W_k^{(n)}+\mathcal E_{\mathrm{Stein},k}^{(n)},
\end{equation}
with $\|H_k^{(n)}\|_{\mathrm{op}}
\le C_H(C_*)$ and
$\|\mathcal E_{\mathrm{Stein},k}^{(n)}\|_{\mathrm{op}}\le C_E(C_*)d^{-3/2}$.
Therefore,
\begin{equation}
\|\bar g_k^{(n)}\|_{\mathrm{op}}
\le\frac{C_H(C_*)C_*}{d}
+C_E(C_*)d^{-3/2}.
\label{eq:mean-gradient-bound-stopped}
\end{equation}
Since $1-a_d=\frac{\alpha_d\gamma}{d}$,
we have
\begin{equation}
\alpha_d\sum_{r=0}^{\infty}a_d^r=\frac{d}{\gamma}.
\label{eq:discount-geometric-sum}
\end{equation}
Combining
\eqref{eq:stopped-unrolled-W}--\eqref{eq:discount-geometric-sum},
on
$\mathcal A_{\mathrm{data}}\cap\mathcal A_{\mathrm{noise}}$ and for
every $m\leq J\wedge N_T$,
\begin{align}
\|W_k^{(m)}\|_{\mathrm{op}}
\le\|W_k^{(0)}\|_{\mathrm{op}}
+\frac{C_H(C_*)C_*}{\gamma}+\frac{C_E(C_*)}{\gamma\sqrt d}+\frac{dr_d}{\gamma}+\delta
\le
C_0+\frac{C_H(C_*)C_*}{\gamma}+\delta+o(1).
\label{eq:closing-bound}
\end{align}
We now fix $C_*:=2C_0+1$ and choose $\gamma$ sufficiently
large (by Assumption \ref{assum:tied}.4) so that
\begin{equation}
\frac{C_H(C_*)C_*}{\gamma}
\le
\frac{C_*-C_0}{4}.
\label{eq:gamma-choice}
\end{equation}
For all sufficiently large $d$, \eqref{eq:closing-bound} then gives
\begin{equation}
\max_k\|W_k^{(m)}\|_{\mathrm{op}}<C_*
\qquad
\text{for every }m\leq J\wedge N_T
\end{equation}
on $\mathcal A_{\mathrm{data}}\cap\mathcal A_{\mathrm{noise}}$.
Suppose, on this event, that $J\le N_T$. Taking $m=J$ in
\eqref{eq:closing-bound} yields
\begin{equation}
\max_k\|W_k^{(J)}\|_{\mathrm{op}}<C_*,
\end{equation}
contradicting the definition \eqref{eq:stopping-time-J}. Therefore, $J>N_T$ on $\mathcal A_{\mathrm{data}}\cap\mathcal A_{\mathrm{noise}}$ for all
sufficiently large $d$.

Finally,
\begin{equation}
\mathbb P(J\le N_T)\le
\mathbb P(\mathcal A_{\mathrm{data}}^c)
+\mathbb P(\mathcal A_{\mathrm{noise}}^c)\to0
\end{equation}
by \eqref{eq:data-event-high-prob} and
\eqref{eq:noise-event-high-prob}. Hence
\begin{equation}
\mathbb P\left(\sup_{0\le n\le N_T}
\max_k\|W_k^{(n)}\|_{\mathrm{op}}\le C_*
\right)
\longrightarrow1,
\end{equation}
which proves the lemma.
\end{proof}

\subsubsection{Moment expansion}
Let $C_*$ be the constant in Lemma \ref{lemma:bound-tied-SGD}, and fix $T>0$ and $\rho>C_*^2$. Denote
\begin{equation}
\mathcal E_d(T):=\left\{\sup_{0\le n\le N_T}
\max_{1\le k\le K}\|W_k^{(n)}\|_{\mathrm{op}}
\le C_*\right\},
\qquad N_T:=\left\lfloor\frac{T}{\tau_d}\right\rfloor .
\end{equation}
Then $\mathbb P(\mathcal E_d(T))\to1$ by
Lemma \ref{lemma:bound-tied-SGD}.

For the compactness argument later on, it is convenient to introduce the separable closed subspace
\begin{equation}
X_\rho^0:=\left\{\mu\in X_\rho:\lim_{w\to\infty}
\sup_{|\alpha|=w}\frac{|\mu_\alpha|}{\rho^w}=0
\right\}.
\end{equation}
Equipped with the norm inherited from $X_\rho$, \(X_\rho^0\) is a separable Banach space.

Since the empirical moment sequence need not belong to $X_\rho^0$ on the exceptional event $\mathcal E_d(T)^c$, we introduce
an auxiliary spectrally bounded process. For the fixed time horizon $T$, define
\begin{equation}
\check W_k^{(n)}:=\begin{cases}
W_k^{(n)}, & \text{on }\mathcal E_d(T),\\
W_k^{(0)}, & \text{on }\mathcal E_d(T)^c,
\end{cases}
\qquad
0\le n\le N_T,
\label{eq:auxiliary-W}
\end{equation}
and let
\begin{equation}
\check S_k^{(n)}:=\check W_k^{(n)}(\check W_k^{(n)})^T.
\end{equation}
For a multi-index $\alpha=(k_1,\ldots,k_w)$, define
\begin{equation}
\check\mu_\alpha^{(n)}:=\frac1d
\Tr\left[\check S_{k_1}^{(n)}\cdots\check S_{k_w}^{(n)}\right],
\end{equation}
and let $\check\mu^{(d)}(t)$ denote its piecewise-constant interpolation,
\begin{equation}
\check\mu^{(d)}(t)=\check\mu^{(n)},
\qquad
t\in[n\tau_d,(n+1)\tau_d).
\label{eq:auxiliary-moment-process}
\end{equation}
Since $C_*>C_0$, we have almost surely
\begin{equation}
\sup_{0\le n\le N_T}\max_{1\le k\le K}\|\check W_k^{(n)}\|_{\mathrm{op}}
\le C_*.
\label{eq:auxiliary-uniform-spectral-bound}
\end{equation}
Consequently, for every $|\alpha|=w$,
\begin{equation}
|\check\mu_\alpha^{(d)}(t)|\le C_*^{2w},\qquad t\in[0,T],
\label{eq:auxiliary-moment-bound}
\end{equation}
and hence, almost surely (w.r.t. the initialization),
\begin{equation}
\check\mu^{(d)}(t)\in X_\rho^0
\qquad\text{for all }t\in[0,T].
\end{equation}
Moreover,
\begin{equation}
\mathbb P\left(
\check\mu^{(d)}(t)=\tilde\mu^{(d)}(t)
\text{ for all }t\in[0,T]\right)\ge\mathbb P(\mathcal E_d(T))\longrightarrow1.
\label{eq:auxiliary-equals-original}
\end{equation}
All compactness arguments below will be carried out for $\check\mu^{(d)}$, and the resulting convergence will subsequently be
transferred to the original empirical trajectory $\tilde\mu^{(d)}$ using \eqref{eq:auxiliary-equals-original}.

We define empirical vector field $V_{d}(\tilde\mu^{(d)}(t))$ as 
\begin{equation}
[V_d(\tilde{\mu}^{(d)})]_\alpha =\frac{1}{d}\sum_{\substack{r=1 \\ k_r \in \mathcal{K}}}^w\Tr\left[S_{k_1}^{(n)}\ldots A_{k_r}^{(n)}\ldots S^{(n)}_{k_w}\right]
    \label{eq:V_d-definition}
\end{equation}
where $A_k^{(n)} := - \frac{1}{4}\left( M_k^{(n)}  S_k^{(n)} +  S_k^{(n)} (M_k^{(n)})^T \right)$ and $M_k^{(n)} := H_k^{(n)} + \gamma I_d$ with $H_k^{(n)}$ defined in Lemma \ref{lemma:bound-gradient}. 

The following lemma establishes the integral equation for the empirical trajectory \eqref{eq:interpolation-tied}.
\begin{lemma}
\label{lemma:macro-tied-SGD}
For every $t\in[0,T]$, the
auxiliary process satisfies
\begin{equation}
\check\mu^{(d)}(t)=\check\mu^{(d)}(0)+\mathbf 1_{\mathcal E_d(T)}
\int_0^tV_d(\tilde\mu^{(d)}(s))\,ds+\check E_{\mathrm{total}}^{(d)}(t).
\label{eq:macro-integral-expansion}
\end{equation}
For every $\rho>C_*^2$, all terms in
\eqref{eq:macro-integral-expansion} belong to $X_\rho^0$ almost surely, and
\begin{equation}
\sup_{t\in[0,T]}\left\|\check E_{\mathrm{total}}^{(d)}(t)\right\|_\rho
\xrightarrow{\mathbb P}0.
\label{eq:macro-total-error}
\end{equation}
\end{lemma}
\begin{proof}
Define
\begin{equation}
J:=\inf\left\{n\ge0:\max_{1\le k\le K}
\|W_k^{(n)}\|_{\mathrm{op}}>C_*\right\}.
\end{equation}
By Lemma \ref{lemma:bound-tied-SGD},
\begin{equation}
\mathbb P(J>N_T)\to1.
\label{eq:macro-stopping-high-prob}
\end{equation}
All estimates below are uniform on $\{n<J\}$.

\paragraph{Step 1: One-step expansion of the attention matrices.}
For $k\in\mathcal K$, Lemma \ref{lemma:bound-gradient} gives
\begin{equation}
\bar g_k^{(n)}=\frac1d H_k^{(n)}W_k^{(n)}+\mathcal E_{\mathrm{Stein},k}^{(n)},
\end{equation}
where $\|H_k^{(n)}\|_{\mathrm{op}}
\le C_H$ and $\|\mathcal E_{\mathrm{Stein},k}^{(n)}\|_{\mathrm{op}}
\le C_Ed^{-3/2}$, uniformly on $\{n<J\}$. Define
\begin{equation}
\xi_k^{(n)}:=g_k^{(n)}-\bar g_k^{(n)},
\qquad
M_k^{(n)}:=H_k^{(n)}+\gamma I_d.
\end{equation}
Then $\mathbb E[\xi_k^{(n)}
\mid\mathcal F_n]=0$,
and the SGD \eqref{eq:SGD} can be written as
\begin{equation}
W_k^{(n+1)}=W_k^{(n)}-\frac{\tau_d}{4}M_k^{(n)}W_k^{(n)}-\alpha_d\xi_k^{(n)}
-\alpha_d\mathcal E_{\mathrm{Stein},k}^{(n)}.
\label{eq:W-update-macro}
\end{equation}
For $k\notin\mathcal K$, $W_k^{(n+1)}=W_k^{(n)}$. Let
\begin{equation}
\delta W_k^{(n)}:=W_k^{(n+1)}-W_k^{(n)}.
\end{equation}
Expanding $S_k^{(n+1)}=\left(W_k^{(n)}+\delta W_k^{(n)}\right)
\left(W_k^{(n)}+\delta W_k^{(n)}\right)^T$ gives
\begin{equation}
S_k^{(n+1)}=S_k^{(n)}+\tau_d A_k^{(n)}+\alpha_d B_k^{(n)}+C_k^{(n)},
\label{eq:S-one-step-expansion}
\end{equation}
where, for $k\in\mathcal K$,
\begin{equation}
A_k^{(n)}:=-\frac14
\left(M_k^{(n)}S_k^{(n)}+S_k^{(n)}(M_k^{(n)})^T\right),
\label{eq:A-definition}
\end{equation}
\begin{equation}
B_k^{(n)}:=-\left(\xi_k^{(n)}(W_k^{(n)})^T+W_k^{(n)}(\xi_k^{(n)})^T\right),
\label{eq:B-definition}
\end{equation}
and
\begin{equation}
C_k^{(n)}:=-\alpha_d
\left(\mathcal E_{\mathrm{Stein},k}^{(n)}
(W_k^{(n)})^T+W_k^{(n)}(\mathcal E_{\mathrm{Stein},k}^{(n)})^T\right)+
\delta W_k^{(n)}(\delta W_k^{(n)})^T.
\label{eq:C-definition}
\end{equation}
For $k\notin\mathcal K$, we set
\begin{equation}
A_k^{(n)}=B_k^{(n)}=C_k^{(n)}=0.
\end{equation}
On $\{n<J\}$,
\begin{equation}
\|S_k^{(n)}\|_{\mathrm{op}}\le C_*^2,
\qquad
\|M_k^{(n)}\|_{\mathrm{op}}\le C_M
\label{eq:macro-spectral-bounds}
\end{equation}
for a constant $C_M$ independent of $d$ and $n$. In particular,
\begin{equation}
\|A_k^{(n)}\|_{\mathrm{op}}\le C_A.
\label{eq:A-bound}
\end{equation}
We shall also use moment bounds for the stochastic gradient noise. Similarly to the estimates in Lemma \ref{lemma:bound-tied-SGD} we can write
\begin{equation}
\mathbf 1_{\{n<J\}}
\mathbb E\left[\|\xi_k^{(n)}\|_F^p
\,\middle|\,\mathcal F_n
\right]\le
C_p d^{p/2}.
\label{eq:xi-F-moment}
\end{equation}

Fix a multi-index $\alpha=(k_1,\ldots,k_w)$.
Using \eqref{eq:S-one-step-expansion} in every factor of $\mu_\alpha^{(n+1)}=\frac1d\Tr\left[S_{k_1}^{(n+1)}\cdots S_{k_w}^{(n+1)}\right]$,
we separate the expansion into four classes.

The zeroth-order term is $\mu_\alpha^{(n)}$. The terms containing exactly one factor $\tau_d A_{k_r}^{(n)}$ define the empirical drift
\begin{equation}
[V_d(\mu^{(n)})]_\alpha:=\sum_{\substack{1\le r\le w\\ k_r\in\mathcal K}}
\frac1d\Tr\left[S_{k_1}^{(n)}\cdots A_{k_r}^{(n)}\cdots S_{k_w}^{(n)}\right].
\label{eq:Vd-definition-macro}
\end{equation}
The terms containing exactly one factor $\alpha_d B_{k_r}^{(n)}$ define the martingale
\begin{equation}
\eta_\alpha^{(n)}:=\sum_{\substack{1\le r\le w\\ k_r\in\mathcal K}}\frac1d\Tr\left[S_{k_1}^{(n)}\cdots B_{k_r}^{(n)}\cdots
S_{k_w}^{(n)}\right].
\label{eq:eta-definition}
\end{equation}
Since all $S_k^{(n)}$ are $\mathcal F_n$-measurable and
$B_k^{(n)}$ is linear in $\xi_k^{(n)}$, $\mathbb E[\eta_\alpha^{(n)}\mid\mathcal F_n]=0$.

Finally, let $R_\alpha^{(n)}$ denote the sum of all remaining terms in the expansion. Thus $R_\alpha^{(n)}$
contains, in particular,
\begin{itemize}
\item every term involving at least one $C_{k_r}^{(n)}$;
\item every term involving at least two first-order increments chosen from $\{\tau A_{k_r}^{(n)},\alpha B_{k_r}^{(n)}\}$.
\end{itemize}
With this definition we obtain the expansion:
\begin{equation}
\mu_\alpha^{(n+1)}-\mu_\alpha^{(n)}=\tau_d[V_d(\mu^{(n)})]_\alpha
+\alpha_d\eta_\alpha^{(n)}+R_\alpha^{(n)}.
\label{eq:moment-one-step}
\end{equation}

\paragraph{Step 2: Vanishing of the remainder.}
We next bound the remainder. For every fixed multi-index
$\alpha$, there exists a constant $C_\alpha<\infty$ such that,
uniformly on $\{n<J\}$,
\begin{equation}
\mathbb E\left[|R_\alpha^{(n)}|\,\middle|\,\mathcal F_n
\right]\le C_\alpha\left(\alpha_d^2+\alpha_d\tau_d+\tau_d^2+\alpha_d d^{-3/2}\right).
\label{eq:remainder-one-step-bound}
\end{equation}
To see this, first bound terms containing at least two $B$-factors, let $P_0,P_1,\cdots$ denote arbitrary products of the remaining $S$- and $A$-matrices. On $\{n<J\}$, their operator norms are bounded by a constant depending only on $\alpha$. 
We obtain, for fixed $r_B\ge2$,
\begin{equation}
\frac1d\mathbb E\left[
\left|\Tr(P_0B_1P_1\cdots B_{r_B}P_{r_B})\right|\,\middle|\,\mathcal F_n\right]\le C_\alpha d^{(r_B-2)/2}.
\label{eq:multiple-B-bound}
\end{equation}
Hence the corresponding contribution is bounded by
\begin{equation}
C_\alpha \alpha_d^{r_B}\tau_d^{r_A}
d^{(r_B-2)/2}\le C_\alpha\alpha_d^2
\end{equation}
for all sufficiently large $d$, because $\alpha_d\sqrt d\to0$.
A term containing exactly one $\alpha_d B$ and at least one $\tau_d A$ is bounded by $C_\alpha\alpha_d\tau_d$, while a term containing at least two $\tau_d A$ factors is bounded by $C_\alpha\tau_d^2$.

It remains to consider terms containing at least one $C$-factor.
By \eqref{eq:C-definition}, each $C_k^{(n)}$ is a finite sum of terms
of the following types:
\begin{equation*}
O_{\mathrm{op}}(\alpha_d d^{-3/2}),
\qquad\tau_d^2\,U_1U_2^T,
\qquad\alpha_d\tau_d\,U_1\xi_k^T,
\qquad\alpha_d^2\,\xi_k\xi_k^T,
\end{equation*}
together with terms involving
$\mathcal E_{\mathrm{Stein},k}^{(n)}$ that are of smaller order.
Here the $U_i$ have uniformly bounded operator norm and
$\mathcal O(\sqrt d)$ Frobenius norm. Applying the same trace inequalities, Hölder's inequality, and \eqref{eq:xi-F-moment}, every product containing at least one $C$-factor is bounded in conditional expectation by $C_\alpha
\left(\alpha_d^2+\alpha_d\tau_d+\tau_d^2+\alpha_d d^{-3/2}\right)$, with any additional stochastic increment producing only an additional
factor of order $\alpha_d\sqrt d=o(1)$. This proves
\eqref{eq:remainder-one-step-bound}.

Define
\begin{equation}\label{eq:E_rem}
E_{\mathrm{rem},\alpha}^{(d)}(t):=\sum_{n=0}^{\lfloor t/\tau\rfloor-1}R_\alpha^{(n)}.
\end{equation}
Replace $R_\alpha^{(n)}$ temporarily by $R_\alpha^{(n)}\mathbf 1_{\{n<J\}}$. By \eqref{eq:remainder-one-step-bound},
\begin{align}
\mathbb E\left[\sum_{n=0}^{N_T-1}\mathbf 1_{\{n<J\}}
|R_\alpha^{(n)}|\right]\le N_T C_\alpha\left(\alpha_d^2+\alpha_d\tau_d+\tau_d^2+\alpha_d d^{-3/2}
\right)\le C_{\alpha,T}\left(d\alpha_d+\alpha_d+\frac{\alpha_d}{d}+d^{-1/2}\right).
\label{eq:cumulative-remainder-L1}
\end{align}
Since $\alpha_d\le\frac{c_0}{d\log d}$, the right-hand side tends to zero. Hence, by Markov's inequality,
\begin{equation}
\sup_{t\in[0,T]}\left|E_{\mathrm{rem},\alpha}^{(d)}(t)\right|
\xrightarrow{\mathbb P}0
\label{eq:fixed-remainder-vanish}
\end{equation}
on $\{J>N_T\}$. Together with
\eqref{eq:macro-stopping-high-prob}, this proves
\eqref{eq:fixed-remainder-vanish} for the original trajectory.

\paragraph{Step 3: Vanishing of the first-order martingale noise.}
Define
\begin{equation}\label{eq:E_noise}
E_{\mathrm{noise},\alpha}^{(d)}(t):=\sum_{n=0}^{\lfloor t/\tau_d\rfloor-1}\alpha_d\eta_\alpha^{(n)}.
\end{equation}
We claim that, for every fixed $\alpha$,
\begin{equation}
\mathbf 1_{\{n<J\}}\mathbb E\left[
(\eta_\alpha^{(n)})^2\,\middle|\,\mathcal F_n\right]\le\frac{C_\alpha}{d^2}.
\label{eq:eta-variance}
\end{equation}
Indeed, every constituent of
\eqref{eq:eta-definition} is a linear functional of $\xi_k^{(n)}$.
Since $\xi_k^{(n)}=g_k^{(n)}-\mathbb E[g_k^{(n)}\mid\mathcal F_n]$, it suffices to bound the second moment of the corresponding linear functional of $g_k^{(n)}$. After inserting
$g_k^{(n)}=\sum_{i,j}\partial_{(ijk)}\mathcal L(G^{(n)})
X_{ij}^{(n)}W_k^{(n)}$,
each term takes the form
\begin{equation}
\partial_{(ijk)}\mathcal L(G^{(n)})Z_{ij}(A),
\qquad
Z_{ij}(A):=\frac1d\Tr\left[A X_{ij}^{(n)}
\right],
\label{eq:Zij-definition}
\end{equation}
where, for fixed $\alpha$, $\|A\|_{\mathrm{op}}\le C_\alpha$.
Since $X_{ij}^{(n)}$ is a centered Gaussian quadratic-form matrix, standard Gaussian moment estimates give
\begin{equation}
\mathbb E\left[|Z_{ij}(A)|^4
\,\middle|\,\mathcal F_n\right]\le\frac{C_\alpha}{d^4}.
\label{eq:Z-fourth-moment}
\end{equation}
On the other hand, Assumption \ref{assum:tied}.1 and the uniform
moment bounds for $G^{(n)}$ imply
\begin{equation}
\mathbb E\left[\left|\partial_{(ijk)}\mathcal L(G^{(n)})\right|^4\,\middle|\,
\mathcal F_n\right]\le C_\alpha.
\end{equation}
Hence, by Cauchy--Schwarz,
\begin{equation}
\mathbb E\left[\left|\partial_{(ijk)}\mathcal L(G^{(n)})Z_{ij}(A)\right|^2\,\middle|\,\mathcal F_n\right]\le
\frac{C_\alpha}{d^2}.
\end{equation}
Since $K,L$, and the length of $\alpha$ are fixed, this proves
\eqref{eq:eta-variance}.

Therefore the stopped process
\begin{equation}
M_{\alpha,m}:=\sum_{n=0}^{m-1}\alpha_d\eta_\alpha^{(n)}\mathbf 1_{\{n<J\}}
\end{equation}
is a square-integrable martingale with predictable quadratic variation bounded by
\begin{align}
\mathbb E[\langle M_\alpha\rangle_{N_T}]\le
N_T\alpha_d^2\frac{C_\alpha}{d^2}\le
C_{\alpha,T}\frac{\alpha_d}{d}.
\end{align}
Doob's maximal inequality yields, for every $\varepsilon>0$,
\begin{equation}
\mathbb P\left(\max_{m\le N_T}|M_{\alpha,m}|>\varepsilon\right)\le
\frac{C_{\alpha,T}\alpha_d}{\varepsilon^2d}\to0
\end{equation}
Using again $\mathbb P(J>N_T)\to1$, we conclude that
\begin{equation}
\sup_{t\in[0,T]}
\left|E_{\mathrm{noise},\alpha}^{(d)}(t)\right|
\xrightarrow{\mathbb P}0.
\label{eq:fixed-noise-vanish}
\end{equation}

\paragraph{Step 4: From the discrete recursion to the integral equation.}
For $m(t):=\left\lfloor\frac{t}{\tau}\right\rfloor$,
summing \eqref{eq:moment-one-step} from $n=0$ to $m(t)-1$ gives
\begin{align}
\tilde\mu_\alpha^{(d)}(t)=\tilde\mu_\alpha^{(d)}(0)+\tau_d\sum_{n=0}^{m(t)-1}
[V_d(\mu^{(n)})]_\alpha+E_{\mathrm{noise},\alpha}^{(d)}(t)+E_{\mathrm{rem},\alpha}^{(d)}(t).
\end{align}
Since the interpolation \eqref{eq:interpolation-tied} is piecewise constant,
\begin{align}
\int_0^t[V_d(\tilde\mu^{(d)}(s))]_\alpha\,ds
=\tau_d\sum_{n=0}^{m(t)-1}
[V_d(\mu^{(n)})]_\alpha+(t-m(t)\tau_d)
[V_d(\mu^{(m(t))})]_\alpha.
\end{align}
Define the discretization error
\begin{equation}
E_{\mathrm{disc},\alpha}^{(d)}(t):=-(t-m(t)\tau_d)[V_d(\mu^{(m(t))})]_\alpha.
\end{equation}
Thus we define $\check E_{\mathrm{total},\alpha}^{(d)}:=\mathbf 1_{\mathcal E_d(T)}E_{\mathrm{total},\alpha}^{(d)}$ with
\begin{equation}
E_{\mathrm{total},\alpha}^{(d)}=E_{\mathrm{noise},\alpha}^{(d)}+E_{\mathrm{rem},\alpha}^{(d)}+E_{\mathrm{disc},\alpha}^{(d)},
\end{equation}
and \eqref{eq:macro-integral-expansion} follows coordinatewise.

On $\{n<J\}$, \eqref{eq:A-bound} implies that for $|\alpha|=w$,
\begin{equation}
|[V_d(\mu^{(n)})]_\alpha|\le B w C_*^{2w}
\label{eq:Vd-coordinate-bound}
\end{equation}
for a constant $B$ independent of $d,n$, and $w$. Consequently, for every fixed $\alpha$,
\begin{equation}
\sup_{t\in[0,T]}
|E_{\mathrm{disc},\alpha}^{(d)}(t)|\le\tau B|\alpha|C_*^{2|\alpha|}\to0.
\end{equation}
Combining this with
\eqref{eq:fixed-remainder-vanish} and
\eqref{eq:fixed-noise-vanish}, we have, for every fixed multi-index
$\alpha$,
\begin{equation}
\sup_{t\in[0,T]}
|E_{\mathrm{total},\alpha}^{(d)}(t)|
\xrightarrow{\mathbb P}0.
\label{eq:coordinate-total-error}
\end{equation}

\paragraph{Step 5: Convergence in $X_\rho^0$.}
We now upgrade the coordinatewise convergence to convergence in $X_\rho^0$ on $\mathcal E_d(T)$.
On $\mathcal E_d(T)$,
\begin{equation}
\sup_{t\in[0,T]}\sup_{|\alpha|=w}|\tilde\mu_\alpha^{(d)}(t)|\le C_*^{2w}.
\label{eq:moment-tail-bound-macro}
\end{equation}
Together with \eqref{eq:Vd-coordinate-bound} and
\begin{equation}
E_{\mathrm{total}}^{(d)}(t)=\tilde\mu^{(d)}(t)-\tilde\mu^{(d)}(0)
-\int_0^tV_d(\tilde\mu^{(d)}(s))\,ds,
\label{eq:E-total-integral}
\end{equation}
this yields, on $\mathcal E_d(T)$,
\begin{equation}
\sup_{t\in[0,T]}
\sup_{|\alpha|=w}
|E_{\mathrm{total},\alpha}^{(d)}(t)|
\le
(2+TBw)C_*^{2w}.
\label{eq:Etotal-tail-bound}
\end{equation}
Since $\rho>C_*^2$,
\begin{equation}
(2+TBw)
\left(\frac{C_*^2}{\rho}\right)^w
\longrightarrow0
\qquad\text{as }w\to\infty.
\end{equation}

Fix $\varepsilon>0$. Choose $w_\star$ sufficiently large that
\begin{equation}
\sup_{w>w_\star}(2+TBw)\left(\frac{C_*^2}{\rho}\right)^w<\frac{\varepsilon}{2}.
\label{eq:tail-choice-W0}
\end{equation}
For degrees $w\le w_\star$, there are only finitely many multi-indices,
since $K$ is fixed. Hence
\eqref{eq:coordinate-total-error} and a finite union bound imply
\begin{equation}
\max_{|\alpha|\le w_\star}
\sup_{t\in[0,T]}\frac{|E_{\mathrm{total},\alpha}^{(d)}(t)|
}{\rho^{|\alpha|}}\xrightarrow{\mathbb P}0
\end{equation}
on $\mathcal E_d(T)$.

By $\check E_{\mathrm{total}}^{(d)}(t):=\mathbf 1_{\mathcal E_d(T)}
E_{\mathrm{total}}^{(d)}(t)$,
$\check E_{\mathrm{total}}^{(d)}(t)=0$ on
$\mathcal E_d(T)^c$. On $\mathcal E_d(T)$,
\eqref{eq:Etotal-tail-bound}--\eqref{eq:tail-choice-W0}
show that
$\check E_{\mathrm{total}}^{(d)}(t)\in X_\rho^0$ and that
\begin{equation}
\sup_{t\in[0,T]}
\left\|\check E_{\mathrm{total}}^{(d)}(t)
\right\|_\rho
\xrightarrow{\mathbb P}0.
\label{eq:checked-total-error-vanish}
\end{equation}

Finally, on $\mathcal E_d(T)$ we have
$\check\mu^{(d)}=\tilde\mu^{(d)}$, and hence \eqref{eq:E-total-integral} gives
\begin{equation}
\check\mu^{(d)}(t)
=
\check\mu^{(d)}(0)
+
\mathbf 1_{\mathcal E_d(T)}
\int_0^t
V_d(\tilde\mu^{(d)}(s))\,ds
+
\check E_{\mathrm{total}}^{(d)}(t).
\end{equation}
On $\mathcal E_d(T)^c$, both the integral term and
$\check E_{\mathrm{total}}^{(d)}$ vanish, while by definition
$\check\mu^{(d)}(t)=\check\mu^{(d)}(0)$.
Therefore the same identity holds almost surely on the whole
probability space. This proves the lemma.
\end{proof}

\subsubsection{Compactness of the empirical trajectories}
\begin{lemma}
\label{lemma:compactness}
Under the conditions of Theorem \ref{theo:tied-SGD}, for any finite
$T>0$ and any $\rho>C_*^2$, the sequence of auxiliary moment processes $\{\check\mu^{(d)}(\cdot)\}_{d\ge1}$ is $C$-tight in $D([0,T];X_\rho^0)$\footnote{We say that a sequence of processes is $C$-tight in $D([0,T];X_\rho^0)$ if it is tight (relatively compact in distribution) and every subsequential weak limit is supported on $C([0,T];X_\rho^0)$.}.
\end{lemma}

\begin{proof}
Define the continuous approximation
\begin{equation}
\hat\mu^{(d)}(t):=\check\mu^{(d)}(0)+\mathbf 1_{\mathcal E_d(T)}
\int_0^tV_d(\tilde\mu^{(d)}(s))\,ds.
\label{eq:hat-mu-definition}
\end{equation}
By Lemma \ref{lemma:macro-tied-SGD},
\begin{equation}
\check\mu^{(d)}(t)-\hat\mu^{(d)}(t)=\check E_{\mathrm{total}}^{(d)}(t),
\end{equation}
and therefore
\begin{equation}
\sup_{t\in[0,T]}\left\|\check\mu^{(d)}(t)-\hat\mu^{(d)}(t)
\right\|_\rho\xrightarrow{\mathbb P}0.
\label{eq:check-hat-equivalence}
\end{equation}
We first establish compactness of the family
$\{\hat\mu^{(d)}\}$. By
\eqref{eq:auxiliary-uniform-spectral-bound} and
\eqref{eq:Vd-coordinate-bound}, for every multi-index
$\alpha$ of length $w$,
\begin{equation}
\sup_{t\in[0,T]}|\hat\mu_\alpha^{(d)}(t)|\le(1+TBw)C_*^{2w}=:a_w
\qquad\text{almost surely}.
\label{eq:hat-moment-bound}
\end{equation}
Indeed, on $\mathcal E_d(T)$ this follows from Lemma \ref{lemma:bound-tied-SGD}, while on $\mathcal E_d(T)^c$ the process $\hat\mu^{(d)}$ is constant.

Since
\begin{equation}
\frac{a_w}{\rho^w}=(1+TBw)\left(\frac{C_*^2}{\rho}\right)^w\longrightarrow0,
\end{equation}
the set
\begin{equation}
\mathcal K_{T,\rho}:=\left\{
\mu\in X_\rho^0:\sup_{|\alpha|=w}|\mu_\alpha|\le a_w\text{ for every }w\ge1\right\}
\label{eq:compact-moment-set}
\end{equation}
is compact in $X_\rho^0$.
This is because for every $\varepsilon>0$, the coordinates of degree larger than a sufficiently large $M$ contribute at most $\varepsilon$ uniformly over $\mathcal K_{T,\rho}$, whereas the coordinates of degree at most $M$ form a bounded subset of a finite-dimensional space.

Moreover, for $0\le s<t\le T$ and $|\alpha|=w$,
\begin{equation}
|\hat\mu_\alpha^{(d)}(t)-\hat\mu_\alpha^{(d)}(s)|
\le B|t-s|\,wC_*^{2w}.
\end{equation}
Consequently,
\begin{align}
\left\|\hat\mu^{(d)}(t)-\hat\mu^{(d)}(s)
\right\|_\rho\le B|t-s|
\sup_{w\ge1}w\left(\frac{C_*^2}{\rho}\right)^w
=:L_\rho|t-s|,
\label{eq:hat-equicontinuity}
\end{align}
where $L_\rho<\infty$ is independent of $d$.

Thus the trajectories $\hat\mu^{(d)}$ take values almost surely in
the compact set $\mathcal K_{T,\rho}$ and are uniformly
Lipschitz in $X_\rho^0$. By the Arzel\`a--Ascoli theorem, the family $\{\hat\mu^{(d)}\}$ is tight in
$C([0,T];X_\rho^0)$.

Finally, \eqref{eq:check-hat-equivalence} shows that
$\check\mu^{(d)}$ is asymptotically equivalent to the tight family
$\hat\mu^{(d)}$ in the uniform metric. It follows that
$\{\check\mu^{(d)}\}_{d\ge1}$ is tight in
$D([0,T];X_\rho^0)$, and every subsequential weak limit is supported on
$C([0,T];X_\rho^0)$. Hence the sequence is $C$-tight.
\end{proof}

\subsubsection{Global well-posedness}
We show below that $V_d$ converges uniformly to $V$ under the spectral bounds in Lemma \ref{lemma:bound-tied-SGD}.

\begin{lemma}
\label{lemma:drift_convergence}
Fix $C_*>0$ and $\rho>C_*^2$. Let
\begin{equation}
\mathcal A_d(C_*):=\left\{
(S_1,\ldots,S_K):S_k=W_kW_k^T,\ \max_{1\le k\le K}\|W_k\|_{\mathrm{op}}\le C_*\right\}.
\end{equation}
For $\mathbf S\in\mathcal A_d(C_*)$, let $\mu(\mathbf S)$ denote its associated moment sequence. Then
\begin{equation}
\sup_{\mathbf S\in\mathcal A_d(C_*)}
\left\|V_d(\mu(\mathbf S))-V(\mu(\mathbf S))
\right\|_\rho\longrightarrow0
\qquad
\text{as }d\to\infty.
\end{equation}
Consequently, on the spectral-bound event $\mathcal{E}_d(T)$,
\begin{equation}
\sup_{t\in[0,T]}
\left\|V_d(\tilde\mu^{(d)}(t))-V(\tilde\mu^{(d)}(t))\right\|_\rho
\xrightarrow{\mathbb P}0.
\end{equation}
\end{lemma}
\begin{proof}
Define
\begin{equation}
\Gamma_{kl}^{(d)}:=\frac12
\sum_{i,j,i',j'}
\left(\mathcal{C}_{ii'}\mathcal{C}_{jj'}+\mathcal{C}_{ij'}\mathcal{C}_{ji'}\right)
\mathbb E_x\left[\partial^2_{(ijk),(i'j'l)}\mathcal L(G^{(d)})\right].
\end{equation}
By the definition of $H_k$,
\begin{equation}
H_k=4\sum_{l=1}^K\Gamma_{kl}^{(d)}S_l.
\end{equation}
Hence the empirical vector field \eqref{eq:V_d-definition} can be written as
\begin{equation}
[V_d(\mu)]_\alpha
=-\sum_{\substack{r=1\\k_r\in\mathcal K}}^w
\sum_{l=1}^K\Gamma_{k_rl}^{(d)}
\left(\mu_{\alpha^{r,l,-}}+\mu_{\alpha^{r,l,+}}
\right)-\frac{\gamma}{2}|\alpha|_{\mathcal K}\mu_\alpha,
\end{equation}
where $\alpha^{r,l,-}$ and $\alpha^{r,l,+}$ denote the two multi-indices obtained by inserting $l$ before and after $k_r$, respectively.

The limiting vector field has the same form with $\Gamma_{kl}^{(d)}$
replaced by $\Gamma_{kl}(q)$ defined in \eqref{eq:Gamma-definition}.
Therefore, defining
\begin{equation}
\delta_d(C_*):=\sup_{\mathbf S\in\mathcal A_d(C_*)}\max_{k,l}\left|\Gamma_{kl}^{(d)}-\Gamma_{kl}(q(\mathbf S))\right|,
\end{equation}
we have, for every $|\alpha|=w$,
\begin{equation}
|[V_d(\mu)-V(\mu)]_\alpha|\le2Kw\,\delta_d(C_*)\,C_*^{2(w+1)}.
\end{equation}
Thus,
\begin{equation}
\|V_d(\mu)-V(\mu)\|_\rho\le2KC_*^2\delta_d(C_*)\sup_{w\ge1}w\left(\frac{C_*^2}{\rho}\right)^w.
\end{equation}
It remains to prove $\delta_d(C_*)\to0$.

Suppose otherwise. Then there exist $\epsilon>0$, a subsequence $d_m$, and
$\mathbf S^{(d_m)}\in\mathcal A_{d_m}(C_*)$ such that
\begin{equation}
\max_{k,l}\left|\Gamma_{kl}^{(d_m)}-\Gamma_{kl}(q^{(d_m)})
\right|\ge\epsilon.
\label{eq:gamma-bound-epsilon}
\end{equation}
Since $K$ is fixed and $|q_{kl}^{(d_m)}|\le C_*^4$,
after taking a further subsequence we may assume
$q^{(d_m)}\to q^*\in\mathbb S_+^K$.

Define the active index set
\begin{equation}
I:=\{k\in\{1,\ldots,K\}:q_{kk}^*>0\}.
\end{equation}
For every $k\in I$,
\begin{equation}
\frac1{d_m}\|S_k^{(d_m)}\|_F^2=q_{kk}^{(d_m)}\longrightarrow q_{kk}^*>0.
\end{equation}
Since $\|S_k^{(d_m)}\|_{\mathrm{op}}\le C_*^2$, it follows that
\begin{equation}
\frac{\|S_k^{(d_m)}\|_{\mathrm{op}}}
{\|S_k^{(d_m)}\|_F}
\longrightarrow0.
\end{equation}
Hence Corollary \ref{cor:symmetric} applies to
$k\in I$. For $k\notin I$, we instead have
$q_{kk}^{(d_m)}\to0$, which gives
\begin{equation}
G_{ijk}^{(d_m)}\longrightarrow0
\qquad\text{in }L^2.
\end{equation}
Moreover, all covariances involving an inactive index converge to zero. Combining them,
\begin{equation}
G^{(d_m)}\to G^*\sim\mathcal N(0,\Sigma(q^*))
\label{eq:G-subsequence-limit}
\end{equation}
in distribution. Therefore, by Assumption \ref{assum:tied}.1 and uniform integrability,
\begin{equation}
\mathbb E\left[\partial^2_{(ijk),(i'j'l)}
\mathcal L(G^{(d_m)})\right]\longrightarrow\mathbb E\left[\partial^2_{(ijk),(i'j'l)}\mathcal L(G^*)\right].
\end{equation}
Consequently, by the definitions of $\Gamma^{(d_m)}$ and
$\Gamma(q^*)$,
\begin{equation}
\Gamma_{kl}^{(d_m)}\longrightarrow\Gamma_{kl}(q^*)\qquad\text{for every }k,l.
\end{equation}
This contradicts \eqref{eq:gamma-bound-epsilon}, and hence
\begin{equation}
\delta_d(C_*)\to0.
\end{equation}
The claimed convergence in $X_\rho$ follows.
\end{proof}

The next lemma establishes the global well-posedness of \eqref{eq:tied_macroscopic_equation}.
\begin{lemma}
\label{lemma:global_existence}
Under the conditions of Theorem \ref{theo:tied-SGD}, let $C_*$ be the uniform spectral bound in Lemma \ref{lemma:bound-tied-SGD}. For every finite $T>0$ and every pair $C_*^2<\rho_-<\rho_+$, there exists a unique admissible trajectory $\bar\mu\in C([0,T];X_{\rho_-}^0)$ such that
\begin{equation}
\bar\mu(t)=\bar\mu(0)+\int_0^tV(\bar\mu(s))\,ds,
\qquad t\in[0,T],
\label{eq:limiting-integral-equation}
\end{equation}
where the identity is understood in \(X_{\rho_+}^0\). Moreover,
\begin{equation}
q(\bar\mu(t))\succeq0,
\qquad t\in[0,T],
\end{equation}
and $|\bar\mu_\alpha(t)|\le C_*^{2|\alpha|}$
for every multi-index $\alpha$.
\end{lemma}

\begin{proof}
Fix $T>0$ and $C_*^2<\rho_-<\rho_+$. By Lemma
\ref{lemma:compactness}, the sequence
$\{\check\mu^{(d)}\}_{d\ge1}$ is $C$-tight in
$D([0,T];X_{\rho_-}^0)$. Hence, from any subsequence we may extract a further subsequence, still indexed by $d$, such that
\begin{equation}
\check\mu^{(d)}\to\mu^*
\qquad\text{in }D([0,T];X_{\rho_-}^0),
\label{eq:check-mu-subsequence-convergence}
\end{equation}
where $\mu^*\in C([0,T];X_{\rho_-}^0)$ almost surely.

We next rewrite the auxiliary equation in terms of the
limiting vector field. By Lemma \ref{lemma:macro-tied-SGD},
\begin{equation}
\check\mu^{(d)}(t)=\check\mu^{(d)}(0)+\mathbf 1_{\mathcal E_d(T)}
\int_0^tV_d(\tilde\mu^{(d)}(s))\,ds+\check E_{\mathrm{total}}^{(d)}(t),
\label{eq:check-empirical-integral-equation}
\end{equation}
with
\begin{equation}
\sup_{t\in[0,T]}\left\|\check E_{\mathrm{total}}^{(d)}(t)
\right\|_{\rho_+}\xrightarrow{\mathbb P}0.
\label{eq:check-total-error-rhoplus}
\end{equation}
Define
\begin{align}
R_d(t):=\check E_{\mathrm{total}}^{(d)}(t)+\int_0^t
\left[\mathbf 1_{\mathcal E_d(T)}V_d(\tilde\mu^{(d)}(s))-V(\check\mu^{(d)}(s))\right]ds.
\label{eq:global-Rd-definition}
\end{align}
Then
\begin{equation}
\check\mu^{(d)}(t)=\check\mu^{(d)}(0)+\int_0^t
V(\check\mu^{(d)}(s))\,ds+R_d(t).
\label{eq:approx-limit-integral-equation}
\end{equation}
We claim that
\begin{equation}
\sup_{t\in[0,T]}\|R_d(t)\|_{\rho_+}\xrightarrow{\mathbb P}0.
\label{eq:global-Rd-vanish}
\end{equation}
Indeed, on $\mathcal E_d(T)$ we have $\check\mu^{(d)}=\tilde\mu^{(d)}$, and therefore
\begin{align}
\mathbf 1_{\mathcal E_d(T)}\left\|V_d(\tilde\mu^{(d)}(s))-V(\check\mu^{(d)}(s))\right\|_{\rho_+}=\mathbf 1_{\mathcal E_d(T)}\left\|V_d(\tilde\mu^{(d)}(s))-V(\tilde\mu^{(d)}(s))
\right\|_{\rho_+}.
\end{align}
Lemma \ref{lemma:drift_convergence} implies
\begin{equation}
\mathbf 1_{\mathcal E_d(T)}\sup_{s\in[0,T]}\left\|V_d(\tilde\mu^{(d)}(s))-V(\tilde\mu^{(d)}(s))\right\|_{\rho_+}
\xrightarrow{\mathbb P}0.
\label{eq:drift-convergence-on-good-event}
\end{equation}
On the other hand, by
\eqref{eq:auxiliary-uniform-spectral-bound},
the auxiliary process always satisfies
\begin{equation}
|\check\mu_\alpha^{(d)}(s)|\le C_*^{2|\alpha|}.
\end{equation}
Since $\widetilde\Gamma$ is bounded on the corresponding bounded set of second moments, there exists a constant
$C_V=C_V(C_*,\rho_+)<\infty$ such that
\begin{equation}
\sup_{d\ge1}\sup_{s\in[0,T]}\|V(\check\mu^{(d)}(s))\|_{\rho_+}\le C_V\qquad\text{almost surely}.
\label{eq:V-auxiliary-uniform-bound}
\end{equation}
Consequently,
\begin{equation}
\mathbf 1_{\mathcal E_d(T)^c}\sup_{s\in[0,T]}\|V(\check\mu^{(d)}(s))\|_{\rho_+}
\xrightarrow{\mathbb P}0,
\end{equation}
because $\mathbb P(\mathcal E_d(T)^c)\to0$.
Together with
\eqref{eq:check-total-error-rhoplus} and
\eqref{eq:drift-convergence-on-good-event}, this proves
\eqref{eq:global-Rd-vanish}.

We also have $\check\mu^{(d)}(0)=\tilde\mu^{(d)}(0)$.
Assumption \ref{assum:tied}.3 gives convergence in probability of every coordinate, while Assumption \ref{assum:tied}.2 gives
$|\check\mu_\alpha^{(d)}(0)|\le C_0^{2|\alpha|}$.
Since $\rho_->C_*^2>C_0^2$, the contribution of sufficiently high
degrees is uniformly exponentially small, which yields
\begin{equation}
\check\mu^{(d)}(0)\xrightarrow{\mathbb P}\bar\mu(0)\qquad\text{in }X_{\rho_-}^0.
\label{eq:initial-Xrho-convergence}
\end{equation}

Because $X_{\rho_-}^0$ is separable, we may apply the Skorokhod representation theorem to the joint laws of $\left(\check\mu^{(d)},\check\mu^{(d)}(0),R_d\right)$.
After passing to a further subsequence and realizing these random
elements on a common probability space, we may assume that, almost
surely,
\begin{equation}
\check\mu^{(d)}\to\mu^*\qquad
\text{in }D([0,T];X_{\rho_-}^0),
\end{equation}
\begin{equation}
\check\mu^{(d)}(0)\longrightarrow\bar\mu(0)\qquad\text{in }X_{\rho_-}^0,
\end{equation}
and
\begin{equation}
\sup_{t\in[0,T]}\|R_d(t)\|_{\rho_+}
\longrightarrow0.
\end{equation}
Since $\mu^*$ is continuous, convergence in the Skorokhod topology implies
\begin{equation}
\sup_{t\in[0,T]}\left\|\check\mu^{(d)}(t)-\mu^*(t)\right\|_{\rho_-}\longrightarrow0
\qquad\text{almost surely}.
\label{eq:check-uniform-Skorokhod}
\end{equation}

The vector field $V$ satisfies the Cauchy estimate established in
Lemma \ref{lemma:local_existence}. Since the convergent trajectories
remain in a bounded subset of $X_{\rho_-}^0$, there exists
$L<\infty$ such that
\begin{equation}
\|V(\mu)-V(\nu)\|_{\rho_+}\le\frac{L}{\rho_+-\rho_-}
\|\mu-\nu\|_{\rho_-}
\end{equation}
for all states $\mu,\nu$ under consideration. Hence
\begin{align}
\sup_{t\in[0,T]}\left\|\int_0^t\left[V(\check\mu^{(d)}(s))-V(\mu^*(s))\right]ds\right\|_{\rho_+}\le
\frac{LT}{\rho_+-\rho_-}
\sup_{s\in[0,T]}\left\|\check\mu^{(d)}(s)-\mu^*(s)\right\|_{\rho_-}\longrightarrow0.
\end{align}
Passing to the limit in
\eqref{eq:approx-limit-integral-equation} yields
\begin{equation}
\mu^*(t)=\bar\mu(0)+\int_0^t
V(\mu^*(s))\,ds,\qquad t\in[0,T],
\label{eq:mu-star-limit-equation}
\end{equation}
as an identity in $X_{\rho_+}^0$.

It remains to verify that the limiting trajectory is admissible.
For every finite $d$ and every $t\in[0,T]$,
\begin{equation}
q_{kl}(\check\mu^{(d)}(t))=\frac1d\Tr\left[\check S_k^{(d)}(t)\check S_l^{(d)}(t)\right].
\end{equation}
Therefore
\begin{equation}
q(\check\mu^{(d)}(t))\succeq0
\qquad
\text{almost surely for every }d,t.
\end{equation}
Since the PSD cone is closed and convergence in
$X_{\rho_-}^0$ implies convergence of second-order coordinates,
\begin{equation}
q(\mu^*(t))\succeq0,\qquad t\in[0,T].
\end{equation}
Thus
\begin{equation}
\mathscr P(q(\mu^*(t)))=q(\mu^*(t)),
\end{equation}
and consequently
\begin{equation}
\widetilde\Gamma(q(\mu^*(t)))=\Gamma(q(\mu^*(t))).
\end{equation}
Therefore $\mu^*$ satisfies the original equation
\eqref{eq:tied_macroscopic_equation}.
Moreover, by
\eqref{eq:auxiliary-moment-bound},
\begin{equation}
|\check\mu_\alpha^{(d)}(t)|\le C_*^{2|\alpha|}
\qquad
\text{almost surely for every }d,t.
\end{equation}
Passing to the limit gives
\begin{equation}
|\mu_\alpha^*(t)|\le C_*^{2|\alpha|},\qquad
\alpha,\ t\in[0,T].
\label{eq:moment-bound}
\end{equation}

We finally establish uniqueness. Suppose that $\mu_1$ and $\mu_2$ are two admissible solutions of \eqref{eq:tied_macroscopic_equation}. Both are also solutions of \eqref{eq:V-definition}. By Lemma \ref{lemma:local_existence}, the solution of \eqref{eq:V-definition} is locally unique in the Banach scale. Hence, if $\mu_1(t_0)=\mu_2(t_0)$ at some $t_0<T$, the two solutions coincide on a finite interval to the right of $t_0$. Starting from $t_0=0$ and applying this local uniqueness with a contradiction argument shows that
\begin{equation}
\mu_1(t)=\mu_2(t),\qquad t\in[0,T].
\end{equation}
Thus the admissible solution is unique.

Since the initial condition $\bar{\mu}(0)$ is deterministic, uniqueness implies that the random limit $\mu^*$ is almost surely equal to a deterministic trajectory, which we denote by $\bar\mu$ as a unique admissible trajectory.
\end{proof}

\subsubsection{Final proof of Theorem \ref{theo:tied-SGD}}
Fix an arbitrary finite horizon $T>0$, and choose
\begin{equation}
C_*^2<\rho_-<\rho_+.
\end{equation}
By Lemma \ref{lemma:compactness}, the auxiliary trajectories
$\{\check\mu^{(d)}\}_{d\ge1}$ are $C$-tight in
$D([0,T];X_{\rho_-}^0)$.

By Lemma \ref{lemma:global_existence}, every subsequential weak limit is the same deterministic path $\bar\mu$, and therefore
\begin{equation}
\check\mu^{(d)}\to\bar\mu\qquad\text{in }
D([0,T];X_{\rho_-}^0).
\end{equation}
Since $\bar\mu$ is deterministic, convergence in distribution to
$\bar\mu$ is equivalent to convergence in probability:
\begin{equation}
d_{J_1}\left(\check\mu^{(d)},\bar\mu\right)\xrightarrow{\mathbb P}0,
\end{equation}
where $d_{J_1}$ denotes the Skorokhod $J_1$ metric.
Moreover, $\bar\mu$ is continuous. Convergence in the Skorokhod
topology to a continuous limit implies convergence in the uniform
topology. Hence
\begin{equation}
\sup_{t\in[0,T]}\left\|\check\mu^{(d)}(t)-\bar\mu(t)
\right\|_{\rho_-}\xrightarrow{\mathbb P}0.
\label{eq:auxiliary-uniform-convergence}
\end{equation}
For every fixed multi-index $\alpha$,
\begin{equation}
\sup_{t\in[0,T]}
\left|\check\mu_\alpha^{(d)}(t)-\bar\mu_\alpha(t)\right|\le\rho_-^{|\alpha|}\sup_{t\in[0,T]}\left\|\check\mu^{(d)}(t)-\bar\mu(t)
\right\|_{\rho_-},
\end{equation}
and therefore
\begin{equation}
\sup_{t\in[0,T]}
\left|\check\mu_\alpha^{(d)}(t)-\bar\mu_\alpha(t)
\right|\xrightarrow{\mathbb P}0.
\label{eq:auxiliary-coordinate-convergence}
\end{equation}
It remains only to transfer this convergence from the auxiliary
trajectory to the original trajectory. By
\eqref{eq:auxiliary-equals-original}, for every fixed multi-index $\alpha$ and every $\varepsilon>0$,
\begin{align}
\mathbb P\left(\sup_{t\in[0,T]}\left|\tilde\mu_\alpha^{(d)}(t)-\bar\mu_\alpha(t)\right|>\varepsilon
\right)\le\mathbb P(\mathcal E_d(T)^c)+\mathbb P\left(
\sup_{t\in[0,T]}\left|\check\mu_\alpha^{(d)}(t)-\bar\mu_\alpha(t)
\right|>\varepsilon\right).
\end{align}
The first term tends to zero by Lemma \ref{lemma:bound-tied-SGD}, while the second tends to zero by \eqref{eq:auxiliary-coordinate-convergence}. Consequently,
\begin{equation}
\lim_{d\to\infty}
\mathbb P\left(\sup_{t\in[0,T]}\left|\tilde\mu_\alpha^{(d)}(t)-\bar\mu_\alpha(t)\right|>\varepsilon\right)=0.
\end{equation}
This proves Theorem \ref{theo:tied-SGD}.

\subsection{Proof of Corollary \ref{cor:truncation-tied}}
\label{app:proof-cor-truncation-tied}
We define a projection operator $\Pi_M: X_s \to X_s$ for a $M \ge 2$, such that for any multi-index $\alpha$:
\begin{equation}
[\Pi_M \mu]_\alpha = \begin{cases} \mu_\alpha, & \text{if } |\alpha| \le M \\ 0, & \text{if } |\alpha| > M \end{cases}.
\end{equation}
Let $R:=C_*^2$. By Lemma \ref{lemma:global_existence},
\begin{equation}
|\bar\mu_\alpha(t)|\le R^{|\alpha|},
\qquad t\ge0.
\label{eq:true-moment-bound-truncation}
\end{equation}
In particular, $\|\bar q(t)\|_F\le K R^2$.

We first show that the finite-dimensional truncation is equivalently
the restriction of the extended infinite-dimensional ODE \eqref{eq:V-definition} to an
invariant subspace. Indeed, define
\begin{equation}
Y_M:=\left\{\mu:\mu_\alpha=0\text{ for every }|\alpha|>M
\right\}.
\end{equation}
If $\mu\in Y_M$ and $|\alpha|>M$, then every moment appearing in $[V(\mu)]_\alpha$ has degree either $|\alpha|$ or
$|\alpha|+1$, and hence vanishes. Therefore, $[V(\mu)]_\alpha=0$ for all $|\alpha|>M$.
Thus $Y_M$ is invariant under \eqref{eq:V-definition}, and $\mu^{(M)}$ is precisely the solution of
\begin{equation}
\frac{d}{dt}\mu^{(M)}(t)=V(\mu^{(M)}(t)),
\qquad
\mu^{(M)}(0)=\Pi_M\bar\mu(0).
\label{eq:truncated-full-equation}
\end{equation}
We next establish stability with respect to this truncation.
Let
\begin{equation}
\mathcal U:=\left\{q\in\mathbb R^{K\times K}:\|q\|_F\le KR^2+1\right\}.
\end{equation}
Since $\widetilde\Gamma$ is locally bounded and locally Lipschitz (see Lemma \ref{lemma:local_existence}), there exist finite constants $C_\Gamma,L_\Gamma$ such that
\begin{equation}
\max_{k,l}|\widetilde\Gamma_{kl}(q)|\le C_\Gamma,
\qquad q\in\mathcal U,
\label{eq:Gamma-uniform-truncation}
\end{equation}
and
\begin{equation}
\max_{k,l}
\left|\widetilde\Gamma_{kl}(q)-\widetilde\Gamma_{kl}(q')
\right|\le L_\Gamma\|q-q'\|_F,
\qquad q,q'\in\mathcal U.
\label{eq:Gamma-Lipschitz-truncation}
\end{equation}
Choose $s_0>R$ and $\gamma$ sufficiently large (Assumption \ref{assum:tied}.4) such that
\begin{equation}
\kappa:=\frac{4KC_\Gamma s_0}{\gamma}+\frac{4K^2L_\Gamma s_0^2R}{\gamma}<1.
\label{eq:truncation-kappa}
\end{equation}
Define the error
\begin{equation}
e^{(M)}(t):=\mu^{(M)}(t)-\bar\mu(t)
\end{equation}
and
\begin{equation}
E_M(t):=\|e^{(M)}(t)\|_{s_0}=\sup_{\alpha}\frac{|e_\alpha^{(M)}(t)|}{s_0^{|\alpha|}}.
\label{eq:EM-definition}
\end{equation}
Let $q^{(M)}(t):=q(\mu^{(M)}(t))$. As long as $q^{(M)}(t)\in\mathcal U$, we have
\begin{equation}
\|q^{(M)}(t)-\bar q(t)\|_F\le K s_0^2 E_M(t).
\label{eq:q-error-from-E}
\end{equation}
For a multi-index $\alpha=(k_1,\ldots,k_w)$, denote
\begin{equation}
r_\alpha:=|\alpha|_{\mathcal K}.
\end{equation}
Let $\alpha^{i,l,-}$ and $\alpha^{i,l,+}$ denote the
multi-indices obtained by inserting $l$ before and
after $k_i$, respectively. Subtracting the equations for
$\mu^{(M)}$ and $\bar\mu$, and using
\(\widetilde\Gamma(\bar q)=\Gamma(\bar q)\), gives
\begin{align}
\frac{d}{dt}e_\alpha^{(M)}=-\sum_{\substack{1\le i\le w\\k_i\in\mathcal K}}\sum_{l=1}^K\widetilde\Gamma_{k_il}(q^{(M)})\left(e_{\alpha^{i,l,-}}^{(M)}+e_{\alpha^{i,l,+}}^{(M)}\right)-\sum_{\substack{1\le i\le w\\k_i\in\mathcal K}}\sum_{l=1}^K\left[\widetilde\Gamma_{k_il}(q^{(M)})-\Gamma_{k_il}(\bar q)\right]\left(\bar\mu_{\alpha^{i,l,-}}+\bar\mu_{\alpha^{i,l,+}}\right)-\frac{\gamma}{2}r_\alpha e_\alpha^{(M)}.
\label{eq:error-equation-truncation}
\end{align}
If \(r_\alpha=0\), the right-hand side vanishes and hence
\begin{equation}
e_\alpha^{(M)}(t)=e_\alpha^{(M)}(0).
\end{equation}
Suppose now that \(r_\alpha\ge1\). Applying the variation-of-constants formula to \eqref{eq:error-equation-truncation} and using
\eqref{eq:Gamma-uniform-truncation} gives
\begin{align}
\frac{|e_\alpha^{(M)}(t)|}{s_0^w}
\le{}&\frac{|e_\alpha^{(M)}(0)|}{s_0^w}+2KC_\Gamma r_\alpha s_0\int_0^te^{-\frac{\gamma}{2}r_\alpha(t-u)}
E_M(u)\,du\nonumber\\
&+\frac{1}{s_0^w}\int_0^te^{-\frac{\gamma}{2}r_\alpha(t-u)}
\sum_{\substack{i:k_i\in\mathcal K}}
\sum_{l=1}^K\left|\widetilde\Gamma_{k_il}(q^{(M)}(u))
-\Gamma_{k_il}(\bar q(u))\right|\times
\left(|\bar\mu_{\alpha^{i,l,-}}(u)|+|\bar\mu_{\alpha^{i,l,+}}(u)|
\right)
\,du.
\label{eq:error-duhamel-first}
\end{align}
By \eqref{eq:Gamma-Lipschitz-truncation},
\eqref{eq:q-error-from-E}, and
\eqref{eq:true-moment-bound-truncation},
the second integral is bounded by
\begin{equation}
2K^2L_\Gamma s_0^2R\,r_\alpha\left(\frac{R}{s_0}\right)^w
\int_0^te^{-\frac{\gamma}{2}r_\alpha(t-u)}
E_M(u)\,du.
\end{equation}
Since
\begin{equation}
r_\alpha\int_0^te^{-\frac{\gamma}{2}r_\alpha(t-u)}\,du
\le\frac{2}{\gamma},
\end{equation}
and \(R/s_0<1\), we obtain
\begin{equation}
\frac{|e_\alpha^{(M)}(t)|}{s_0^w}\le\frac{|e_\alpha^{(M)}(0)|}{s_0^w}+\kappa\sup_{0\le u\le t}E_M(u),
\label{eq:error-coordinate-stability}
\end{equation}
with \(\kappa\) defined in \eqref{eq:truncation-kappa}. The same
inequality trivially holds when \(r_\alpha=0\).

Taking the supremum over \(\alpha\) and over \(0\le u\le t\), we
therefore obtain
\begin{equation}
\sup_{0\le u\le t}E_M(u)\le E_M(0)+\kappa\sup_{0\le u\le t}E_M(u).
\end{equation}
Since \(\kappa<1\),
\begin{equation}
\sup_{0\le u\le t}E_M(u)\le\frac{1}{1-\kappa}E_M(0).
\label{eq:error-stability-final}
\end{equation}
It remains to estimate the initial truncation error. By definition,
\begin{equation}
e_\alpha^{(M)}(0)=\begin{cases}
0,&|\alpha|\le M,\\
-\bar\mu_\alpha(0),&|\alpha|>M.
\end{cases}
\end{equation}
Using \eqref{eq:true-moment-bound-truncation},
\begin{align}
E_M(0)=\sup_{w>M}\sup_{|\alpha|=w}\frac{|\bar\mu_\alpha(0)|}{s_0^w}\le\sup_{w>M}\left(\frac{R}{s_0}\right)^w
=\left(\frac{R}{s_0}\right)^{M+1}.
\label{eq:initial-truncation-error}
\end{align}
Combining \eqref{eq:error-stability-final} and
\eqref{eq:initial-truncation-error} gives
\begin{equation}
\sup_{0\le t\le T}\|\mu^{(M)}(t)-\bar\mu(t)\|_{s_0}\le\frac{1}{1-\kappa}
\left(\frac{R}{s_0}\right)^{M+1},
\label{eq:truncation-error-prebootstrap}
\end{equation}
provided \(q^{(M)}(t)\in\mathcal U\).

We finally close this condition by a standard bootstrap argument.
Since
\begin{equation}
\|q^{(M)}(t)-\bar q(t)\|_F\le\frac{Ks_0^2}{1-\kappa}
\left(\frac{R}{s_0}\right)^{M+1},
\end{equation}
the right-hand side is smaller than \(1\) for all sufficiently
large \(M\). Because
\(\|\bar q(t)\|_F\le KR^2\), it follows that
\begin{equation}
\|q^{(M)}(t)\|_F<KR^2+1,
\end{equation}
so \(q^{(M)}(t)\) cannot leave \(\mathcal U\). Moreover,
\eqref{eq:truncation-error-prebootstrap} bounds every coordinate of
\eqref{eq:truncated-tied-system}, so its solution cannot undergo finite-time blow-up. Hence, for all sufficiently large \(M\), the solution of \eqref{eq:truncated-tied-system} exists on
\([0,T]\), and \eqref{eq:truncation-error-prebootstrap} holds
throughout the interval. This finishes the proof.

\subsection{An exactly solvable example}
\label{app:solvable}
Consider the squared loss $\mathcal{L}(G_1,G_2)=(G_1-G_2)^2$ without weight decay ($\gamma=0$). It is a simplified setting where $L=1$ and $K=2$. We assume only the first index is learnable ($\mathcal{K}=\{1\}$). The potential evaluates to:
\begin{equation}
\Phi(q)=\mathbb{E}_{z_1,z_2}||z_1-z_2||^2 = 2q_{11} + 2q_{22} - 4q_{12}.
\end{equation}
Because the partial derivatives of $\Phi(q)$ are constant, the dynamics in \eqref{eq:tied_macroscopic_equation} form an infinite-dimensional linear ODE hierarchy. Truncating this system at a finite degree $N$ yields a closed linear ODE with polynomial solutions. At truncation degree $N=1$, the solution is static: $\bar\mu_{1}(t) = \bar\mu_{1}(0)$. At $N=2$, the dynamics activate:
\begin{equation}
\begin{aligned}
&\bar q_{11}(t) = \bar q_{11}(0),\ \bar q_{12}(t) = \bar q_{12}(0),\\
&\bar\mu_{1}(t) = \bar\mu_{1}(0) - 4t [ \bar q_{11}(0) - \bar q_{12}(0) ].
\end{aligned}
\end{equation}
Extending to $N=3$, the solution captures higher-order curvature:
\begin{equation}
\begin{aligned}
&\bar q_{11}(t) = \bar q_{11}(0) - 8t [ \bar\mu_{111}(0) - \bar\mu_{112}(0) ],\ \bar q_{12}(t) = \bar q_{12}(0) - 4 t [ \bar\mu_{112}(0) - \bar\mu_{122}(0) ],\\
&\bar\mu_{1}(t) = \bar\mu_{1}(0) - 4t [ \bar q_{11}(0) - q_{12}(0) ] + 8 t^2 [ 2\bar\mu_{111}(0) - 3\bar\mu_{112}(0) + \bar\mu_{122}(0) ].
\end{aligned}
\end{equation}
Remarkably, these truncated polynomial solutions match the Taylor expansion of the original matrix flow evaluated around $t=0$. To see this, note that the gradient flow for the student forms a matrix Riccati differential equation:
\begin{equation}
\frac{d}{dt}S_1 = 2 (S_1S_2 + S_2S_1) - 4 S_1^2.
\label{eq:matrix-riccati}
\end{equation}
Assuming that $S_1(0)$ and $S_2$ are invertible, \eqref{eq:matrix-riccati} admits a closed-form solution:
\begin{equation}
S_1(t) = \left[ e^{-2 S_2 t} S_1(0)^{-1} e^{-2 S_2 t} + S_2^{-1} \left( I - e^{-4 S_2 t} \right) \right]^{-1},
\end{equation}
Taking the first-order Taylor expansion of the inner inverse matrix around $t=0$ yields:
\begin{equation}
\begin{aligned}
S_1(t)^{-1} &\approx (I - 2 S_2 t) S_1(0)^{-1} (I - 2 S_2 t) + S_2^{-1} ( 4 S_2 t )\\
&\approx S_1(0)^{-1} + \left( 4 I - 2 S_2 S_1(0)^{-1} - 2 S_1(0)^{-1} S_2 \right) t.
\end{aligned}
\end{equation}
Using the matrix identity $(A+tB)^{-1} \approx A^{-1} - A^{-1} B A^{-1} t$, we recover the first-order expansion for $S_1(t)$:
\begin{equation}
S_1(t)\approx S_1(0)+(2S_1(0) S_2 + 2 S_2 S_1(0) - 4 S_1(0)^2)t.
\end{equation}
Taking the normalized traces of this matrix expansion reproduces the linear terms of the truncated moments:
\begin{equation}
\begin{aligned}
&\bar q_{11}(t) = \bar q_{11}(0) - 8 t [ \bar\mu_{111}(0) - \bar\mu_{112}(0) ],\ \bar q_{12}(t) = \bar q_{12}(0) - 4 t [ \bar\mu_{112}(0) - \bar\mu_{122}(0) ],\\
&\bar\mu_{1}(t) = \bar\mu_{1}(0) - 4 t [ \bar q_{11}(0) - \bar q_{12}(0) ].
\end{aligned}
\end{equation}
Numerical results are shown in Figure \ref{fig:tied-theory-linear}.
\begin{figure}
\centering
\includegraphics[width=0.8\linewidth]{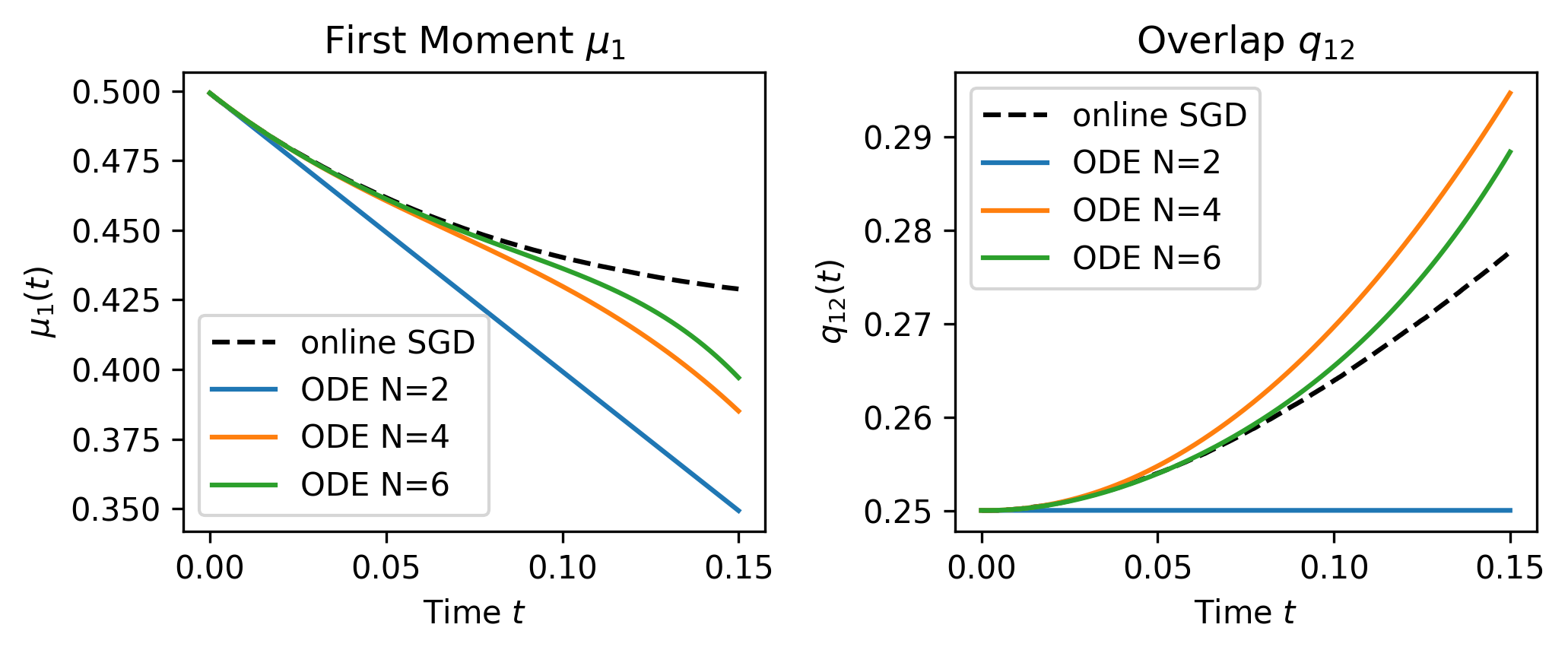}
\caption{Comparison between online SGD and the theory at different truncation orders $N$ for the linear activation, $L=1$ and the MSE loss. As $N$ grows larger, the theory is better aligned with the experiments. We choose $d=400$, learning rate $0.01$, batch size $512$ and no weight decay ($\gamma=0$).}
\label{fig:tied-theory-linear}
\end{figure}

The closed-form solution also provides a direct interpretation of why the moment dynamics is intrinsically infinite-dimensional. To see this, consider the special case $S_2=cI_d$ for some $c>0$. Then each eigenvalue $\lambda(t)$ of $S_1(t)$ evolves independently according to $\dot\lambda(t)=4c\lambda(t)-4\lambda(t)^2$,
and hence
\begin{equation}
\lambda(t)=\frac{c\lambda(0)}{\lambda(0)+(c-\lambda(0))e^{-4ct}}.
\end{equation}
Therefore, if $\rho_0$ denotes the limiting empirical spectral
distribution of $S_1(0)$, the first moment satisfies
\begin{equation}
\mu_1(t)=\int\frac{c\lambda}{\lambda+(c-\lambda)e^{-4ct}}
\,d\rho_0(\lambda).
\label{eq:solvable-spectral-transform}
\end{equation}
Restricting $\rho_0$ to finitely supported measures
$\rho_0=\sum_{j=0}^{M}p_j\delta_{\lambda_j}$ with distinct
$\lambda_j\neq c$, the functions
$t\longmapsto\frac{c\lambda_j}{\lambda_j+(c-\lambda_j)e^{-4ct}}$
are linearly independent. Consequently, the trajectory \(\{\mu_1(t):t\ge0\}\) uniquely determines the weights \((p_0,\ldots,p_M)\). 

Suppose, for contradiction, that these trajectories admitted a fixed \(r\)-dimensional ODE representation whose initial state depends continuously on the initial spectrum. For each \(M\), the map from the interior of the \(M\)-dimensional simplex \((p_0,\ldots,p_M)\) to the ODE initial state must then be injective, since identical ODE initial states generate identical trajectories, while the trajectory uniquely determines the weights. For \(M>r\), however, no continuous injective map from an \(M\)-dimensional open set into \(\mathbb R^r\) exists (as a consequence of Brouwer's invariance of domain theorem). Since \(M\) is arbitrary, no fixed finite-dimensional ODE with a continuous initialization map can reproduce this family of trajectories. This shows that the infinite-dimensionality of \eqref{eq:tied_macroscopic_equation} is intrinsic and cannot in general be removed by replacing the moment hierarchy with a fixed finite set of order parameters.

\subsection{Proof of Corollary \ref{cor:tied_weak_recovery_sgd}}
\label{app:proof_tied_weak_recovery_sgd}
Let
\begin{equation}
\bar p_{12}(t)=\bar q_{12}(t)-\bar t_1(t)\bar t_2,
\end{equation}
where \(\bar t_2\) is constant because the teacher is fixed. We first derive its limiting dynamics from \eqref{eq:tied_macroscopic_equation}.

For the order parameter
\(\bar q_{12}=\bar\mu_{(1,2)}\), only the student index \(1\) is
trainable. Therefore, from \eqref{eq:tied_macroscopic_equation} we have
\begin{align}
\frac{d}{dt}\bar q_{12}=-2\Gamma_{11}(\bar q)\bar\mu_{(1,1,2)}-2\Gamma_{12}(\bar q)\bar\mu_{(1,2,2)}-\frac{\gamma}{2}\bar q_{12}.
\label{eq:q12-dynamics-weak}
\end{align}
Similarly, for \(\bar t_1=\bar\mu_{(1)}\),
\begin{equation}
\frac{d}{dt}\bar t_1=-2\Gamma_{11}(\bar q)\bar q_{11}-2\Gamma_{12}(\bar q)\bar q_{12}-\frac{\gamma}{2}\bar t_1.
\label{eq:t1-dynamics-weak}
\end{equation}
Subtracting \(\bar t_2\) times \eqref{eq:t1-dynamics-weak} from
\eqref{eq:q12-dynamics-weak} gives
\begin{align}
\frac{d}{dt}\bar p_{12}=-2\Gamma_{11}(\bar q)\left(\bar\mu_{(1,1,2)}-\bar t_2\bar q_{11}\right)-2\Gamma_{12}(\bar q)\left(\bar\mu_{(1,2,2)}-\bar t_2\bar q_{12}\right)-\frac{\gamma}{2}\bar p_{12}.
\label{eq:p12-before-centered}
\end{align}
Define the centered third-order moments
\begin{equation}
\bar{\mathring\mu}_{(1,1,2)}:=\lim_{d\to\infty}
\frac1d\Tr[\mathring S_1^2\mathring S_2],
\qquad
\bar{\mathring\mu}_{(1,2,2)}:=\lim_{d\to\infty}\frac1d\Tr[\mathring S_1\mathring S_2^2].
\end{equation}
Expanding \(S_k=\mathring S_k+\bar t_kI\) yields
\begin{equation}
\bar\mu_{(1,1,2)}-\bar t_2\bar q_{11}=\bar{\mathring\mu}_{(1,1,2)}+2\bar t_1\bar p_{12},
\end{equation}
and
\begin{equation}
\bar\mu_{(1,2,2)}-\bar t_2\bar q_{12}=\bar{\mathring\mu}_{(1,2,2)}+\bar t_2\bar p_{12}+\bar t_1\bar p_{22}.
\end{equation}
Consequently,
\begin{align}
\frac{d}{dt}\bar p_{12}=-2\Gamma_{11}(\bar q)
\left(\bar{\mathring\mu}_{(1,1,2)}+2\bar t_1\bar p_{12}\right)-
2\Gamma_{12}(\bar q)\left(\bar{\mathring\mu}_{(1,2,2)}+\bar t_2\bar p_{12}+\bar t_1\bar p_{22}\right)-\frac{\gamma}{2}\bar p_{12}.
\label{eq:p12-centered-dynamics}
\end{align}
At \(t=0\), Assumption \ref{assum:tied_initialization} gives
\begin{equation}
\bar p_{12}(0)=\bar{\mathring\mu}_{(1,1,2)}(0)=\bar{\mathring\mu}_{(1,2,2)}(0)=0.
\end{equation}
Hence \eqref{eq:p12-centered-dynamics} reduces to
\begin{equation}
\left.\frac{d}{dt}\bar p_{12}(t)\right|_{t=0}=-2\Gamma_{12}(\bar q(0))
\bar t_1(0)\bar p_{22}.
\label{eq:p12-initial-drift}
\end{equation}
By Assumption \ref{assum:tied_initialization}, this quantity is nonzero. Set
\begin{equation}
v_0:=\left|\left.\frac{d}{dt}\bar p_{12}(t)\right|_{t=0}\right|=2|\Gamma_{12}(\bar q(0))|\bar t_1(0)\bar p_{22}>0.
\label{eq:v0-weak}
\end{equation}
Since the limiting trajectory satisfies \eqref{eq:tied_macroscopic_equation} and the relevant finite-order components of \(V(\bar\mu(t))\) are continuous in \(t\), \(\bar p_{12}\) is continuously differentiable. By continuity there exists \(T_0>0\), independent of \(d\), such that
\begin{equation}
\left| \frac{d}{dt}\bar p_{12}(t)\right| \le \frac{v_0}{2}, \qquad 0\le t\le T_0.
\end{equation}
In particular, by continuity \(\frac{d}{dt}\bar p_{12}(t)\) has the same sign as \(\frac{d}{dt}\bar p_{12}(0)\) throughout \([0,T_0]\).
Since \(\bar p_{12}(0)=0\),
\begin{equation}
|\bar p_{12}(T_0)|\ge\frac{v_0T_0}{2}.
\label{eq:positive-overlap-T0}
\end{equation}
Choose $c:=\frac{v_0T_0}{4}>0$ and $T_+:=T_0$. Then
\begin{equation}
|\bar p_{12}(T_+)|\ge2c.
\label{eq:upper-recovery-limit}
\end{equation}

On the other hand, because
\(\bar p_{12}(0)=0\) and \(\bar p_{12}\) is continuous, there exists
\(T_-\in(0,T_+)\) such that
\begin{equation}
\sup_{0\le t\le T_-}|\bar p_{12}(t)|<\frac c2.
\label{eq:lower-recovery-limit}
\end{equation}
Let \(\tilde p_{12}^{(d)}(t)\) denote the piecewise-constant interpolation
of the empirical structural overlap. Since
\begin{equation}
\tilde p_{12}^{(d)}=\tilde q_{12}^{(d)}-\tilde t_1^{(d)}\tilde t_2^{(d)},
\end{equation}
Theorem \ref{theo:tied-SGD} implies that,
for every fixed finite \(T\),
\begin{equation}
\sup_{t\in[0,T]}
\left|\tilde p_{12}^{(d)}(t)-\bar p_{12}(t)\right|\xrightarrow{\mathbb P}0.
\label{eq:p12-uniform-convergence}
\end{equation}
Applying \eqref{eq:p12-uniform-convergence} on \([0,T_+]\) and using
\eqref{eq:upper-recovery-limit}, we obtain
\begin{equation}
\mathbb P\left(|\tilde p_{12}^{(d)}(T_+)|\ge c\right)\longrightarrow1.
\end{equation}
Hence, with probability tending to one,
\begin{equation}
N_{\mathrm{wr}}^{(d)}(c)\le\left\lceil\frac{T_+}{\tau_d}\right\rceil=\left\lceil\frac{T_+d}{4\alpha_d}\right\rceil.
\label{eq:weak-recovery-upper-hit}
\end{equation}
Likewise, by \eqref{eq:lower-recovery-limit} and
\eqref{eq:p12-uniform-convergence},
\begin{equation}
\mathbb P\left(\sup_{0\le t\le T_-}|\tilde p_{12}^{(d)}(t)|<c\right)\longrightarrow1.
\end{equation}
Therefore, with probability tending to one,
\begin{equation}
N_{\mathrm{wr}}^{(d)}(c)>\left\lfloor\frac{T_-}{\tau_d}\right\rfloor=\left\lfloor\frac{T_-d}{4\alpha_d}\right\rfloor.
\label{eq:weak-recovery-lower-hit}
\end{equation}
Combining \eqref{eq:weak-recovery-upper-hit} and
\eqref{eq:weak-recovery-lower-hit} proves
\eqref{eq:weak-recovery-hitting-bound}, and hence
\begin{equation}
N_{\mathrm{wr}}^{(d)}(c)=\Theta_{\mathbb P}\left(\frac d{\alpha_d}\right).
\end{equation}
Finally, choosing
\begin{equation}
\alpha_d=\Theta\left(\frac1{d\log d}\right),
\end{equation}
which is the largest learning rate covered by
Theorem \ref{theo:tied-SGD}, gives
\begin{equation}
N_{\mathrm{wr}}^{(d)}(c)=\Theta_{\mathbb P}(d^2\log d).
\end{equation}
This completes the proof.

\subsection{Convergence rate of strong recovery}
\label{app:strong-recovery}
In this section we study the convergence rate after weak recovery. Because the weight decay depends on the first moments of the matrices, it is convenient to introduce
\begin{equation}
t_k:=\mu_{(k)}=\frac1d\Tr[S_k],
\qquad
q_{kl}:=\mu_{(k,l)}=\frac1d\Tr[S_kS_l].
\end{equation}
The limiting regularized population loss is defined as
\begin{equation}
\Psi(t,q):=\Phi(q)+\frac{\gamma}{2}
\sum_{k\in\mathcal K}t_k.
\label{eq:regularized-potential-strong}
\end{equation}
Let \(\mathcal Q_+^{(1,2)}\) denote the asymptotically realizable domain of the pair \((t,q)\), namely the closure of all limits generated by sequences of PSD matrices
\(S_k=W_kW_k^T\succeq0\), subject to the fixed teacher components when
\(k\notin\mathcal K\). We define
\begin{equation}
\Psi^*:=\inf_{(t,q)\in\mathcal Q_+^{(1,2)}}\Psi(t,q).
\label{eq:global-optimum-strong}
\end{equation}
We refer to the convergence to the global minimizers of \(\Psi\) as ``strong recovery''.

Let $\bar t_k(t):=\bar\mu_{(k)}(t)$ and $\bar q_{kl}(t):=\bar\mu_{(k,l)}(t)$ denote the corresponding components of the admissible solution of \eqref{eq:tied_macroscopic_equation}. We impose an effective
Polyak--\L ojasiewicz condition along the trajectory.
\begin{assumption}
\label{assum:effective_pl}
There exist constants \(T_0<\infty\) and \(c>0\) such that, for every \(t\ge T_0\),
\begin{equation}
-\frac{d}{dt}\Psi(\bar t(t),\bar q(t))\ge2c\left(\Psi(\bar t(t),\bar q(t))-\Psi^*\right).
\label{eq:effective_pl}
\end{equation}
\end{assumption}
For a finite \(d\), define the regularized population loss
\begin{equation}
\Psi_d(W):=\mathbb E_x\left[\mathcal L\left(\left\{\frac{1}{\sqrt d}\Tr\left[W_kW_k^T(x_jx_i^T-\mathcal C_{ij}I_d)\right]\right\}_{i,j,k=1}^{L,L,K}\right)\right]+\frac{\gamma}{2d}\sum_{k\in\mathcal K}\|W_k\|_F^2.
\label{eq:finite-d-population-risk}
\end{equation}
\begin{corollary}
\label{cor:strong_recovery}
Suppose that the conditions of Theorem \ref{theo:tied-SGD} and
Assumption \ref{assum:effective_pl} hold. Let $n_0:=\left\lceil\frac{T_0d}{4\alpha_d}\right\rceil$. Then, for every fixed target accuracy \(\epsilon>0\), there exists
\begin{equation}
n_\epsilon=n_0+O\left(\frac{d}{\alpha_d}\log_+\frac1\epsilon\right),
\label{eq:strong-recovery-sample-complexity}
\end{equation}
such that
\begin{equation}
\lim_{d\to\infty}\mathbb P\left(\Psi_d(W^{(n_\epsilon)})-\Psi^*>\epsilon\right)=0.
\label{eq:strong-recovery-risk-convergence}
\end{equation}
where $\log_+(x):=\max\{0,\log x\}$.
In particular, if \(T_0=\Theta(1)\), reaching an \(\epsilon\)-suboptimal population loss requires at most $O\left(\frac{d}{\alpha_d}\log\frac1\epsilon\right)$ additional samples.
\end{corollary}
\begin{proof}
We first analyze the limiting trajectory in Theorem \ref{theo:tied-SGD}. Define
\begin{equation}
\Delta(t):=\Psi(\bar t(t),\bar q(t))-\Psi^*\ge0.
\end{equation}
By Assumption \ref{assum:effective_pl}, for every \(t\ge T_0\),
\begin{equation}
\frac{d}{dt}\Delta(t)\le-2c\Delta(t).
\end{equation}
Hence Gr\"onwall's inequality gives
\begin{equation}
\Delta(t)\le\Delta(T_0)e^{-2c(t-T_0)},
\qquad t\ge T_0.
\label{eq:deterministic-exponential-decay}
\end{equation}
Let
\begin{equation}
\Delta_0:=\Psi(\bar t(T_0),\bar q(T_0))-\Psi^*.
\end{equation}
If \(\Delta_0=0\), there is nothing to prove. Otherwise, for any fixed \(\epsilon>0\), define
\begin{equation}
T_\epsilon:=T_0+\frac{1}{2c}\log_+\left(
\frac{2\Delta_0}{\epsilon}\right).
\label{eq:T-epsilon-strong}
\end{equation}
Then
\begin{equation}
\Psi(\bar t(T_\epsilon),\bar q(T_\epsilon))-\Psi^*\le\frac{\epsilon}{2}.
\label{eq:deterministic-epsilon-gap}
\end{equation}
We next transfer this bound to the finite-dimensional population loss. Recall that the discrete steps and continuous times are related by $t=n\tau_d$ with $\tau_d=\frac{4\alpha_d}{d}$. For \(t\ge0\), define
$W^{(d)}(t):=W^{(\lfloor t/\tau_d\rfloor)}$. For every fixed finite \(T\), Theorem \ref{theo:tied-SGD}, together with the uniform spectral bound of Lemma \ref{lemma:bound-tied-SGD}, implies
\begin{equation}
\sup_{t\in[0,T]}\left|\Psi_d(W^{(d)}(t))-\Psi(\bar t(t),\bar q(t))\right|\xrightarrow{\mathbb P}0.
\label{eq:population-risk-uniform-convergence}
\end{equation}
This is because on the spectral-bound event, the same argument used in Lemma \ref{lemma:drift_convergence}, together with Corollary
\ref{cor:symmetric} and uniform integrability, gives uniformly over the spectrally bounded class
\begin{equation}
\mathbb E_x[\mathcal L(G(x;W))]-\Phi(q(W))
\longrightarrow0.
\end{equation}
The regularization term is exactly
\begin{equation}
\frac{\gamma}{2d}\sum_{k\in\mathcal K}\|W_k\|_F^2=\frac{\gamma}{2}
\sum_{k\in\mathcal K}t_k.
\end{equation}
The convergence of the first- and second-order moments supplied by
Theorem \ref{theo:tied-SGD}, together with the continuity of
\(\Phi\), therefore yields
\eqref{eq:population-risk-uniform-convergence}.

For the fixed accuracy \(\epsilon>0\), the time
\(T_\epsilon\) in \eqref{eq:T-epsilon-strong} is finite and independent of \(d\). Thus \eqref{eq:population-risk-uniform-convergence}, evaluated on \([0,T_\epsilon]\), implies
\begin{equation}
\mathbb P\left(\left|\Psi_d(W^{(d)}(T_\epsilon))-\Psi(\bar t(T_\epsilon),\bar q(T_\epsilon))\right|>\frac{\epsilon}{2}\right)\longrightarrow0.
\end{equation}
Set $n_\epsilon:=\left\lceil\frac{T_\epsilon d}{4\alpha_d}\right\rceil$. Therefore, using
\eqref{eq:deterministic-epsilon-gap},
\begin{equation}
\mathbb P\left(\Psi_d(W^{(n_\epsilon)})-\Psi^*>\epsilon\right)\longrightarrow0.
\end{equation}
Finally, by \eqref{eq:T-epsilon-strong} and the definition of $n_\epsilon$,
\begin{align}
n_\epsilon-n_0&=O\left(\frac{d}{\alpha_d}\log_+\frac{1}{\epsilon}\right),
\end{align}
which proves the claimed sample complexity.
\end{proof}
\begin{example}
Consider the setting \(\gamma=0\) and $\mathcal L(G_1,G_2)=(G_1-G_2)^2$ with one learnable student \(S_1=W_1W_1^T\) and one fixed teacher \(S_2\). Then
\begin{equation}
\Phi(q)=2q_{11}+2q_{22}-4q_{12}=2\lim_{d\to\infty}\frac1d\Tr[(S_1-S_2)^2].
\end{equation}
If the teacher is normalized by \(q_{22}=1\), the global minimum \(\Phi^*=0\) is attained at perfect recovery,
\(q_{11}=q_{12}=1\).

At finite \(d\), up to the normalization induced by the time scaling, the corresponding population loss is proportional to $\frac1d\Tr[(S_1-S_2)^2]$. Its gradient with respect to \(W_1\) satisfies
\begin{equation}
\|\nabla_{W_1}\mathcal L\|_F^2
\propto\frac1{d^2}\Tr[S_1(S_1-S_2)^2].
\end{equation}
Hence the effective PL condition reduces, up to fixed normalization constants, to a bound of the form
\begin{equation}
\Tr[S_1(S_1-S_2)^2]\ge c_0\Tr[(S_1-S_2)^2]
\label{eq:quadratic-effective-pl}
\end{equation}
for some \(c_0>0\).
A sufficient condition for
\eqref{eq:quadratic-effective-pl} is 
\(S_1\succeq c_0I_d\). More generally, it requires non-degeneracy only along the directions relevant to \(S_1-S_2\).
\end{example}

\section{Online SGD over S}
\label{app:sgd-over-S}
In this section, we study online SGD directly over the attention matrices
$\{S_k\}_{k=1}^K$. At each step $n$, a fresh sample
$x^{(n)}\sim\mathcal N(0,\mathcal C\otimes I_d)$ is drawn, and for
$k\in\{1,\ldots,K\}$ we consider the update
\begin{equation}
S_k^{(n+1)}=S_k^{(n)}-\alpha_d\chi_k\left(\nabla_{S_k}\mathcal L(G^{(n)}(x^{(n)}))+\frac{\gamma}{d}S_k^{(n)}\right),
\label{eq:S-sgd}
\end{equation}
where $\chi_k:=\mathbf 1_{\{k\in\mathcal K\}}$ so that $S_k$ remains fixed whenever $k\notin\mathcal K$. The tensor $G$ in \eqref{eq:S-sgd} is defined as 
\begin{equation}
G_{ijk}^{(n)}:=\frac1{\sqrt d}\Tr\left[S_k^{(n)}\left(x_j^{(n)}(x_i^{(n)})^T-\mathcal C_{ij}I_d\right)\right].
\label{eq:G-S-sgd}
\end{equation}
We introduce the time as $\tau_d:=\frac{\alpha_d}{d}$ and $t:=n\tau_d$.

The order parameters are defined as:
\begin{equation}
t_k^{(d)}:=\frac1d\Tr[S_k],
\qquad\Psi_{kl}^{(d)}:=\frac1d\Tr[S_kS_l],
\qquad\Omega_{kl}^{(d)}:=\frac1d\Tr[S_kS_l^T].
\label{eq:S-order-parameters}
\end{equation}
Moreover, writing $S_k^{+}:=\frac{S_k+S_k^T}{2}$ and $S_k^{-}:=\frac{S_k-S_k^T}{2}$, we have $\frac{\Omega_{kl}+\Psi_{kl}}{2}=\frac1d
\langle S_k^{+},S_l^{+}\rangle_F$ and $\frac{\Omega_{kl}-\Psi_{kl}}{2}=\frac1d\langle S_k^{-},S_l^{-}\rangle_F$.
Consequently, the feasible domain of the order parameters is
\begin{equation}
\mathcal D:=\left\{(\Psi,\Omega)\in\mathbb S^K\times\mathbb S^K:
\frac{\Omega+\Psi}{2}\succeq0,\quad\frac{\Omega-\Psi}{2}\succeq0\right\}.
\label{eq:S-feasible-domain}
\end{equation}
For $(\Psi,\Omega)\in\mathcal D$, define the potential
\begin{equation}
\Phi(\Psi,\Omega):=\mathbb E_{G_{\Psi,\Omega}}\left[\mathcal L(G_{\Psi,\Omega})\right],
\label{eq:S-potential}
\end{equation}
where $G_{\Psi,\Omega}\in\mathbb R^{L\times L\times K}$ is a centered
Gaussian tensor with covariance $\operatorname{Cov}\left((G_{\Psi,\Omega})_{ijk},(G_{\Psi,\Omega})_{i'j'l}\right)=\mathcal C_{ii'}\mathcal C_{jj'}\Omega_{kl}+\mathcal C_{ij'}\mathcal C_{ji'}\Psi_{kl}$. We denote by $\nabla_\Psi\Phi$ and $\nabla_\Omega\Phi$ the corresponding
matrix derivatives w.r.t. $\Psi$ and $\Omega$.
\begin{assumption}
\label{assum:S-sgd}
The following conditions hold.
\begin{itemize}
\item[1.] The loss $\mathcal L:\mathbb R^{L\times L\times K}\to\mathbb R$
is four times continuously differentiable, and $\mathcal L$ together
with its derivatives up to fourth order have at most polynomial growth.
\item[2.] There exists a constant $C_0>0$, independent of $d$, such that
$\max_{1\le k\le K}\|S_k^{(0)}\|_{\mathrm{op}}\le C_0$ almost surely.
\item[3.] The initial order parameters converge in probability to
deterministic limits: $t_k^{(d)}(0)\to t_k(0)$, $\Psi_{kl}^{(d)}(0)\to\bar\Psi_{kl}(0)$ and $\Omega_{kl}^{(d)}(0)\to\bar\Omega_{kl}(0)$
for every $k,l$.
\item[4.] The regularization is sufficiently large, i.e.,
$\gamma>\gamma_*(C_0,\mathcal C,\mathcal L,K,L)$ for some
$\gamma_*$ independent of $d$.
\end{itemize}
\end{assumption}
Let $\tilde S_k^{(d)}(t)$ denote the piecewise-constant interpolation
\begin{equation}
\tilde S_k^{(d)}(t)=S_k^{(n)},
\qquad
t\in[n\tau_d,(n+1)\tau_d),
\end{equation}
and define $\tilde t^{(d)}$, $\tilde\Psi^{(d)}$, and
$\tilde\Omega^{(d)}$ in the same way according to \eqref{eq:S-order-parameters}.

The following theorem shows that, unlike Theorems \ref{theo:tied-SGD} and \ref{theo:untied-SGD}, direct optimization over $S$ closes at the level of finitely many first- and second-order parameters.
\begin{theorem}
\label{theo:S-sgd-dynamics}
Under Assumption \ref{assum:S-sgd}, suppose that the learning rate satisfies $\alpha_d\le\frac{c_0}{d\log d}$ and $\alpha_d=\Omega(d^{-\iota})$
for some constants $\iota>0$ and $c_0>0$ sufficiently small.
Then, for every finite $T>0$, the empirical order parameters converge uniformly in probability to a deterministic trajectory
$(\bar t(t),\bar\Psi(t),\bar\Omega(t))$.
More precisely, for every $k,l$ and every $\epsilon>0$,
\begin{align}
\lim_{d\to\infty}\mathbb P\left(\sup_{t\in[0,T]}|\tilde t_k^{(d)}(t)-\bar t_k(t)|+|\tilde\Psi_{kl}^{(d)}(t)-\bar\Psi_{kl}(t)|+|\tilde\Omega_{kl}^{(d)}(t)-\bar\Omega_{kl}(t)|>\epsilon\right)=0.
\end{align}
The limiting trajectory is the unique solution of
\begin{align}
\frac{d}{dt}\bar t_k&=-2\chi_k
\sum_{m=1}^K\left([\nabla_\Omega\Phi]_{km}+[\nabla_\Psi\Phi]_{km}
\right)\bar t_m-\gamma\chi_k\bar t_k,
\label{eq:S-t-dynamics}
\\
\frac{d}{dt}\bar\Psi_{kl}
&=-2\chi_k\sum_{m=1}^K\left([\nabla_\Omega\Phi]_{km}\bar\Psi_{ml}+[\nabla_\Psi\Phi]_{km}\bar\Omega_{ml}\right)
-2\chi_l\sum_{m=1}^K\left([\nabla_\Omega\Phi]_{lm}\bar\Psi_{km}+[\nabla_\Psi\Phi]_{lm}\bar\Omega_{km}
\right)
-\gamma(\chi_k+\chi_l)\bar\Psi_{kl},
\label{eq:S-Psi-dynamics}
\\
\frac{d}{dt}\bar\Omega_{kl}
&=
-2\chi_k
\sum_{m=1}^K\left([\nabla_\Omega\Phi]_{km}\bar\Omega_{ml}+[\nabla_\Psi\Phi]_{km}\bar\Psi_{ml}
\right)-2\chi_l
\sum_{m=1}^K
\left([\nabla_\Omega\Phi]_{lm}\bar\Omega_{km}+[\nabla_\Psi\Phi]_{lm}\bar\Psi_{km}\right)-\gamma(\chi_k+\chi_l)\bar\Omega_{kl},
\label{eq:SE-S-sgd}
\end{align}
where the gradients of $\Phi$ are evaluated at $(\bar\Psi(t),\bar\Omega(t))\in\mathcal D$ for $t\in[0,T]$.
\end{theorem}

Building upon this uniform convergence, now we consider the weak recovery. First define the traceless
parts
\begin{equation}
\mathring S_k:=S_k-t_kI_d,
\end{equation}
and the structural overlaps
\begin{equation}
p_{\Psi,kl}^{(d)}:=\frac1d\Tr[\mathring S_k\mathring S_l]=\Psi_{kl}^{(d)}-t_k^{(d)}t_l^{(d)},
\end{equation}
\begin{equation}
p_{\Omega,kl}^{(d)}:=\frac1d\Tr[\mathring S_k\mathring S_l^T]=\Omega_{kl}^{(d)}-t_k^{(d)}t_l^{(d)}.
\label{eq:S-structural-overlaps}
\end{equation}
It is useful to resolve these overlaps into their symmetric and
skew-symmetric channels:
\begin{equation}
p_{+,12}:=\frac{p_{\Omega,12}+p_{\Psi,12}}2,
\qquad p_{-,12}:=\frac{p_{\Omega,12}-p_{\Psi,12}}2.
\label{eq:S-plus-minus-overlap}
\end{equation}
For the teacher, define the corresponding signal strengths
\begin{equation}
a_+:=\frac{p_{\Omega,22}+p_{\Psi,22}}2=\lim_{d\to\infty}\frac1d\left\|\frac{\mathring S_2+\mathring S_2^T}{2}\right\|_F^2,
\end{equation}
and
\begin{equation}
a_-:=\frac{p_{\Omega,22}-p_{\Psi,22}}2=\lim_{d\to\infty}\frac1d\left\|\frac{\mathring S_2-\mathring S_2^T}{2}\right\|_F^2.
\label{eq:S-teacher-channel-strength}
\end{equation}
Thus $a_+,a_-\ge0$. The following is the uninformative initialization and informative teacher assumption.
\begin{assumption}
\label{assum:S-initialization}
As $d\to\infty$, the initial student $S_1(0)$ and the fixed teacher
$S_2$ satisfy
\begin{equation}
\bar t_1(0)=0,\qquad\bar p_{\Psi,12}(0)=0,\qquad\bar p_{\Omega,12}(0)=0.
\label{eq:S-uninformative-initialization}
\end{equation}
Moreover, the teacher is informative in the sense that $a_++a_->0$.
\end{assumption}
For example, \eqref{eq:S-uninformative-initialization} is satisfied by the standard Gaussian initialization with centered
i.i.d.\ entries of variance $\Theta(d^{-1})$.

Define the effective gradients in the symmetric and
skew-symmetric parts by
\begin{equation}
g_+:=[\nabla_\Omega\Phi]_{12}+[\nabla_\Psi\Phi]_{12},
\qquad g_-:=[\nabla_\Omega\Phi]_{12}-[\nabla_\Psi\Phi]_{12}.
\label{eq:S-channel-gradients}
\end{equation}
For a fixed $c>0$, define the weak-recovery sample complexity
\begin{equation}
N_{\mathrm{wr},S}^{(d)}(c):=\inf\left\{
n\ge0:\max\left(|p_{\Psi,12}^{(d,n)}|,|p_{\Omega,12}^{(d,n)}|\right)\ge c\right\}.
\label{eq:S-weak-recovery-hitting-time}
\end{equation}
The following corollary gives the weak recovery conditions.

\begin{corollary}
\label{cor:weak-recovery-S-sgd}
Consider \eqref{eq:S-sgd} with
$\mathcal K=\{1\}$, where $S_1$ is the trainable student and $S_2$ is a fixed teacher. Suppose the
conditions of Theorem \ref{theo:S-sgd-dynamics} and Assumption \ref{assum:S-initialization} hold.

\begin{itemize}
\item[1.]If, at initialization,
\begin{equation}
a_+g_+(0)\neq0
\qquad\text{or}\qquad
a_-g_-(0)\neq0,
\label{eq:S-first-order-signal}
\end{equation}
then there exist constants
$c>0$ and $0<T_-<T_+<\infty$, independent of $d$, such that
\begin{equation}
\lim_{d\to\infty}\mathbb P\left(\left\lfloor\frac{T_-d}{\alpha_d}\right\rfloor<N_{\mathrm{wr},S}^{(d)}(c)
\le\left\lceil\frac{T_+d}{\alpha_d}\right\rceil\right)
=1.
\label{eq:S-successful-recovery}
\end{equation}
In particular, choosing the largest learning-rate scaling $\alpha_d=\Theta\left(\frac1{d\log d}\right)$ covered by Theorem \ref{theo:S-sgd-dynamics} gives the weak recovery sample complexity
\begin{equation}
N_{\mathrm{wr},S}^{(d)}(c)=\Theta_{\mathbb P}(d^2\log d).
\end{equation}

\item[2.] Let the uninformative manifold be $\mathcal U:=\left\{(t,\Psi,\Omega):t_1=0,\Psi_{12}=0,\Omega_{12}=0\right\}$.
Suppose that
\begin{equation}
[\nabla_\Omega\Phi(\Psi,\Omega)]_{12}=[\nabla_\Psi\Phi(\Psi,\Omega)]_{12}=0
\qquad
\text{for every state in }\mathcal U.
\label{eq:S-higher-order-condition}
\end{equation}
Then for every fixed $T<\infty$ and every $\epsilon>0$,
\begin{equation}
\lim_{d\to\infty}\mathbb P\left(\sup_{t\in[0,T]}
\max\left(|p_{\Psi,12}^{(d)}(t)|,|p_{\Omega,12}^{(d)}(t)|\right)>\epsilon\right)=0.
\label{eq:S-failure-recovery}
\end{equation}
In particular, for any learning-rates covered by
Theorem \ref{theo:S-sgd-dynamics}, $S$-SGD cannot achieve weak
recovery within $\mathcal O(d^2\log d)$ samples.
\end{itemize}
\end{corollary}

\subsection{Proof of Theorem \ref{theo:S-sgd-dynamics}}
The proof combines uniform spectral bounds based on the Matrix Freedman inequality with a martingale expansion of the empirical order parameters, and concludes by applying Grönwall's inequality to the resulting approximate ODE.

\subsubsection{Uniform spectral bounds}
\begin{lemma}
\label{lemma:bound-S}
Under the conditions of Theorem \ref{theo:S-sgd-dynamics}, there exists
a constant $C_*>C_0$, independent of $d$ and $T$, such that for every
fixed $T>0$,
\begin{equation}
\lim_{d\to\infty}\mathbb P\left(\sup_{0\le n\le N_T}\max_{1\le k\le K}\|S_k^{(n)}\|_{\mathrm{op}}\le C_*\right)=1,
\label{eq:S-spectral-bound}
\end{equation}
where $N_T:=\left\lfloor\frac{T}{\tau_d}\right\rfloor=\left\lfloor\frac{Td}{\alpha_d}\right\rfloor$.
\end{lemma}
The proof follows essentially the same stopping-time and Matrix Freedman argument as Lemma \ref{lemma:bound-tied-SGD}. The main difference is that direct optimization over $S_k$ yields an expected gradient of the form $d^{-1}H_k$ rather than $d^{-1}H_kW_k$ (see Lemma \ref{lemma:bound-gradient}); moreover, since $S_k$ need not be symmetric, $H_k$ contains both $S_l$ and $S_l^T$.

\begin{proof}
The matrices with $k\notin\mathcal K$ remain fixed, and hence satisfy
$\|S_k^{(n)}\|_{\mathrm{op}}\le C_0$. It therefore suffices to control
the trainable indices $k\in\mathcal K$.

Fix $C_*>C_0$, to be chosen below, and define the stopping time
\begin{equation}
J:=\inf\left\{n\ge0:\max_{1\le k\le K}\|S_k^{(n)}\|_{\mathrm{op}}>C_*
\right\}.
\label{eq:S-stopping-time}
\end{equation}
On $\{n<J\}$,
\begin{equation}
\max_k\|S_k^{(n)}\|_{\mathrm{op}}\le C_*.
\label{eq:S-pre-stopping-bound}
\end{equation}

\paragraph{Step 1: Truncation of the fresh sample.}
As in the proof of Lemma \ref{lemma:bound-tied-SGD}, define
\begin{equation}
\mathcal A_n:=\left\{\|x^{(n)}\|^2\le c_xd\right\}
\cap\left\{\max_{i,j,k}|G_{ijk}^{(n)}|\le c_G\log d\right\},
\label{eq:S-data-truncation}
\end{equation}
where $c_x,c_G>0$ are sufficiently large constants. Similarly to Lemma \ref{lemma:bound-tied-SGD}, we have
\begin{equation}
\mathbf 1_{\{n<J\}}\mathbb P(\mathcal A_n^c\mid\mathcal F_n)\le p_d,
\label{eq:S-truncation-tail}
\end{equation}
where $p_d\le C(e^{-c_1d}+e^{-c_2(\log d)^2})$. Consequently, defining
\begin{equation}
\mathcal A_{\mathrm{data}}:=\bigcap_{n=0}^{N_T-1}\left(\{n\ge J\}\cup\mathcal A_n\right),
\end{equation}
we have
\begin{equation}
\mathbb P(\mathcal A_{\mathrm{data}}^c)\le N_Tp_d
\longrightarrow0,
\label{eq:S-data-high-prob}
\end{equation}
because $\alpha_d=\Omega(d^{-\iota})$.

\paragraph{Step 2: Conditional drift and stopped martingale noise.}
For $k\in\mathcal K$, define the gradient
\begin{equation}
g_k^{(n)}:=\nabla_{S_k}\mathcal L(G^{(n)})=
\sum_{i,j=1}^L\partial_{(ijk)}\mathcal L(G^{(n)})X_{ij}^{(n)},
\label{eq:S-gradient}
\end{equation}
where $X_{ij}^{(n)}:=\frac1{\sqrt d}\left(x_i^{(n)}(x_j^{(n)})^T-\mathcal C_{ij}I_d\right)$. Let $\bar g_k^{(n)}:=\mathbb E[g_k^{(n)}\mid\mathcal F_n]$. The same calculation as in Lemma
\ref{lemma:bound-gradient} gives
\begin{equation}
\bar g_k^{(n)}=\frac1dH_k^{(n)}+\mathcal E_{\mathrm{Stein},k}^{(n)},
\label{eq:S-mean-gradient-expansion}
\end{equation}
where
\begin{align}
H_k^{(n)}:=\sum_{i,j,i',j'=1}^L
\sum_{l=1}^K\mathbb E_x\left[\partial^2_{(ijk),(i'j'l)}\mathcal L(G^{(n)})\right]\left(
\mathcal C_{ii'}\mathcal C_{jj'}S_l^{(n)}+\mathcal C_{ij'}\mathcal C_{ji'}(S_l^{(n)})^T\right).
\label{eq:S-H-definition}
\end{align}
Using Assumption \ref{assum:S-sgd}.1 and the uniform moment bounds for $G^{(n)}$ under \eqref{eq:S-pre-stopping-bound}, the same estimates as
in Lemma \ref{lemma:bound-gradient} yield
\begin{equation}
\|H_k^{(n)}\|_{\mathrm{op}}\le C_H(C_*),
\qquad
\|\mathcal E_{\mathrm{Stein},k}^{(n)}\|_{\mathrm{op}}\le C_E(C_*)d^{-3/2},
\label{eq:S-mean-gradient-bounds}
\end{equation}
uniformly on $\{n<J\}$.

Define the stopped and truncated gradient
\begin{equation}
\widehat g_k^{(n)}:=\mathbf 1_{\{n<J\}}
g_k^{(n)}\mathbf 1_{\mathcal A_n},
\end{equation}
its conditional mean
\begin{equation}
\widehat{\bar g}_k^{(n)}:=\mathbb E[\widehat g_k^{(n)}\mid\mathcal F_n],
\end{equation}
and the martingale difference
\begin{equation}
\widehat\xi_k^{(n)}:=\widehat g_k^{(n)}-\widehat{\bar g}_k^{(n)}.
\end{equation}
Then
\begin{equation}
\widehat{\bar g}_k^{(n)}=\mathbf 1_{\{n<J\}}
\left(\bar g_k^{(n)}-\mathcal T_k^{(n)}\right),\qquad
\mathcal T_k^{(n)}:=\mathbb E\left[g_k^{(n)}\mathbf 1_{\mathcal A_n^c}
\mid\mathcal F_n
\right].
\label{eq:S-tail-definition}
\end{equation}
As in the proof of Lemma \ref{lemma:bound-tied-SGD}, Hölder's inequality
and moment bounds give
\begin{equation}
\mathbf 1_{\{n<J\}}\max\left\{\left\|\mathbb E[\widehat\xi_k^{(n)}
(\widehat\xi_k^{(n)})^T\mid\mathcal F_n]\right\|_{\mathrm{op}},\left\|\mathbb E[(\widehat\xi_k^{(n)})^T\widehat\xi_k^{(n)}\mid\mathcal F_n]\right\|_{\mathrm{op}}\right\}\le C_g(C_*).
\label{eq:S-gradient-variance}
\end{equation}
On $\mathcal A_n\cap\{n<J\}$, we also have
\begin{equation}
\|\widehat g_k^{(n)}\|_{\mathrm{op}}\le C\sqrt d\left(1+(\log d)^r\right)
=:B_d
\end{equation}
for some fixed integer $r$, and consequently
\begin{equation}
\|\widehat\xi_k^{(n)}\|_{\mathrm{op}}\le2B_d.
\label{eq:S-noise-increment}
\end{equation}

Finally, by Cauchy--Schwarz and \eqref{eq:S-truncation-tail},
the truncation error satisfies
\begin{equation}
\mathbf 1_{\{n<J\}}\|\mathcal T_k^{(n)}\|_{\mathrm{op}}
\le r_d,\qquad r_d=o(d^{-M})
\label{eq:S-tail-error}
\end{equation}
for every fixed $M>0$.

\paragraph{Step 3: Uniform control of the martingale noise.}
Set $a_d:=1-\frac{\alpha_d\gamma}{d}=1-\tau_d\gamma$.
For sufficiently large $d$, $a_d\in(0,1)$. Define $D_{k,n}:=-\alpha_d\widehat\xi_k^{(n)}$
and, for each terminal time $m\ge1$,
\begin{equation}
Y_{k,m}:=\sum_{n=0}^{m-1}
a_d^{m-1-n}D_{k,n}.
\label{eq:S-discounted-noise}
\end{equation}
For each fixed $m$, the summands form a square matrix martingale
difference sequence. By \eqref{eq:S-noise-increment},
\begin{equation}
\|a_d^{m-1-n}D_{k,n}\|_{\mathrm{op}}\le2\alpha_dB_d=:R_d,
\end{equation}
where $R_d=o((\log d)^{-1})$ under $\alpha_d\le c_0(d\log d)^{-1}$.
Moreover, by \eqref{eq:S-gradient-variance},
\begin{align}
\sigma_d^2\le C_g\alpha_d^2\sum_{r=0}^{\infty}a_d^{2r}=\frac{C_g\alpha_d^2}{1-a_d^2}
\le C\frac{d\alpha_d}{\gamma}
\le\frac{Cc_0}{\gamma\log d}.
\label{eq:S-discounted-variance}
\end{align}
Hence the Matrix Freedman inequality gives, for every fixed
$\delta>0$,
\begin{equation}
\mathbb P(\|Y_{k,m}\|_{\mathrm{op}}\ge\delta)\le2d\exp\left[
-\frac{\delta^2/2}{\sigma_d^2+R_d\delta/3}\right]\le2d^{1-c\delta^2/c_0}
\label{eq:S-Freedman}
\end{equation}
for all sufficiently large $d$. Since
$N_T\le C_Td^{1+\iota}$ and $K$ is fixed, a union bound gives
\begin{equation}
\mathbb P\left(\max_{k\in\mathcal K}
\max_{1\le m\le N_T}\|Y_{k,m}\|_{\mathrm{op}}>\delta\right)\le C_Td^{2+\iota-c\delta^2/c_0}.
\end{equation}
Choosing $c_0>0$ sufficiently small yields
\begin{equation}
\mathbb P(\mathcal A_{\mathrm{noise}}^c)\longrightarrow0,
\label{eq:S-noise-high-prob}
\end{equation}
where
\begin{equation}
\mathcal A_{\mathrm{noise}}:=\left\{\max_{k\in\mathcal K}\max_{1\le m\le N_T}\|Y_{k,m}\|_{\mathrm{op}}\le\delta\right\}.
\end{equation}

\paragraph{Step 4: Closing the stopping-time argument.}
On $\mathcal A_{\mathrm{data}}$, for every $m\le J$ and every $n<m$,
\begin{equation}
g_k^{(n)}=\bar g_k^{(n)}-\mathcal T_k^{(n)}+\widehat\xi_k^{(n)}.
\end{equation}
For $k\in\mathcal K$, unrolling \eqref{eq:S-sgd} therefore gives
\begin{equation}
S_k^{(m)}=a_d^mS_k^{(0)}-\alpha_d
\sum_{n=0}^{m-1}a_d^{m-1-n}\bar g_k^{(n)}+\alpha_d
\sum_{n=0}^{m-1}a_d^{m-1-n}\mathcal T_k^{(n)}+Y_{k,m}.
\label{eq:S-unrolled}
\end{equation}
Since $\alpha_d\sum_{r=0}^{\infty}a_d^r=\frac{d}{\gamma}$, \eqref{eq:S-mean-gradient-bounds} and \eqref{eq:S-tail-error} imply that, on
$\mathcal A_{\mathrm{data}}\cap\mathcal A_{\mathrm{noise}}$ and for
$m\leq J\wedge N_T$,
\begin{align}
\|S_k^{(m)}\|_{\mathrm{op}}\le C_0+\frac{C_H(C_*)}{\gamma}+\frac{C_E(C_*)}{\gamma\sqrt d}+\frac{dr_d}{\gamma}+\delta=C_0+\frac{C_H(C_*)}{\gamma}+\delta+o(1).
\label{eq:S-closing-bound}
\end{align}

Fix, for instance, $C_*:=2C_0+1$ and $\delta:=\frac{C_*-C_0}{4}$. By Assumption \ref{assum:S-sgd}.4, $\gamma$ can be chosen sufficiently
large so that
\begin{equation}
\frac{C_H(C_*)}{\gamma}\le\frac{C_*-C_0}{4}.
\end{equation}
Then, for all sufficiently large $d$,
\begin{equation}
\max_k\|S_k^{(m)}\|_{\mathrm{op}}<C_*
\qquad\text{for every }m\leq J\wedge N_T
\end{equation}
on
$\mathcal A_{\mathrm{data}}\cap\mathcal A_{\mathrm{noise}}$.

If $J\le N_T$, taking $m=J$ gives
$\max_k\|S_k^{(J)}\|_{\mathrm{op}}<C_*$, contradicting the definition
of $J$. Hence $J>N_T$ on $\mathcal A_{\mathrm{data}}\cap\mathcal A_{\mathrm{noise}}$ for all
sufficiently large $d$.

Finally,
\begin{equation}
\mathbb P(J\le N_T)\le\mathbb P(\mathcal A_{\mathrm{data}}^c)+\mathbb P(\mathcal A_{\mathrm{noise}}^c)
\longrightarrow0
\end{equation}
by \eqref{eq:S-data-high-prob} and
\eqref{eq:S-noise-high-prob}. This proves the lemma.
\end{proof}

\subsubsection{Moment expansion}
Let
\begin{equation}
\mathbf Q^{(d)}(t):=\left(
t^{(d)}(t),\Psi^{(d)}(t),\Omega^{(d)}(t)
\right)
\end{equation}
denote the piecewise-constant interpolation of the empirical order
parameters, and let $F=(F_t,F_\Psi,F_\Omega)$ denote the finite-dimensional vector field on the right-hand sides of
\eqref{eq:S-t-dynamics}--\eqref{eq:SE-S-sgd}.

\begin{lemma}
\label{lemma:macro-S-sgd}
Under the conditions of Theorem \ref{theo:S-sgd-dynamics}, for every
fixed $T>0$,
\begin{equation}
\mathbf Q^{(d)}(t)=\mathbf Q^{(d)}(0)+\int_0^t
F(\mathbf Q^{(d)}(s))\,ds+E_{\mathrm{total}}^{(d)}(t),
\qquad t\in[0,T],
\label{eq:S-macro-integral}
\end{equation}
where
\begin{equation}
\sup_{t\in[0,T]}\left\|E_{\mathrm{total}}^{(d)}(t)\right\|_\infty\xrightarrow{\mathbb P}0.
\label{eq:S-macro-error}
\end{equation}
\end{lemma}
The proof follows the same martingale-decomposition and drift-convergence strategy as Lemma \ref{lemma:macro-tied-SGD}. The main simplification is that direct optimization over $S$ closes at the finite-dimensional level of $(t,\Psi,\Omega)$, so no higher-order moments or weighted moment-space estimates are required; the only additional feature is the simultaneous appearance of $S_l$ and $S_l^T$ due to the possible asymmetry of $S_l$.
\begin{proof}
Let $C_*$ be the constant in Lemma \ref{lemma:bound-S}, and define
\begin{equation}
J:=\inf\left\{n\ge0:\max_{1\le k\le K}\|S_k^{(n)}\|_{\mathrm{op}}>C_*\right\}.
\end{equation}
Then, for $N_T:=\left\lfloor\frac{T}{\tau_d}\right\rfloor=\left\lfloor\frac{Td}{\alpha_d}\right\rfloor$, Lemma \ref{lemma:bound-S} gives
\begin{equation}
\mathbb P(J>N_T)\longrightarrow1.
\label{eq:S-macro-stopping}
\end{equation}
All estimates below are uniform on $\{n<J\}$.

\paragraph{Step 1: Conditional drift.}
Recall from the proof of Lemma \ref{lemma:bound-S} that
\begin{equation}
\bar g_k^{(n)}=\frac1dH_k^{(n)}+\mathcal E_{\mathrm{Stein},k}^{(n)},
\qquad\|\mathcal E_{\mathrm{Stein},k}^{(n)}\|_{\mathrm{op}}\le C_Ed^{-3/2},
\label{eq:S-macro-mean-gradient}
\end{equation}
where
\begin{align}
H_k^{(n)}=\sum_{i,j,i',j',m}\mathbb E_x\left[\partial^2_{(ijk),(i'j'm)}\mathcal L(G^{(n)})\right]
\left(\mathcal C_{ii'}\mathcal C_{jj'}S_m^{(n)}+\mathcal C_{ij'}\mathcal C_{ji'}(S_m^{(n)})^T\right).
\end{align}
Define the finite-dimensional coefficients
\begin{align}
\Gamma_{\Omega,km}^{(d)}&:=\frac12\sum_{i,j,i',j'}\mathcal C_{ii'}\mathcal C_{jj'}\,\mathbb E_x\left[\partial^2_{(ijk),(i'j'm)}\mathcal L(G^{(n)})\right],
\label{eq:S-Gamma-Omega-d}
\\
\Gamma_{\Psi,km}^{(d)}&:=\frac12\sum_{i,j,i',j'}\mathcal C_{ij'}\mathcal C_{ji'}\,\mathbb E_x\left[\partial^2_{(ijk),(i'j'm)}\mathcal L(G^{(n)})\right].
\label{eq:S-Gamma-Psi-d}
\end{align}
Then
\begin{equation}
H_k^{(n)}=2\sum_{m=1}^K\left(\Gamma_{\Omega,km}^{(d)}S_m^{(n)}+\Gamma_{\Psi,km}^{(d)}(S_m^{(n)})^T\right).
\label{eq:S-H-Gamma}
\end{equation}
We claim that, uniformly over the spectrally bounded class,
\begin{align}
\delta_d(C_*):=\sup_{\max_m\|S_m\|_{\mathrm{op}}\le C_*}
\max_{k,m}\Big(&\left|\Gamma_{\Omega,km}^{(d)}-[\nabla_\Omega\Phi(\Psi,\Omega)]_{km}\right|+
\left|\Gamma_{\Psi,km}^{(d)}-[\nabla_\Psi\Phi(\Psi,\Omega)]_{km}\right|\Big)
\longrightarrow0.
\label{eq:S-uniform-drift-convergence}
\end{align}
The proof is the same subsequence argument as in Lemma
\ref{lemma:drift_convergence}.

\paragraph{Step 2: One-step expansion of the order parameters.}
Define
\begin{equation}
\xi_k^{(n)}:=g_k^{(n)}-\bar g_k^{(n)}.
\end{equation}
Using $\tau_d=\alpha_d/d$, the SGD update can be written as
\begin{equation}
\Delta S_k^{(n)}:=S_k^{(n+1)}-S_k^{(n)}=-\tau_d\chi_k\left(H_k^{(n)}+\gamma S_k^{(n)}\right)-\alpha_d\chi_k\xi_k^{(n)}-\alpha_d\chi_k\mathcal E_{\mathrm{Stein},k}^{(n)}.
\label{eq:S-macro-increment}
\end{equation}
For the normalized trace,
\begin{equation}
t_k^{(n+1)}-t_k^{(n)}=\frac1d\Tr[\Delta S_k^{(n)}].
\end{equation}
For the two second-order overlaps,
\begin{align}
\Psi_{kl}^{(n+1)}-\Psi_{kl}^{(n)}
&=\frac1d\Tr[
\Delta S_k^{(n)}S_l^{(n)}+S_k^{(n)}\Delta S_l^{(n)}]+\frac1d\Tr[\Delta S_k^{(n)}\Delta S_l^{(n)}],
\label{eq:S-Psi-one-step}
\\
\Omega_{kl}^{(n+1)}-\Omega_{kl}^{(n)}&=\frac1d\Tr[\Delta S_k^{(n)}(S_l^{(n)})^T+S_k^{(n)}(\Delta S_l^{(n)})^T]+
\frac1d\Tr[\Delta S_k^{(n)}(\Delta S_l^{(n)})^T].
\label{eq:S-Omega-one-step}
\end{align}
Substituting \eqref{eq:S-H-Gamma} gives, for example,
\begin{equation}
\frac1d\Tr[H_k^{(n)}]=2\sum_m\left(\Gamma_{\Omega,km}^{(d)}+\Gamma_{\Psi,km}^{(d)}\right)t_m^{(n)},
\end{equation}
\begin{equation}
\frac1d\Tr[H_k^{(n)}S_l^{(n)}]=2\sum_m\left(\Gamma_{\Omega,km}^{(d)}\Psi_{ml}^{(n)}+\Gamma_{\Psi,km}^{(d)}\Omega_{ml}^{(n)}\right),
\end{equation}
and
\begin{equation}
\frac1d\Tr[H_k^{(n)}(S_l^{(n)})^T]=2\sum_m\left(\Gamma_{\Omega,km}^{(d)}\Omega_{ml}^{(n)}+\Gamma_{\Psi,km}^{(d)}\Psi_{ml}^{(n)}\right).
\end{equation}
Therefore the first-order terms coincide with
$\tau_dF(\mathbf Q^{(n)})$, up to an error
$\tau_d\varepsilon_{\mathrm{drift}}^{(n)}$ satisfying
\begin{equation}
\mathbf 1_{\{n<J\}}\|\varepsilon_{\mathrm{drift}}^{(n)}\|_\infty
\le C(C_*)\delta_d(C_*).
\label{eq:S-drift-error}
\end{equation}
The first-order martingale terms are
\begin{align}
\eta_{t,k}^{(n)}&:=-\frac{\chi_k}{d}\Tr[\xi_k^{(n)}],
\\
\eta_{\Psi,kl}^{(n)}&:=-\frac1d\Tr\left[\chi_k\xi_k^{(n)}S_l^{(n)}+\chi_lS_k^{(n)}\xi_l^{(n)}\right],
\\
\eta_{\Omega,kl}^{(n)}&:=-\frac1d\Tr\left[\chi_k\xi_k^{(n)}(S_l^{(n)})^T+\chi_lS_k^{(n)}(\xi_l^{(n)})^T\right].
\end{align}
Collect them into $\eta^{(n)}$. The remaining terms in
\eqref{eq:S-Psi-one-step}--\eqref{eq:S-Omega-one-step}, are denoted
by $R^{(n)}$. Hence
\begin{equation}
\mathbf Q^{(n+1)}-\mathbf Q^{(n)}=\tau_dF(\mathbf Q^{(n)})+\alpha_d\eta^{(n)}+R^{(n)}+\tau_d\varepsilon_{\mathrm{drift}}^{(n)}.
\label{eq:S-Q-one-step}
\end{equation}

\paragraph{Step 3: Vanishing of the martingale noise.}
As in Step 3 of the proof of Lemma \ref{lemma:macro-tied-SGD}, every coordinate of $\eta^{(n)}$ is a finite sum of linear functionals of $\xi_k^{(n)}$. The same moment estimate therefore gives
\begin{equation}
\mathbf 1_{\{n<J\}}\mathbb E\left[
|\eta_a^{(n)}|^2\,\middle|\,\mathcal F_n\right]\le\frac{C}{d^2}
\label{eq:S-eta-variance}
\end{equation}
for every coordinate $a$ of $(t,\Psi,\Omega)$.

Thus the stopped process
\begin{equation}
M_a^{(m)}:=\sum_{n=0}^{m-1}\alpha_d\eta_a^{(n)}
\mathbf 1_{\{n<J\}}
\end{equation}
is a square-integrable martingale with
\begin{equation}
\mathbb E[\langle M_a\rangle_{N_T}]\le
N_T\alpha_d^2\frac{C}{d^2}\le C_T\frac{\alpha_d}{d}.
\end{equation}
By Doob's maximal inequality,
\begin{equation}
\max_a\sup_{m\le N_T}|M_a^{(m)}|\xrightarrow{\mathbb P}0.
\label{eq:S-martingale-vanish}
\end{equation}
Here the maximum is over only finitely many coordinates, since $K$ is fixed.

\paragraph{Step 4: Vanishing of the higher-order residuals.}
On $\{n<J\}$, the same moment estimates as in Lemma
\ref{lemma:macro-tied-SGD} give
\begin{equation}
\mathbb E\left[\|R^{(n)}\|_\infty\,\middle|\,\mathcal F_n
\right]\le C\left(\alpha_d^2+\alpha_d\tau_d+\tau_d^2+\alpha_dd^{-3/2}\right).
\label{eq:S-remainder-bound}
\end{equation}
The first three terms arise from the products of the
drift and stochastic increments in
\eqref{eq:S-Psi-one-step}--\eqref{eq:S-Omega-one-step}, while the last
term comes from the residual in
\eqref{eq:S-macro-increment}.

Since $N_T=O\left(\frac d{\alpha_d}\right)$ and $\tau_d=\frac{\alpha_d}{d}$, we obtain
\begin{align}
\mathbb E\left[\sum_{n=0}^{N_T-1}\mathbf 1_{\{n<J\}}\|R^{(n)}\|_\infty\right]&\le C_T\left(d\alpha_d+\alpha_d+\frac{\alpha_d}{d}+d^{-1/2}\right)\longrightarrow0,
\label{eq:S-cumulative-remainder}
\end{align}
because $\alpha_d\le c_0(d\log d)^{-1}$.
Hence, by Markov's inequality,
\begin{equation}
\sup_{m\le N_T}\left\|\sum_{n=0}^{m-1}R^{(n)}
\right\|_\infty\xrightarrow{\mathbb P}0.
\label{eq:S-remainder-vanish}
\end{equation}
Likewise, by \eqref{eq:S-drift-error},
\begin{equation}
\sup_{m\le N_T}\left\|\tau_d\sum_{n=0}^{m-1}\varepsilon_{\mathrm{drift}}^{(n)}
\right\|_\infty\le TC(C_*)\delta_d(C_*)
\longrightarrow0.
\label{eq:S-cumulative-drift-error}
\end{equation}

\paragraph{Step 5: Integral equation.}
For $m(t):=\left\lfloor\frac{t}{\tau_d}\right\rfloor$, summing \eqref{eq:S-Q-one-step} gives
\begin{align}
\mathbf Q^{(d)}(t)=\mathbf Q^{(d)}(0)+\tau_d\sum_{n=0}^{m(t)-1}F(\mathbf Q^{(n)})+E_{\mathrm{noise}}^{(d)}(t)+E_{\mathrm{rem}}^{(d)}(t)+E_{\mathrm{drift}}^{(d)}(t),
\end{align}
where $E_{\mathrm{noise}}^{(d)}(t):=\alpha_d\sum_{n=0}^{m(t)-1}\eta^{(n)}$, $E_{\mathrm{rem}}^{(d)}(t):=\sum_{n=0}^{m(t)-1}R^{(n)}$ and $E_{\mathrm{drift}}^{(d)}(t):=\tau_d\sum_{n=0}^{m(t)-1}\varepsilon_{\mathrm{drift}}^{(n)}$.

Because $\mathbf Q^{(d)}$ is piecewise constant,
\begin{align}
\int_0^tF(\mathbf Q^{(d)}(s))\,ds=\tau_d\sum_{n=0}^{m(t)-1}F(\mathbf Q^{(n)})+(t-m(t)\tau_d)F(\mathbf Q^{(m(t))}).
\end{align}
On $\{J>N_T\}$, the spectral bound implies that
$\mathbf Q^{(d)}(t)$ remains in a bounded subset, on which $F$ is bounded. Therefore the discretization error $E_{\mathrm{disc}}^{(d)}(t):=(t-m(t)\tau_d)F(\mathbf Q^{(m(t))})$ satisfies
\begin{equation}
\sup_{t\in[0,T]}\|E_{\mathrm{disc}}^{(d)}(t)\|_\infty\le C\tau_d
\longrightarrow0.
\label{eq:S-disc-error}
\end{equation}
Combining
\eqref{eq:S-martingale-vanish},
\eqref{eq:S-remainder-vanish},
\eqref{eq:S-cumulative-drift-error},
and \eqref{eq:S-disc-error}, we conclude that $\sup_{t\in[0,T]}\left\|E_{\mathrm{total}}^{(d)}(t)
\right\|_\infty\xrightarrow{\mathbb P}0$, where
\begin{equation}
E_{\mathrm{total}}^{(d)}=E_{\mathrm{noise}}^{(d)}+E_{\mathrm{rem}}^{(d)}+E_{\mathrm{drift}}^{(d)}+E_{\mathrm{disc}}^{(d)}.
\end{equation}
This proves the lemma.
\end{proof}

\subsubsection{Final proof of Theorem \ref{theo:S-sgd-dynamics}}
Fix an arbitrary finite time horizon $T>0$, and let $C_*$ be the
constant in Lemma \ref{lemma:bound-S}. Define
\begin{equation}
\mathcal E_d(T):=\left\{\sup_{0\le n\le N_T}
\max_{1\le k\le K}\|S_k^{(n)}\|_{\mathrm{op}}\le C_*\right\},
\qquad N_T=\left\lfloor\frac{Td}{\alpha_d}\right\rfloor .
\end{equation}
By Lemma \ref{lemma:bound-S},
\begin{equation}
\mathbb P(\mathcal E_d(T))\longrightarrow1.
\label{eq:S-final-spectral-event}
\end{equation}
On $\mathcal E_d(T)$, the empirical order parameters satisfy
\begin{equation}
|t_k^{(d)}(t)|\le C_*,
\qquad |\Psi_{kl}^{(d)}(t)|\le C_*^2,
\qquad |\Omega_{kl}^{(d)}(t)|\le C_*^2,
\label{eq:S-final-Q-bound}
\end{equation}
uniformly over $t\in[0,T]$.

We first address the fact that the vector field $F$ is naturally
defined only on the domain $\mathbb R^K\times\mathcal D$. For
$(\Psi,\Omega)\in\mathbb S^K\times\mathbb S^K$, write
\begin{equation}
Q^+:=\frac{\Omega+\Psi}{2},
\qquad
Q^-:=\frac{\Omega-\Psi}{2},
\end{equation}
and define
\begin{equation}
\widehat Q^\pm:=\Pi_{\mathbb S_+^K}(Q^\pm),
\qquad
\widehat\Omega:=\widehat Q^++\widehat Q^-,
\qquad
\widehat\Psi:=\widehat Q^+-\widehat Q^-.
\label{eq:S-projection-extension}
\end{equation}
Using
$\nabla_\Psi\Phi(\widehat\Psi,\widehat\Omega)$ and
$\nabla_\Omega\Phi(\widehat\Psi,\widehat\Omega)$ in
\eqref{eq:S-t-dynamics}--\eqref{eq:SE-S-sgd} defines an extension
$\widetilde F$ of $F$ to the finite-dimensional space
$\mathcal X:=\mathbb R^K\times\mathbb S^K\times\mathbb S^K.$
By Assumption \ref{assum:S-sgd}.1, Price's theorem, and the
non-expansiveness of the PSD projection, $\widetilde F$ is locally
Lipschitz on $\mathcal X$.

Choose a fixed radius $R$ larger than the bound in
\eqref{eq:S-final-Q-bound}. By composing $\widetilde F$ with the
projection onto the closed ball $B_R\subset\mathcal X$, we
obtain a globally Lipschitz vector field $\widehat F$ which agrees
with $F$ on all states satisfying
\eqref{eq:S-final-Q-bound}. Let $L_R<\infty$ denote a global
Lipschitz constant of $\widehat F$.

Consider the globally well-posed ODE
\begin{equation}
\frac{d}{dt}\bar{\mathbf Q}(t)=\widehat F(\bar{\mathbf Q}(t)),
\qquad
\bar{\mathbf Q}(0)=\left(\bar t(0),\bar\Psi(0),\bar\Omega(0)\right),
\label{eq:S-extended-ODE}
\end{equation}
where $\bar{\mathbf Q}=(\bar t,\bar\Psi,\bar\Omega)$.

By Assumption \ref{assum:S-sgd}.3 and the fact that $K$ is fixed,
\begin{equation}
\left\|\mathbf Q^{(d)}(0)-\bar{\mathbf Q}(0)\right\|_\infty
\xrightarrow{\mathbb P}0.
\label{eq:S-final-initial-convergence}
\end{equation}
On the other hand, Lemma \ref{lemma:macro-S-sgd} gives
\begin{equation}
\mathbf Q^{(d)}(t)=\mathbf Q^{(d)}(0)+\int_0^t
F(\mathbf Q^{(d)}(s))\,ds+E_{\mathrm{total}}^{(d)}(t),
\label{eq:S-final-empirical-integral}
\end{equation}
with
\begin{equation}
\sup_{t\in[0,T]}\left\|E_{\mathrm{total}}^{(d)}(t)\right\|_\infty\xrightarrow{\mathbb P}0.
\label{eq:S-final-error-convergence}
\end{equation}

On $\mathcal E_d(T)$, every empirical state satisfies \eqref{eq:S-final-Q-bound}. Hence \begin{equation}
F(\mathbf Q^{(d)}(t))=\widehat F(\mathbf Q^{(d)}(t)),
\qquad t\in[0,T].
\end{equation}
Subtracting the integral form of \eqref{eq:S-extended-ODE} from
\eqref{eq:S-final-empirical-integral}, we obtain on
$\mathcal E_d(T)$
\begin{align}
\left\|\mathbf Q^{(d)}(t)-\bar{\mathbf Q}(t)\right\|_\infty
&\le\left\|\mathbf Q^{(d)}(0)-\bar{\mathbf Q}(0)
\right\|_\infty+\sup_{s\le T}\left\|E_{\mathrm{total}}^{(d)}(s)\right\|_\infty+L_R\int_0^t\left\|\mathbf Q^{(d)}(s)-\bar{\mathbf Q}(s)\right\|_\infty ds.
\end{align}
Therefore, by Gr\"onwall's inequality,
\begin{align}
\sup_{t\in[0,T]}\left\|\mathbf Q^{(d)}(t)-\bar{\mathbf Q}(t)\right\|_\infty\le e^{L_RT}\left(\left\|\mathbf Q^{(d)}(0)-\bar{\mathbf Q}(0)\right\|_\infty+\sup_{t\in[0,T]}\left\|E_{\mathrm{total}}^{(d)}(t)\right\|_\infty\right)
\label{eq:S-final-Gronwall}
\end{align}
on $\mathcal E_d(T)$. Combining
\eqref{eq:S-final-spectral-event},
\eqref{eq:S-final-initial-convergence}, and
\eqref{eq:S-final-error-convergence} yields
\begin{equation}
\sup_{t\in[0,T]}\left\|\mathbf Q^{(d)}(t)-\bar{\mathbf Q}(t)
\right\|_\infty\xrightarrow{\mathbb P}0.
\label{eq:S-final-uniform-convergence}
\end{equation}

It remains to show that $\bar{\mathbf Q}$ is a solution of
the original ODE rather than merely a solution of its extension.
For every finite $d$ and every $t$, $\left(\Psi^{(d)}(t),\Omega^{(d)}(t)
\right)\in\mathcal D$. Since $\mathcal D$ is closed, \eqref{eq:S-final-uniform-convergence} implies $\left(\bar\Psi(t),\bar\Omega(t)\right)\in\mathcal D
$ for $t\in[0,T]$. Likewise, \eqref{eq:S-final-Q-bound} and
\eqref{eq:S-final-uniform-convergence} imply that
$\bar{\mathbf Q}(t)$ remains inside the region where
$\widehat F=F$. Consequently,
\begin{equation}
\frac{d}{dt}\bar{\mathbf Q}(t)=F(\bar{\mathbf Q}(t)),
\qquad t\in[0,T],
\end{equation}
and hence $\bar{\mathbf Q}$ satisfies
\eqref{eq:S-t-dynamics}--\eqref{eq:SE-S-sgd}.

Finally, the locally Lipschitz extension $\widetilde F$ implies local
uniqueness of the ODE. Thus the solution $\bar{\mathbf Q}$ is unique.
This proves Theorem \ref{theo:S-sgd-dynamics}.

\subsection{Proof of Corollary \ref{cor:weak-recovery-S-sgd}}
Recall that $S_1$ is trainable and $S_2$ is fixed. For brevity, write
\begin{equation}
A_{kl}(t):=[\nabla_\Omega\Phi(\bar\Psi(t),\bar\Omega(t))]_{kl},
\qquad
B_{kl}(t):=
[\nabla_\Psi\Phi(\bar\Psi(t),\bar\Omega(t))]_{kl}.
\end{equation}
Since $S_2$ is fixed, $\bar t_2,p_{\Psi,22},p_{\Omega,22},a_+,a_-$ are constant in time.

We first derive the limiting dynamics of the structural overlaps. From
\eqref{eq:S-t-dynamics}, with $\chi_1=1$ and $\chi_2=0$,
\begin{equation}
\frac{d}{dt}\bar t_1=-2(A_{11}+B_{11})\bar t_1-2(A_{12}+B_{12})\bar t_2-\gamma\bar t_1.
\label{eq:S-weak-t1}
\end{equation}
Similarly, \eqref{eq:S-Psi-dynamics} and
\eqref{eq:SE-S-sgd} give
\begin{align}
\frac{d}{dt}\bar\Psi_{12}
&=-2A_{11}\bar\Psi_{12}-2B_{11}\bar\Omega_{12}-2A_{12}\bar\Psi_{22}-2B_{12}\bar\Omega_{22}-\gamma\bar\Psi_{12},
\label{eq:S-weak-Psi12}
\\
\frac{d}{dt}\bar\Omega_{12}
&=-2A_{11}\bar\Omega_{12}-2B_{11}\bar\Psi_{12}-2A_{12}\bar\Omega_{22}-2B_{12}\bar\Psi_{22}-\gamma\bar\Omega_{12}.
\label{eq:S-weak-Omega12}
\end{align}
Recall
\begin{equation}
\bar p_{\Psi,12}=\bar\Psi_{12}-\bar t_1\bar t_2,
\qquad
\bar p_{\Omega,12}=\bar\Omega_{12}-\bar t_1\bar t_2.
\end{equation}
Subtracting $\bar t_2$ times \eqref{eq:S-weak-t1} from
\eqref{eq:S-weak-Psi12} and \eqref{eq:S-weak-Omega12}, respectively, we obtain
\begin{align}
\frac{d}{dt}\bar p_{\Psi,12}&=-2A_{11}\bar p_{\Psi,12}
-2B_{11}\bar p_{\Omega,12}-2A_{12}p_{\Psi,22}
-2B_{12}p_{\Omega,22}-\gamma\bar p_{\Psi,12},
\label{eq:S-pPsi-dynamics}
\\
\frac{d}{dt}\bar p_{\Omega,12}
&=-2A_{11}\bar p_{\Omega,12}-2B_{11}\bar p_{\Psi,12}-2A_{12}p_{\Omega,22}-2B_{12}p_{\Psi,22}-\gamma\bar p_{\Omega,12}.
\label{eq:S-pOmega-dynamics}
\end{align}
It is convenient to use
\begin{equation}
\bar p_+:=\frac{\bar p_{\Omega,12}+\bar p_{\Psi,12}}2,
\qquad
\bar p_-:=\frac{\bar p_{\Omega,12}-\bar p_{\Psi,12}}2.
\end{equation}
Define
\begin{equation}
h_+(t):=A_{11}(t)+B_{11}(t),
\qquad
h_-(t):=A_{11}(t)-B_{11}(t),
\end{equation}
and recall
\begin{equation}
g_+(t):=A_{12}(t)+B_{12}(t),
\qquad
g_-(t):=A_{12}(t)-B_{12}(t).
\end{equation}
Adding and subtracting
\eqref{eq:S-pPsi-dynamics}--\eqref{eq:S-pOmega-dynamics} yields the
decoupled equations
\begin{align}
\frac{d}{dt}\bar p_+&=-\left(2h_+(t)+\gamma\right)\bar p_+
-2a_+g_+(t),
\label{eq:S-p-plus-dynamics}
\\
\frac{d}{dt}\bar p_-&=-\left(2h_-(t)+\gamma\right)\bar p_-
-2a_-g_-(t).
\label{eq:S-p-minus-dynamics}
\end{align}
By Assumption \ref{assum:S-initialization}, $\bar p_+(0)=\bar p_-(0)=0$. Consequently,
\begin{equation}
\dot{\bar p}_+(0)=-2a_+g_+(0),
\qquad
\dot{\bar p}_-(0)=-2a_-g_-(0).
\label{eq:S-channel-initial-drift}
\end{equation}
We now prove the two claims separately.

\paragraph{Case 1.}
Suppose \eqref{eq:S-first-order-signal} holds. Then there exists $\sigma\in\{+,-\}$ such that
\begin{equation}
v_\sigma:=\dot{\bar p}_\sigma(0)\neq0.
\end{equation}
By Theorem \ref{theo:S-sgd-dynamics}, the limiting trajectory is continuously differentiable, and hence $\dot{\bar p}_\sigma(t)$ is
continuous. Therefore there exists $T_0>0$, independent of $d$, such that
\begin{equation}
\left|\dot{\bar p}_\sigma(t)-v_\sigma\right|\le\frac{|v_\sigma|}{2},
\qquad0\le t\le T_0.
\end{equation}
In particular, $\dot{\bar p}_\sigma(t)$ has the same sign as $v_\sigma$ throughout $[0,T_0]$ and
\begin{equation}
|\dot{\bar p}_\sigma(t)|\ge\frac{|v_\sigma|}{2}.
\end{equation}
Since $\bar p_\sigma(0)=0$,
\begin{equation}
|\bar p_\sigma(T_0)|\ge\frac{|v_\sigma|T_0}{2}.
\label{eq:S-positive-channel-overlap}
\end{equation}
Moreover,
\begin{equation}
\max\left(|\bar p_{\Psi,12}|,
|\bar p_{\Omega,12}|\right)\ge\max(|\bar p_+|,|\bar p_-|).
\label{eq:S-channel-to-overlap}
\end{equation}
Set $c:=\frac{|v_\sigma|T_0}{4}$ and $T_+:=T_0$. Then
\begin{equation}
\max\left(|\bar p_{\Psi,12}(T_+)|,|\bar p_{\Omega,12}(T_+)|\right)\ge2c.
\label{eq:S-upper-recovery-limit}
\end{equation}

On the other hand, because both structural overlaps vanish at initialization and are continuous, there exists
$T_-\in(0,T_+)$ such that
\begin{equation}
\sup_{0\le t\le T_-}\max\left(|\bar p_{\Psi,12}(t)|,
|\bar p_{\Omega,12}(t)|\right)<\frac c2.
\label{eq:S-lower-recovery-limit}
\end{equation}

Let
$\tilde p_{\Psi,12}^{(d)}(t)$ and
$\tilde p_{\Omega,12}^{(d)}(t)$ denote the corresponding empirical piecewise-constant interpolations. Since
\begin{equation}
\tilde p_{\Psi,12}^{(d)}=\tilde\Psi_{12}^{(d)}-\tilde t_1^{(d)}\tilde t_2^{(d)},
\qquad
\tilde p_{\Omega,12}^{(d)}=\tilde\Omega_{12}^{(d)}-\tilde t_1^{(d)}\tilde t_2^{(d)},
\end{equation}
Theorem \ref{theo:S-sgd-dynamics} implies
\begin{equation}
\sup_{t\in[0,T]}\max\left(\left|\tilde p_{\Psi,12}^{(d)}(t)-\bar p_{\Psi,12}(t)
\right|,
\left|\tilde p_{\Omega,12}^{(d)}(t)-\bar p_{\Omega,12}(t)\right|\right)
\xrightarrow{\mathbb P}0
\label{eq:S-structural-uniform-convergence}
\end{equation}
for every fixed $T<\infty$.

Applying \eqref{eq:S-structural-uniform-convergence} on $[0,T_+]$
and using \eqref{eq:S-upper-recovery-limit}, we obtain
\begin{equation}
\mathbb P\left(\max\left(|\tilde p_{\Psi,12}^{(d)}(T_+)|,
|\tilde p_{\Omega,12}^{(d)}(T_+)|\right)
\ge c\right)\longrightarrow1.
\end{equation}
Hence, with probability tending to one,
\begin{equation}
N_{\mathrm{wr},S}^{(d)}(c)\le\left\lceil
\frac{T_+}{\tau_d}\right\rceil=\left\lceil
\frac{T_+d}{\alpha_d}\right\rceil .
\label{eq:S-upper-hitting-time}
\end{equation}
Likewise, \eqref{eq:S-lower-recovery-limit} and
\eqref{eq:S-structural-uniform-convergence} imply
\begin{equation}
\mathbb P\left(\sup_{0\le t\le T_-}\max\left(|\tilde p_{\Psi,12}^{(d)}(t)|,|\tilde p_{\Omega,12}^{(d)}(t)|
\right)<c\right)
\longrightarrow1.
\end{equation}
Therefore, with probability tending to one,
\begin{equation}
N_{\mathrm{wr},S}^{(d)}(c)>\left\lfloor
\frac{T_-d}{\alpha_d}\right\rfloor .
\label{eq:S-lower-hitting-time}
\end{equation}
Combining \eqref{eq:S-upper-hitting-time} and
\eqref{eq:S-lower-hitting-time} proves
\eqref{eq:S-successful-recovery}. In particular, $N_{\mathrm{wr},S}^{(d)}(c)=\Theta_{\mathbb P}\left(\frac d{\alpha_d}\right)$. Choosing $\alpha_d=\Theta\left(\frac1{d\log d}\right)$ gives
\begin{equation}
N_{\mathrm{wr},S}^{(d)}(c)=\Theta_{\mathbb P}(d^2\log d).
\end{equation}

\paragraph{Case 2.}
Suppose now that \eqref{eq:S-higher-order-condition} holds. Consider the uninformative manifold
\begin{equation}
\mathcal U=\left\{(t,\Psi,\Omega):t_1=0,\quad\Psi_{12}=0,\quad\Omega_{12}=0
\right\}.
\end{equation}
On $\mathcal U$, we have
\begin{equation}
p_{\Psi,12}=p_{\Omega,12}=p_+=p_-=0.
\end{equation}
Moreover, by \eqref{eq:S-higher-order-condition}, $A_{12}=B_{12}=0$. Therefore
\eqref{eq:S-weak-t1} gives $\dot{\bar t}_1=0$ while
\eqref{eq:S-p-plus-dynamics}--\eqref{eq:S-p-minus-dynamics} give $\dot{\bar p}_+=\dot{\bar p}_-=0$.
Thus the limiting vector field is tangent to $\mathcal U$.

By uniqueness of the solution established in
Theorem \ref{theo:S-sgd-dynamics}, and since Assumption
\ref{assum:S-initialization} places the initial state on
$\mathcal U$, the limiting trajectory remains on this manifold:
\begin{equation}
\bar t_1(t)=0,\qquad
\bar\Psi_{12}(t)=\bar\Omega_{12}(t)=0,
\qquad t\ge0.
\end{equation}
Equivalently,
\begin{equation}
\bar p_{\Psi,12}(t)=\bar p_{\Omega,12}(t)=0,
\qquad t\ge0.
\label{eq:S-limiting-no-recovery}
\end{equation}
Combining \eqref{eq:S-limiting-no-recovery} with
\eqref{eq:S-structural-uniform-convergence}, for every fixed
$T<\infty$ and every $\epsilon>0$,
\begin{equation}
\mathbb P\left(
\sup_{t\in[0,T]}
\max\left(|\tilde p_{\Psi,12}^{(d)}(t)|,
|\tilde p_{\Omega,12}^{(d)}(t)|\right)>\epsilon\right)\longrightarrow0,
\end{equation}
which proves \eqref{eq:S-failure-recovery}.

Finally, let $C>0$ be arbitrary and consider $M_d:=\left\lfloor Cd^2\log d\right\rfloor$ SGD steps, satisfying
\begin{equation}
M_d\tau_d=M_d\frac{\alpha_d}{d}
\le Cd\log d\,\alpha_d\le Cc_0
\end{equation}
under the learning-rate condition of
Theorem \ref{theo:S-sgd-dynamics}. Hence, for every fixed
recovery threshold $c>0$,
\begin{equation}
\mathbb P\left(N_{\mathrm{wr},S}^{(d)}(c)\le Cd^2\log d\right)\longrightarrow0.
\end{equation}
Since $C>0$ is arbitrary, $S$-SGD cannot achieve weak recovery within $\mathcal O(d^2\log d)$ samples.

This completes the proof.

\section{Theory in Section \ref{sec:untied}}
\label{app:untied}
\subsection{Proof of Theorem \ref{theo:untied-SGD}}
\subsubsection{A stronger version of Theorem \ref{theo:general}}
We first estimate the error scaling in Theorem \ref{theo:general}. It is closely related to  existing work on the multivariate Gaussian approximation \cite{nourdin2010multivariate,noreddine2011gaussian}, and we present its independent proof in Appendix \ref{app:proof-gaussian-approximation} for completeness.
\begin{lemma}
\label{lemma:quantitative-gaussian-approximation}
Fix constants \(C_*,R_*<\infty\). Consider deterministic matrices
\(M_k\in\mathbb R^{d\times d}\) satisfying $\max_k\|M_k\|_{\mathrm{op}}\le C_*$ and $\max_k\frac1{\sqrt d}|\Tr(M_k)|\le R_*$. Let $G_{ijk}^{(d)}=\frac1{\sqrt d}x_i^TM_kx_j$, where \(x\sim\mathcal N(0,\mathcal C\otimes I_d)\), and let \(G_{\mathrm G}^{(d)}\) be a Gaussian tensor having the same mean and covariance as \(G^{(d)}\).

Then for every \(f\in C^3(\mathbb R^{L\times L\times K})\) whose derivatives up to third order have at most polynomial growth, there exists a constant $C=C(C_*,R_*,\mathcal C,K,L,f)$ independent of \(d\) and of the matrices \(M_k\), such that
\begin{equation}
\left|\mathbb E f(G^{(d)})-\mathbb E f(G_{\mathrm G}^{(d)})\right|\le\frac{C}{\sqrt d}.
\end{equation}
\end{lemma}

\subsubsection{Local well-posedness}
\label{app:untied-local-well-posedness}
We first introduce an extension of the vector field in \eqref{eq:slow_tau_mu} outside the set of admissible covariance parameters. Define
\begin{equation}
\mathfrak Q:=\left\{(q^{(1)},q^{(2)})\in\mathbb S^K\times\mathbb S^K:
q^{(1)}+q^{(2)}\succeq0,\quad q^{(1)}-q^{(2)}\succeq0\right\}.
\label{eq:untied-covariance-cone}
\end{equation}
Let $\mathscr P:\mathbb R^{K\times K}\times\mathbb R^{K\times K}\longrightarrow\mathfrak Q$
denote the orthogonal projection onto the closed convex cone
$\mathfrak Q$. In particular,
\begin{equation}
\|\mathscr P(q)-\mathscr P(q')\|_F\le \|q-q'\|_F.
\label{eq:untied-projection-lipschitz}
\end{equation}
We next define the extended ODE for \eqref{eq:slow_tau_mu}. Let $q(\mu):=\bigl(q^{(1)}(\mu),q^{(2)}(\mu)\bigr)$ and $\widehat q(\mu):=\mathscr P(q(\mu))$. We set
\begin{equation}
\widetilde m_{\mathcal K}^\star(\mu):=m_{\mathcal K}^\star(\widehat q(\mu)),
\label{eq:untied-extended-mstar}
\end{equation}
with the non-trainable components fixed at $m_{\setminus\mathcal K}(0)$, and define
\begin{equation}
a_{kl}^{(r)}(\mu):=\left[\nabla_{q^{(r)}}\Phi\left(\widetilde m^\star(\mu),\widehat q(\mu)\right)\right]_{kl},\qquad r=1,2.
\label{eq:untied-extended-a}
\end{equation}
For $k\in\mathcal K$, let
\begin{equation}
\psi_k^{\mathrm{ext}}(\mu):=2\sum_{l=1}^K\sum_{r=1}^2a_{kl}^{(r)}(\mu)\left[\mu_{(M_l^{(r)},B_k)}+\mu_{(A_k,M_l^{(r)})}\right],
\label{eq:untied-extended-psi}
\end{equation}
and
\begin{equation}
\beta_k^{\mathrm{ext}}(\mu):=-\frac{\psi_k^{\mathrm{ext}}(\mu)}
{t_k^A(\mu)+t_k^B(\mu)}.
\label{eq:untied-extended-beta}
\end{equation}
The extended vector field $V_{\mathrm{ext}}$ is defined by
\begin{align}
[V_{\mathrm{ext}}(\mu)]_\alpha=-\sum_{\substack{i=1\\k_i\in\mathcal K}}^w
\beta_{k_i}^{\mathrm{ext}}(\mu)\mathcal T_i^{\mathrm{bias}}(\mu)-2\sum_{\substack{i=1\\k_i\in\mathcal K}}^w\sum_{l=1}^K\sum_{r=1}^2a_{k_i l}^{(r)}(\mu)\mathcal T_{i,l,r}^{\mathrm{diff}}(\mu)-2\gamma|\alpha|_{\mathcal K}\mu_\alpha .
\label{eq:untied-extended-vector-field}
\end{align}
We refer to
\begin{equation}
\frac{d\mu}{d\tau}=V_{\mathrm{ext}}(\mu)
\label{eq:untied-extended-slow-system}
\end{equation}
as an extension of the original ODE \eqref{eq:slow_tau_mu}.

Whenever $q(\mu)\in\mathfrak Q$, the projection acts trivially,
$\widehat q(\mu)=q(\mu)$. Hence, on the admissible covariance set,
$\widetilde m_{\mathcal K}^\star(\mu)
=m_{\mathcal K}^\star(q(\mu))$, and
\eqref{eq:untied-extended-slow-system} agrees with the original slow
system \eqref{eq:untied-slaving}--\eqref{eq:slow_tau_mu}.

For $s>0$, define the open subset
\begin{equation}
\mathcal U_s:=\left\{\mu\in X_s:\mathscr P(q(\mu))\in\mathcal Q,\quad\min_{k\in\mathcal K}\bigl(t_k^A(\mu)+t_k^B(\mu)\bigr)>0\right\}.
\label{eq:untied-local-domain}
\end{equation}
\begin{lemma}
\label{lemma:local-untied-well-posed}
Let $s_0>0$ and $\mu_0\in\mathcal U_{s_0}$. Under Assumptions \ref{assum:untied} and \ref{assum:untied_convex}, there exist $T_{\rm local}>0$ and a strictly increasing continuous function $s:[0,T_{\rm local})\to(s_0,\infty)$ with $s(0)=s_0$ such that \eqref{eq:untied-extended-slow-system} admits a unique solution $\mu(\tau)\in X_{s(\tau)}$ for $0\le\tau<T_{\rm local}$ satisfying $\mu(\tau)\in\mathcal U_{s(\tau)}$.
\end{lemma}
\begin{proof}
Fix $s_0>0$ and $\mu_0\in\mathcal U_{s_0}$, and write
\begin{equation}
q_0:=\mathscr P(q(\mu_0)),\qquad
d_0:=\min_{k\in\mathcal K}\left(t_k^A(\mu_0)+t_k^B(\mu_0)\right)>0.
\label{eq:untied-local-q0-d0}
\end{equation}
Since $q_0\in\mathcal Q$, we may choose a neighborhood
$\mathcal V_q$ of $q_0$ in $\mathfrak Q$ whose closure is contained
in $\mathcal Q$.

We first prove the local regularity of the slaved mean $m_{\mathcal K}^\star(q)$. By Assumption \ref{assum:untied_convex}, for every $q\in\mathcal V_q$ the map $m_{\mathcal K}\mapsto\Phi(m,q)$ is uniformly strongly convex on $\mathcal M$ and admits a critical point $m_{\mathcal K}^\star(q)\in\mathcal M$, which is therefore unique. Moreover, on compact subsets of $\mathcal M\times\mathcal V_q$ there exists $L_q<\infty$ such that
\begin{equation}
\left\|\nabla_{m_{\mathcal K}}\Phi(m,q)-\nabla_{m_{\mathcal K}}\Phi(m,q')\right\|\le L_q\|q-q'\|_F.
\label{eq:untied-m-gradient-q-lipschitz}
\end{equation}
Uniform strong convexity gives
\begin{equation}
\left\langle\nabla_{m_{\mathcal K}}\Phi(m,q)-\nabla_{m_{\mathcal K}}\Phi(m',q),m-m'\right\rangle\ge\lambda_0\|m-m'\|^2.
\label{eq:untied-strong-monotonicity}
\end{equation}
Applying \eqref{eq:untied-m-gradient-q-lipschitz} and \eqref{eq:untied-strong-monotonicity} to
$m=m^\star(q)$ and $m'=m^\star(q')$ yields
\begin{equation}
\|m_{\mathcal K}^\star(q)-m_{\mathcal K}^\star(q')\|\le\frac{L_q}{\lambda_0}\|q-q'\|_F.
\label{eq:untied-mstar-local-lipschitz}
\end{equation}
Hence $m_{\mathcal K}^\star$ is locally Lipschitz on $\mathcal V_q$.

Since $\mathscr P$ is non-expansive and the finite-order coordinates
of $\mu$ are continuous on $X_s$, there exists a neighborhood $\mathcal V_\mu$ of $\mu_0$ in $X_{s_0}$ such that for every $\mu\in\mathcal V_\mu$,
\begin{equation}
\mathscr P(q(\mu))\in\mathcal V_q,\qquad t_k^A(\mu)+t_k^B(\mu)\ge\frac{d_0}{2},\quad k\in\mathcal K.
\label{eq:untied-local-neighborhood-bounds}
\end{equation}
It follows from \eqref{eq:untied-mstar-local-lipschitz}, Assumption \ref{assum:untied}.1, and the definition that $a_{kl}^{(r)}(\mu)$ is locally bounded and locally Lipschitz on $\mathcal V_\mu$. The same is true for $\psi_k^{\rm ext}(\mu)$, and by the lower bound
\eqref{eq:untied-local-neighborhood-bounds}, for $\beta_k^{\rm ext}(\mu)$.

Then following Lemma \ref{lemma:local_existence}, we can verify that $V_{\rm ext}$ is locally bounded and locally Lipschitz from $X_s$ to $X_{s'}$. The Ovsyannikov theorem therefore yields $T_{\rm local}>0$ and a strictly increasing continuous scale $s(\tau)$ for which \eqref{eq:untied-extended-slow-system} has a unique solution $\mu(\tau)\in X_{s(\tau)}$ on $[0,T_{\rm local})$.
\end{proof}

\subsubsection{Uniform spectral bounds}
Now we prove that the SGD trajectory does not blow up, i.e. the spectral norm of the weights remains uniformly bounded across the training steps.
\paragraph{Step 1: A lower bound for the traces.}
\begin{lemma}
\label{lemma:empirical-balancing}
Fix \(T>0\), and let $N_T:=\left\lfloor\frac{Td}{\alpha_d}\right\rfloor$.
Let \(J\) be any stopping time such that, on \(\{n<J\}\),
\begin{equation}
\max_{1\le k\le K}\left\{\|U_k^{(n)}\|_{\mathrm{op}},\|V_k^{(n)}\|_{\mathrm{op}}\right\}\le C_*,
\qquad\|m^{(n)}\|_\infty\le R_*
\label{eq:untied-stopped-bounds}
\end{equation}
for some constants \(C_*,R_*<\infty\) independent of \(d,T,\gamma\). Under the conditions of Theorem \ref{theo:untied-SGD}, there exists \(c_D>0\), independent of \(d\), such that
\begin{equation}
\lim_{d\to\infty}\mathbb P\left(\inf_{\substack{0\le n<N_T\wedge J\\k\in\mathcal K}}\left(t_k^{A,(n)}+t_k^{B,(n)}\right)
\ge c_D\right)=1.
\label{eq:empirical-D-lower-bound}
\end{equation}
\end{lemma}

\begin{proof}
Fix \(k\in\mathcal K\), and suppress the index \(k\) when there is no
ambiguity. Let
\begin{equation}
g_U^{(n)}:=\nabla_U\mathcal L(G^{(n)}),
\qquad
g_V^{(n)}:=\nabla_V\mathcal L(G^{(n)})
\end{equation}
denote the stochastic gradients. The SGD updates can be written as
\begin{equation}
U^{(n+1)}=a_dU^{(n)}-\alpha_d g_U^{(n)},
\qquad
V^{(n+1)}=a_dV^{(n)}-\alpha_d g_V^{(n)},
\label{eq:untied-update-balancing}
\end{equation}
where $a_d:=1-\frac{\alpha_d\gamma}{d}$. The key observation is that
\begin{align}
\left\langle g_U^{(n)},U^{(n)}\right\rangle_F=
\sum_{i,j=1}^L\partial_{ijk}\mathcal L(G^{(n)})G_{ijk}^{(n)}
=\left\langle g_V^{(n)},V^{(n)}\right\rangle_F.
\label{eq:exact-balancing-gradient}
\end{align}
Therefore, defining $\delta_k^{(n)}:=t_k^{A,(n)}-t_k^{B,(n)}$,
\eqref{eq:untied-update-balancing} gives the recursion
\begin{equation}
\delta^{(n+1)}=a_d^2\delta^{(n)}+r^{(n)},
\qquad
r^{(n)}:=\frac{\alpha_d^2}{d}
\left(\|g_U^{(n)}\|_F^2-\|g_V^{(n)}\|_F^2\right).
\label{eq:delta-exact-recursion}
\end{equation}
On \(\{n<J\}\), the bounds in \eqref{eq:untied-stopped-bounds}
imply uniform bounds on all moments of \(G^{(n)}\). By Assumption \ref{assum:untied}.1, the same holds for the polynomially growing first derivatives of \(\mathcal L\). Since $g_U^{(n)}=\frac1{\sqrt d}\sum_{i,j=1}^L\partial_{ijk}\mathcal L(G^{(n)})x_i^{(n)}(x_j^{(n)})^T V^{(n)}$, and similarly for \(g_V^{(n)}\), the estimates in Lemma \ref{lemma:bound-gradient} give
\begin{equation}
\mathbf 1_{\{n<J\}}\mathbb E\left[\|g_U^{(n)}\|_F^2+\|g_V^{(n)}\|_F^2
\,\middle|\,\mathcal F_n\right]\le C d,
\label{eq:gradient-Frobenius-second-moment-untied}
\end{equation}
where \(C\) depends on \(C_*,R_*,\mathcal C,\mathcal L,K,L\), but not
on \(d\) or \(n\).

Unrolling \eqref{eq:delta-exact-recursion}, for every
\(m<N_T\wedge J\),
\begin{equation}
\delta^{(m)}-a_d^{2m}\delta^{(0)}=\sum_{n=0}^{m-1}
a_d^{2(m-1-n)}r^{(n)}.
\end{equation}
Hence
\begin{align}
\sup_{m<N_T\wedge J}
\left|\delta^{(m)}-a_d^{2m}\delta^{(0)}\right|
&\le\sum_{n=0}^{N_T-1}\mathbf 1_{\{n<J\}}|r^{(n)}|.
\end{align}
Using \eqref{eq:gradient-Frobenius-second-moment-untied},
\begin{align}
\mathbb E\left[\sum_{n=0}^{N_T-1}\mathbf 1_{\{n<J\}}|r^{(n)}|
\right]\le\frac{\alpha_d^2}{d}
\sum_{n=0}^{N_T-1}Cd\le
C_T d\alpha_d\longrightarrow0,
\label{eq:balancing-residual-L1}
\end{align}
because \(\alpha_d\le c_0(d\log d)^{-1}\).
Markov's inequality proves 
\begin{equation}
\max_{k\in\mathcal K}\sup_{0\le n<N_T\wedge J}\left|\delta_k^{(n)}-a_d^{2n}\delta_k^{(0)}\right|\xrightarrow{\mathbb P}0.
\label{eq:empirical-balancing}
\end{equation}
Since
\(K=\Theta(1)\), the convergence holds simultaneously for all
\(k\in\mathcal K\).

Finally, Assumptions \ref{assum:untied}.3--\ref{assum:untied}.4 imply
\begin{equation}
\mathbb P\left(\min_{k\in\mathcal K}
|\delta_k^{(0)}|\ge\frac{c_{\mathrm{bal}}}{2}\right)\longrightarrow1.
\label{eq:initial-empirical-imbalance}
\end{equation}
Moreover, uniformly for \(n\le N_T\),
\begin{equation}
a_d^{2n}=\left(1-\frac{\alpha_d\gamma}{d}\right)^{2n}
\ge e^{-3\gamma T}
\end{equation}
for all sufficiently large \(d\). Combining this with
\eqref{eq:empirical-balancing}, with probability tending to one,
\begin{equation}
\inf_{\substack{0\le n<N_T\wedge J\\k\in\mathcal K}}
|\delta_k^{(n)}|\ge\frac{c_{\mathrm{bal}}}{4}e^{-3\gamma T}.
\end{equation}
Since
\begin{equation}
t_k^{A,(n)}+t_k^{B,(n)}\ge|t_k^{A,(n)}-t_k^{B,(n)}|=|\delta_k^{(n)}|,
\end{equation}
we obtain \eqref{eq:empirical-D-lower-bound}, for example with $c_D:=\frac{c_{\mathrm{bal}}}{4}e^{-3\gamma T}$.
\end{proof}

\paragraph{Step 2: Finite-dimensional Stein expansion.}
\begin{lemma}
\label{lemma:finite-dimensional-Stein-untied}
Fix a compact set \(\mathfrak D\Subset\mathcal M\times\mathcal Q\) and a constant
\(C_*<\infty\). Let \(J\) be a stopping time such that, on
\(\{n<J\}\), 
\begin{equation}
\left(m_{\mathcal K}^{(n)},q^{(1,n)},q^{(2,n)}\right)\in\mathfrak D,\quad
\max_{1\le l\le K}\left\{\|U_l^{(n)}\|_{\rm op},\|V_l^{(n)}\|_{\rm op}\right\}\le C_*.
\label{eq:stopping-time}
\end{equation}
Write \(\mathbb E_n[\cdot]=\mathbb E[\cdot\mid\mathcal F_n]\), and define $b_k^{(n)}:=\sum_{i,j=1}^L\mathcal C_{ij}\,\mathbb E_n\left[\partial_{ijk}\mathcal L(G^{(n)})\right]$.
For \(r=1,2\), set $\mathcal C^{(1)}_{ij,i'j'}:=\mathcal C_{ii'}\mathcal C_{jj'}$ and $\mathcal C^{(2)}_{ij,i'j'}:=\mathcal C_{ij'}\mathcal C_{ji'}$ and define
$a_{kl}^{(r,d,n)}:=\frac12
\sum_{i,j,i',j'=1}^L
\mathcal C^{(r)}_{ij,i'j'}\,\mathbb E_n\left[\partial^2_{(ijk),(i'j'l)}\mathcal L(G^{(n)})\right]$.

Then, uniformly for \(n<J\) and \(k\in\mathcal K\),
\begin{align}
\bar g_{U,k}^{(n)}&=\frac{b_k^{(n)}}{\sqrt d}V_k^{(n)}+\frac{2}{d}\sum_{l=1}^K\sum_{r=1}^2a_{kl}^{(r,d,n)}M_l^{(r,n)}V_k^{(n)}+\mathcal E_{U,k}^{(n)},
\label{eq:finite-Stein-U}
\\
\bar g_{V,k}^{(n)}&=\frac{b_k^{(n)}}{\sqrt d}U_k^{(n)}+\frac{2}{d}\sum_{l=1}^K\sum_{r=1}^2a_{kl}^{(r,d,n)}M_l^{(\bar r,n)}U_k^{(n)}+\mathcal E_{V,k}^{(n)}.
\label{eq:finite-Stein-V}
\end{align}
Moreover, there exists $C=C(C_*,\mathfrak D,\mathcal C,\mathcal L,K,L)<\infty$
such that $\max_{k,l,r}|a_{kl}^{(r,d,n)}|+\max_k|b_k^{(n)}|\le C$,\\ $\max_{k\in\mathcal K}\left\{\|\mathcal E_{U,k}^{(n)}\|_{\rm op},\|\mathcal E_{V,k}^{(n)}\|_{\rm op}\right\}\le Cd^{-3/2}$ and $\max_{1\le k\le K}
\left|b_k^{(n)}-\frac{\partial\Phi}{\partial m_k}\left(m^{(n)},q^{(1,n)},q^{(2,n)}\right)\right|\le \frac{C}{\sqrt d}$.
\end{lemma}
\begin{proof}
Throughout the proof, we condition on \(\mathcal F_n\) and suppress the
superscript \((n)\). Conditional on \(\mathcal F_n\), all weights are
deterministic and $x=(x_1,\ldots,x_L)\sim\mathcal N(0,\mathcal C\otimes I_d)$.
Write $f_{ijk}(x):=\partial_{ijk}\mathcal L(G(x))$.

The stochastic gradients are
\begin{align}
g_{U,k}&=\frac1{\sqrt d}\sum_{i,j=1}^Lf_{ijk}(x)x_ix_j^TV_k,
\label{eq:finite-stochastic-gradient-U}
\\
g_{V,k}&=\frac1{\sqrt d}\sum_{i,j=1}^Lf_{ijk}(x)x_jx_i^TU_k.
\label{eq:finite-stochastic-gradient-V}
\end{align}
Similarly to the expansion used in Lemma \ref{lemma:bound-gradient}, we can obtain
\begin{align}
\bar g_{U,k}=\frac1{\sqrt d}\sum_{i,j=1}^L\mathcal C_{ij}\mathbb E[f_{ijk}]V_k\nonumber+\frac1d\sum_{\substack{i,j,i',j'=1\\l=1}}^{L,L,L,L,K}\mathbb E\left[\partial^2_{(ijk),(i'j'l)}\mathcal L(G)\right]\left(\mathcal C_{ii'}\mathcal C_{jj'}M_l+\mathcal C_{ij'}\mathcal C_{ji'}M_l^T\right)V_k+\mathcal E_{U,k}.
\label{eq:finite-Stein-U-expanded}
\end{align}
and similarly for \(g_{V,k}\). This proves \eqref{eq:finite-Stein-U} and \eqref{eq:finite-Stein-V}.

The remainders are controlled following Lemma \ref{lemma:bound-gradient}, which gives $\|\mathcal E_{U,k}\|_{\mathrm{op}}+\|\mathcal E_{V,k}\|_{\mathrm{op}}\le Cd^{-3/2}$.

We now identify the coefficients. Let \(G_{\mathrm G}\) be the Gaussian
tensor having the same mean and covariance as \(G\). Then we have
\begin{equation}
\frac{\partial\Phi}{\partial m_k}=\sum_{i,j=1}^L\mathcal C_{ij}\mathbb E
\left[\partial_{ijk}\mathcal L(G_{\mathrm G})\right].
\label{eq:finite-Price-m}
\end{equation}
and
\begin{align}
\left[\nabla_{q^{(1)}}\Phi\right]_{kl}&=\frac12\sum_{i,j,i',j'=1}^L\mathcal C_{ii'}\mathcal C_{jj'}\mathbb E\left[\partial^2_{(ijk),(i'j'l)}\mathcal L(G_{\mathrm G})\right],
\label{eq:finite-Price-q1}
\\
\left[\nabla_{q^{(2)}}\Phi\right]_{kl}&=\frac12\sum_{i,j,i',j'=1}^L\mathcal C_{ij'}\mathcal C_{ji'}
\mathbb E\left[\partial^2_{(ijk),(i'j'l)}\mathcal L(G_{\mathrm G})\right].
\label{eq:finite-Price-q2}
\end{align}
by Price's theorem. Applying Lemma
\ref{lemma:quantitative-gaussian-approximation} finishes the proof.
\end{proof}

\paragraph{Step 3: One-step increment.}
For any process \(Y^{(n)}\), we define $\Delta Y^{(n)}:=Y^{(n+1)}-Y^{(n)}$.

\begin{lemma}
\label{lemma:untied-one-step-estimates}
Under the conditions of Theorem \ref{theo:untied-SGD}, let \(J\) be a
stopping time satisfying \eqref{eq:stopping-time}. Set $g^{(n)}:=\nabla_{m_{\mathcal K}}\Phi\left(m^{(n)},q^{(1,n)},q^{(2,n)}\right)$ and $D^{(n)}:=\operatorname{diag}\left(t_k^{A,(n)}+t_k^{B,(n)}
\right)_{k\in\mathcal K}$, and $\vartheta^{(n)}:=\left(
t^{A,(n)},t^{B,(n)},q^{(1,n)},q^{(2,n)}\right)$.
Then, uniformly on \(\{n<J\}\), 
\begin{equation}
\mathbb E_n\left[\Delta m_{\mathcal K}^{(n)}\right]=
-\alpha_dD^{(n)}\left(g^{(n)}+\frac{e^{(n)}}{\sqrt d}\right)+r_m^{(n)},
\label{eq:untied-one-step-m-drift}
\end{equation}
where $\|e^{(n)}\|\le C$ and $\|r_m^{(n)}\|\le C\left(\frac{(1+\gamma)\alpha_d}{d}+\alpha_d^2\right)$. Moreover,
\begin{equation}
\mathbb E_n\left[\left\|\Delta m_{\mathcal K}^{(n)}-\mathbb E_n[\Delta m_{\mathcal K}^{(n)}]\right\|^2\right]\le C\alpha_d^2,
\label{eq:untied-one-step-m-noise}
\end{equation}
and
\begin{equation}
\left\|\mathbb E_n[\Delta\vartheta^{(n)}]\right\|\le
C\left(\frac{\alpha_d}{\sqrt d}\|g^{(n)}\|+\frac{(1+\gamma)\alpha_d}{d}\right),
\label{eq:untied-one-step-slow-drift}
\end{equation}
\begin{equation}
\mathbb E_n\left[\|\Delta\vartheta^{(n)}\|^2\right]\le
C(1+\gamma)^2\frac{\alpha_d^2}{d}.
\label{eq:untied-one-step-slow-second-moment}
\end{equation}
Here $C=C(C_*,\mathfrak D,\mathcal C,\mathcal L,K,L)<\infty$ is independent of \(d,n,T\), and \(\gamma\).
\end{lemma}
\begin{proof}
Throughout the proof, we condition on \(\mathcal F_n\) and suppress the
superscript \((n)\). Thus, for any state variable \(Y^{(n)}\), we write
$Y:=Y^{(n)}$, $Y^+:=Y^{(n+1)}$ and $\Delta Y:=Y^+-Y$. We write $a_d:=1-\frac{\alpha_d\gamma}{d}$. All constants below are uniform over the stopped class \eqref{eq:stopping-time}. 

The spectral bounds \eqref{eq:stopping-time} imply $\max_{1\le l\le K}\left\{\|M_l\|_{\rm op}, \|A_l\|_{\rm op},\|B_l\|_{\rm op}\right\}\le C_*^2$. Together with the compactness assumption on \((m_{\mathcal K},q)\), this gives uniform bounds on all fixed moments of \(G\). Assumption \ref{assum:untied}.1 therefore gives uniform bounds on all fixed moments of the derivatives of \(\mathcal L(G)\).

We shall repeatedly use the following elementary conditional Gaussian
estimates. If \(P\) is \(\mathcal F_n\)-measurable and
\(\|P\|_{\rm op}\le C\), then, for every fixed polynomially growing
function \(F(G)\) arising below,
\begin{equation}
\mathbb E_n\left[\left|\frac1d F(G)x_i^TPx_j\right|^2\right]\le C,\qquad
\mathbb E_n\left[\left|\frac1{d^{3/2}}F(G)x_i^TPx_j\right|^2\right]\le\frac{C}{d},
\label{eq:untied-basic-trace-moments}
\end{equation}
and
\begin{equation}
\mathbb E_n\left|x_i^TM_l^{(r)}x_j\right|^p\le C_p d^{p/2},
\qquad r=1,2,
\label{eq:untied-M-quadratic-moment}
\end{equation}
which follows from \(\|M_l\|_{\rm op}=\mathcal O(1)\),
\(\|M_l\|_F=\mathcal O(\sqrt d)\), and \(\Tr M_l=\mathcal O(\sqrt d)\).

For \(k\in\mathcal K\), by definition we have
\begin{equation}
M_k^{+}=a_d^2M_k-\alpha_da_d\left(g_{U,k}V_k^T+U_kg_{V,k}^T\right)+\alpha_d^2g_{U,k}g_{V,k}^T.
\label{eq:untied-M-exact-one-step}
\end{equation}
Consequently,
\begin{align}
\mathbb E_n[\Delta m_k]=(a_d^2-1)m_k-\frac{\alpha_da_d}{\sqrt d}
\left[\Tr(\bar g_{U,k}V_k^T)+\Tr(U_k\bar g_{V,k}^T)\right]+
\frac{\alpha_d^2}{\sqrt d}\mathbb E_n\left[\Tr(g_{U,k}g_{V,k}^T)\right],
\label{eq:untied-m-exact-one-step}
\end{align}
where $\bar g_{U,k}:=\mathbb E_n[g_{U,k}]$, $\bar g_{V,k}:=\mathbb E_n[g_{V,k}]$. Let \(b_k\) and \(a_{kl}^{(r,d)}\) be the coefficients from Lemma \ref{lemma:finite-dimensional-Stein-untied}, and define
\begin{equation}
\psi_k^{(d)}:=2\sum_{l=1}^K\sum_{r=1}^2a_{kl}^{(r,d)}
\left[\mu_{(M_l^{(r)},B_k)}+\mu_{(A_k,M_l^{(r)})}\right].
\label{eq:untied-finite-psi-one-step}
\end{equation}
Lemma \ref{lemma:finite-dimensional-Stein-untied} gives
\begin{align}
\frac1{\sqrt d}\left[\Tr(\bar g_{U,k}V_k^T)+\Tr(U_k\bar g_{V,k}^T)
\right]=D_kb_k+\frac{\psi_k^{(d)}}{\sqrt d}+\rho_k,
\label{eq:untied-mean-trace-expansion}
\end{align}
where $D_k:=t_k^A+t_k^B$ and $|\rho_k|\le\frac{C}{d}$. Set
\begin{equation}
\zeta_k:=\sqrt d\left[b_k-\frac{\partial\Phi}{\partial m_k}
\left(m,q^{(1)},q^{(2)}\right)\right].
\label{eq:untied-zeta-one-step}
\end{equation}
Lemma \ref{lemma:finite-dimensional-Stein-untied} gives $|\zeta_k|\le C$. Moreover, since \(A_k,B_k\succeq0\),
\begin{align}
\left|\mu_{(M_l^{(r)},B_k)}\right|\le
\|M_l\|_{\rm op}\,t_k^B,\quad\left|\mu_{(A_k,M_l^{(r)})}\right|\le\|M_l\|_{\rm op}\,t_k^A.
\end{align}
The coefficients \(a_{kl}^{(r,d)}\) are uniformly bounded (Lemma \ref{lemma:finite-dimensional-Stein-untied}), and hence $|\psi_k^{(d)}|\le CD_k$. Define
\begin{equation}
e_k:=\zeta_k+\frac{\psi_k^{(d)}}{D_k},
\label{eq:untied-e-one-step}
\end{equation}
where the ratio is set equal to zero when \(D_k=0\). Thus we have $|e_k|\le C$.
Finally, define $\chi_k:=\frac1{\sqrt d}\mathbb E_n\left[\Tr(g_{U,k}g_{V,k}^T)\right]$. Expanding the two gradients and using
\eqref{eq:untied-M-quadratic-moment} gives $|\chi_k|\le C$.

Substituting $b_k=\frac{\partial\Phi}{\partial m_k}
+\frac{\zeta_k}{\sqrt d}$ and \eqref{eq:untied-mean-trace-expansion} into
\eqref{eq:untied-m-exact-one-step}, we obtain
\begin{equation}
\mathbb E_n[\Delta m_k]=-\alpha_dD_k\left(\frac{\partial\Phi}{\partial m_k}
+\frac{e_k}{\sqrt d}\right)+r_{m,k},
\end{equation}
where
\begin{align}
r_{m,k}=(a_d^2-1)m_k+\alpha_d(1-a_d)D_k\left(\frac{\partial\Phi}{\partial m_k}
+\frac{e_k}{\sqrt d}\right)-\alpha_da_d\rho_k+\alpha_d^2\chi_k.
\end{align}
By $1-a_d=\frac{\alpha_d\gamma}{d}$, $|a_d^2-1|\le
\frac{2\alpha_d\gamma}{d}+\frac{\alpha_d^2\gamma^2}{d^2}$ and \eqref{eq:stopping-time},
\begin{equation}
|r_{m,k}|\le C\left(\frac{(1+\gamma)\alpha_d}{d}+\alpha_d^2\right)
\end{equation}
for all sufficiently large \(d\). Summing over the finitely many
trainable indices proves \eqref{eq:untied-one-step-m-drift}.

Now we define
\begin{equation}
L_k:=\frac1{\sqrt d}\left[\Tr(g_{U,k}V_k^T)+\Tr(U_kg_{V,k}^T)\right],\qquad
Q_k:=\frac1{\sqrt d}\Tr(g_{U,k}g_{V,k}^T).
\end{equation}
Thus
\begin{equation}
\Delta m_k=(a_d^2-1)m_k-\alpha_da_dL_k+\alpha_d^2Q_k.
\end{equation}
By \eqref{eq:untied-basic-trace-moments}, $\mathbb E_n|L_k|^2\le C$. The argument used for $\chi$ also gives $\mathbb E_n|Q_k|^2\le C$. Therefore, using \(a_d\le1\) for all sufficiently large \(d\),
\begin{equation}
\mathbb E_n\left[\left|\Delta m_k-\mathbb E_n[\Delta m_k]\right|^2\right]\le C\alpha_d^2+C\alpha_d^4\le C\alpha_d^2.
\end{equation}
Since \(|\mathcal K|=\mathcal O(1)\), this proves
\eqref{eq:untied-one-step-m-noise}.

Recall that for a trainable index \(k\), 
\begin{align}
M_k^+&=a_d^2M_k-\alpha_da_d\left(g_{U,k}V_k^T+U_kg_{V,k}^T\right)
+\alpha_d^2g_{U,k}g_{V,k}^T,
\label{eq:untied-M-update-slow-coordinates}
\\
A_k^+&=a_d^2A_k-\alpha_da_d\left(g_{U,k}U_k^T+U_kg_{U,k}^T\right)+\alpha_d^2g_{U,k}g_{U,k}^T,
\label{eq:untied-A-update-slow-coordinates}
\\
B_k^+&=a_d^2B_k-\alpha_da_d\left(g_{V,k}V_k^T+V_kg_{V,k}^T\right)+\alpha_d^2g_{V,k}g_{V,k}^T.
\label{eq:untied-B-update-slow-coordinates}
\end{align}
The preceding expansions imply, for every
\(Z_k\in\{M_k,M_k^T,A_k,B_k\}\),
\begin{equation}
\left\|\mathbb E_n[\Delta Z_k]\right\|_{\rm op}\le C\left(\frac{\alpha_d}{\sqrt d}\|g\|+\frac{(1+\gamma)\alpha_d}{d}+\alpha_d^2\right).
\label{eq:untied-letter-mean-increment}
\end{equation}
They also give
\begin{equation}
\mathbb E_n\left[\|\Delta Z_k\|_F^2\right]\le C(1+\gamma)^2\alpha_d^2d.
\label{eq:untied-letter-F-increment}
\end{equation}
Using \eqref{eq:untied-letter-mean-increment},
\eqref{eq:untied-letter-F-increment}, and $\frac1d\left|\Tr(PQ)\right|\le\frac1d\|P\|_F\|Q\|_F$, we obtain
\begin{align}
\left\|\mathbb E_n[\Delta\vartheta]\right\|\le C\left(\frac{\alpha_d}{\sqrt d}\|g\|+\frac{(1+\gamma)\alpha_d}{d}+\alpha_d^2\right).
\end{align}
The learning-rate condition in Theorem \ref{theo:untied-SGD} implies \(\alpha_d^2\le C\alpha_d/d\) for all sufficiently large \(d\), and hence proves \eqref{eq:untied-one-step-slow-drift}.

It remains to estimate the second moments. Expanding
\eqref{eq:untied-M-update-slow-coordinates}--\eqref{eq:untied-B-update-slow-coordinates} shows that
\begin{equation}
\Delta\vartheta
=
\alpha_d\Lambda+R_\vartheta,
\label{eq:untied-theta-linear-remainder}
\end{equation}
where each component of \(\Lambda\) is a finite sum of terms of the
form $\frac1{d^{3/2}}\partial_{ijk}\mathcal L(G)\,x_i^TPx_j$ for an \(\mathcal F_n\)-measurable matrix \(P\) satisfying \(\|P\|_{\rm op}\le C\). Thus
\eqref{eq:untied-basic-trace-moments} gives
\begin{equation}
\mathbb E_n\|\Lambda\|^2\le\frac{C}{d}.
\label{eq:untied-Lambda-bound}
\end{equation}
The remainder \(R_\vartheta\) consists of the weight-decay terms and terms containing at least two weight increments. Similarly we obtain
\begin{equation}
\mathbb E_n\left[\|R_\vartheta\|^2\right]\le
C(1+\gamma)^2\left(\frac{\alpha_d^2}{d^2}+\alpha_d^4\right).
\label{eq:untied-theta-remainder-second-moment}
\end{equation}
Since \(\alpha_d^2\le C/d\) for all sufficiently large \(d\),
\eqref{eq:untied-theta-linear-remainder}--\eqref{eq:untied-theta-remainder-second-moment}
give
\begin{equation}
\mathbb E_n\left[\|\Delta\vartheta\|^2\right]\le C(1+\gamma)^2\frac{\alpha_d^2}{d}.
\end{equation}
This proves
\eqref{eq:untied-one-step-slow-second-moment}.
\end{proof}

\paragraph{Step 4: A bound for the mean variable.}
\begin{lemma}
\label{lemma:coarse-localization-untied}
Under the conditions of Theorem \ref{theo:untied-SGD}, let \(J\) be a stopping time satisfying \eqref{eq:stopping-time}.
Write $q^{(n)}:=\left(q^{(1,n)},q^{(2,n)}\right)$ and $N_T:=
\left\lfloor\frac{Td}{\alpha_d}\right\rfloor$.
Then there exist constants $C_{\rm loc}=C_{\rm loc}(C_*,\mathfrak D,\mathcal C,\mathcal L,K,L)<\infty$ and $c_{\rm loc}=c_{\rm loc}(T,\gamma,c_{\rm bal},\lambda_0)>0$, independent of \(d\), such that for any $\epsilon>0$
\begin{equation}
\lim_{d\to\infty}\mathbb P\left(\sup_{0\le n<N_T\wedge J}\left[\left\|m_{\mathcal K}^{(n)}-m_{\mathcal K}^{\star}(q^{(n)})\right\|-C_{\rm loc}e^{-c_{\rm loc}\alpha_dn}\right]_+>\epsilon
\right)=0.
\label{eq:coarse-localization-conclusion}
\end{equation}
\end{lemma}

\begin{proof}
Define the excess population loss
\begin{equation}
\mathscr E(m_{\mathcal K},q):=\Phi(m_{\mathcal K},q)-\Phi(m_{\mathcal K}^{\star}(q),q).
\label{eq:coarse-energy-definition}
\end{equation}
Let \(\mathfrak D_q\Subset\mathcal Q\) denote the projection of
\(\mathfrak D\) onto the \(q\)-coordinates. By the implicit function theorem and the positive definiteness of
\(\nabla_{m_{\mathcal K}}^2\Phi\) (Assumption \ref{assum:untied_convex}), the map $q\longmapsto m_{\mathcal K}^{\star}(q)$ is continuously differentiable on a neighborhood of \(\mathfrak D_q\). The envelope identity gives
\begin{align}
\nabla_{m_{\mathcal K}}\mathscr E(m_{\mathcal K},q)&=\nabla_{m_{\mathcal K}}\Phi(m,q),
\label{eq:coarse-energy-m-gradient}
\\
\nabla_q\mathscr E(m_{\mathcal K},q)&=\nabla_q\Phi(m,q)-\nabla_q\Phi(m^\star(q),q).
\label{eq:coarse-energy-q-gradient}
\end{align}
Assumption \ref{assum:untied}.1 implies \(\mathscr E\) has bounded second derivatives on a compact neighborhood of \eqref{eq:stopping-time}.

Strong convexity and compactness give constants
\(\Lambda<\infty\) and \(C_{\mathscr E}<\infty\) such that, uniformly over \eqref{eq:stopping-time},
\begin{align}
\frac{\lambda_0}{2}\left\|m_{\mathcal K}-m_{\mathcal K}^{\star}(q)\right\|^2&\le\mathscr E(m_{\mathcal K},q)\le\frac{\Lambda}{2}\left\|m_{\mathcal K}-m_{\mathcal K}^{\star}(q)\right\|^2,
\label{eq:coarse-energy-distance-equivalence}
\\
\left\|\nabla_{m_{\mathcal K}}\Phi(m,q)\right\|^2&\ge2\lambda_0\mathscr E(m_{\mathcal K},q),
\label{eq:coarse-PL-bound}
\\
\left\|\nabla_q\mathscr E(m_{\mathcal K},q)\right\|&\le C_{\mathscr E}\left\|m_{\mathcal K}-m_{\mathcal K}^{\star}(q)\right\|\le\frac{C_{\mathscr E}}{\lambda_0}
\left\|\nabla_{m_{\mathcal K}}\Phi(m,q)\right\|.
\label{eq:coarse-energy-q-gradient-bound}
\end{align}
The last inequality follows from $\left\|
\nabla_{m_{\mathcal K}}\Phi(m,q)\right\|\ge\lambda_0\left\|m_{\mathcal K}-m_{\mathcal K}^{\star}(q)\right\|$.

For \(k\in\mathcal K\), let $D_k^{(n)}:=t_k^{A,(n)}+t_k^{B,(n)}$ and $D^{(n)}:=\operatorname{diag}\left(D_k^{(n)}\right)_{k\in\mathcal K}$. The spectral bound in
\eqref{eq:stopping-time} implies $0\le D_k^{(n)}\le 2C_*^2$. By Lemma \ref{lemma:empirical-balancing}, there exists  $c_D=c_D(T,\gamma,c_{\rm bal})>0$ such that the event
\begin{equation}
\mathcal B_d:=\left\{\inf_{\substack{0\le n<N_T\wedge J\\k\in\mathcal K}}\left(t_k^{A,(n)}+t_k^{B,(n)}\right)\ge c_D\right\}
\end{equation}
satisfies $\mathbb P(\mathcal B_d)\longrightarrow1$. We work on \(\mathcal B_d\).
Set $\mathscr E_n:=\mathscr E\left(m_{\mathcal K}^{(n)},q^{(n)}\right)$. For \(n<J\), Taylor's expansion and the boundedness of the second derivatives of \(\mathscr E\) (according to Assumption \ref{assum:untied}.1) give
\begin{align}
\mathbb E_n[\mathscr E_{n+1}-\mathscr E_n]=\left\langle
g^{(n)},\mathbb E_n[\Delta m_{\mathcal K}^{(n)}]\right\rangle+\left\langle\nabla_q\mathscr E\left(m_{\mathcal K}^{(n)},q^{(n)}\right),\mathbb E_n[\Delta q^{(n)}]\right\rangle+R_{\mathscr E}^{(n)},
\label{eq:coarse-energy-Taylor}
\end{align}
where
\begin{equation}
|R_{\mathscr E}^{(n)}|\le C\,
\mathbb E_n\left[\|\Delta m_{\mathcal K}^{(n)}\|^2+\|\Delta q^{(n)}\|^2\right]\le C(1+\gamma)^2\alpha_d^2.
\label{eq:coarse-energy-Taylor-remainder}
\end{equation}
Here and below \(\mathbb E_n[\cdot]=\mathbb E[\cdot\mid\mathcal F_n]\).
Using Lemma \ref{lemma:untied-one-step-estimates}, $D_k^{(n)}\le 2C_*^2$, and
\eqref{eq:coarse-energy-q-gradient-bound}, we obtain
\begin{align}
\mathbb E_n[\mathscr E_{n+1}-\mathscr E_n]
\le-\alpha_d(g^{(n)})^TD^{(n)}g^{(n)}+C\frac{\alpha_d}{\sqrt d}
\left(\|g^{(n)}\|+\|g^{(n)}\|^2\right)+C(1+\gamma)\frac{\alpha_d}{d}+C(1+\gamma)^2\alpha_d^2.
\label{eq:coarse-energy-pre-contraction}
\end{align}
All quantities appearing here are uniformly bounded over
\(\mathfrak D\). Since \(D^{(n)}\succeq c_DI\) for \(n<J\), the \(d^{-1/2}\)-terms to be absorbed into the leading term for all sufficiently large \(d\). Moreover, \(\alpha_d^2\le\alpha_d/d\) under the learning-rate assumption.
Consequently, there exist $c_1=c_1(T,\gamma,c_{\rm bal},\lambda_0)>0$ and \(C_1=C_1(T,\gamma,C_*,\mathfrak D)<\infty\) such that
\begin{align}
\mathbb E_n[\mathscr E_{n+1}-\mathscr E_n]\le-\frac{c_D\alpha_d}{2}\|g^{(n)}\|^2+C_1\frac{\alpha_d}{d}\le-c_1\alpha_d\mathscr E_n+C_1\frac{\alpha_d}{d},
\label{eq:coarse-energy-contraction}
\end{align}
for $n<J$, where the second inequality uses
\eqref{eq:coarse-PL-bound}.

Define the stopped martingale differences
\begin{equation}
\xi^{(n+1)}:=\mathbf 1_{\{n<J\}}\left[\mathscr E_{n+1}-\mathscr E_n-\mathbb E_n[\mathscr E_{n+1}-\mathscr E_n]\right].
\label{eq:coarse-energy-martingale-difference}
\end{equation}
Because \(\mathscr E\) has bounded gradient, Lemma \ref{lemma:untied-one-step-estimates} implies
\begin{equation}
\mathbb E_n\left[|\xi^{(n+1)}|^2\right]\le C\alpha_d^2.
\label{eq:coarse-energy-martingale-variance}
\end{equation}
Thus, with $\rho_d:=1-c_1\alpha_d$, \eqref{eq:coarse-energy-contraction} gives, for
\(n<J\),
\begin{equation}
\mathscr E_{n+1}\le\rho_d\mathscr E_n+C_1\frac{\alpha_d}{d}+\xi^{(n+1)}.
\label{eq:coarse-energy-recursion}
\end{equation}
To control this recursion uniformly over \(N_T=\Theta(d/\alpha_d)\) steps, divide the time interval into blocks of length $\ell_d:=\left\lceil\frac1{\alpha_d}\right\rceil$ and 
define $B_d:=\left\lceil\frac{N_T}{\ell_d}\right\rceil\le C_Td$.
For a block starting at \(s=b\ell_d\), set $S_{b,r}:=\sum_{j=s}^{s+r-1}\xi^{(j+1)}$ for $0\le r\le\ell_d$, with the convention that increments after \(J\) are zero. Doob's maximal inequality and \eqref{eq:coarse-energy-martingale-variance} give
\begin{align}
\mathbb P\left(\max_{0\le b<B_d}\max_{0\le r\le\ell_d}|S_{b,r}|>\varepsilon_d\right)
&\le\frac{CB_d\ell_d\alpha_d^2}{\varepsilon_d^2},
\label{eq:coarse-block-Doob}
\end{align}
where we choose $\varepsilon_d:=(d\alpha_d)^{1/4}$. Since \(B_d=\mathcal O(d)\), \(\ell_d=\mathcal O(\alpha_d^{-1})\), and \(d\alpha_d\to0\),
\begin{equation}
\frac{B_d\ell_d\alpha_d^2}{\varepsilon_d^2}\le C_T(d\alpha_d)^{1/2}
\longrightarrow0.
\label{eq:coarse-block-Doob-vanish}
\end{equation}
On the complementary high-probability event, Abel summation gives,
for any \(s<m\) in the same block,
\begin{equation}
\left|\sum_{j=s}^{m-1}\rho_d^{\,m-1-j}\xi^{(j+1)}
\right|\le2\max_{1\le r\le m-s}|S_{b,r}|\le2\varepsilon_d.
\label{eq:coarse-Abel-martingale}
\end{equation}
Moreover, $\rho_d^{\ell_d}\le e^{-c_1\alpha_d\ell_d}\le e^{-c_1}<1$. Iterating \eqref{eq:coarse-energy-recursion} block by block, using
\eqref{eq:coarse-Abel-martingale} and $\frac{\alpha_d}{d}
\sum_{r=0}^{\infty}\rho_d^r=\frac{1}{c_1d}$, yields
\begin{equation}
\mathscr E_n\le C_0\rho_d^n+C\left(\varepsilon_d+\frac1d\right),\qquad0\le n<N_T\wedge J,
\label{eq:coarse-energy-uniform-bound}
\end{equation}
where $C_0:=\sup_{(m_{\mathcal K},q)\in\mathfrak D}\mathscr E(m_{\mathcal K},q)
<\infty$. Using \eqref{eq:coarse-energy-distance-equivalence} and
\(\rho_d^n\le e^{-c_1\alpha_dn}\), we conclude that, on an event whose
probability tends to one,
\begin{align}
\left\|m_{\mathcal K}^{(n)}-m_{\mathcal K}^{\star}(q^{(n)})\right\|\le
C_{\rm loc}e^{-\frac{c_1}{2}\alpha_dn}+C\left[(d\alpha_d)^{1/8}+d^{-1/2}\right]
\label{eq:coarse-distance-quantitative}
\end{align}
uniformly for \(0\le n<N_T\wedge J\). The term in square brackets tends to zero, which proves \eqref{eq:coarse-localization-conclusion}, with \(c_{\rm loc}=c_1/2\). Removing \(\mathcal B_d\) is legitimate because
\(\mathbb P(\mathcal B_d^c)\to0\).
\end{proof}

\paragraph{Step 5: A refined bound for the mean gradients.}
From Lemma \ref{lemma:coarse-localization-untied}, we can obtain the following bound for the mean gradients
\begin{equation}
\lim_{d\to\infty}\mathbb P\left(\sup_{0\le n<N_T\wedge J}
\left[\|b_k^{(n)}\|-C_{\rm loc}e^{-c_{\rm loc}\alpha_dn}\right]_+>\epsilon\right)=0
\end{equation}
by the boundedness of the Hessian. However, we would require a refined bound given by the following lemma.
\begin{lemma}
\label{lemma:local-mean-gradient-refinement}
Under the conditions of Theorem \ref{theo:untied-SGD}, fix \(T>0\) and let \(J\) be a
stopping time satisfying \eqref{eq:stopping-time}. Set $N_T:=\left\lfloor\frac{Td}{\alpha_d}\right\rfloor$ and $b_{\mathcal K}^{(n)}:=\left(b_k^{(n)}\right)_{k\in\mathcal K}$,
where \(b_k^{(n)}\) is defined in Lemma \ref{lemma:finite-dimensional-Stein-untied}. Then there exist
constants $C_{\rm tr}=C_{\rm tr}(T,\gamma,C_*,\mathfrak D)
<\infty,c_{\rm tr}=c_{\rm tr}(T,\gamma,C_*,\mathfrak D)
>0,C_b=C_b(C_*,\mathfrak D,\mathcal C,\mathcal L,K,L)<\infty,
c_0^\star=c_0^\star(T,\gamma,C_*,\mathfrak D)
>0$ such that, whenever \(c_0\le c_0^\star\),
\begin{equation}
\lim_{d\to\infty}\mathbb P\left(\sup_{0\le n<N_T\wedge J}
\left[\|b_{\mathcal K}^{(n)}\|-C_{\rm tr}e^{-c_{\rm tr}\alpha_dn}\right]_+>\frac{C_b}{\sqrt d}
\right)=0.
\label{eq:local-b-refinement}
\end{equation}
\end{lemma}

\begin{proof}
Write $q^{(n)}:=\left(q^{(1,n)},q^{(2,n)}\right)$ and $D^{(n)}:=\operatorname{diag}\left(t_k^{A,(n)}+t_k^{B,(n)}\right)_{k\in\mathcal K}$. All constants below are uniform over \eqref{eq:stopping-time}. By Lemma \ref{lemma:empirical-balancing}, there exists 
$c_D=c_D(T,\gamma,c_{\rm bal})>0$ such that the event $\mathcal B_d:=\left\{\inf_{\substack{0\le n<N_T\wedge J\\k\in\mathcal K}}\left(t_k^{A,(n)}+t_k^{B,(n)}\right)\ge c_D\right\}$ satisfies $\mathbb P(\mathcal B_d)\longrightarrow1$. We work on \(\mathcal B_d\).

Let $H_\star(q):=\nabla_{m_{\mathcal K}}^2\Phi\left(m^\star(q),q\right)$ and $S(q):=H_\star(q)^{1/2}$. By Assumption \ref{assum:untied_convex}, the implicit function theorem, and Assumption \ref{assum:untied}.1, the maps
$q\longmapsto m_{\mathcal K}^\star(q),q\longmapsto H_\star(q),
q\longmapsto S(q)$ are $C^2$ on a neighborhood of the projection of \(\mathfrak D\) onto the \(q\)-coordinates. In particular, on that neighborhood,
\begin{equation}
\lambda_0I\preceq H_\star(q)\preceq\Lambda I,
\qquad\|S(q)\|_{\rm op}+\|S(q)^{-1}\|_{\rm op}\le C,
\label{eq:local-refinement-S-bounds}
\end{equation}
and \(S,S^{-1}\), and \(m^\star\) are Lipschitz.
Define
\begin{equation}
x^{(n)}:=m_{\mathcal K}^{(n)}-m_{\mathcal K}^\star(q^{(n)}),
\qquad
y^{(n)}:=S(q^{(n)})x^{(n)}.
\label{eq:local-refinement-xy}
\end{equation}
Taylor expansion gives
\begin{equation}
g^{(n)}:=\nabla_{m_{\mathcal K}}\Phi\left(m^{(n)},q^{(n)}\right)=H_\star(q^{(n)})x^{(n)}+\rho_g^{(n)},
\qquad
\|\rho_g^{(n)}\|\le C\|x^{(n)}\|^2.
\label{eq:local-refinement-gradient-linearization}
\end{equation}
Fix \(r_0>0\), to be chosen sufficiently small below. For convenience, write
\begin{equation}
S_n:=S(q^{(n)}),\qquad
H_n:=H_\star(q^{(n)})=S_n^2,\qquad
\mathcal D_n:=D^{(n)}.
\end{equation}
Define
\begin{equation}
A^{(n)}:=S_n\mathcal D_nS_n,\qquad
u^{(n)}:=S_n^{-1}e^{(n)},
\label{eq:local-refinement-A-u-definition}
\end{equation}
where \(e^{(n)}\) is the bounded vector appearing in Lemma \ref{lemma:untied-one-step-estimates}. On the event \(\mathcal B_d\), we have $c_DI\preceq \mathcal D_n\preceq 2C_*^2I$. Together with \eqref{eq:local-refinement-S-bounds}, this yields
\begin{equation}
a_-I\preceq A^{(n)}\preceq a_+I,
\qquad
a_-:=\lambda_0c_D>0,
\qquad
a_+:=2\Lambda C_*^2.
\label{eq:local-refinement-A-bounds}
\end{equation}
Moreover,
\begin{equation}
\sup_{n<J}\|u^{(n)}\|\le\sup_{q}\|S(q)^{-1}\|_{\rm op}
\sup_{n<J}\|e^{(n)}\|=:C_u<\infty.
\label{eq:local-refinement-u-bound}
\end{equation}
We next derive the recursion for \(y^{(n)}\). Define
\begin{equation}
F(m_{\mathcal K},q):=S(q)\left(m_{\mathcal K}-m_{\mathcal K}^\star(q)\right),
\end{equation}
so that
\begin{equation}
y^{(n)}=F(m_{\mathcal K}^{(n)},q^{(n)}).
\end{equation}
By Lemma \ref{lemma:untied-one-step-estimates},
\begin{equation}
\mathbb E_n[\Delta m_{\mathcal K}^{(n)}]=-\alpha_d\mathcal D_n
\left(g^{(n)}+\frac{e^{(n)}}{\sqrt d}\right)+r_m^{(n)},
\label{eq:local-refinement-m-drift}
\end{equation}
where $\|r_m^{(n)}\|\le C\left(\frac{(1+\gamma)\alpha_d}{d}+\alpha_d^2\right)$. Furthermore,
\begin{equation}
\left\|\mathbb E_n[\Delta q^{(n)}]\right\|\le C\left(
\frac{\alpha_d}{\sqrt d}\|g^{(n)}\|+\frac{(1+\gamma)\alpha_d}{d}\right),
\label{eq:local-refinement-q-drift}
\end{equation}
and
\begin{equation}
\mathbb E_n\|\Delta q^{(n)}\|^2\le C(1+\gamma)^2\frac{\alpha_d^2}{d}.
\label{eq:local-refinement-q-second-moment}
\end{equation}
Also,
\begin{equation}
\mathbb E_n\|\Delta m_{\mathcal K}^{(n)}\|^2\le C\alpha_d^2.
\label{eq:local-refinement-m-second-moment}
\end{equation}
Since \(F\) is \(C^2\) on a compact neighborhood before the stopping time, a second-order Taylor expansion gives
\begin{align}
\mathbb E_n[y^{(n+1)}-y^{(n)}]=S_n\mathbb E_n[\Delta m_{\mathcal K}^{(n)}]+\partial_qF(m_{\mathcal K}^{(n)},q^{(n)})
\left[\mathbb E_n[\Delta q^{(n)}]\right]+\rho_F^{(n)},
\label{eq:local-refinement-F-Taylor}
\end{align}
where, uniformly before the stopping time,
\begin{equation}
\|\rho_F^{(n)}\|\le C\mathbb E_n\left[\|\Delta m_{\mathcal K}^{(n)}\|^2+\|\Delta q^{(n)}\|^2\right]\le C(1+\gamma)^2\alpha_d^2.
\label{eq:local-refinement-F-Taylor-remainder}
\end{equation}
Using \eqref{eq:local-refinement-gradient-linearization},
\begin{equation}
g^{(n)}=H_nx^{(n)}+\rho_g^{(n)},
\qquad
\|\rho_g^{(n)}\|\le C\|x^{(n)}\|^2,
\end{equation}
and the identities $S_n\mathcal D_nH_nx^{(n)}=S_n\mathcal D_nS_ny^{(n)}=A^{(n)}y^{(n)}$, $S_n\mathcal D_ne^{(n)}=A^{(n)}u^{(n)}$, we obtain, whenever \(n<J\) and \(\|y^{(n)}\|\le r_0\),
\begin{equation}
\mathbb E_n[y^{(n+1)}-y^{(n)}]=-\alpha_dA^{(n)}y^{(n)}
-\frac{\alpha_d}{\sqrt d}A^{(n)}u^{(n)}+r_y^{(n)},
\label{eq:local-refinement-conditional-drift}
\end{equation}
where
\begin{equation}
\|r_y^{(n)}\|\le
C\left[\alpha_d\|y^{(n)}\|^2+\frac{\alpha_d}{\sqrt d}\|y^{(n)}\|+\frac{(1+\gamma)\alpha_d}{d}+(1+\gamma)^2\alpha_d^2\right].
\label{eq:local-refinement-ry-bound}
\end{equation}
Here we used \eqref{eq:local-refinement-S-bounds} and the fact that
\(\partial_qF\) is uniformly bounded before the stopping time.

Define
\begin{equation}
\xi^{(n+1)}:=y^{(n+1)}-\mathbb E_n[y^{(n+1)}].
\label{eq:local-refinement-xi-definition}
\end{equation}
Then $\mathbb E_n[\xi^{(n+1)}]=0$. Since \(F\) is Lipschitz before the stopping time, Lemma \ref{lemma:untied-one-step-estimates} and
\eqref{eq:local-refinement-m-second-moment}--
\eqref{eq:local-refinement-q-second-moment} imply
\begin{equation}
\mathbb E_n\|\xi^{(n+1)}\|^2\le C\alpha_d^2.
\label{eq:local-refinement-xi-variance}
\end{equation}
Consequently,
\begin{equation}
y^{(n+1)}=\left(I-\alpha_dA^{(n)}\right)y^{(n)}-\frac{\alpha_d}{\sqrt d}A^{(n)}u^{(n)}+\xi^{(n+1)}+r_y^{(n)}.
\label{eq:local-refinement-recursion}
\end{equation}
We first use Lemma \ref{lemma:coarse-localization-untied} to enter the region in which \eqref{eq:local-refinement-recursion} can be iterated. Let $C_S:=\sup_{q}\|S(q)\|_{\rm op}<\infty$, where the supremum is taken over a compact neighborhood of the \(q\)-projection of \(\mathfrak D\). Choose \(t_0=t_0(T,\gamma,r_0)\)
sufficiently large that $C_SC_{\rm loc}e^{-c_{\rm loc}t_0}<\frac{r_0}{8}$ and set $n_0:=\left\lceil\frac{t_0}{\alpha_d}\right\rceil$.
Applying Lemma \ref{lemma:coarse-localization-untied} with \(\epsilon=r_0/(8C_S)\), and using \(\alpha_dn_0\ge t_0\), gives
\begin{equation}
\mathbb P\left(\sup_{n_0\le n<N_T\wedge J}\|y^{(n)}\|>\frac{r_0}{4}
\right)\longrightarrow0.
\label{eq:local-refinement-entry}
\end{equation}
Here and below, the supremum over an empty index set is understood to
be zero.

Starting from the entry time \(n_0\), define the local exit time
\begin{equation}
J_{\rm loc}:=
\inf\left\{n\ge n_0:\|y^{(n)}\|>r_0\right\}.
\label{eq:local-refinement-local-exit}
\end{equation}
All local estimates below are carried out up to
\(N_T\wedge J\wedge J_{\rm loc}\). We divide the time interval after \(n_0\) into blocks. Set
\begin{equation}
\ell_d:=\left\lceil\frac{1}{a_-\alpha_d}\right\rceil,
\qquad
s_b:=n_0+b\ell_d,
\label{eq:local-refinement-block-definition}
\end{equation}
and, whenever \(s_b<N_T\wedge J\wedge J_{\rm loc}\), let
$I_b:=\left\{s_b,\ldots,\min\left(s_b+\ell_d-1,\,N_T\wedge J\wedge J_{\rm loc}-1\right)\right\}$ to be the $b-$th block.

We next analyze the dynamics on each block. For a block \(I_b\), set
\begin{equation}
\bar A_b:=A^{(s_b)},\qquad P_b:=I-\alpha_d\bar A_b.
\label{eq:local-refinement-Pb}
\end{equation}
By \eqref{eq:local-refinement-A-bounds}, for all sufficiently large \(d\),
\begin{equation}
0\preceq P_b\preceq (1-a_-\alpha_d)I,
\qquad\|P_b^j\|_{\rm op}\le 1,\quad j\ge0.
\label{eq:local-refinement-Pb-contraction}
\end{equation}
For \(n\in I_b\), rewrite \eqref{eq:local-refinement-recursion} as
\begin{equation}
y^{(n+1)}=P_b y^{(n)}-\frac{\alpha_d}{\sqrt d}\bar A_bu^{(n)}+\xi^{(n+1)}+\rho_b^{(n)},
\label{eq:local-refinement-frozen-recursion}
\end{equation}
where
\begin{equation}
\rho_b^{(n)}:=-\alpha_d\bigl(A^{(n)}-\bar A_b\bigr)y^{(n)}
-\frac{\alpha_d}{\sqrt d}\bigl(A^{(n)}-\bar A_b\bigr)u^{(n)}+r_y^{(n)}.
\label{eq:local-refinement-rho-b}
\end{equation}
For \(0\le r\le |I_b|\), define
\begin{equation}
R_{b,r}:=\sum_{j=0}^{r-1}P_b^{\,r-1-j}\rho_b^{(s_b+j)}
\label{eq:local-refinement-Rbr-definition}
\end{equation}
with $R_{b,0}:=0$. Iterating \eqref{eq:local-refinement-frozen-recursion} within the
block gives
\begin{align}
y^{(s_b+r)}=P_b^r y^{(s_b)}-\frac1{\sqrt d}\sum_{j=0}^{r-1}P_b^{r-1-j}\alpha_d\bar A_bu^{(s_b+j)}+\sum_{j=0}^{r-1}P_b^{r-1-j}\xi^{(s_b+j+1)}+R_{b,r}.
\label{eq:local-refinement-block-Duhamel}
\end{align}

We first control the second term of \eqref{eq:local-refinement-block-Duhamel}. Diagonalize
$\bar A_b=Q_b\operatorname{diag}(\lambda_1,\ldots,\lambda_p)Q_b^T$, where $p:=|\mathcal K|$. By \eqref{eq:local-refinement-A-bounds}, \(\lambda_i\in[a_-,a_+]\). For every coordinate \(i\), using \eqref{eq:local-refinement-u-bound},
\begin{align}
\left|\sum_{j=0}^{r-1}(1-\alpha_d\lambda_i)^{r-1-j}\alpha_d\lambda_i\left(Q_b^Tu^{(s_b+j)}\right)_i\right|\le C_u\sum_{j=0}^{r-1}(1-\alpha_d\lambda_i)^{r-1-j}\alpha_d\lambda_i\le C_u.
\end{align}
Consequently,
\begin{equation}
\max_{0\le r\le |I_b|}\left\|\sum_{j=0}^{r-1}P_b^{r-1-j}\alpha_d\bar A_bu^{(s_b+j)}\right\|\le\sqrt p\,C_u.
\label{eq:local-refinement-forcing-bound}
\end{equation}

We next show that \(A^{(n)}\) is approximately frozen on each block. Write
\begin{equation}
\Delta\vartheta^{(n)}=\mathbb E_n[\Delta\vartheta^{(n)}]+\zeta^{(n+1)},\qquad
\mathbb E_n[\zeta^{(n+1)}]=0.
\end{equation}
By Lemma \ref{lemma:untied-one-step-estimates}, uniformly before the stopping time,
\begin{equation}
\left\|\mathbb E_n[\Delta\vartheta^{(n)}]\right\|\le C\left(\frac{\alpha_d}{\sqrt d}+\frac{(1+\gamma)\alpha_d}{d}\right),
\label{eq:local-refinement-theta-block-drift}
\end{equation}
where we used the boundedness of \(g^{(n)}\) on \eqref{eq:stopping-time}, while
\begin{equation}
\mathbb E_n\|\zeta^{(n+1)}\|^2\le C(1+\gamma)^2\frac{\alpha_d^2}{d}.
\label{eq:local-refinement-theta-block-noise}
\end{equation}
Since $\ell_d=\left\lceil\frac{1}{a_-\alpha_d}\right\rceil=O\!\left((a_-\alpha_d)^{-1}\right)$, the accumulated conditional drift over a single block is \(o(1)\).
Moreover, Doob's maximal inequality applied on each block, followed by a union bound over the \(\mathcal O(d)\) blocks, yields
\begin{equation}
\max_b\sup_{n\in I_b}\left\|\vartheta^{(n)}-\vartheta^{(s_b)}\right\|\xrightarrow{\mathbb P}0.
\label{eq:local-refinement-theta-freezing}
\end{equation}
Since $A(\vartheta)=S(q)D(\vartheta)S(q)$ is Lipschitz on the compact stopped region
\eqref{eq:stopping-time}, it follows that
\begin{equation}
\delta_{A,d}:=\max_b\sup_{n\in I_b}\left\|A^{(n)}-\bar A_b\right\|\xrightarrow{\mathbb P}0.
\label{eq:local-refinement-A-freezing}
\end{equation}

We now bound the block remainder \(R_{b,r}\) (the last term of \eqref{eq:local-refinement-block-Duhamel}). Recall from \eqref{eq:local-refinement-ry-bound} that, on the region \(\|y^{(n)}\|\le r_0\),
\begin{equation}
\|r_y^{(n)}\|\le C\left[\alpha_d\|y^{(n)}\|^2+\frac{\alpha_d}{\sqrt d}\|y^{(n)}\|+\frac{(1+\gamma)\alpha_d}{d}+(1+\gamma)^2\alpha_d^2\right].
\label{eq:local-refinement-ry-bound-recalled}
\end{equation}
Using
\eqref{eq:local-refinement-Rbr-definition}, \eqref{eq:local-refinement-Pb-contraction},
\eqref{eq:local-refinement-rho-b}, \eqref{eq:local-refinement-u-bound}, and
\eqref{eq:local-refinement-ry-bound-recalled}, we obtain, for every
\(0\le r\le |I_b|\),
\begin{align}
\|R_{b,r}\|\le C r\alpha_d\left(\delta_{A,d}+r_0+\frac1{\sqrt d}\right)\max_{0\le j\le r}\|y^{(s_b+j)}\|+C r\alpha_d\left[\frac{\delta_{A,d}}{\sqrt d}+\frac{1+\gamma}{d}+(1+\gamma)^2\alpha_d\right].
\label{eq:local-refinement-Rbr-prebound}
\end{align}
Since \(r\le |I_b|\le\ell_d\) and \(\alpha_d\ell_d\le C/a_-\), this implies
\begin{align}
\|R_{b,r}\|\le\frac{C}{a_-}\left(r_0+\delta_{A,d}+\frac1{\sqrt d}\right)\max_{0\le j\le r}
\|y^{(s_b+j)}\|+\frac{C}{a_-}\left[\frac{\delta_{A,d}}{\sqrt d}+\frac{1+\gamma}{d}+(1+\gamma)^2\alpha_d\right].
\label{eq:local-refinement-Rbr-bound}
\end{align}
Choose \(r_0=r_0(T,\gamma)>0\) sufficiently small so that $\frac{C}{a_-}r_0\le\frac1{16}$.
By \eqref{eq:local-refinement-A-freezing}, \(\delta_{A,d}\xrightarrow{\mathbb P}0\), and therefore
\begin{equation}
\frac{C}{a_-}\left(\delta_{A,d}+\frac1{\sqrt d}\right)\le\frac1{16}
\label{eq:local-refinement-small-random-coefficient}
\end{equation}
with probability tending to one. Furthermore, for fixed \(T,\gamma\), the learning-rate condition \(\alpha_d\le c_0(d\log d)^{-1}\) gives
\begin{equation}
\frac{\delta_{A,d}}{\sqrt d}+\frac{1+\gamma}{d}+(1+\gamma)^2\alpha_d=\frac{o_{\mathbb P}(1)}{\sqrt d}.
\label{eq:local-refinement-independent-remainder-small}
\end{equation}
Consequently, there exists a nonnegative random sequence
\(\varepsilon_d\xrightarrow{\mathbb P}0\) such that, with probability
tending to one,
\begin{equation}
\|R_{b,r}\|\le\frac18\max_{0\le j\le r}\|y^{(s_b+j)}\|+\frac{\varepsilon_d}{\sqrt d},
\qquad 0\le r\le |I_b|,
\label{eq:local-refinement-block-remainder}
\end{equation}
uniformly over all blocks \(b\).

Finally, We control the martingale term (the third term of \eqref{eq:local-refinement-block-Duhamel}). Truncate the fresh sample at step \(n\) by $\mathcal A_n=\left\{\|x^{(n)}\|^2\le c_xd\right\}\cap\left\{\max_{i,j,k}|G_{ijk}^{(n)}|
\le c_G\log d\right\}$.
The probability of any truncation failure before \(N_T\wedge J\) tends to zero. On the truncated event,
\begin{equation}
\|\xi^{(n+1)}\|\le C\alpha_d(1+(\log d)^r)
\label{eq:local-refinement-xi-increment}
\end{equation}
for some fixed \(r<\infty\). For each block, Freedman's inequality applied to \eqref{eq:local-refinement-block-Duhamel} gives
\begin{align}
\mathbb P\left(\left|\sum_{j=0}^{r-1}
(1-\alpha_d\lambda_i)^{r-1-j}\left(Q_b^T\xi^{(s_b+j+1)}\right)_i\right|>\frac{\eta}{\sqrt d}\right)\le2\exp\left[-\frac{c\eta^2a_-}{d\alpha_d}\right]
\label{eq:local-refinement-Freedman}
\end{align}
for every fixed \(\eta>0\) and all sufficiently large \(d\). Here we used
\begin{equation}
\sum_{j\ge0}(1-\alpha_d\lambda_i)^{2j}
\mathbb E_{s_b+j}\left[\left|\left(Q_b^T\xi^{(s_b+j+1)}\right)_i\right|^2\right]\le C\frac{\alpha_d}{a_-}.
\end{equation}
We now make the estimate uniform over all blocks and all terminal
times within each block. Let $\eta_0:=\frac{1}{\sqrt p}$ and $p:=|\mathcal K|$.
The total number of pairs \((b,r)\), with \(b\) ranging over the
blocks and \(0\le r\le |I_b|\), is at most \(N_T+1\). Therefore,
a union bound gives
\begin{align}
\mathbb P\Bigg(\max_b\max_{0\le r\le |I_b|}
\max_{1\le i\le p}\left|\sum_{j=0}^{r-1}
(1-\alpha_d\lambda_i)^{r-1-j}\left(Q_b^T\xi^{(s_b+j+1)}\right)_i
\right|>\frac{\eta_0}{\sqrt d}\Bigg)\le2p(N_T+1)
\exp\left[-\frac{c\eta_0^2a_-}{d\alpha_d}\right].
\label{eq:local-refinement-Freedman-union}
\end{align}
By the lower learning-rate bound \(\alpha_d=\Omega(d^{-\iota})\), $N_T\le C_Td^{1+\iota}$, while $\frac{1}{d\alpha_d}\ge\frac{\log d}{c_0}$ by the upper learning-rate bound. Hence
\begin{align}
2p(N_T+1)\exp\left[-\frac{c\eta_0^2a_-}{d\alpha_d}\right]\le C_Td^{\,1+\iota-c\eta_0^2a_-/c_0}.
\end{align}
We may therefore choose \(c_0^\star=c_0^\star(T,\gamma,C_*,\mathfrak D)>0\) sufficiently small so that
\begin{equation}
\frac{c\eta_0^2a_-}{c_0^\star}>1+\iota.
\label{eq:local-refinement-c0-choice}
\end{equation}
Whenever \(c_0\le c_0^\star\), the right-hand side of
\eqref{eq:local-refinement-Freedman-union} converges to zero.
Let $v_{b,r}:=\sum_{j=0}^{r-1}P_b^{r-1-j}\xi^{(s_b+j+1)}$.
Since $P_b=Q_b\operatorname{diag}(1-\alpha_d\lambda_1,\ldots,
1-\alpha_d\lambda_p)Q_b^T$,
we have
\begin{equation}
(Q_b^Tv_{b,r})_i=\sum_{j=0}^{r-1}(1-\alpha_d\lambda_i)^{r-1-j}
(Q_b^T\xi^{(s_b+j+1)})_i.
\end{equation}
Hence, on the high-probability event obtained by \eqref{eq:local-refinement-Freedman-union},
\begin{equation}
\max_{b,r,i}|(Q_b^Tv_{b,r})_i|\le\frac{\eta_0}{\sqrt d}.
\end{equation}
Choosing $\eta_0=p^{-1/2}$ and using the orthogonality of $Q_b$,
\begin{align}
\|v_{b,r}\|=\|Q_b^Tv_{b,r}\|\le\sqrt p\,
\max_{1\le i\le p}|(Q_b^Tv_{b,r})_i|\le\frac1{\sqrt d}.
\end{align}
Therefore,
\begin{equation}
\max_b\max_{0\le r\le |I_b|}
\left\|\sum_{j=0}^{r-1}P_b^{r-1-j}\xi^{(s_b+j+1)}
\right\|\le\frac1{\sqrt d}
\label{eq:local-refinement-uniform-martingale}
\end{equation}
with probability tending to one.

Combining
\eqref{eq:local-refinement-block-remainder},
\eqref{eq:local-refinement-forcing-bound}, and
\eqref{eq:local-refinement-uniform-martingale}, we obtain, on an event whose probability tends to one,
\begin{equation}
\max_{0\le r\le |I_b|}\|y^{(s_b+r)}\|
\le2\|y^{(s_b)}\|+\frac{C_y}{\sqrt d},
\label{eq:local-refinement-within-block}
\end{equation}
where $C_y=C_y(C_*,\mathfrak D,\mathcal C,\mathcal L,K,L)$ is independent of \(\gamma\).

Moreover, for a full block, i.e. whenever
\(s_b+\ell_d<N_T\wedge J\wedge J_{\rm loc}\), we additionally have $\|P_b^{\ell_d}\|_{\rm op}\le(1-a_-\alpha_d)^{\ell_d}
\le e^{-1}$. Evaluating \eqref{eq:local-refinement-block-Duhamel} at the block endpoint and using \eqref{eq:local-refinement-within-block} therefore gives
\begin{equation}
\|y^{(s_b+\ell_d)}\|\leq\|P_b^{\ell_d}\|_{\rm op}\|y^{(s_b)}\|+\frac{1}{8}\max_{0\leq j\leq\ell_d}\|y^{(s_b+j)}\|+\frac{C}{\sqrt{d}}\leq\theta\|y^{(s_b)}\|+\frac{C}{\sqrt{d}},
\end{equation}
where we use \eqref{eq:local-refinement-block-remainder} for the first inequality and choose $\theta:=e^{-1}+\frac{1}{4}\in(0,1)$ for the second inequality. Iteration gives
\begin{equation}
\|y^{(s_b)}\|\le\theta^b\|y^{(s_0)}\|+\frac{C_y}{(1-\theta)\sqrt d}.
\label{eq:local-refinement-block-iteration}
\end{equation}

We now close the local stopping-time argument. At the entry time, $\|y^{(n_0)}\|\le\frac{r_0}{4}$. By \eqref{eq:local-refinement-within-block}, for all sufficiently large \(d\),
$\max_{0\le r\le |I_0|}\|y^{(n_0+r)}\|<r_0$. Moreover, by
\eqref{eq:local-refinement-block-iteration}, the value at the next block endpoint is again at most \(r_0/4\) for all sufficiently large \(d\). Repeating the same argument block by block shows that $J_{\rm loc}>N_T\wedge J$ with probability tending to one. Hence the local stopping time $J_{\rm loc}$ may be removed.

The number of blocks between \(n_0\) and \(n\) is bounded
below by \(a_-\alpha_d(n-n_0)-2\). Hence
\eqref{eq:local-refinement-within-block}--
\eqref{eq:local-refinement-block-iteration} imply
\begin{equation}
\|y^{(n)}\|\le C_{\rm tr}'e^{-c_{\rm tr}'\alpha_dn}+\frac{C_y'}{\sqrt d},
\qquad n_0\le n<N_T\wedge J,
\label{eq:local-refinement-y-conclusion}
\end{equation}
where $c_{\rm tr}'=-a_-\log\theta>0$. The constants \(C_{\rm tr}'\) and \(c_{\rm tr}'\) may depend on \(T,\gamma\), whereas \(C_y'\) is independent of \(\gamma\).

We finally convert the estimate on \(y^{(n)}\) into the desired
estimate on \(b_{\mathcal K}^{(n)}\). By \eqref{eq:local-refinement-gradient-linearization} and
Lemma \ref{lemma:finite-dimensional-Stein-untied},
\begin{align}
\|b_{\mathcal K}^{(n)}\|\le\|g^{(n)}\|+\frac{C}{\sqrt d}\le
C\|y^{(n)}\|+C\|y^{(n)}\|^2+\frac{C}{\sqrt d}.
\label{eq:local-refinement-b-from-y}
\end{align}
Combining this with \eqref{eq:local-refinement-y-conclusion}, there exist constants \(C_1<\infty\), \(C_b<\infty\), and $c_{\rm tr}:=c_{\rm tr}'>0$, such that, with probability tending to one,
\begin{equation}
\|b_{\mathcal K}^{(n)}\|\le C_1e^{-c_{\rm tr}\alpha_dn}
+\frac{C_b}{\sqrt d},
\qquad n_0\le n<N_T\wedge J.
\label{eq:local-refinement-b-after-entry}
\end{equation}
Here \(C_b\) can be chosen independently of \(\gamma\).

It remains to cover the initial interval \(0\le n<n_0\wedge J\).
By Lemma \ref{lemma:finite-dimensional-Stein-untied}, there exists
\(B_0<\infty\), independent of \(d\), such that
\begin{equation}
\sup_{0\le n<J}\|b_{\mathcal K}^{(n)}\|\le B_0.
\label{eq:local-refinement-b-early-bound}
\end{equation}
For all sufficiently large \(d\), \(\alpha_d\le1\), and hence
\begin{equation}
n<n_0\quad\Longrightarrow\quad\alpha_dn<\alpha_dn_0
\le t_0+\alpha_d\le t_0+1.
\end{equation}
Choose $C_{\rm tr}\ge\max\left\{C_1,\,B_0e^{c_{\rm tr}(t_0+1)}\right\}$. Then, for every \(0\le n<n_0\wedge J\),
$B_0\le C_{\rm tr}e^{-c_{\rm tr}\alpha_dn}$. Combining this with
\eqref{eq:local-refinement-b-after-entry} yields
\begin{equation}
\|b_{\mathcal K}^{(n)}\|\le C_{\rm tr}e^{-c_{\rm tr}\alpha_dn}
+\frac{C_b}{\sqrt d},
\qquad0\le n<N_T\wedge J,
\end{equation}
on an event whose probability tends to one. This proves
\eqref{eq:local-b-refinement}. Removing \(\mathcal B_d\) is legitimate because \(\mathbb P(\mathcal B_d^c)\to0\).
\end{proof}

\paragraph{Step 6: Final spectral bounds}
\begin{lemma}
\label{lemma:spectral_bounds_untied}
Fix $T>0$. Let $\mathfrak D\Subset\mathcal M\times\mathcal Q$ be a compact set containing the limiting trajectory, and define
\begin{equation}
J_{\rm mac}:=\inf\left\{n\ge0:\left(m_{\mathcal K}^{(n)},q^{(1,n)},q^{(2,n)}\right)\notin\mathfrak D\right\}.
\label{eq:J-mac}
\end{equation}
Under the conditions of Theorem \ref{theo:untied-SGD}, there exists a constant $C_*>C_0$, independent of $d$, such that
\begin{equation}
\lim_{d\to\infty}\mathbb P\left(\sup_{0\le n<N_T\wedge J_{\rm mac}}\max_{1\le k\le K}\left\{\|U_k^{(n)}\|_{\rm op},\|V_k^{(n)}\|_{\rm op}\right\}\le C_*\right)=1,
\label{eq:stopped-untied-spectral-bound}
\end{equation}
where $N_T:=\left\lfloor\frac{Td}{\alpha_d}\right\rfloor$.
\end{lemma}

\begin{proof}
The weights with $k\notin\mathcal K$ remain fixed throughout training, so it suffices to consider $k\in\mathcal K$.

Fix $C_*>C_0$, to be chosen below, and define the spectral stopping time
\begin{equation}
J_{\rm sp}:=\inf\left\{n\ge0:\max_{1\le k\le K}
\left\{\|U_k^{(n)}\|_{\rm op},\|V_k^{(n)}\|_{\rm op}\right\}>C_*\right\},
\label{eq:untied-spectral-stopping-time}
\end{equation}
and set $J:=J_{\rm sp}\wedge J_{\rm mac}$. All estimates below are uniform on $\{n<J\}$.

For $k\in\mathcal K$, denote the stochastic gradients by
\begin{equation}
g_{U,k}^{(n)}:=\nabla_{U_k}\mathcal L(G^{(n)}),
\qquad
g_{V,k}^{(n)}:=\nabla_{V_k}\mathcal L(G^{(n)}),
\end{equation}
and their conditional means by
$\bar g_{U,k}^{(n)}:=\mathbb E[g_{U,k}^{(n)}\mid\mathcal F_n]$ and $\bar g_{V,k}^{(n)}:=\mathbb E[g_{V,k}^{(n)}\mid\mathcal F_n]$. Lemma \ref{lemma:finite-dimensional-Stein-untied} gives
\begin{align}
\bar g_{U,k}^{(n)}&=\frac{b_k^{(n)}}{\sqrt d}V_k^{(n)}+\frac1d H_{U,k}^{(n)}V_k^{(n)}+\mathcal E_{U,k}^{(n)},
\label{eq:mean-gradient-U-spectral}
\\
\bar g_{V,k}^{(n)}&=\frac{b_k^{(n)}}{\sqrt d}U_k^{(n)}+\frac1d H_{V,k}^{(n)}U_k^{(n)}+\mathcal E_{V,k}^{(n)}
\label{eq:mean-gradient-V-spectral}
\end{align}
with $\max_{k\in\mathcal K}\left\{\|H_{U,k}^{(n)}\|_{\rm op},\|H_{V,k}^{(n)}\|_{\rm op}\right\}\le C_H$ and $\max_{k\in\mathcal K}\left\{\|\mathcal E_{U,k}^{(n)}\|_{\rm op},\|\mathcal E_{V,k}^{(n)}\|_{\rm op}\right\}\le C_Ed^{-3/2}$ uniformly on $\{n<J\}$. Here $C_H,C_E$ are constants independent of $d$ and $n$.

By Lemma \ref{lemma:local-mean-gradient-refinement}: there exist constants $C_{\rm tr}(T,\gamma),c_{\rm tr}(T,\gamma)>0$ and $C_b(T)<\infty$ such that, with probability tending to one,
\begin{equation}
\max_{k\in\mathcal K}|b_k^{(n)}|\le C_{\rm tr}(T,\gamma)e^{-c_{\rm tr}(T,\gamma)\alpha_dn}+\frac{C_b(T)}{\sqrt d},
\qquad0\le n<N_T\wedge J.
\label{eq:b-coarse-bound-spectral}
\end{equation}

Define
\begin{equation}
\mathcal A_n:=\left\{\|x^{(n)}\|^2\le c_xd\right\}\cap
\left\{\max_{i,j,k}|G_{ijk}^{(n)}|\le c_G\log d\right\},
\label{eq:untied-data-truncation}
\end{equation}
where $c_x,c_G$ are sufficiently large constants. Thus we have
\begin{equation}
\mathbb P\left(\exists\,n<N_T\wedge J:\mathcal A_n^c\right)\longrightarrow0.
\label{eq:untied-data-high-prob}
\end{equation}
Denote the complementary high-probability event by
$\mathcal A_{\rm data}$.

Define the stopped and truncated gradients
$\widehat g_{U,k}^{(n)}:=\mathbf 1_{\{n<J\}}
g_{U,k}^{(n)}\mathbf 1_{\mathcal A_n}$, $\widehat g_{V,k}^{(n)}:=\mathbf 1_{\{n<J\}}g_{V,k}^{(n)}
\mathbf 1_{\mathcal A_n}$, their conditional means
$\widehat{\bar g}_{U,k}^{(n)}$, $\widehat{\bar g}_{V,k}^{(n)}$, and the martingale differences $\widehat\xi_{U,k}^{(n)}:=\widehat g_{U,k}^{(n)}-\widehat{\bar g}_{U,k}^{(n)}$, $\widehat\xi_{V,k}^{(n)}:=\widehat g_{V,k}^{(n)}-
\widehat{\bar g}_{V,k}^{(n)}$.

Set $a_d:=1-\frac{\alpha_d\gamma}{d}$. Define the discounted martingale sums
\begin{align}
Y_{U,k,m}:=-\alpha_d\sum_{n=0}^{m-1}a_d^{m-1-n}\widehat\xi_{U,k}^{(n)},\quad
Y_{V,k,m}:=-\alpha_d\sum_{n=0}^{m-1}a_d^{m-1-n}\widehat\xi_{V,k}^{(n)}.
\end{align}
Similarly to Lemma \ref{lemma:bound-tied-SGD} we have
\begin{equation}
\mathbb P\left(\max_{\substack{k\in\mathcal K\\1\le m\le N_T}}\max\left\{\|Y_{U,k,m}\|_{\rm op},\|Y_{V,k,m}\|_{\rm op}\right\}>\delta\right)\longrightarrow0.
\label{eq:untied-noise-high-prob}
\end{equation}
Denote the complementary event by
$\mathcal A_{\rm noise}$.
Also the truncation-tail corrections are small:
\begin{equation}
\mathbf 1_{\{n<J\}}\max\left\{\| \mathbb E\left[g_{U,k}^{(n)}\mathbf 1_{\mathcal A_n^c}\,\middle|\,\mathcal F_n\right\|_{\rm op},\|\mathbb E\left[g_{V,k}^{(n)}\mathbf 1_{\mathcal A_n^c}\,\middle|\,\mathcal F_n
\right]\|_{\rm op}\right\}\le r_d,
\label{eq:untied-tail-error}
\end{equation}
where $r_d=o(d^{-M})$ for every fixed $M>0$.

On $\mathcal A_{\rm data}$, for every $m\le J$ and $n<m$ unrolling the SGD recursion gives
\begin{align}
U_k^{(m)}=a_d^mU_k^{(0)}-\frac{\alpha_d}{\sqrt d}\sum_{n=0}^{m-1}a_d^{m-1-n}b_k^{(n)}V_k^{(n)}-
\frac{\alpha_d}{d}\sum_{n=0}^{m-1}a_d^{m-1-n}H_{U,k}^{(n)}V_k^{(n)}+Y_{U,k,m}+R_{U,k,m},
\label{eq:unrolled-U-untied}
\end{align}
where, uniformly for $m\leq J\wedge N_T$, $\|R_{U,k,m}\|_{\rm op}\le\alpha_d\sum_{n=0}^{m-1}a_d^{m-1-n}\left(C_Ed^{-3/2}+r_d\right)=o(1)$. The analogous expansion holds for $V_k^{(m)}$.

We now bound the second term of \eqref{eq:unrolled-U-untied}. On the high-probability event where \eqref{eq:b-coarse-bound-spectral} holds, the first term of \eqref{eq:b-coarse-bound-spectral} contributes as
\begin{align}
\frac{\alpha_d}{\sqrt d}\sum_{n=0}^{m-1}a_d^{m-1-n}C_{\rm tr}(T,\gamma)e^{-c_{\rm tr}(T,\gamma)\alpha_dn}\|V_k^{(n)}\|_{\rm op}\le\frac{C(T,\gamma)}{\sqrt d}=o(1),
\label{eq:b-transient-weight-bound}
\end{align}
uniformly in $m\leq J\wedge N_T$, where we use
$\sum_{n\ge0}e^{-c_{\rm tr}(T,\gamma)\alpha_dn}\le C(T,\gamma)/\alpha_d$.
Using $\alpha_d\sum_{r=0}^{\infty}a_d^r=\frac{d}{\gamma}$, the second term of \eqref{eq:b-coarse-bound-spectral} contributes as
\begin{equation}
\frac{\alpha_d}{\sqrt d}\sum_{n=0}^{m-1}
a_d^{m-1-n}\frac{C_b(T)}{\sqrt d}
\|V_k^{(n)}\|_{\rm op}\le\frac{C_b(T)C_*}{\gamma}.
\label{eq:b-stationary-weight-bound}
\end{equation}
The third term of \eqref{eq:unrolled-U-untied} can be bounded as
\begin{equation}
\frac{\alpha_d}{d}\sum_{n=0}^{m-1}a_d^{m-1-n}\|H_{U,k}^{(n)}V_k^{(n)}\|_{\rm op}\le\frac{C_HC_*}{\gamma}.
\label{eq:H-weight-bound}
\end{equation}
Combining
\eqref{eq:unrolled-U-untied}--
\eqref{eq:H-weight-bound}, on
$\mathcal A_{\rm data}\cap\mathcal A_{\rm noise}$ and the
high-probability event
\eqref{eq:b-coarse-bound-spectral},
\begin{equation}
\|U_k^{(m)}\|_{\rm op}\le C_0+\frac{(C_b(T)+C_H)C_*}{\gamma}+\delta+o(1),
\qquad m\leq J\wedge N_T.
\label{eq:U-closing-untied}
\end{equation}
The same estimate holds for $V_k^{(m)}$.

Choose $C_*:=2C_0+1$ and $\delta:=\frac{C_*-C_0}{4}$.
By Assumption \ref{assum:untied}.5, choose $\gamma$ sufficiently large (which might depend on $T$) such that
\begin{equation}
\frac{(C_b(T)+C_H)C_*}{\gamma}\le\frac{C_*-C_0}{4}.
\label{eq:untied-gamma-closing}
\end{equation}
Therefore, for all sufficiently large $d$,
\begin{equation}
\max_{k\in\mathcal K}\max\left\{\|U_k^{(m)}\|_{\rm op},
\|V_k^{(m)}\|_{\rm op}\right\}<C_*
\end{equation}
for every $m\leq J\wedge N_T$ on an event whose probability tends to one. Hence
\begin{equation}
\mathbb P\left(J_{\rm sp}\le J_{\rm mac}\wedge N_T\right)\longrightarrow0.
\end{equation}
This is precisely
\eqref{eq:stopped-untied-spectral-bound}.
\end{proof}

\subsubsection{Fast-dynamics convergence}
\label{app:proof-untied-fast-convergence}
We now prove the convergence on the fast timescale. Fix a finite fast-time horizon \(\widetilde T>0\), and set $N_{\widetilde T}^{\mathrm f}:=\left\lfloor\frac{\widetilde T}{\alpha_d}\right\rfloor$. Let $D_0:=\operatorname{diag}\left(
\bar t_k^A(0)+\bar t_k^B(0)
\right)_{k\in\mathcal K}$, $q_0:=\left(\bar q^{(1)}(0),\bar q^{(2)}(0)\right)$ and let \(\tilde m_{\mathcal K}(t)\) be the solution of
\begin{equation}
\frac{d\tilde m_{\mathcal K}}{dt}=-D_0\nabla_{m_{\mathcal K}}
\Phi\left(\tilde m(t),q_0\right),
\qquad
\tilde m_{\mathcal K}(0)=\lim_{d\to\infty}m_{\mathcal K}^{(d)}(0),
\label{eq:fast-limit-ode-proof}
\end{equation}
with the non-trainable components fixed at their initial values.

By Assumption \ref{assum:untied_convex},
\begin{equation}
\frac{d}{dt}\Phi(\tilde m(t),q_0)=-
\nabla_{m_{\mathcal K}}\Phi(\tilde m(t),q_0)^T
D_0\nabla_{m_{\mathcal K}}\Phi(\tilde m(t),q_0)
\le0.
\end{equation}
Hence the trajectory \(\{\tilde m_{\mathcal K}(t):0\le t\le\widetilde T\}\) remains in a compact subset of \(\mathcal M\).
Choose compact neighborhoods $\mathfrak D_{\mathrm f}
\Subset\mathcal M\times\mathcal Q$
and \(C_*\) in Lemma
\ref{lemma:spectral_bounds_untied} such that the deterministic trajectory
\((\tilde m_{\mathcal K}(t),q_0)\), \(0\le t\le\widetilde T\), lies in
the interior of \(\mathfrak D_{\mathrm f}\). Define the stopping times
\begin{align}
J_{\mathrm f}&:=\inf\left\{n\ge0:
\left(m_{\mathcal K}^{(n)},q^{(1,n)},q^{(2,n)}\right)
\notin\mathfrak D_{\mathrm f}\right\},
\\
J_{\mathrm sp}&:=\inf\left\{n\ge0:\max_k\left\{
\|U_k^{(n)}\|_{\rm op},\|V_k^{(n)}\|_{\rm op}\right\}>C_*\right\},
\end{align}
and set \(J:=J_{\mathrm f}\wedge J_{\mathrm sp}\).
The spectral estimate of Lemma
\ref{lemma:spectral_bounds_untied} gives
\begin{equation}
\mathbb P\left(
J_{\mathrm sp}\le J_{\mathrm f}\wedge
N_{\widetilde T}^{\mathrm f}
\right)\longrightarrow0.
\label{eq:fast-spectral-no-exit}
\end{equation}

We first show that all slow coordinates remain frozen on the fast timescale. Recall $\vartheta^{(n)}:=\left(t^{A,(n)},t^{B,(n)},q^{(1,n)},q^{(2,n)}\right)$.
Write
\begin{equation}
\Delta\vartheta^{(n)}=\mathbb E_n[\Delta\vartheta^{(n)}]+\zeta_\vartheta^{(n+1)},
\qquad
\mathbb E_n[\zeta_\vartheta^{(n+1)}]=0.
\end{equation}
On \(\{n<J\}\), Lemma
\ref{lemma:untied-one-step-estimates} and the boundedness of \(\nabla_{m_{\mathcal K}}\Phi\) on \(\mathfrak D_{\mathrm f}\) give
\begin{equation}
\left\|\mathbb E_n[\Delta\vartheta^{(n)}]
\right\|\le C\left(\frac{\alpha_d}{\sqrt d}
+\frac{(1+\gamma)\alpha_d}{d}\right)
\label{eq:fast-theta-drift}
\end{equation}
and
\begin{equation}
\mathbb E_n\left[\|\zeta_\vartheta^{(n+1)}\|^2\right]
\le C(1+\gamma)^2\frac{\alpha_d^2}{d}.
\label{eq:fast-theta-noise}
\end{equation}
Therefore the accumulated conditional drift over
\(N_{\widetilde T}^{\mathrm f}=\mathcal O(\alpha_d^{-1})\) steps satisfies
\begin{equation}
\sup_{m\le N_{\widetilde T}^{\mathrm f}}
\left\|\sum_{n=0}^{m-1}\mathbf 1_{\{n<J\}}
\mathbb E_n[\Delta\vartheta^{(n)}]\right\|\le C_{\widetilde T}
\left(d^{-1/2}+\frac{1+\gamma}{d}\right)
=o(1).
\label{eq:fast-theta-drift-sum}
\end{equation}
On the other hand, Doob's maximal inequality and
\eqref{eq:fast-theta-noise} imply
\begin{align}
\mathbb E\left[\sup_{m\le N_{\widetilde T}^{\mathrm f}}\left\|
\sum_{n=0}^{m-1}\mathbf 1_{\{n<J\}}\zeta_\vartheta^{(n+1)}\right\|^2\right]\le C N_{\widetilde T}^{\mathrm f}
\frac{\alpha_d^2}{d}\le C_{\widetilde T}\frac{\alpha_d}{d}\longrightarrow0.
\label{eq:fast-theta-martingale}
\end{align}
Consequently,
\begin{equation}
\sup_{0\le n\le N_{\widetilde T}^{\mathrm f}\wedge J}
\left\|\vartheta^{(n)}-\vartheta^{(0)}\right\|
\xrightarrow{\mathbb P}0.
\label{eq:fast-theta-freezing}
\end{equation}
Together with Assumption \ref{assum:untied}.3, this yields
\begin{equation}
\sup_{0\le n\le N_{\widetilde T}^{\mathrm f}\wedge J}
\left\|\vartheta^{(n)}-\bar\vartheta(0)\right\|
\xrightarrow{\mathbb P}0.
\label{eq:fast-theta-freezing-limit}
\end{equation}
In particular,
\begin{equation}
\sup_{0\le n\le N_{\widetilde T}^{\mathrm f}\wedge J}
\left(\|D^{(n)}-D_0\|+\|q^{(n)}-q_0\|\right)\xrightarrow{\mathbb P}0.
\label{eq:fast-Dq-freezing}
\end{equation}

We next turn to the mean variables. Define
\begin{equation}
F(m):=-D_0\nabla_{m_{\mathcal K}}\Phi(m,q_0).
\end{equation}
By Lemma \ref{lemma:untied-one-step-estimates}, on \(\{n<J\}\),
\begin{equation}
\mathbb E_n[\Delta m_{\mathcal K}^{(n)}]=\alpha_d F(m^{(n)})+\alpha_d\rho_m^{(n)},
\label{eq:fast-m-Euler}
\end{equation}
where, using \eqref{eq:fast-Dq-freezing}, the local Lipschitz continuity of \(\nabla_{m_{\mathcal K}}\Phi\), and the bound
\(\|e^{(n)}\|\le C\),
\begin{equation}
\sup_{0\le n<N_{\widetilde T}^{\mathrm f}\wedge J}\|\rho_m^{(n)}\|\xrightarrow{\mathbb P}0.
\label{eq:fast-m-drift-error}
\end{equation}
This is because
\begin{align}
\|\rho_m^{(n)}\|\le C\|D^{(n)}-D_0\|+C\|q^{(n)}-q_0\|+\frac{C}{\sqrt d}+C\left(\frac{1+\gamma}{d}+\alpha_d\right).
\end{align}

Let
\begin{equation}
\zeta_m^{(n+1)}:=\Delta m_{\mathcal K}^{(n)}-\mathbb E_n[\Delta m_{\mathcal K}^{(n)}].
\end{equation}
Again by Lemma \ref{lemma:untied-one-step-estimates}, $\mathbb E_n\left[\|\zeta_m^{(n+1)}\|^2\right]\le C\alpha_d^2$. Hence Doob's maximal inequality gives
\begin{equation}
\sup_{m\le N_{\widetilde T}^{\mathrm f}}
\left\|\sum_{n=0}^{m-1}\mathbf 1_{\{n<J\}}\zeta_m^{(n+1)}\right\|
\xrightarrow{\mathbb P}0,
\label{eq:fast-m-martingale}
\end{equation}
because its quadratic variation is bounded by $
CN_{\widetilde T}^{\mathrm f}\alpha_d^2
\le C_{\widetilde T}\alpha_d
\longrightarrow0$.

Summing \eqref{eq:fast-m-Euler}, we therefore obtain, uniformly for \(m\le N_{\widetilde T}^{\mathrm f}\wedge J\),
\begin{equation}
m_{\mathcal K}^{(m)}=m_{\mathcal K}^{(0)}+\alpha_d\sum_{n=0}^{m-1}
F(m^{(n)})+o_{\mathbb P}(1),
\label{eq:fast-Euler-integral}
\end{equation}
where the error is uniform in \(m\).
Since \(F\) is Lipschitz on the compact neighborhood
\(\mathfrak D_{\mathrm f}\), the standard discrete error estimate, together with Assumption \ref{assum:untied}.3 and the Gr\"onwall inequality (see the proof of Theorem \ref{theo:S-sgd-dynamics}), gives
\begin{equation}
\sup_{0\le n\le
N_{\widetilde T}^{\mathrm f}\wedge J}
\left\|m_{\mathcal K}^{(n)}-\tilde m_{\mathcal K}(n\alpha_d)\right\|
\xrightarrow{\mathbb P}0.
\label{eq:fast-stopped-convergence}
\end{equation}

It remains to remove the stopping time. By construction, the deterministic trajectory
\(\{(\tilde m_{\mathcal K}(t),q_0):0\le t\le\widetilde T\}\) has positive distance from
\(\partial\mathfrak D_{\mathrm f}\).
Equations \eqref{eq:fast-Dq-freezing} and
\eqref{eq:fast-stopped-convergence} therefore imply
\begin{equation}
\mathbb P\left(J_{\mathrm f}\le N_{\widetilde T}^{\mathrm f}\wedge J_{\mathrm sp}\right)\longrightarrow0.
\label{eq:fast-macro-no-exit}
\end{equation}
Combining \eqref{eq:fast-spectral-no-exit} and
\eqref{eq:fast-macro-no-exit}, we conclude that
\begin{equation}
\mathbb P\left(
J\le N_{\widetilde T}^{\mathrm f}
\right)\longrightarrow0.
\end{equation}
Hence the stopping may be removed from
\eqref{eq:fast-stopped-convergence}. Using the piecewise-constant
interpolation \eqref{eq:untied-fast-interpolation}, we obtain
\begin{equation}
\sup_{t\in[0,\widetilde T]}\left\|\tilde m^{(d)}(t)-\tilde m(t)\right\|_\infty\xrightarrow{\mathbb P}0,
\end{equation}
which proves \eqref{eq:untied-fast-convergence}.

\subsubsection{Moment expansion}
\paragraph{Step 1: One-step expansion}
For \(k\in\mathcal K\), define the centered gradient noises $\xi_{U,k}^{(n)}:=g_{U,k}^{(n)}-\bar g_{U,k}^{(n)}$, $\xi_{V,k}^{(n)}:=g_{V,k}^{(n)}-\bar g_{V,k}^{(n)}$, where $\bar g_{U,k}^{(n)}:=\mathbb E_n[g_{U,k}^{(n)}]$,
$\bar g_{V,k}^{(n)}=\mathbb E_n[g_{V,k}^{(n)}]$.
Let $a_d:=1-\frac{\alpha_d\gamma}{d}$. For \(Z_k\in\{M_k,M_k^T,A_k,B_k\}\), define
\begin{equation}
\mathscr N_k^{(n)}(Z_k):=\begin{cases}
-a_d\left(\xi_{U,k}^{(n)}V_k^T+U_k(\xi_{V,k}^{(n)})^T\right),& Z_k=M_k,
\\[1mm]
-a_d\left(V_k(\xi_{U,k}^{(n)})^T+\xi_{V,k}^{(n)}U_k^T\right),& Z_k=M_k^T,
\\[1mm]
-a_d\left(\xi_{U,k}^{(n)}U_k^T+U_k(\xi_{U,k}^{(n)})^T\right),& Z_k=A_k,
\\[1mm]
-a_d\left(\xi_{V,k}^{(n)}V_k^T+V_k(\xi_{V,k}^{(n)})^T\right),& Z_k=B_k.
\end{cases}
\label{eq:untied-letter-noise}
\end{equation}
For \(k\notin\mathcal K\), set
\(\mathscr N_k^{(n)}(Z_k)=0\). Define the finite-dimensional raw vector field as
\begin{align}
[\mathcal V_{\mathrm{raw}}^{(d,n)}]_\alpha:=-\sum_{\substack{i=1\\k_i\in\mathcal K}}^w
\sqrt d\,b_{k_i}^{(n)}\mathcal T_i^{\mathrm{bias}}(\mu^{(n)})-2\sum_{\substack{i=1\\k_i\in\mathcal K}}^w
\sum_{l=1}^K\sum_{r=1}^2a_{k_il}^{(r,d,n)}\mathcal T_{i,l,r}^{\mathrm{diff}}(\mu^{(n)})-2\gamma|\alpha|_{\mathcal K}\mu_\alpha^{(n)}.
\label{eq:untied-raw-vector-field}
\end{align}
\begin{lemma}
\label{lemma:untied-raw-moment-expansion}
Under the conditions of Theorem \ref{theo:untied-SGD}, let \(J\) be a
stopping time satisfying \eqref{eq:stopping-time}. Fix a finite word $\alpha=(Z_1,\ldots,Z_w)\in\mathscr Z^w$, and let \(k_i\) denote the index carried by the letter \(Z_i\). Then, uniformly on \(\{n<J\}\),
\begin{equation}
\mu_\alpha^{(n+1)}-\mu_\alpha^{(n)}=\frac{\alpha_d}{d}
[\mathcal V_{\mathrm{raw}}^{(d,n)}]_\alpha+\alpha_d\eta_\alpha^{(n)}+R_\alpha^{(n)},
\label{eq:untied-raw-moment-one-step}
\end{equation}
where
\begin{equation}
\eta_\alpha^{(n)}:=\sum_{\substack{i=1\\k_i\in\mathcal K}}^w
\frac1d\Tr\left[Z_1^{(n)}\cdots Z_{i-1}^{(n)}
\mathscr N_{k_i}^{(n)}(Z_i)Z_{i+1}^{(n)}\cdots Z_w^{(n)}\right]
\label{eq:untied-moment-martingale}
\end{equation}
is a martingale difference: $\mathbb E_n[\eta_\alpha^{(n)}]=0$.
Moreover, for every fixed word \(\alpha\), there exists
\(C_\alpha<\infty\), independent of \(d\) and \(n\), such that
\begin{equation}
\mathbf 1_{\{n<J\}}\mathbb E_n\left[
|\eta_\alpha^{(n)}|^2\right]\le\frac{C_\alpha}{d},
\label{eq:untied-moment-martingale-variance}
\end{equation}
and
\begin{equation}
\mathbf 1_{\{n<J\}}\mathbb E_n\left[|R_\alpha^{(n)}|
\right]\le C_\alpha\left[(1+\gamma)^2\alpha_d^2+\frac{\alpha_d}{d^{3/2}}\right].
\label{eq:untied-raw-moment-remainder}
\end{equation}
\end{lemma}
\begin{proof}
Throughout the proof we condition on \(\mathcal F_n\) and suppress the superscript \((n)\). Thus \(U_k,V_k,M_k,A_k,B_k\) are deterministic
conditional on \(\mathcal F_n\), while the stochastic gradients depend on the fresh sample \(x^{(n)}\). All estimates below are uniform over \eqref{eq:stopping-time}.

For \(k\in\mathcal K\), the SGD recursion can be written as $U_k^+=a_dU_k-\alpha_d g_{U,k}$ and $V_k^+=a_dV_k-\alpha_d g_{V,k}$. Therefore
\begin{align}
M_k^+&=a_d^2M_k-\alpha_da_d\left(g_{U,k}V_k^T+U_kg_{V,k}^T\right)+\alpha_d^2g_{U,k}g_{V,k}^T,
\label{eq:untied-raw-M-exact}
\\
A_k^+&=a_d^2A_k-\alpha_da_d\left(g_{U,k}U_k^T+U_kg_{U,k}^T\right)+\alpha_d^2g_{U,k}g_{U,k}^T,
\label{eq:untied-raw-A-exact}
\\
B_k^+&=a_d^2B_k-\alpha_da_d\left(g_{V,k}V_k^T+V_kg_{V,k}^T\right)+\alpha_d^2g_{V,k}g_{V,k}^T.
\label{eq:untied-raw-B-exact}
\end{align}
Write $g_{U,k}=\bar g_{U,k}+\xi_{U,k}$ and $g_{V,k}=\bar g_{V,k}+\xi_{V,k}$. By Lemma \ref{lemma:finite-dimensional-Stein-untied},
\begin{align}
\bar g_{U,k}&=\frac{b_k}{\sqrt d}V_k+\frac{2}{d}
\sum_{l=1}^K\sum_{r=1}^2a_{kl}^{(r,d)}M_l^{(r)}V_k+\mathcal E_{U,k},
\label{eq:untied-raw-Stein-U}
\\
\bar g_{V,k}&=\frac{b_k}{\sqrt d}U_k+\frac{2}{d}
\sum_{l=1}^K\sum_{r=1}^2a_{kl}^{(r,d)}M_l^{(\bar r)}U_k+\mathcal E_{V,k},
\label{eq:untied-raw-Stein-V}
\end{align}
where
\begin{equation}
|b_k|+\max_{l,r}|a_{kl}^{(r,d)}|\le C,
\qquad
\|\mathcal E_{U,k}\|_{\rm op}+\|\mathcal E_{V,k}\|_{\rm op}\le Cd^{-3/2}.
\end{equation}
Substituting \eqref{eq:untied-raw-Stein-U}--\eqref{eq:untied-raw-Stein-V}
into \eqref{eq:untied-raw-M-exact}, using $(M_l^{(\bar r)})^T=M_l^{(r)}$, we obtain
\begin{align}
\Delta M_k=-\frac{\alpha_db_k}{\sqrt d}(A_k+B_k)
-\frac{2\alpha_d}{d}\sum_{l=1}^K\sum_{r=1}^2
a_{kl}^{(r,d)}\left(M_l^{(r)}B_k+A_kM_l^{(r)}\right)-\frac{2\alpha_d\gamma}{d}M_k+\alpha_d\mathscr N_k(M_k)+\mathscr R_k(M_k).
\label{eq:untied-raw-letter-M}
\end{align}
Taking the transpose gives
\begin{align}
\Delta M_k^T=-\frac{\alpha_db_k}{\sqrt d}(A_k+B_k)-\frac{2\alpha_d}{d}\sum_{l=1}^K\sum_{r=1}^2a_{kl}^{(r,d)}\left(B_kM_l^{(\bar r)}+M_l^{(\bar r)}A_k\right)-\frac{2\alpha_d\gamma}{d}M_k^T+\alpha_d\mathscr N_k(M_k^T)+\mathscr R_k(M_k^T).
\label{eq:untied-raw-letter-MT}
\end{align}
Likewise, substituting
\eqref{eq:untied-raw-Stein-U} into
\eqref{eq:untied-raw-A-exact} gives
\begin{align}
\Delta A_k=-\frac{\alpha_db_k}{\sqrt d}(M_k+M_k^T)-\frac{2\alpha_d}{d}\sum_{l=1}^K\sum_{r=1}^2a_{kl}^{(r,d)}\left(M_l^{(r)}M_k^T+M_kM_l^{(\bar r)}\right)-\frac{2\alpha_d\gamma}{d}A_k+\alpha_d\mathscr N_k(A_k)+\mathscr R_k(A_k),
\label{eq:untied-raw-letter-A}
\end{align}
whereas
\eqref{eq:untied-raw-Stein-V} and
\eqref{eq:untied-raw-B-exact} give
\begin{align}
\Delta B_k=-\frac{\alpha_db_k}{\sqrt d}(M_k+M_k^T)-\frac{2\alpha_d}{d}\sum_{l=1}^K\sum_{r=1}^2
a_{kl}^{(r,d)}\left(M_l^{(\bar r)}M_k+M_k^TM_l^{(r)}\right)-\frac{2\alpha_d\gamma}{d}B_k+\alpha_d\mathscr N_k(B_k)+\mathscr R_k(B_k).
\label{eq:untied-raw-letter-B}
\end{align}
We next make the remainders \(\mathscr R_k(Z_k)\) explicit. For \(Z_k=M_k\),
\begin{align}
\mathscr R_k(M_k)=&\alpha_d(1-a_d)\Bigg[\frac{b_k}{\sqrt d}(A_k+B_k)+\frac{2}{d}\sum_{l=1}^K\sum_{r=1}^2
a_{kl}^{(r,d)}\left(M_l^{(r)}B_k+A_kM_l^{(r)}\right)\Bigg]\nonumber\\
&-\alpha_da_d\left(\mathcal E_{U,k}V_k^T+U_k\mathcal E_{V,k}^T
\right)+\frac{\alpha_d^2\gamma^2}{d^2}M_k+\alpha_d^2g_{U,k}g_{V,k}^T,
\label{eq:untied-letter-remainder-M-explicit}
\end{align}
which follows directly from $a_d^2-1=-\frac{2\alpha_d\gamma}{d}+\frac{\alpha_d^2\gamma^2}{d^2}$ and $1-a_d=\frac{\alpha_d\gamma}{d}$. Moreover, $\mathscr R_k(M_k^T)=\mathscr R_k(M_k)^T$.
Similarly,
\begin{align}
\mathscr R_k(A_k)
={}&\alpha_d(1-a_d)\Bigg[\frac{b_k}{\sqrt d}(M_k+M_k^T)+\frac{2}{d}\sum_{l=1}^K\sum_{r=1}^2a_{kl}^{(r,d)}
\left(M_l^{(r)}M_k^T+M_kM_l^{(\bar r)}
\right)\Bigg]\nonumber\\
&-\alpha_da_d\left(\mathcal E_{U,k}U_k^T+U_k\mathcal E_{U,k}^T
\right)+\frac{\alpha_d^2\gamma^2}{d^2}A_k+\alpha_d^2g_{U,k}g_{U,k}^T,
\label{eq:untied-letter-remainder-A-explicit}
\end{align}
and
\begin{align}
\mathscr R_k(B_k)={}&\alpha_d(1-a_d)
\Bigg[\frac{b_k}{\sqrt d}(M_k+M_k^T)+\frac{2}{d}
\sum_{l=1}^K\sum_{r=1}^2a_{kl}^{(r,d)}\left(M_l^{(\bar r)}M_k+M_k^TM_l^{(r)}\right)\Bigg]
\nonumber\\
&-\alpha_da_d\left(\mathcal E_{V,k}V_k^T+V_k\mathcal E_{V,k}^T
\right)+\frac{\alpha_d^2\gamma^2}{d^2}B_k+\alpha_d^2g_{V,k}g_{V,k}^T.
\label{eq:untied-letter-remainder-B-explicit}
\end{align}
We now estimate these remainders in the form needed for the moment expansion. Let \(P\) be any \(\mathcal F_n\)-measurable \(d\times d\) matrix satisfying $\|P\|_{\rm op}\le C_P$. On the stopped \eqref{eq:stopping-time}, all letters satisfy
\begin{equation}
\max_{1\le l\le K}\left\{\|M_l\|_{\rm op},\|A_l\|_{\rm op},
\|B_l\|_{\rm op}\right\}\le C_*^2.
\label{eq:untied-remainder-letter-op-bound}
\end{equation}
Hence every fixed product of the letters appearing in
\eqref{eq:untied-letter-remainder-M-explicit}--
\eqref{eq:untied-letter-remainder-B-explicit} has uniformly bounded operator norm, i.e., $\frac1d|\Tr(PQ)|\le\|P\|_{\rm op}\|Q\|_{\rm op}\le C$ for every such product \(Q\).

Since $1-a_d=\frac{\alpha_d\gamma}{d}$, the terms are therefore bounded by
\begin{align}
\alpha_d(1-a_d)
\left(
\frac{C}{\sqrt d}+\frac{C}{d}
\right)
\le
C\alpha_d^2\gamma
\left(
d^{-3/2}+d^{-2}
\right).
\label{eq:untied-remainder-ad-correction}
\end{align}
Likewise,
\begin{equation}
\frac{\alpha_d^2\gamma^2}{d^2}
\frac1d|\Tr(PZ_k)|
\le
C\frac{\alpha_d^2\gamma^2}{d^2}.
\label{eq:untied-remainder-decay-correction}
\end{equation}
For the remainders, Lemma
\ref{lemma:finite-dimensional-Stein-untied} gives $\|\mathcal E_{U,k}\|_{\rm op}+\|\mathcal E_{V,k}\|_{\rm op}\le Cd^{-3/2}$.
Since \(\|U_k\|_{\rm op},\|V_k\|_{\rm op}\le C_*\), we obtain, for example,
\begin{align}
\frac1d\left|\Tr(P\mathcal E_{U,k}V_k^T)
\right|\le\|P\mathcal E_{U,k}V_k^T\|_{\rm op}
\le C_P\|\mathcal E_{U,k}\|_{\rm op}\|V_k\|_{\rm op}\le Cd^{-3/2}.
\end{align}
The same estimate applies to all the other Stein-remainder terms, and hence their total contribution is bounded by $C\alpha_d d^{-3/2}$.

Finally, consider the quadratic-gradient terms. For instance, by Cauchy--Schwarz and
\eqref{eq:gradient-Frobenius-second-moment-untied},
\begin{align}
\mathbb E_n\left[\frac1d\left|\Tr\left[P g_{U,k}g_{V,k}^T\right]\right|\right]\le\frac{C_P}{d}\left(\mathbb E_n\|g_{U,k}\|_F^2\right)^{1/2}\left(\mathbb E_n\|g_{V,k}\|_F^2\right)^{1/2}\le C.
\label{eq:untied-remainder-quadratic-gradient}
\end{align}
The same argument gives
\begin{equation}
\mathbb E_n\left[\frac1d\left|\Tr\left[P g_{U,k}g_{U,k}^T\right]\right|+\frac1d\left|\Tr\left[P g_{V,k}g_{V,k}^T\right]\right|\right]\le C.
\end{equation}
Thus the quadratic-gradient part of the letter remainder contributes at most \(C\alpha_d^2\).

Combining
\eqref{eq:untied-remainder-ad-correction}--
\eqref{eq:untied-remainder-quadratic-gradient}, we conclude that, uniformly on \(\{n<J\}\),
\begin{equation}
\mathbb E_n\left[\left|\frac1d\Tr\left[
P\mathscr R_k(Z_k)\right]\right|\right]\le C_P
\left[(1+\gamma)^2\alpha_d^2+\frac{\alpha_d}{d^{3/2}}
\right],\qquad Z_k\in\{M_k,M_k^T,A_k,B_k\}.
\label{eq:untied-letter-remainder-trace-bound}
\end{equation}

Fix $\alpha=(Z_1,\ldots,Z_w)$. Writing $Z_i^+=Z_i+\Delta Z_i$ in $\mu_\alpha^+=\frac1d
\Tr[Z_1^+\cdots Z_w^+]$ and expanding the product gives
\begin{align}
\mu_\alpha^+-\mu_\alpha=\sum_{i=1}^w\frac1d\Tr\left[
Z_1\cdots Z_{i-1}\Delta Z_iZ_{i+1}\cdots Z_w\right]+R_{\alpha,\mathrm{mult}},
\label{eq:untied-word-product-expansion}
\end{align}
where \(R_{\alpha,\mathrm{mult}}\) is the sum of all terms containing increments at two or more positions.

Insert
\eqref{eq:untied-raw-letter-M}--
\eqref{eq:untied-raw-letter-B}
into the first-order sum in
\eqref{eq:untied-word-product-expansion}. By the definition
\eqref{eq:untied-substitution}, we obtain
\begin{equation}
\mu_\alpha^+-\mu_\alpha=\frac{\alpha_d}{d}[\mathcal{V}_{\mathrm{raw}}^{(d)}]_\alpha+\alpha_d\eta_\alpha+R_\alpha.
\end{equation}
\(\eta_\alpha^{(n)}\) is defined in
\eqref{eq:untied-moment-martingale}. By $\mathbb E_n[\mathscr N_{k_i}(Z_i)]=0$ we immediately obtain $\mathbb E_n[\eta_\alpha]=0$. Moreover, on \eqref{eq:stopping-time} all letters satisfy $\max_{Z\in\mathscr Z}\|Z\|_{\rm op}\le C_*^2$, which gives $\mathbb E_n|\eta_\alpha|^2
\le\frac{C_\alpha}{d}$ and proves
\eqref{eq:untied-moment-martingale-variance}. $R_\alpha$ is the sum of the first-order letter
remainders and $R_{\alpha,\mathrm{mult}}$.

The final step is to control the first-order letter remainders and \(R_{\alpha,\mathrm{mult}}\). By \eqref{eq:untied-letter-remainder-trace-bound}, the sum of all first-order letter remainders satisfies
\begin{equation}
\mathbb E_n\left|\sum_{\substack{i=1\\k_i\in\mathcal K}}^w
\frac1d\Tr\left[Z_1\cdots Z_{i-1}\mathscr R_{k_i}(Z_i)
Z_{i+1}\cdots Z_w\right]\right|\le
C_\alpha\left[(1+\gamma)^2\alpha_d^2+\frac{\alpha_d}{d^{3/2}}\right].
\label{eq:untied-first-order-letter-remainder}
\end{equation}
We next consider products containing at least two letter increments. The spectral bounds \eqref{eq:stopping-time}, polynomial growth of the derivatives of \(\mathcal L\), and Gaussian moment estimates imply, for every fixed
\(p<\infty\),
\begin{equation}
\mathbb E_n\left[\|\Delta Z_k\|_F^p\right]\le
C_p(1+\gamma)^p\alpha_d^p d^{p/2},
\qquad Z_k\in\{M_k,M_k^T,A_k,B_k\}.
\label{eq:untied-letter-increment-F-moment}
\end{equation}
The same estimate holds with the Frobenius norm replaced by the
operator norm, since \(\|\cdot\|_{\rm op}\le\|\cdot\|_F\).

Consider a term in \(R_{\alpha,\mathrm{mult}}\) containing
\(r\ge2\) increments. H\"older's inequality gives
\begin{equation}
\frac1d\mathbb E_n
\left|\Tr\left[P_0\Delta Z_{i_1}P_1\cdots\Delta Z_{i_r}P_r
\right]\right|\le C_\alpha(1+\gamma)^r\alpha_d^r d^{(r-2)/2},
\label{eq:untied-multiple-increment-bound}
\end{equation}
where the \(P_j\)'s are products of unchanged letters.
Since the length \(w\) is fixed and $\alpha_d\sqrt d\longrightarrow0$, under the learning-rate condition of Theorem
\ref{theo:untied-SGD}, for every \(2\le r\le w\),
\begin{equation}
\alpha_d^r d^{(r-2)/2}=\alpha_d^2(\alpha_d\sqrt d)^{r-2}\le C\alpha_d^2
\end{equation}
for all sufficiently large \(d\). Summing over the finitely many
multi-increment terms therefore gives
\begin{equation}
\mathbb E_n|R_{\alpha,\mathrm{mult}}|\le C_\alpha(1+\gamma)^2\alpha_d^2.
\label{eq:untied-multiple-increment-total}
\end{equation}
Combining
\eqref{eq:untied-first-order-letter-remainder} and
\eqref{eq:untied-multiple-increment-total}, and defining
\(R_\alpha^{(n)}\) to be their sum, yields
\begin{equation}
\mathbb E_n|R_\alpha^{(n)}|\le C_\alpha\left[(1+\gamma)^2\alpha_d^2+\frac{\alpha_d}{d^{3/2}}\right].
\end{equation}
This proves
\eqref{eq:untied-raw-moment-remainder}, and completes the proof.
\end{proof}

\paragraph{Step 2: Vanishing errors}
For a fixed word \(\alpha\in\mathscr Z^w\), define the stopped
interpolation
\begin{equation}
\tilde\mu_{\alpha,J}^{(d)}(\tau):=\mu_\alpha^{(m_d(\tau)\wedge J)},\qquad
m_d(\tau):=\left\lfloor\frac{\tau}{\tau_d}\right\rfloor,
\qquad
0\le\tau\le T.
\label{eq:untied-stopped-moment-interpolation}
\end{equation}
Also define the piecewise-constant stopped finite-dimensional vector field by
\begin{equation}
[\widetilde{\mathcal V}_{\mathrm{raw},J}^{(d)}(\tau)]_\alpha:=\mathbf 1_{\{m_d(\tau)<J\}}
[\mathcal V_{\mathrm{raw}}^{(d,m_d(\tau))}]_\alpha.
\label{eq:untied-stopped-raw-vector-field}
\end{equation}
The following lemma proves the integral equation for the interpolated moments, in the spirit of Lemma \ref{lemma:macro-tied-SGD}.
\begin{lemma}
\label{lemma:untied-moment-error}
Under the conditions of Theorem \ref{theo:untied-SGD}, let \(J\) be a stopping time satisfying \eqref{eq:stopping-time}. Fix \(T>0\), and set $\tau_d:=\frac{\alpha_d}{d}$, $N_T:=\left\lfloor\frac{T}{\tau_d}\right\rfloor=\left\lfloor\frac{Td}{\alpha_d}\right\rfloor$.
Then there exists \(E_{\alpha}^{(d)}:[0,T]\to\mathbb R\) such that
\begin{equation}
\tilde\mu_{\alpha,J}^{(d)}(\tau)=\mu_\alpha^{(0)}+\int_0^\tau
[\widetilde{\mathcal V}_{\mathrm{raw},J}^{(d)}(s)]_\alpha\,ds+E_{\alpha}^{(d)}(\tau),
\qquad
0\le\tau\le T,
\label{eq:untied-stopped-raw-integral-equation}
\end{equation}
and $\sup_{\tau\in[0,T]}|E_{\alpha}^{(d)}(\tau)|\xrightarrow{\mathbb P}0$.
\end{lemma}
\begin{proof}
Fix a word
\[
\alpha=(Z_1,\ldots,Z_w).
\]
Throughout the proof all estimates are uniform on \eqref{eq:stopping-time}. Recall from Lemma
\ref{lemma:untied-raw-moment-expansion} that, on \(\{n<J\}\),
\begin{equation}
\mu_\alpha^{(n+1)}-\mu_\alpha^{(n)}=\tau_d
[\mathcal V_{\mathrm{raw}}^{(d,n)}]_\alpha+\alpha_d\eta_\alpha^{(n)}+R_\alpha^{(n)},
\label{eq:untied-error-starting-recursion}
\end{equation}
where
\begin{equation}
\mathbb E_n[\eta_\alpha^{(n)}]=0,
\qquad
\mathbb E_n|\eta_\alpha^{(n)}|^2
\le\frac{C_\alpha}{d},
\label{eq:untied-error-noise-input}
\end{equation}
and
\begin{equation}
\mathbb E_n|R_\alpha^{(n)}|\le C_\alpha\left[(1+\gamma)^2\alpha_d^2+\frac{\alpha_d}{d^{3/2}}\right].
\label{eq:untied-error-remainder-input}
\end{equation}

For \(m\le N_T\), the stopped process satisfies 
\begin{equation}
\mu_\alpha^{(m\wedge J)}-\mu_\alpha^{(0)}=\sum_{n=0}^{m-1}
\mathbf 1_{\{n<J\}}\left(\mu_\alpha^{(n+1)}-\mu_\alpha^{(n)}\right).
\label{eq:untied-stopped-telescoping}
\end{equation}
Substituting
\eqref{eq:untied-error-starting-recursion} into
\eqref{eq:untied-stopped-telescoping} gives
\begin{align}
\mu_\alpha^{(m\wedge J)}
=\mu_\alpha^{(0)}+\tau_d
\sum_{n=0}^{m-1}\mathbf 1_{\{n<J\}}[\mathcal V_{\mathrm{raw}}^{(d,n)}]_\alpha+
\sum_{n=0}^{m-1}\alpha_d\mathbf 1_{\{n<J\}}
\eta_\alpha^{(n)}+\sum_{n=0}^{m-1}
\mathbf 1_{\{n<J\}}R_\alpha^{(n)}.
\label{eq:untied-stopped-summed-recursion}
\end{align}
Accordingly, define $E_{\mathrm{noise},\alpha}^{(d)}(\tau)
:=\sum_{n=0}^{m_d(\tau)-1}
\alpha_d\mathbf 1_{\{n<J\}}\eta_\alpha^{(n)}$
and $E_{\mathrm{rem},\alpha}^{(d)}(\tau)
:=\sum_{n=0}^{m_d(\tau)-1}
\mathbf 1_{\{n<J\}}R_\alpha^{(n)}$.

Since \(J\) is a stopping time,
\(\mathbf 1_{\{n<J\}}\) is \(\mathcal F_n\)-measurable. Hence $\alpha_d\mathbf 1_{\{n<J\}}\eta_\alpha^{(n)}$ is a martingale difference sequence.
Let
\begin{equation}
M_{\alpha,m}^{(d)}:=\sum_{n=0}^{m-1}
\alpha_d\mathbf 1_{\{n<J\}}\eta_\alpha^{(n)},
\qquad0\le m\le N_T.
\end{equation}
Using
\eqref{eq:untied-error-noise-input}, its predictable quadratic variation satisfies
\begin{align}
\mathbb E\left[\left\langle M_\alpha^{(d)}
\right\rangle_{N_T}\right]\le
\sum_{n=0}^{N_T-1}\alpha_d^2
\mathbb E\left[\mathbf 1_{\{n<J\}}
\mathbb E_n|\eta_\alpha^{(n)}|^2\right]
\le N_TC_\alpha\frac{\alpha_d^2}{d}
\le C_{\alpha,T}\alpha_d.
\label{eq:untied-moment-noise-qv}
\end{align}
Since \(\alpha_d\to0\), Doob's maximal inequality gives, for every
\(\varepsilon>0\),
\begin{align}
\mathbb P\left(\sup_{\tau\in[0,T]}
|E_{\mathrm{noise},\alpha}^{(d)}(\tau)|>\varepsilon\right)\le
\mathbb P\left(\max_{m\le N_T}|M_{\alpha,m}^{(d)}|>\varepsilon\right)\le
\frac{C_{\alpha,T}\alpha_d}{\varepsilon^2}
\longrightarrow0.
\end{align}
Therefore
\begin{equation}
\sup_{\tau\in[0,T]}
|E_{\mathrm{noise},\alpha}^{(d)}(\tau)|
\xrightarrow{\mathbb P}0.
\label{eq:untied-moment-noise-vanish}
\end{equation}

By
\eqref{eq:untied-error-remainder-input},
\begin{align}
\mathbb E\left[\sum_{n=0}^{N_T-1}
\mathbf 1_{\{n<J\}}|R_\alpha^{(n)}|\right]\le N_TC_\alpha\left[(1+\gamma)^2\alpha_d^2+\frac{\alpha_d}{d^{3/2}}\right]\le
C_{\alpha,T}\left[(1+\gamma)^2d\alpha_d+d^{-1/2}\right].
\label{eq:untied-cumulative-remainder-bound}
\end{align}
For the fixed time horizon \(T\),
\(\gamma\) is independent of \(d\). Moreover,
$d\alpha_d\le\frac{c_0}{\log d}\longrightarrow0$
under the learning-rate assumption. Hence the right-hand side of \eqref{eq:untied-cumulative-remainder-bound} tends to zero.

Since $\sup_{\tau\in[0,T]}
|E_{\mathrm{rem},\alpha}^{(d)}(\tau)|
\le\sum_{n=0}^{N_T-1}\mathbf 1_{\{n<J\}}
|R_\alpha^{(n)}|$,
Markov's inequality yields
\begin{equation}
\sup_{\tau\in[0,T]}|E_{\mathrm{rem},\alpha}^{(d)}(\tau)|\xrightarrow{\mathbb P}0.
\label{eq:untied-moment-remainder-vanish}
\end{equation}

For \(\tau\in[0,T]\), write $m:=m_d(\tau)$.
By the definition
\eqref{eq:untied-stopped-raw-vector-field},
\begin{align}
\int_0^\tau[\widetilde{\mathcal V}_{\mathrm{raw},J}^{(d)}(s)]_\alpha\,ds=
\tau_d\sum_{n=0}^{m-1}\mathbf 1_{\{n<J\}}[\mathcal V_{\mathrm{raw}}^{(d,n)}]_\alpha+
(\tau-m\tau_d)\mathbf 1_{\{m<J\}}
[\mathcal V_{\mathrm{raw}}^{(d,m)}]_\alpha.
\label{eq:untied-raw-riemann-sum}
\end{align}
We therefore define the discretization error
\begin{equation}
E_{\mathrm{disc},\alpha}^{(d)}(\tau):=-(\tau-m\tau_d)\mathbf 1_{\{m<J\}}[\mathcal V_{\mathrm{raw}}^{(d,m)}]_\alpha.
\label{eq:untied-moment-discretization-error}
\end{equation}
By \eqref{eq:stopping-time}, all letters satisfy
$\max_{Z\in\mathscr Z}\|Z\|_{\rm op}\le C_*^2$.
Consequently, for the fixed \(\alpha\), $T_i^{\mathrm{bias}}$ and $T_i^{\mathrm{diff}}$ in \eqref{eq:untied-raw-vector-field} are uniformly bounded. In addition,
Lemma \ref{lemma:finite-dimensional-Stein-untied} gives $\max_{k}|b_k^{(n)}|+\max_{k,l,r}|a_{kl}^{(r,d,n)}|\le C$ uniformly on \(\{n<J\}\). Hence
\begin{equation}
\mathbf 1_{\{n<J\}}\left|[\mathcal V_{\mathrm{raw}}^{(d,n)}]_\alpha
\right|\le C_\alpha\left(\sqrt d+1+\gamma\right).
\label{eq:untied-raw-drift-coarse-bound}
\end{equation}
Since $0\le\tau-m\tau_d<\tau_d
=\frac{\alpha_d}{d}$, we obtain 
\begin{align}
\sup_{\tau\in[0,T]}|E_{\mathrm{disc},\alpha}^{(d)}(\tau)|\le\frac{\alpha_d}{d}
C_\alpha(\sqrt d+1+\gamma)\le
C_\alpha\left(\frac{\alpha_d}{\sqrt d}+\frac{(1+\gamma)\alpha_d}{d}\right)
\longrightarrow0.
\label{eq:untied-moment-discretization-vanish}
\end{align}

Finally, combining
\eqref{eq:untied-stopped-summed-recursion} and
\eqref{eq:untied-raw-riemann-sum} gives
\begin{equation}
\tilde\mu_{\alpha,J}^{(d)}(\tau)=\mu_\alpha^{(0)}+\int_0^\tau[\widetilde{\mathcal V}_{\mathrm{raw},J}^{(d)}(s)]_\alpha\,ds+E_{\mathrm{noise},\alpha}^{(d)}(\tau)+E_{\mathrm{rem},\alpha}^{(d)}(\tau)+E_{\mathrm{disc},\alpha}^{(d)}(\tau).
\end{equation}
Defining $E_{\alpha}^{(d)}=E_{\mathrm{noise},\alpha}^{(d)}+E_{\mathrm{rem},\alpha}^{(d)}
+E_{\mathrm{disc},\alpha}^{(d)}$ and using
\eqref{eq:untied-moment-noise-vanish},
\eqref{eq:untied-moment-remainder-vanish},
\eqref{eq:untied-moment-discretization-vanish} finishes the proof.
\end{proof}

\paragraph{Step 3: Effective bias of the fast flow}
For \(k\in\mathcal K\), define $D_k^{(n)}:=t_k^{A,(n)}+t_k^{B,(n)}$ and
\begin{equation}
\psi_k^{(d,n)}:=2\sum_{l=1}^K\sum_{r=1}^2
a_{kl}^{(r,d,n)}\left[\mu_{(M_l^{(r)},B_k)}^{(n)}+\mu_{(A_k,M_l^{(r)})}^{(n)}\right].
\label{eq:untied-finite-effective-psi}
\end{equation}
Define the finite-dimensional effective bias by
\begin{equation}
\beta_k^{(d,n)}:=-\frac{\psi_k^{(d,n)}}{D_k^{(n)}}
\label{eq:untied-finite-effective-beta}
\end{equation}
whenever \(D_k^{(n)}>0\), and set
\(\beta_k^{(d,n)}:=0\) when \(D_k^{(n)}=0\). The following lemma proves its convergence.
\begin{lemma}
\label{lemma:untied-effective-bias}
Under the conditions of Theorem \ref{theo:untied-SGD}, let \(J\) be a
stopping time satisfying \eqref{eq:stopping-time}, and fix \(T>0\). Set $\tau_d:=\frac{\alpha_d}{d}$ and $N_T:=\left\lfloor\frac{Td}{\alpha_d}\right\rfloor$.
Then, 
\begin{align}
&\sup_{0\le\tau\le T}
\left|
\tau_d
\sum_{n=0}^{m_d(\tau)-1}
\mathbf 1_{\{n<J\}}
\sum_{\substack{i=1\\k_i\in\mathcal K}}^w
\left(
\sqrt d\,b_{k_i}^{(n)}
-
\beta_{k_i}^{(d,n)}
\right)
\mathcal T_i^{\mathrm{bias}}(\mu^{(n)})
\right|
\xrightarrow{\mathbb P}0,
\label{eq:untied-effective-bias-summed}
\end{align}
where $m_d(\tau):=\left\lfloor\frac{\tau d}{\alpha_d}\right\rfloor$.
\end{lemma}
\begin{proof}
Fix a trainable position \(i\) and write $k:=k_i$, $H_n:=H_{\alpha,i}^{(n)}:=\mathcal T_i^{\mathrm{bias}}(\mu^{(n)})$, $D_n:=D_k^{(n)}$, $
b_n:=b_k^{(n)}$ and $\beta_n:=\beta_k^{(d,n)}$.
All constants below may depend on the fixed word \(\alpha\), but are independent of \(d\) and \(n\).

Returning to the proof of Lemma
\ref{lemma:untied-one-step-estimates}, the expansion
\eqref{eq:untied-mean-trace-expansion} gives 
\begin{equation}
\mathbb E_n[\Delta m_k^{(n)}]=-\alpha_d\mathcal D_n b_n
-\frac{\alpha_d}{\sqrt d}\psi_k^{(d,n)}+r_{k}^{(n)},
\label{eq:untied-effective-bias-m-increment}
\end{equation}
where, uniformly on \(\{n<J\}\),
\begin{equation}
|r_k^{(n)}|\le
C\left(\frac{(1+\gamma)\alpha_d}{d}+\alpha_d^2\right).
\label{eq:untied-effective-bias-m-remainder}
\end{equation}
Using $\psi_k^{(d,n)}=-D_n\beta_n$, we may rewrite
\eqref{eq:untied-effective-bias-m-increment} as
\begin{equation}
\mathbb E_n[\Delta m_k^{(n)}]=-\frac{\alpha_d\mathcal D_n}{\sqrt d}
\left(\sqrt d\,b_n-\beta_n\right)+r_k^{(n)}.
\label{eq:untied-effective-bias-key-identity}
\end{equation}
Let $\zeta_{m,k}^{(n+1)}:=\Delta m_k^{(n)}-\mathbb E_n[\Delta m_k^{(n)}]$. By Lemma \ref{lemma:untied-one-step-estimates},
\begin{equation}
\mathbb E_n\left[|\zeta_{m,k}^{(n+1)}|^2\right]\le C\alpha_d^2.
\label{eq:untied-effective-bias-m-noise-bound}
\end{equation}
Therefore
\eqref{eq:untied-effective-bias-key-identity} gives
\begin{align}
\tau_d H_n\left(\sqrt d\,b_n-\beta_n\right)=-\frac1{\sqrt d}\frac{H_n}{D_n}\Delta m_k^{(n)}+\frac1{\sqrt d}\frac{H_n}{D_n}\zeta_{m,k}^{(n+1)}+\frac1{\sqrt d}\frac{H_n}{D_n}r_k^{(n)}.
\label{eq:untied-effective-bias-decomposition}
\end{align}
Thus it remains to show that the three terms on the right vanish after summation.

By Lemma \ref{lemma:empirical-balancing}, there exists \(c_D=c_D(T,\gamma,c_{\rm bal})>0\) such that
\begin{equation}
\mathbb P\left(\inf_{\substack{0\le n<N_T\wedge J\\k\in\mathcal K}}D_k^{(n)}\ge c_D\right)\longrightarrow1.
\label{eq:untied-effective-bias-D-event}
\end{equation}
Moreover, Lemma \ref{lemma:local-mean-gradient-refinement} gives
constants \(C_{\rm tr},c_{\rm tr},C_b>0\) such that
\begin{equation}
\mathbb P\left(\sup_{0\le n<N_T\wedge J}
\left[\|b_{\mathcal K}^{(n)}\|-C_{\rm tr}e^{-c_{\rm tr}\alpha_dn}\right]_+>\frac{C_b}{\sqrt d}
\right)\longrightarrow0.
\label{eq:untied-effective-bias-b-event}
\end{equation}
Now define
\begin{align}
J_D&:=\inf\left\{n\ge0:\min_{k\in\mathcal K}D_k^{(n)}<c_D\right\},
\\
J_b&:=\inf\left\{n\ge0:\|b_{\mathcal K}^{(n)}\|>C_{\rm tr}e^{-c_{\rm tr}\alpha_dn}+\frac{2C_b}{\sqrt d}\right\},
\end{align}
and set $\widehat J:=J\wedge J_D\wedge J_b$. Equations
\eqref{eq:untied-effective-bias-D-event} and
\eqref{eq:untied-effective-bias-b-event} imply
\begin{equation}
\mathbb P\left(\widehat J<N_T\wedge J\right)\longrightarrow0.
\label{eq:untied-effective-bias-extra-stopping-remove}
\end{equation}
It therefore suffices to prove the desired convergence with \(J\) replaced by \(\widehat J\).
Before \(\widehat J\), $D_n\ge c_D$, $|b_n|\le C_{\rm tr}e^{-c_{\rm tr}\alpha_dn}+\frac{2C_b}{\sqrt d}$, and also
$|H_n|\le C_\alpha$, $|m_k^{(n)}|\le C$ by \eqref{eq:stopping-time}.

Define $C_n:=\frac{H_n}{D_n}$ satisfying $|C_n|\le C_\alpha$
for $n<\widehat J$. To control the second term of \eqref{eq:untied-effective-bias-decomposition}, consider
\begin{equation}
M_m:=\frac1{\sqrt d}\sum_{n=0}^{m-1}\mathbf 1_{\{n<\widehat J\}}
C_n\zeta_{m,k}^{(n+1)}.
\end{equation}
Since \(C_n\mathbf 1_{\{n<\widehat J\}}\) is
\(\mathcal F_n\)-measurable, \(M_m\) is a martingale. By
\eqref{eq:untied-effective-bias-m-noise-bound},
\begin{align}
\mathbb E[\langle M\rangle_{N_T}]\le\frac{C_\alpha}{d}N_T\alpha_d^2\le C_{\alpha,T}\alpha_d.
\end{align}
Hence Doob's maximal inequality yields
\begin{equation}
\max_{m\le N_T}|M_m|\xrightarrow{\mathbb P}0.
\label{eq:untied-effective-bias-m-martingale-vanish}
\end{equation}
For the third term of \eqref{eq:untied-effective-bias-decomposition}, using
\eqref{eq:untied-effective-bias-m-remainder},
\begin{align}
\frac1{\sqrt d}\sum_{n=0}^{N_T-1}
\mathbf 1_{\{n<\widehat J\}}
|C_nr_k^{(n)}|\le\frac{C_\alpha}{\sqrt d}N_T\left(\frac{(1+\gamma)\alpha_d}{d}+\alpha_d^2\right)\le
C_{\alpha,T}\left(\frac{1+\gamma}{\sqrt d}+\sqrt d\,\alpha_d\right)\longrightarrow0,
\label{eq:untied-effective-bias-m-remainder-vanish}
\end{align}
because \(\alpha_d\le c_0(d\log d)^{-1}\).

For the following we control the first term of \eqref{eq:untied-effective-bias-decomposition}. Since \(H_n\) is a finite linear combination of moments of the same degree, Lemma \ref{lemma:untied-raw-moment-expansion} applies to \(H_n\). The same lemma applies to $D_n=t_k^{A,(n)}+t_k^{B,(n)}$.
Using the smooth map $(H,D)\longmapsto\frac{H}{D}$ on the region \(|H|\le C_\alpha\), \(D\ge c_D\), together with a
first-order Taylor expansion, we obtain, for \(n<\widehat J\),
\begin{equation}
C_{n+1}-C_n=A_C^{(n)}+\zeta_C^{(n+1)}+\rho_C^{(n)},
\label{eq:untied-effective-bias-C-decomposition}
\end{equation}
where \(A_C^{(n)}\) is \(\mathcal F_n\)-measurable,
\(\zeta_C^{(n+1)}\) is a martingale difference, and
\begin{align}
|A_C^{(n)}|&\le C_\alpha\left[\frac{\alpha_d}{\sqrt d}
\|b_{\mathcal K}^{(n)}\|+\frac{(1+\gamma)\alpha_d}{d}+(1+\gamma)^2\alpha_d^2\right],
\label{eq:untied-effective-bias-C-drift}
\\
\mathbb E_n\left[|\zeta_C^{(n+1)}|^2\right]&\le C_\alpha\frac{\alpha_d^2}{d},
\label{eq:untied-effective-bias-C-noise}
\\
\mathbb E_n\left[|\rho_C^{(n)}|\right]&\le C_\alpha\left[(1+\gamma)^2\alpha_d^2+\frac{\alpha_d}{d^{3/2}}\right].
\label{eq:untied-effective-bias-C-remainder}
\end{align}
The first term on the right-hand side of
\eqref{eq:untied-effective-bias-C-drift} comes from \eqref{eq:untied-raw-vector-field}, while
\eqref{eq:untied-effective-bias-C-noise} and
\eqref{eq:untied-effective-bias-C-remainder} follow from
\eqref{eq:untied-moment-martingale-variance} and
\eqref{eq:untied-raw-moment-remainder}, together with the quadratic Taylor remainder of \(H/D\).

For $n<J_b$ we have
\begin{align}
\sum_{n=0}^{N_T-1}\mathbf 1_{\{n<\widehat J\}}|A_C^{(n)}|
&\le C_\alpha\frac{\alpha_d}{\sqrt d}\sum_{n\ge0}\left(C_{\rm tr}e^{-c_{\rm tr}\alpha_dn}+\frac{2C_b}{\sqrt d}\right)+C_{\alpha,T}\left(1+\gamma+d\alpha_d\right)=\mathcal O(1),
\label{eq:untied-effective-bias-C-total-drift}
\end{align}
uniformly in \(d\). Here we use
$\alpha_d\sum_{n\ge0}e^{-c_{\rm tr}\alpha_dn}\le C$
and $N_T\frac{\alpha_d}{d}=\mathcal O(1)$, $N_T\alpha_d^2=\mathcal O(d\alpha_d)=o(1)$.
Moreover,
\eqref{eq:untied-effective-bias-C-remainder} gives
\begin{equation}
\sum_{n=0}^{N_T-1}\mathbf 1_{\{n<\widehat J\}}
|\rho_C^{(n)}|\xrightarrow{\mathbb P}0,
\label{eq:untied-effective-bias-C-total-remainder}
\end{equation}

For every \(m\le N_T\),
\begin{align}
\sum_{n=0}^{m-1}\mathbf 1_{\{n<\widehat J\}}
C_n\Delta m_k^{(n)}=C_{m_\ast-1}m_k^{(m_\ast)}-C_0m_k^{(0)}-\sum_{n=1}^{m_\ast-1}m_k^{(n)}
\left(C_n-C_{n-1}\right),
\label{eq:untied-effective-bias-summation-by-parts}
\end{align}
where $m_\ast:=m\wedge\widehat J$ with the first term interpreted as zero when \(m_\ast=0\).

The boundary terms in
\eqref{eq:untied-effective-bias-summation-by-parts} are uniformly \(\mathcal O(1)\), and hence become \(o(1)\) after multiplication by
\(d^{-1/2}\).

For the \(C_n-C_{n-1}\) term, using the boundedness of \(m_k^{(n)}\) and \eqref{eq:untied-effective-bias-C-total-drift},
\begin{equation}
\frac1{\sqrt d}\sup_{m\le N_T}\left|\sum_{n=1}^{m_\ast-1}m_k^{(n)}A_C^{(n-1)}\right|\longrightarrow0.
\label{eq:untied-effective-bias-ibp-drift}
\end{equation}
The contribution of \(\rho_C^{(n)}\) vanishes by
\eqref{eq:untied-effective-bias-C-total-remainder}.

It remains to treat the martingale part (the second term of \eqref{eq:untied-effective-bias-C-decomposition}). Write $m_k^{(n)}=m_k^{(n-1)}+\Delta m_k^{(n-1)}$. Then
\begin{align}
\sum_{n=1}^{m_\ast-1}m_k^{(n)}\zeta_C^{(n)}=\sum_{n=1}^{m_\ast-1}m_k^{(n-1)}\zeta_C^{(n)}+\sum_{n=1}^{m_\ast-1}
\Delta m_k^{(n-1)}\zeta_C^{(n)}.
\label{eq:untied-effective-bias-ibp-martingale-split}
\end{align}
The first term of \eqref{eq:untied-effective-bias-ibp-martingale-split} is a martingale with predictable quadratic variation bounded by \(C_{\alpha,T}\alpha_d\), because
\(m_k^{(n-1)}\) is uniformly bounded. Hence it is upper bounded by \(O_{\mathbb P}(\sqrt{\alpha_d})\).
For the second term of \eqref{eq:untied-effective-bias-ibp-martingale-split}, Lemma \ref{lemma:untied-one-step-estimates} gives $\mathbb E_{n-1}\left[|\Delta m_k^{(n-1)}|^2
\right]\le C\alpha_d^2$ before the stopping time. Together with
\eqref{eq:untied-effective-bias-C-noise} and Cauchy--Schwarz,
\begin{align}
\mathbb E\left[\sum_{n=1}^{N_T\wedge\widehat J-1}
\left|\Delta m_k^{(n-1)}\zeta_C^{(n)}
\right|\right]\le CN_T\frac{\alpha_d^2}{\sqrt d}
\le C_T\alpha_d\sqrt d.
\end{align}
Therefore
\begin{equation}
\frac1{\sqrt d}\sum_{n=1}^{N_T\wedge\widehat J-1}
\left|\Delta m_k^{(n-1)}\zeta_C^{(n)}\right|\xrightarrow{\mathbb P}0,
\label{eq:untied-effective-bias-cross-term}
\end{equation}
since \(\alpha_d\to0\).
Combining these estimates with
\eqref{eq:untied-effective-bias-summation-by-parts}, we obtain
\begin{equation}
\sup_{m\le N_T}
\left|
\frac1{\sqrt d}
\sum_{n=0}^{m-1}
\mathbf 1_{\{n<\widehat J\}}
C_n\Delta m_k^{(n)}
\right|
\xrightarrow{\mathbb P}0.
\label{eq:untied-effective-bias-m-increment-vanish}
\end{equation}
Summing
\eqref{eq:untied-effective-bias-decomposition} from \(n=0\) to
\(m_d(\tau)-1\), and using
\eqref{eq:untied-effective-bias-m-martingale-vanish},
\eqref{eq:untied-effective-bias-m-remainder-vanish}, and
\eqref{eq:untied-effective-bias-m-increment-vanish}, gives
\begin{equation}
\sup_{0\le\tau\le T}\left|\tau_d\sum_{n=0}^{m_d(\tau)-1}
\mathbf 1_{\{n<\widehat J\}}H_n\left(\sqrt d\,b_n-\beta_n\right)\right|\xrightarrow{\mathbb P}0.
\end{equation}
Finally,
\eqref{eq:untied-effective-bias-extra-stopping-remove} allows us to replace \(\widehat J\) by \(J\). Summing over the trainable positions proves
\eqref{eq:untied-effective-bias-summed}.
\end{proof}
\paragraph{Step 4: Final results}
Finally we can write an approximate integral equation for the moments. 
For a word
\(\alpha=(Z_1,\ldots,Z_w)\in\mathscr Z^w\), define the finite-dimensional effective vector field by
\begin{align}
[\mathcal V_{\mathrm{eff}}^{(d,n)}]_\alpha:=-\sum_{\substack{i=1\\k_i\in\mathcal K}}^w
\beta_{k_i}^{(d,n)}\mathcal T_i^{\mathrm{bias}}(\mu^{(n)})-2\sum_{\substack{i=1\\k_i\in\mathcal K}}^w\sum_{l=1}^K\sum_{r=1}^2a_{k_il}^{(r,d,n)}
\mathcal T_{i,l,r}^{\mathrm{diff}}(\mu^{(n)})-2\gamma|\alpha|_{\mathcal K}\mu_\alpha^{(n)}.
\label{eq:untied-effective-vector-field-finite}
\end{align}
Let
\begin{equation}
[\widetilde{\mathcal V}_{\mathrm{eff},J}^{(d)}(\tau)]_\alpha:=\mathbf 1_{\{m_d(\tau)<J\}}
[\mathcal V_{\mathrm{eff}}^{(d,m_d(\tau))}]_\alpha
\label{eq:untied-stopped-effective-vector-field}
\end{equation}
and $m_d(\tau):=\left\lfloor\frac{\tau d}{\alpha_d}\right\rfloor$.
\begin{lemma}
\label{lemma:untied-approximate-integral-equation}
Under the conditions of Theorem \ref{theo:untied-SGD}, let \(J\) be a
stopping time satisfying \eqref{eq:stopping-time}, and fix \(T>0\). Set $\tau_d:=\frac{\alpha_d}{d}$ and $N_T:=\left\lfloor\frac{Td}{\alpha_d}\right\rfloor$.
Then, for every \(\alpha\), there exists
\(\mathcal E_\alpha^{(d)}:[0,T]\to\mathbb R\) such that
\begin{equation}
\tilde\mu_{\alpha,J}^{(d)}(\tau)=\mu_\alpha^{(0)}+\int_0^\tau
[\widetilde{\mathcal V}_{\mathrm{eff},J}^{(d)}(s)]_\alpha\,ds+\mathcal E_\alpha^{(d)}(\tau),
\qquad
0\le\tau\le T,
\label{eq:untied-approximate-effective-integral-equation}
\end{equation}
and $\sup_{\tau\in[0,T]}|\mathcal E_\alpha^{(d)}(\tau)|\xrightarrow{\mathbb P}0$.
\end{lemma}

\begin{proof}
By Lemma \ref{lemma:untied-moment-error}, for every fixed \(\alpha\),
\begin{equation}
\tilde\mu_{\alpha,J}^{(d)}(\tau)=\mu_\alpha^{(0)}+\int_0^\tau
[\widetilde{\mathcal V}_{\mathrm{raw},J}^{(d)}(s)]_\alpha\,ds+E_\alpha^{(d)}(\tau),
\label{eq:untied-approximate-proof-raw}
\end{equation}
where $\sup_{\tau\in[0,T]}
|E_\alpha^{(d)}(\tau)|\xrightarrow{\mathbb P}0$.
It therefore suffices to prove
\begin{equation}
\sup_{\tau\in[0,T]}\left|\int_0^\tau
\left([\widetilde{\mathcal V}_{\mathrm{raw},J}^{(d)}(s)]_\alpha-[\widetilde{\mathcal V}_{\mathrm{eff},J}^{(d)}(s)]_\alpha\right)ds\right|\xrightarrow{\mathbb P}0.
\label{eq:untied-raw-effective-integrated-difference}
\end{equation}
By the definitions
\eqref{eq:untied-raw-vector-field} and
\eqref{eq:untied-effective-vector-field-finite}, the two vector fields
differ only in the bias term. Thus
\begin{align}
[\mathcal V_{\mathrm{raw}}^{(d,n)}]_\alpha
-[\mathcal V_{\mathrm{eff}}^{(d,n)}]_\alpha=
-\sum_{\substack{i=1\\k_i\in\mathcal K}}^w
\left(\sqrt d\,b_{k_i}^{(n)}-\beta_{k_i}^{(d,n)}
\right)\mathcal T_i^{\mathrm{bias}}(\mu^{(n)}).
\label{eq:untied-raw-effective-difference}
\end{align}
Fix \(\tau\in[0,T]\) and write $m:=m_d(\tau)=\left\lfloor\frac{\tau}{\tau_d}\right\rfloor$.
Since both vector fields are piecewise constant,
\begin{align}
&\int_0^\tau\left([\widetilde{\mathcal V}_{\mathrm{raw},J}^{(d)}(s)]_\alpha-[\widetilde{\mathcal V}_{\mathrm{eff},J}^{(d)}(s)]_\alpha\right)ds=-\tau_d
\sum_{n=0}^{m-1}\mathbf 1_{\{n<J\}}
\sum_{\substack{i=1\\k_i\in\mathcal K}}^w
\left(\sqrt d\,b_{k_i}^{(n)}-\beta_{k_i}^{(d,n)}
\right)\mathcal T_i^{\mathrm{bias}}(\mu^{(n)})+R_{\mathrm{cell},\alpha}^{(d)}(\tau),
\label{eq:untied-raw-effective-riemann-decomposition}
\end{align}
where the final term is
\begin{align}
R_{\mathrm{cell},\alpha}^{(d)}(\tau):=-(\tau-m\tau_d)\mathbf 1_{\{m<J\}}\sum_{\substack{i=1\\k_i\in\mathcal K}}^w
\left(\sqrt d\,b_{k_i}^{(m)}-\beta_{k_i}^{(d,m)}
\right)\mathcal T_i^{\mathrm{bias}}(\mu^{(m)}).
\label{eq:untied-effective-final-cell}
\end{align}
The first term in
\eqref{eq:untied-raw-effective-riemann-decomposition} converges to zero uniformly in \(\tau\) in probability by Lemma \ref{lemma:untied-effective-bias}. Hence it remains only to control \(R_{\mathrm{cell},\alpha}^{(d)}\).

Before the stopping time \eqref{eq:stopping-time}, all letters satisfy a uniform operator-norm bound. Therefore, for the fixed word \(\alpha\), $\left|\mathcal T_i^{\mathrm{bias}}(\mu^{(n)})\right|\le C_\alpha$ uniformly for \(n<J\).
Moreover, Lemma
\ref{lemma:finite-dimensional-Stein-untied} gives $|b_k^{(n)}|\le C$ uniformly on \(\{n<J\}\). We also have a uniform bound on
\(\beta_k^{(d,n)}\). Indeed, since \(A_k,B_k\succeq0\),
\begin{align}
\left|\mu_{(M_l^{(r)},B_k)}^{(n)}
\right|
&\le\|M_l^{(r,n)}\|_{\rm op}\,t_k^{B,(n)},
\\
\left|\mu_{(A_k,M_l^{(r)})}^{(n)}\right|
&\le\|M_l^{(r,n)}\|_{\rm op}\,t_k^{A,(n)}.
\end{align}
The coefficients \(a_{kl}^{(r,d,n)}\) are uniformly bounded on \eqref{eq:stopping-time}, and hence
\begin{equation}
|\psi_k^{(d,n)}|\le C\left(
t_k^{A,(n)}+t_k^{B,(n)}\right)=CD_k^{(n)}.
\label{eq:untied-final-cell-psi-bound}
\end{equation}
Consequently, $|\beta_k^{(d,n)}|\le C$ whenever \(D_k^{(n)}>0\), while the same bound is trivial under our convention \(\beta_k^{(d,n)}=0\) when \(D_k^{(n)}=0\).

Since $0\le\tau-m\tau_d<\tau_d=\frac{\alpha_d}{d}$, the bounds above give
\begin{align}
\sup_{\tau\in[0,T]}|R_{\mathrm{cell},\alpha}^{(d)}(\tau)|\le C_\alpha\frac{\alpha_d}{d}
(\sqrt d+1)\le
C_\alpha\left(\frac{\alpha_d}{\sqrt d}+\frac{\alpha_d}{d}\right)
\longrightarrow0.
\label{eq:untied-effective-final-cell-vanish}
\end{align}
Combining this with Lemma
\ref{lemma:untied-effective-bias} proves
\eqref{eq:untied-raw-effective-integrated-difference}.

Finally, define
\begin{align}
\mathcal E_\alpha^{(d)}(\tau):=E_\alpha^{(d)}(\tau)+\int_0^\tau\left([\widetilde{\mathcal V}_{\mathrm{raw},J}^{(d)}(s)]_\alpha-[\widetilde{\mathcal V}_{\mathrm{eff},J}^{(d)}(s)]_\alpha\right)ds.
\end{align}
Then we finish the proof by 
\eqref{eq:untied-approximate-proof-raw} and
\eqref{eq:untied-raw-effective-integrated-difference}.
\end{proof}

\subsubsection{Limiting slow dynamics}
We now pass from the finite-dimensional effective dynamics \eqref{eq:untied-approximate-effective-integral-equation} to the limiting dynamics. We first establish compactness of the trajectories.
\paragraph{Step 1: Compactness} 
Fix a compact set \(\mathfrak D\Subset\mathcal M\times\mathcal Q\), and recall $J_{\rm mac}
:=\inf\left\{n\ge0:\left(m_{\mathcal K}^{(n)},q^{(1,n)},q^{(2,n)}\right)
\notin\mathfrak D\right\}$.
Let \(C_*\) be the constant in Lemma
\ref{lemma:spectral_bounds_untied}, and define
$J_{\rm sp}:=\inf\left\{n\ge0:\max_{1\le k\le K}
\left\{\|U_k^{(n)}\|_{\rm op},\|V_k^{(n)}\|_{\rm op}\right\}>C_*\right\}$.
For $N_T:=\left\lfloor\frac{Td}{\alpha_d}\right\rfloor$, set
\begin{equation}
\mathcal E_{\rm sp}^{(d)}(T):=\left\{J_{\rm sp}>J_{\rm mac}\wedge N_T\right\}.
\label{eq:untied-slow-spectral-event}
\end{equation}
By Lemma \ref{lemma:spectral_bounds_untied},
\begin{equation}
\mathbb P\left(\mathcal E_{\rm sp}^{(d)}(T)
\right)\longrightarrow1.
\label{eq:untied-slow-spectral-event-prob}
\end{equation}
For \(0\le n\le N_T\), define the auxiliary weights
\begin{equation}
\check U_k^{(n)}:=\begin{cases}
U_k^{(n\wedge J_{\rm mac})},
&\text{on }\mathcal E_{\rm sp}^{(d)}(T),
\\
U_k^{(0)},
&\text{on }\mathcal E_{\rm sp}^{(d)}(T)^c,
\end{cases}
\qquad
\check V_k^{(n)}
:=
\begin{cases}
V_k^{(n\wedge J_{\rm mac})},
&\text{on }\mathcal E_{\rm sp}^{(d)}(T),
\\
V_k^{(0)},
&\text{on }\mathcal E_{\rm sp}^{(d)}(T)^c.
\end{cases}
\label{eq:untied-auxiliary-weights}
\end{equation}
Let $\check M_k^{(n)}=\check U_k^{(n)}(\check V_k^{(n)})^T$, $\check A_k^{(n)}=\check U_k^{(n)}(\check U_k^{(n)})^T$, $\check B_k^{(n)}=\check V_k^{(n)}(\check V_k^{(n)})^T$,
and define the corresponding auxiliary moments
$\check\mu_\alpha^{(n)}:=\frac1d
\Tr\left[\check Z_1^{(n)}\cdots\check Z_w^{(n)}\right]$, $\alpha=(Z_1,\ldots,Z_w)\in\mathscr Z^w$.
Their interpolation is denoted by
\begin{equation}
\check\mu^{(d)}(\tau)=\check\mu^{(n)},\qquad\tau\in\left[\frac{n\alpha_d}{d},\frac{(n+1)\alpha_d}{d}\right).
\end{equation}

\begin{lemma}
\label{lemma:untied-slow-compactness}
Under the conditions of Theorem \ref{theo:untied-SGD}, for every fixed
\(T>0\) and every \(\rho>C_*^2\), the sequence
$\left\{\check\mu^{(d)}(\cdot)
\right\}_{d\ge1}$ is \(C\)-tight in $D([0,T];X_\rho^0)$. Moreover, for every word \(\alpha\) of length \(w\), $\sup_{\tau\in[0,T]}
|\check\mu_\alpha^{(d)}(\tau)|\le C_*^{2w}$ almost surely, and the auxiliary process agrees with the empirical process stopped at
\(J_{\rm mac}\) with probability tending to one.
\end{lemma}
\begin{proof}
By construction and
\eqref{eq:untied-slow-spectral-event}, almost surely
\begin{equation}
\sup_{0\le n\le N_T}\max_{1\le k\le K}\left\{\|\check U_k^{(n)}\|_{\rm op},
\|\check V_k^{(n)}\|_{\rm op}\right\}\le C_*.
\label{eq:untied-auxiliary-uniform-spectral-bound}
\end{equation}
Hence every letter $\check Z_k^{(n)}
\in\left\{\check M_k^{(n)},
(\check M_k^{(n)})^T,\check A_k^{(n)},\check B_k^{(n)}\right\}$
satisfies $\|\check Z_k^{(n)}\|_{\rm op}\le C_*^2$. Therefore, for every word \(|\alpha|=w\),
$|\check\mu_\alpha^{(d)}(\tau)|
\le C_*^{2w}$. In particular, $\check\mu^{(d)}(\tau)\in X_\rho^0$ for $0\le\tau\le T$ and every \(\rho>C_*^2\).

We next record a uniform bound on the finite-\(d\) effective vector field. Before the stopping time, Lemma \ref{lemma:finite-dimensional-Stein-untied} gives $\max_{k,l,r}
|a_{kl}^{(r,d,n)}|\le C$.
Moreover, $|\psi_k^{(d,n)}|\le C
\left(t_k^{A,(n)}+t_k^{B,(n)}\right)$, and hence, under the convention
\(\beta_k^{(d,n)}=0\) when
\(t_k^{A,(n)}+t_k^{B,(n)}=0\),  $|\beta_k^{(d,n)}|\le C$.
Consequently, there exists a constant
\(B<\infty\), independent of \(d,n,w\), such that
\begin{equation}
\mathbf 1_{\{n<J_{\rm sp}\wedge J_{\rm mac}\}}
\sup_{|\alpha|=w}\left|[\mathcal V_{\rm eff}^{(d,n)}]_\alpha\right|\le BwC_*^{2w}.
\label{eq:untied-effective-vector-field-coordinate-bound}
\end{equation}

We now use Lemma
\ref{lemma:untied-approximate-integral-equation}. Set $J:=J_{\rm sp}\wedge J_{\rm mac}$. On \(\mathcal E_{\rm sp}^{(d)}(T)\), stopping at \(J\) agrees with stopping at \(J_{\rm mac}\) throughout the time interval
\([0,T]\). Therefore, for every fixed word \(\alpha\), the auxiliary process admits the representation
\begin{equation}
\check\mu_\alpha^{(d)}(\tau)=\check\mu_\alpha^{(d)}(0)+\mathbf 1_{\mathcal E_{\rm sp}^{(d)}(T)}\int_0^\tau[\widetilde{\mathcal V}_{\rm eff,J}^{(d)}(s)]_\alpha\,ds+\check{\mathcal E}_\alpha^{(d)}(\tau),
\label{eq:untied-auxiliary-approximate-integral}
\end{equation}
where
\begin{equation}
\sup_{\tau\in[0,T]}|\check{\mathcal E}_\alpha^{(d)}(\tau)|
\xrightarrow{\mathbb P}0
\label{eq:untied-auxiliary-coordinate-error}
\end{equation}
for every fixed \(\alpha\). On
\(\mathcal E_{\rm sp}^{(d)}(T)^c\), we set
\(\check{\mathcal E}_\alpha^{(d)}\equiv0\).

Define the continuous approximation
\begin{equation}
\hat\mu^{(d)}(\tau):=\check\mu^{(d)}(0)+\mathbf 1_{\mathcal E_{\rm sp}^{(d)}(T)}
\int_0^\tau\widetilde{\mathcal V}_{\rm eff,J}^{(d)}(s)\,ds.
\label{eq:untied-hat-mu-definition}
\end{equation}
The convergence in
\eqref{eq:untied-auxiliary-coordinate-error} can be upgraded to convergence in
\(X_\rho^0\), exactly as in the tied case (Lemma \ref{lemma:macro-tied-SGD}):
\begin{equation}
\sup_{\tau\in[0,T]}\left\|\check\mu^{(d)}(\tau)-\hat\mu^{(d)}(\tau)\right\|_\rho
\xrightarrow{\mathbb P}0.
\label{eq:untied-check-hat-equivalence}
\end{equation}

It remains to prove compactness of
\(\{\hat\mu^{(d)}\}\). By
\eqref{eq:untied-effective-vector-field-coordinate-bound},
for every \(|\alpha|=w\),
\begin{equation}
\sup_{\tau\in[0,T]}|\hat\mu_\alpha^{(d)}(\tau)|
\le(1+TBw)C_*^{2w}.
\label{eq:untied-hat-coordinate-bound}
\end{equation}
Furthermore, for \(0\le s<t\le T\),
\begin{align}
\left\|\hat\mu^{(d)}(t)-\hat\mu^{(d)}(s)
\right\|_\rho\le B|t-s|\sup_{w\ge1}w\left(\frac{C_*^2}{\rho}\right)^w=:L_\rho|t-s|,
\label{eq:untied-hat-equicontinuity}
\end{align}
where \(L_\rho<\infty\) is independent of \(d\).
The remainder of the argument is identical to the proof of Lemma
\ref{lemma:compactness} in the tied case. Namely, \eqref{eq:untied-hat-coordinate-bound} defines a compact subset of \(X_\rho^0\), because its tail decays exponentially, while
\eqref{eq:untied-hat-equicontinuity} gives uniform equicontinuity. The Arzel\`a--Ascoli theorem therefore implies that $\{\hat\mu^{(d)}\}_{d\ge1}$ is tight in
\(C([0,T];X_\rho^0)\). Together with
\eqref{eq:untied-check-hat-equivalence}, this proves that $\{\check\mu^{(d)}\}_{d\ge1}$ is \(C\)-tight in \(D([0,T];X_\rho^0)\).

Finally, on
\(\mathcal E_{\rm sp}^{(d)}(T)\), the auxiliary process is precisely the empirical moment process stopped at \(J_{\rm mac}\). Since
\(\mathbb P(\mathcal E_{\rm sp}^{(d)}(T))\to1\), the two processes are
asymptotically equivalent on \([0,T]\).
\end{proof}

\paragraph{Step 2: Convergence of the effective vector field}
\begin{lemma}
\label{lemma:untied-effective-drift-convergence}
Under the conditions of Theorem \ref{theo:untied-SGD}, fix
\(0<T_1<T_2\le T\), and let \(J\) be a stopping time satisfying \eqref{eq:stopping-time}. Let \(C_*\) denote the corresponding spectral bound in Lemma \ref{lemma:spectral_bounds_untied} and fix \(\rho>C_*^2\).
Let \(\mathcal V_{\rm eff}^{(d,n)}\) be the finite-\(d\) effective vector field defined in \eqref{eq:untied-effective-vector-field-finite}, and let \(V_{\rm ext}\) be the extended limiting vector field defined in \eqref{eq:untied-extended-vector-field}. Then,
\begin{equation}
\sup_{\tau\in[T_1,T_2]}\mathbf 1_{\{m_d(\tau)<J\}}\left\|\mathcal V_{\rm eff}^{(d,m_d(\tau))}-V_{\rm ext}(\mu^{(m_d(\tau))})\right\|_\rho
\xrightarrow{\mathbb P}0.
\label{eq:untied-effective-drift-convergence-time}
\end{equation}
\end{lemma}

\begin{proof}
Throughout the proof, all estimates are uniform on \eqref{eq:stopping-time}. Write
$q^{(n)}:=\left(q^{(1,n)},q^{(2,n)}\right)$ and $m_\star^{(n)}:=m^\star(q^{(n)})$,
where the non-trainable components of \(m_\star^{(n)}\) remain fixed at their initial values.

For notational convenience, define
\begin{equation}
A_{kl}^{(r)}(m,q):=\left[\nabla_{q^{(r)}}\Phi(m,q)
\right]_{kl},\qquad r=1,2.
\label{eq:untied-limiting-a-coefficient}
\end{equation}
An empirical state \(q^{(n)}\in\mathfrak Q\). This is because for arbitrary \(c=(c_1,\ldots,c_K)\), we can write $M_c:=\sum_{k=1}^Kc_kM_k$ and have
\begin{equation}
c^T(q^{(1)}+q^{(2)})c=\frac1d\left(\Tr[M_cM_c^T]+\Tr[M_c^2]\right)=\frac{2}{d}\|\operatorname{Sym}M_c\|_F^2\ge0,
\end{equation}
Similarly
\begin{equation}
c^T(q^{(1)}-q^{(2)})c=\frac{2}{d}\|\operatorname{Skew}M_c\|_F^2\ge0,
\end{equation}
where $\operatorname{Sym}M_c:=\frac{M_c+M_c^T}{2}$ and $\operatorname{Skew}M_c:=\frac{M_c-M_c^T}{2}$. Hence the projection in \(V_{\rm ext}\) (see Appendix \ref{app:untied-local-well-posedness}) acts trivially along the empirical trajectory:
\begin{equation}
\widehat q(\mu^{(n)})=q^{(n)}.
\label{eq:untied-empirical-q-projection-trivial}
\end{equation}

Recall
\begin{equation}
a_{kl}^{(r,d,n)}:=\frac12
\sum_{i,j,i',j'=1}^L\mathcal C_{ij,i'j'}^{(r)}
\mathbb E_n\left[\partial^2_{(ijk),(i'j'l)}\mathcal L(G^{(n)})\right].
\end{equation}
We first prove its convergence to \(A_{kl}^{(r)}(m^{(n)},q^{(n)})\).

Before the stopping time \eqref{eq:stopping-time}, $\max_k\|M_k^{(n)}\|_{\rm op}\le C_*^2$,
while compactness of \(\mathfrak D\) gives
\begin{equation}
\max_k|m_k^{(n)}|=\max_k\frac1{\sqrt d}|\Tr M_k^{(n)}|\le R_*
\end{equation}
for a constant \(R_*<\infty\). Therefore Lemma
\ref{lemma:quantitative-gaussian-approximation} applies uniformly in \(n<J\).

For fixed indices \((ijk),(i'j'l)\), apply Lemma
\ref{lemma:quantitative-gaussian-approximation} to $\partial^2_{(ijk),(i'j'l)}\mathcal L(G)$.
Assumption \ref{assum:untied}.1 guarantees the required regularity and polynomial growth. By Price's theorem,
\begin{equation}
A_{kl}^{(r)}(m^{(n)},q^{(n)})=\frac12\sum_{i,j,i',j'}\mathcal C_{ij,i'j'}^{(r)}\mathbb E\left[\partial^2_{(ijk),(i'j'l)}\mathcal L(G_{\rm G}^{(n)})\right],
\end{equation}
where \(G_{\rm G}^{(n)}\) is Gaussian with the same mean and covariance as \(G^{(n)}\). Consequently,
\begin{equation}
\sup_{n<J}\max_{k,l,r}\left|a_{kl}^{(r,d,n)}-A_{kl}^{(r)}(m^{(n)},q^{(n)})
\right|\le\frac{C}{\sqrt d}.
\label{eq:untied-a-gaussianization}
\end{equation}
We next replace \(m^{(n)}\) by \(m_\star^{(n)}\). By Lemma
\ref{lemma:finite-dimensional-Stein-untied},
\begin{equation}
\left\|b_{\mathcal K}^{(n)}-\nabla_{m_{\mathcal K}}\Phi(m^{(n)},q^{(n)})
\right\|\le\frac{C}{\sqrt d}.
\label{eq:untied-drift-b-gradient-comparison}
\end{equation}
Lemma \ref{lemma:local-mean-gradient-refinement} gives, with
probability tending to one,
\begin{equation}
\|b_{\mathcal K}^{(n)}\|
\le C_{\rm tr}e^{-c_{\rm tr}\alpha_dn}+\frac{C_b}{\sqrt d},
\qquad0\le n<N_T\wedge J.
\label{eq:untied-drift-refined-b}
\end{equation}
Hence, uniformly for $n\ge\left\lceil\frac{T_1d}{\alpha_d}\right\rceil$, we have $e^{-c_{\rm tr}\alpha_dn}\le e^{-c_{\rm tr}T_1d}$,
and therefore
\begin{equation}
\sup_{\substack{\lceil T_1d/\alpha_d\rceil\le n\le\lfloor T_2d/\alpha_d\rfloor\\
n<J}}\left\|\nabla_{m_{\mathcal K}}
\Phi(m^{(n)},q^{(n)})\right\|=O_{\mathbb P}(d^{-1/2}).
\label{eq:untied-drift-gradient-small}
\end{equation}
Since $\nabla_{m_{\mathcal K}}
\Phi(m_\star^{(n)},q^{(n)})=0$, the uniform strong convexity in Assumption \ref{assum:untied_convex} implies
\begin{equation}
\lambda_0\left\|m_{\mathcal K}^{(n)}-m_{\mathcal K}^\star(q^{(n)})
\right\|\le\left\|\nabla_{m_{\mathcal K}}
\Phi(m^{(n)},q^{(n)})\right\|.
\label{eq:untied-drift-strong-convex-distance}
\end{equation}
Thus
\begin{equation}
\sup_{\substack{\lceil T_1d/\alpha_d\rceil\le n\le\lfloor T_2d/\alpha_d\rfloor\\
n<J}}\left\|m^{(n)}-m_\star^{(n)}\right\|=O_{\mathbb P}(d^{-1/2}).
\label{eq:untied-drift-m-tracking}
\end{equation}

The map $(m,q)\longmapsto
\nabla_{q^{(r)}}\Phi(m,q)$
is Lipschitz on a compact neighborhood of \eqref{eq:stopping-time} and of
\(\{(m^\star(q),q):q\in\operatorname{proj}_q\mathfrak D\}\).
Therefore,
\begin{equation}
\sup_{\substack{\lceil T_1d/\alpha_d\rceil\le n\le\lfloor T_2d/\alpha_d\rfloor\\
n<J}}
\max_{k,l,r}\left|A_{kl}^{(r)}(m^{(n)},q^{(n)})-A_{kl}^{(r)}(m_\star^{(n)},q^{(n)})
\right|\xrightarrow{\mathbb P}0.
\label{eq:untied-a-slaving}
\end{equation}
Combining
\eqref{eq:untied-a-gaussianization} and
\eqref{eq:untied-a-slaving}, we conclude that
\begin{equation}
\Delta_{a,d}:=\sup_{\substack{
\lceil T_1d/\alpha_d\rceil\le n\le
\lfloor T_2d/\alpha_d\rfloor\\
n<J}}\max_{k,l,r}\left|a_{kl}^{(r,d,n)}-A_{kl}^{(r)}(m_\star^{(n)},q^{(n)})
\right|\xrightarrow{\mathbb P}0.
\label{eq:untied-a-uniform-convergence}
\end{equation}

For \(k\in\mathcal K\), recall $\psi_k^{(d,n)}=2\sum_{l=1}^K\sum_{r=1}^2
a_{kl}^{(r,d,n)}\left[\mu_{(M_l^{(r)},B_k)}^{(n)}+\mu_{(A_k,M_l^{(r)})}^{(n)}\right]$
and $\beta_k^{(d,n)}=-\frac{\psi_k^{(d,n)}}{D_k^{(n)}}$, $D_k^{(n)}:=t_k^{A,(n)}+t_k^{B,(n)}$. Now we prove the convergence of $\beta_k^{(d,n)}$ to $\beta_k^{\rm ext}$

Along the empirical trajectory,
\eqref{eq:untied-empirical-q-projection-trivial} implies 
\begin{align}
\left|\psi_k^{(d,n)}-\psi_k^{\rm ext}(\mu^{(n)})
\right|\le2\Delta_{a,d}\sum_{l=1}^K\sum_{r=1}^2\left(\left|\mu_{(M_l^{(r)},B_k)}^{(n)}\right|+\left|\mu_{(A_k,M_l^{(r)})}^{(n)}\right|\right).
\label{eq:untied-psi-difference-pre}
\end{align}

The spectral bound \eqref{eq:stopping-time} gives $\|M_l^{(r,n)}\|_{\rm op}\le C_*^2$.
Since \(A_k,B_k\succeq0\), we have $\left|
\mu_{(M_l^{(r)},B_k)}^{(n)}\right|\le
C_*^2 t_k^{B,(n)}$ and $\left|\mu_{(A_k,M_l^{(r)})}^{(n)}
\right|\le C_*^2 t_k^{A,(n)}$. Therefore
\begin{equation}
\left|\psi_k^{(d,n)}-\psi_k^{\rm ext}(\mu^{(n)})
\right|\le C\Delta_{a,d}D_k^{(n)}.
\label{eq:untied-psi-difference}
\end{equation}
Whenever \(D_k^{(n)}>0\), division by \(D_k^{(n)}\) yields
\begin{equation}
\left|\beta_k^{(d,n)}-\beta_k^{\rm ext}(\mu^{(n)})\right|\le C\Delta_{a,d}.
\label{eq:untied-beta-difference}
\end{equation}
By Lemma \ref{lemma:empirical-balancing},
\(D_k^{(n)}\) is uniformly bounded away from zero before \(J\), for all \(k\in\mathcal K\), with probability tending to one. Hence
\eqref{eq:untied-a-uniform-convergence} implies
\begin{equation}
\Delta_{\beta,d}:=\sup_{\substack{\lceil T_1d/\alpha_d\rceil\le n\le\lfloor T_2d/\alpha_d\rfloor\\n<J}}\max_{k\in\mathcal K}\left|\beta_k^{(d,n)}-\beta_k^{\rm ext}(\mu^{(n)})\right|\xrightarrow{\mathbb P}0.
\label{eq:untied-beta-uniform-convergence}
\end{equation}

Finally, fix \(n<J\) and a word
\(\alpha=(Z_1,\ldots,Z_w)\). The weight-decay terms in \(\mathcal V_{\rm eff}^{(d,n)}\) and
\(V_{\rm ext}(\mu^{(n)})\) coincide exactly. Thus their difference
contains only $\mathcal{T}^{\text{bias}}$ and $\mathcal{T}^{\text{diff}}$ terms.
Under the spectral bounds \eqref{eq:stopping-time}, for every trainable position \(i\), $\left|\mathcal T_i^{\rm bias}(\mu^{(n)})
\right|\le2C_*^{2w}$ and $\left|
\mathcal T_{i,l,r}^{\rm diff}(\mu^{(n)})\right|\le2C_*^{2(w+1)}$.
Consequently,
\begin{align}
\left|[\mathcal V_{\rm eff}^{(d,n)}-V_{\rm ext}(\mu^{(n)})]_\alpha\right|\le Cw
\left(\Delta_{\beta,d}C_*^{2w}+\Delta_{a,d}C_*^{2(w+1)}\right),
\label{eq:untied-drift-coordinate-difference}
\end{align}
where \(C\) depends only on \(K\).

Dividing by \(\rho^w\) and taking the supremum over all words of degree \(w\ge1\) gives
\begin{align}
\left\|
\mathcal V_{\rm eff}^{(d,n)}-V_{\rm ext}(\mu^{(n)})\right\|_\rho
&\le C\left(\Delta_{\beta,d}+C_*^2\Delta_{a,d}
\right)\sup_{w\ge1}w\left(\frac{C_*^2}{\rho}
\right)^w.
\label{eq:untied-drift-Xrho-bound}
\end{align}
Because \(\rho>C_*^2\) and $\sup_{w\ge1}w\left(\frac{C_*^2}{\rho}
\right)^w<\infty$.
Equations
\eqref{eq:untied-a-uniform-convergence} and
\eqref{eq:untied-beta-uniform-convergence} therefore imply
\begin{equation}
\sup_{\substack{\lceil T_1d/\alpha_d\rceil\le n\le\lfloor T_2d/\alpha_d\rfloor\\n<J}}\left\|\mathcal V_{\rm eff}^{(d,n)}-V_{\rm ext}(\mu^{(n)})
\right\|_\rho\xrightarrow{\mathbb P}0.
\label{eq:untied-effective-drift-convergence},
\end{equation}
which proves \eqref{eq:untied-effective-drift-convergence-time}.
\end{proof}

\paragraph{Step 3: Global well-posedness}
We call a pair \((\bar m,\bar\mu)\) an admissible solution on \([0,T]\) if there exist
\(0<s_-<s_+<\infty\) such that
$\bar\mu\in C([0,T];X_{s_-}^0)$,
\begin{equation}
\bar\mu(\tau)=\bar\mu(0)+\int_0^\tau
V_{\rm ext}(\bar\mu(s))\,ds
\label{eq:integral-equation-untied}
\end{equation}
holds as an identity in \(X_{s_+}^0\), and, for every
\(\tau\in[0,T]\),
\begin{equation}
q(\bar\mu(\tau))\in\mathfrak Q,
\qquad
\min_{k\in\mathcal K}
\left(\bar t_k^A(\tau)+\bar t_k^B(\tau)\right)>0,
\label{eq:untied-admissibility}
\end{equation}
with
\begin{equation}
\bar m_{\mathcal K}(\tau)=m_{\mathcal K}^\star(q(\bar\mu(\tau))),
\qquad
\bar m_{\setminus\mathcal K}(\tau)=\bar m_{\setminus\mathcal K}(0).
\label{eq:untied-admissible-slaving}
\end{equation}

\begin{lemma}
\label{lemma:untied-global-slow-dynamics}
Under the conditions of Theorem \ref{theo:untied-SGD}, the extended slow system \eqref{eq:untied-extended-slow-system}, initialized at
\(\bar\mu(0)\), admits a unique admissible solution
$\bar\mu\in C([0,T];X_{\rho_-}^0)$ for every pair
$C_*^2<\rho_-<\rho_+$, where the integral equation \eqref{eq:integral-equation-untied} is understood in \(X_{\rho_+}^0\). Moreover, \(\bar\mu\) also satisfies the original ODE \eqref{eq:untied-slaving}--\eqref{eq:slow_tau_mu}.
\end{lemma}

\begin{proof}
Fix $C_*^2<\rho_-<\rho_+$.
By Assumption \ref{assum:untied}.2, every initial letter satisfies
$\|M_k^{(0)}\|_{\rm op},\|A_k^{(0)}\|_{\rm op},
\|B_k^{(0)}\|_{\rm op}\le C_0^2$.
Hence, for every word \(|\alpha|=w\),
$|\mu_\alpha^{(d)}(0)|\le C_0^{2w}$.
Together with Assumption
\ref{assum:untied}.3 and \(\rho_->C_0^2\), the same tail argument used in the tied case (Lemma \ref{lemma:global_existence}) gives
\begin{equation}
\mu^{(d)}(0)
\xrightarrow{\mathbb P}
\bar\mu(0)
\qquad
\text{in }X_{\rho_-}^0.
\label{eq:untied-initial-Xrho-convergence}
\end{equation}

By Lemma \ref{lemma:untied-slow-compactness}, the auxiliary processes
\(\{\check\mu^{(d)}\}\) are \(C\)-tight in
\(D([0,T];X_{\rho_-}^0)\). Hence, from every subsequence, we may
extract a further subsequence, still indexed by \(d\), such that
\begin{equation}
\check\mu^{(d)}\to\mu^*\qquad
\text{in }D([0,T];X_{\rho_-}^0),
\label{eq:untied-global-subsequence}
\end{equation}
where $\mu^*\in C([0,T];X_{\rho_-}^0)$ almost surely.

As in the proof of Lemma \ref{lemma:global_existence} for the tied
model, by the Skorokhod representation theorem we may assume
\begin{equation}
\sup_{\tau\in[0,T]}
\|\check\mu^{(d)}(\tau)-\mu^*(\tau)\|_{\rho_-}
\longrightarrow0
\qquad\text{almost surely}.
\label{eq:untied-global-uniform-subsequence}
\end{equation}
Fix $0<\varepsilon<t\le T$ and first work before the stopping time \eqref{eq:stopping-time}. Lemma
\ref{lemma:untied-approximate-integral-equation}, together with Lemma
\ref{lemma:untied-slow-compactness}, gives
\begin{align}
\check\mu^{(d)}(t)-\check\mu^{(d)}(\varepsilon)=\int_\varepsilon^t
\mathcal V_{\rm eff}^{(d)}(s)\,ds+R_d(\varepsilon,t),
\label{eq:untied-global-effective-integral}
\end{align}
where
\begin{equation}
\sup_{\varepsilon\le s\le t}
\|R_d(\varepsilon,s)\|_{\rho_+}
\xrightarrow{\mathbb P}0.
\label{eq:untied-global-effective-error}
\end{equation}
By Lemma \ref{lemma:untied-effective-drift-convergence},
\begin{equation}
\sup_{s\in[\varepsilon,t]}
\left\|\mathcal V_{\rm eff}^{(d)}(s)-V_{\rm ext}(\check\mu^{(d)}(s))\right\|_{\rho_+}\xrightarrow{\mathbb P}0.
\label{eq:untied-global-drift-replacement}
\end{equation}
The local Lipschitz estimate from Lemma
\ref{lemma:local-untied-well-posed} implies, on every bounded subset,
\begin{equation}
\|V_{\rm ext}(\mu)-V_{\rm ext}(\nu)\|_{\rho_+}
\le\frac{L}{\rho_+-\rho_-}\|\mu-\nu\|_{\rho_-}.
\label{eq:untied-global-V-Lipschitz}
\end{equation}
Hence
\eqref{eq:untied-global-uniform-subsequence} gives
\begin{equation}
\sup_{s\in[\varepsilon,t]}
\left\|V_{\rm ext}(\check\mu^{(d)}(s))-V_{\rm ext}(\mu^*(s))\right\|_{\rho_+}\longrightarrow0.
\label{eq:untied-global-V-pass-limit}
\end{equation}
Passing to the limit in
\eqref{eq:untied-global-effective-integral} yields
\begin{equation}
\mu^*(t)-\mu^*(\varepsilon)=\int_\varepsilon^t
V_{\rm ext}(\mu^*(s))\,ds.
\label{eq:untied-global-limit-away-zero}
\end{equation}
The effective vector fields \(\mathcal V_{\rm eff}^{(d,n)}\) are uniformly bounded in \(X_{\rho_+}^0\) before the stopping time, as established in the proof of Lemma \ref{lemma:untied-slow-compactness}. Therefore,
letting \(\varepsilon\downarrow0\) in
\eqref{eq:untied-global-limit-away-zero} and using the continuity of \(\mu^*\) gives
\begin{equation}
\mu^*(t)=\bar\mu(0)+\int_0^tV_{\rm ext}(\mu^*(s))\,ds,\qquad
0\le t\le T.
\label{eq:untied-global-limit-integral-equation}
\end{equation}

For every finite \(d\), the empirical covariance belongs to \(\mathfrak Q\) (see the proof of Lemma \ref{lemma:untied-effective-drift-convergence}). Since \(\mathfrak Q\) is closed and convergence in
\(X_{\rho_-}^0\) implies convergence of the second-order coordinates,
\begin{equation}
q(\mu^*(\tau))\in\mathfrak Q,
\qquad
0\le\tau\le T.
\label{eq:untied-global-limit-covariance-cone}
\end{equation}
Next, Lemma \ref{lemma:empirical-balancing} gives a constant \(c_T>0\) such that, with probability tending to one, $t_k^{A,(n)}+t_k^{B,(n)}\ge c_T$ for every trainable \(k\) before the stopping time. Passing to the limit gives
\begin{equation}
t_k^A(\mu^*(\tau))+t_k^B(\mu^*(\tau))
\ge c_T,
\qquad
k\in\mathcal K,\quad0\le\tau\le T.
\label{eq:untied-global-limit-D-lower-bound}
\end{equation}
Finally, Lemma \ref{lemma:untied-slow-compactness} gives $|\check\mu_\alpha^{(d)}(\tau)|
\le C_*^{2|\alpha|}$. Passing to the limit yields
\begin{equation}
|\mu_\alpha^*(\tau)|
\le
C_*^{2|\alpha|}.
\label{eq:untied-global-limit-moment-bound}
\end{equation}
Thus every subsequential limit is admissible. In particular, $\mathscr P(q(\mu^*(\tau)))=q(\mu^*(\tau))$,
so along the limiting trajectory $\widetilde m_{\mathcal K}^\star(\mu^*(\tau))=m_{\mathcal K}^\star(q(\mu^*(\tau)))$.
Hence \eqref{eq:untied-global-limit-integral-equation} is precisely the
original system \eqref{eq:untied-slaving}--\eqref{eq:slow_tau_mu}.

We now prove uniqueness. Suppose that \(\mu_1\) and \(\mu_2\) are two admissible solutions with the same initial condition. Both solve the
extended ODE. By Lemma \ref{lemma:local-untied-well-posed}, if $\mu_1(\tau_0)=\mu_2(\tau_0)$, at some \(\tau_0<T\), then the two trajectories agree on a nontrivial interval to the right of \(\tau_0\). The standard contradiction argument shows that
$\mu_1(\tau)=\mu_2(\tau)$ for $0\le\tau\le T$. Thus the admissible solution is unique. We denote it by
\(\bar\mu\).

It remains to verify that the stopping time \(J_{\rm mac}\) is asymptotically inactive. Recall that \(\mathfrak D\Subset\mathcal M\times\mathcal Q\) is chosen to contain both the compact fast-time set $\mathcal C_{\rm fast} := \left\{ (m_{\mathcal K},q_0): \Phi(m_{\mathcal K},m_{\setminus\mathcal K}(0),q_0) \le \Phi(\widetilde m(0),q_0) \right\}$ and the compact slow trajectory
$\mathcal C_{\rm slow} := \left\{ (m^\star_{\mathcal K}(\bar q(\tau)),\bar q(\tau)): 0\le\tau\le T \right\} $
in its interior. Hence
$\delta:= \operatorname{dist} (\mathcal C_{\rm fast}\cup\mathcal C_{\rm slow}, \partial\mathfrak D)>0$.

By C-tightness (Lemma \ref{lemma:untied-slow-compactness}) and uniqueness of the subsequential limit, the stopped moment trajectory converges uniformly in probability to \(\bar\mu\). In particular,
\begin{equation}
\sup_{0\le n\le N_T\wedge J_{\rm mac}} \|q^{(n)}-\bar q(n\alpha_d/d)\| \xrightarrow{\mathbb P}0.
\end{equation}
Moreover, for every fixed \(T_1>0\), Lemma \ref{lemma:local-mean-gradient-refinement} and the strong convexity in Assumption \ref{assum:untied_convex} imply
\begin{equation}
\sup_{T_1d/\alpha_d\le n\le N_T\wedge J_{\rm mac}} \|m_{\mathcal K}^{(n)} -m^\star_{\mathcal K}(q^{(n)})\| \xrightarrow{\mathbb P}0.
\label{eq:untied-global-mac-tracking}
\end{equation}
Since \(q\mapsto m^\star_{\mathcal K}(q)\) is Lipschitz on the compact subset of \(\mathcal Q\), it follows that
\begin{equation}
\sup_{T_1d/\alpha_d\le n\le N_T\wedge J_{\rm mac}} \left\| (m_{\mathcal K}^{(n)},q^{(n)}) - (m^\star_{\mathcal K}(\bar q(\tau_n)),\bar q(\tau_n)) \right\| \xrightarrow{\mathbb P}0,
\end{equation}
where $\tau_n=\frac{n\alpha_d}{d}$. Therefore, with probability tending to one, the trajectory stays at least \(\delta/2\) away from \(\partial\mathfrak D\) for \(T_1d/\alpha_d\le n\le N_T\).

It remains to control the initial regime and the intermediate regime. Choose a fixed fast-time horizon \(t_{\rm f}>0\) sufficiently large that
$C_{\rm loc}e^{-c_{\rm loc}t_{\rm f}}<\delta/8$,
where \(C_{\rm loc},c_{\rm loc}\) are the constants in Lemma \ref{lemma:coarse-localization-untied}. On \(0\le n\le t_{\rm f}/\alpha_d\), the fast-time convergence and the freezing of the slow coordinates proved in Appendix \ref{app:proof-untied-fast-convergence} imply that the empirical trajectory remains uniformly close to the fast limiting trajectory. Since the latter stays in \(\mathcal C_{\rm fast}\), this part of the trajectory remains at least a $\delta/4$ distance from \(\partial\mathfrak D\) with probability tending to one.

For \(t_{\rm f}/\alpha_d\le n\le T_1d/\alpha_d\), Lemma \ref{lemma:coarse-localization-untied} yields
\begin{equation}
\|m_{\mathcal K}^{(n)} -m^\star_{\mathcal K}(q^{(n)})\| \le C_{\rm loc}e^{-c_{\rm loc}t_{\rm f}} +o_{\mathbb P}(1)
\end{equation}
uniformly before \(J_{\rm mac}\). Together with the uniform convergence of \(q^{(n)}\) to \(\bar q(n\alpha_d/d)\), this places the empirical state in a \(\delta/4\)-neighborhood of \(\mathcal C_{\rm slow}\) throughout the intermediate regime. Combining the three time intervals gives
\begin{equation}
\mathbb P(J_{\rm mac}\le N_T)\longrightarrow0.
\end{equation}
This completes the proof.
\end{proof}

\subsubsection{Final proof of Theorem \ref{theo:untied-SGD}}
The fast-time convergence \eqref{eq:untied-fast-convergence} was established in Appendix \ref{app:proof-untied-fast-convergence}. It remains to prove the convergence on the slow time scale.

Fix $0<T_1<T_2\le T$. Choose $C_*^2<\rho_-<\rho_+$,
where \(C_*\) is the spectral constant appearing in Lemma \ref{lemma:spectral_bounds_untied}.

By Lemma \ref{lemma:untied-slow-compactness}, the auxiliary  moment trajectories $\{\check\mu^{(d)}\}_{d\ge1}$ are \(C\)-tight in $D([0,T_2];X_{\rho_-}^0)$. Lemma \ref{lemma:untied-global-slow-dynamics} shows that every
subsequential weak limit is almost surely equal to the same deterministic admissible solution
$\bar\mu\in C([0,T_2];X_{\rho_-}^0)$ of the slow system \eqref{eq:untied-slaving}--\eqref{eq:slow_tau_mu}. Consequently,
\begin{equation}
\check\mu^{(d)}\to\bar\mu
\qquad\text{in }D([0,T_2];X_{\rho_-}^0).
\label{eq:untied-final-auxiliary-weak-convergence}
\end{equation}
Similarly to the tied case, this implies convergence in the uniform topology. Hence
\begin{equation}
\sup_{\tau\in[0,T_2]}
\left\|\check\mu^{(d)}(\tau)-\bar\mu(\tau)\right\|_{\rho_-}\xrightarrow{\mathbb P}0.
\label{eq:untied-final-auxiliary-uniform}
\end{equation}
In particular, for every fixed word \(\alpha\),
\begin{equation}
\sup_{\tau\in[0,T_2]}
\left|\check\mu_\alpha^{(d)}(\tau)-\bar\mu_\alpha(\tau)\right|\le\rho_-^{|\alpha|}
\sup_{\tau\in[0,T_2]}\left\|\check\mu^{(d)}(\tau)-\bar\mu(\tau)\right\|_{\rho_-},
\end{equation}
and therefore
\begin{equation}
\sup_{\tau\in[0,T_2]}\left|\check\mu_\alpha^{(d)}(\tau)-\bar\mu_\alpha(\tau)\right|\xrightarrow{\mathbb P}0.
\label{eq:untied-final-auxiliary-coordinate}
\end{equation}
We now transfer this convergence to the SGD trajectory. Lemma \ref{lemma:untied-global-slow-dynamics} gives
$\mathbb P(J_{\rm mac}\le N_{T_2})\longrightarrow0$, where $N_{T_2}:=
\left\lfloor\frac{T_2d}{\alpha_d}\right\rfloor$,
while Lemma \ref{lemma:spectral_bounds_untied} gives
$\mathbb P\left(J_{\rm sp}\le J_{\rm mac}\wedge N_{T_2}\right)\longrightarrow0$.
Thus the event
\begin{equation}
\mathcal G_d:=\left\{J_{\rm mac}>N_{T_2},\ 
J_{\rm sp}>N_{T_2}\right\}
\label{eq:untied-final-good-event}
\end{equation}
satisfies $\mathbb P(\mathcal G_d)\longrightarrow1$.
On \(\mathcal G_d\), the auxiliary trajectory coincides with the original SGD trajectory throughout \([0,T_2]\). Hence, for every fixed \(\alpha\) and every \(\varepsilon>0\),
\begin{align}
\mathbb P\left(\sup_{\tau\in[0,T_2]}
\left|\tilde\mu_\alpha^{(d)}(\tau)-\bar\mu_\alpha(\tau)\right|>\varepsilon\right)
\le\mathbb P(\mathcal G_d^c)+\mathbb P\left(
\sup_{\tau\in[0,T_2]}\left|\check\mu_\alpha^{(d)}(\tau)-\bar\mu_\alpha(\tau)\right|>\varepsilon\right)\longrightarrow0.
\end{align}
This proves \eqref{eq:untied-slow-moment-convergence}.

It remains to establish the convergence of the mean variables on \([T_1,T_2]\). \eqref{eq:untied-slow-moment-convergence} implies
\begin{equation}
\sup_{\tau\in[0,T_2]}
\left(\|q^{(1,d)}(\tau)-\bar q^{(1)}(\tau)\|_F+\|q^{(2,d)}(\tau)-\bar q^{(2)}(\tau)\|_F
\right)\xrightarrow{\mathbb P}0.
\label{eq:untied-final-q-convergence}
\end{equation}
On the other hand, recall that (see \eqref{eq:untied-global-mac-tracking})
\begin{equation}
\sup_{\substack{\lceil T_1d/\alpha_d\rceil
\le n\le\lfloor T_2d/\alpha_d\rfloor}}
\left\|m_{\mathcal K}^{(n)}-m_{\mathcal K}^\star(q^{(n)})\right\|
\xrightarrow{\mathbb P}0.
\label{eq:untied-final-m-tracking}
\end{equation}
By Assumption \ref{assum:untied_convex} and the implicit function theorem, the map $q\longmapsto m_{\mathcal K}^\star(q)$ is Lipschitz on a compact neighborhood of the limiting covariance trajectory. Combining this with
\eqref{eq:untied-final-q-convergence} gives
\begin{equation}
\sup_{\tau\in[T_1,T_2]}\left\|m_{\mathcal K}^\star(q^{(d)}(\tau))-m_{\mathcal K}^\star(\bar q(\tau))\right\|\xrightarrow{\mathbb P}0.
\label{eq:untied-final-mstar-convergence}
\end{equation}
Since $\bar m_{\mathcal K}(\tau)=m_{\mathcal K}^\star(\bar q(\tau))$, \eqref{eq:untied-final-m-tracking} and \eqref{eq:untied-final-mstar-convergence} imply
\begin{equation}
\sup_{\tau\in[T_1,T_2]}\left\|\tilde m^{(d)}(\tau d)-\bar m(\tau)\right\|_\infty\xrightarrow{\mathbb P}0.
\end{equation}
This proves
\eqref{eq:untied-slow-m-convergence} and completes the proof of Theorem \ref{theo:untied-SGD}.

\subsection{Truncation to finite orders}
\label{app:untied-truncation}

We now show that the infinite-dimensional dynamics can be approximated by a finite-order truncation. A truncated moment sequence need not correspond to actual matrices, and in particular its covariance
parameters need not belong to the admissible cone \(\mathfrak Q\). For this reason, the truncated system is defined using the extended vector field \(V_{\rm ext}\) introduced in \eqref{eq:untied-extended-vector-field}.

For \(M\ge2\), define the projection
\begin{equation}
[\Pi_M\mu]_\alpha:=\begin{cases}
\mu_\alpha,&|\alpha|\le M,\\
0,&|\alpha|>M.
\end{cases}
\label{eq:untied-truncation-projection}
\end{equation}
Let $Y_M:=\left\{\mu:\mu_\alpha=0\text{ for every }|\alpha|>M\right\}$. Recall that \(V_{\rm ext}\) is well defined on the open set
\begin{equation}
\mathcal U_s=\left\{\mu\in X_s:\mathscr P(q(\mu))\in\mathcal Q,\quad
\min_{k\in\mathcal K}\left(t_k^A(\mu)+t_k^B(\mu)\right)>0\right\},
\end{equation}
where \(\mathscr P\) denotes the projection onto the covariance cone \(\mathfrak Q\).

We first prove a stability property of the extended slow flow \eqref{eq:untied-extended-slow-system}.

\begin{lemma}
\label{lemma:untied-truncation-stability}
Let \(\bar\mu\) be the admissible slow trajectory in
Lemma \ref{lemma:untied-global-slow-dynamics}, and fix \(T>0\). Let \(C_*\) be the spectral constant in Lemma \ref{lemma:spectral_bounds_untied} and set
$R:=C_*^2$. Then there exist constants
$R<s_-<s_+<\infty$, $\varepsilon_T>0$ and 
$C_T<\infty$ such that the following holds.

If \(\nu_0\in X_{s_-}^0\) satisfies
\begin{equation}
\|\nu_0-\bar\mu(0)\|_{s_-}\le\varepsilon_T,
\label{eq:untied-truncation-close-initial}
\end{equation}
then the extended ODE
\begin{equation}
\frac{d\nu}{d\tau}=V_{\rm ext}(\nu),
\qquad
\nu(0)=\nu_0,
\label{eq:untied-truncation-perturbed-flow}
\end{equation}
has a unique solution on \([0,T]\), remains in $\mathcal U_s$, and satisfies
\begin{equation}
\sup_{0\le\tau\le T}\|\nu(\tau)-\bar\mu(\tau)\|_{s_+}\le C_T\|\nu_0-\bar\mu(0)\|_{s_-}.
\label{eq:untied-truncation-stability-bound}
\end{equation}
\end{lemma}

\begin{proof}
By Lemma \ref{lemma:untied-global-slow-dynamics},
\begin{equation}
|\bar\mu_\alpha(\tau)|\le R^{|\alpha|},
\qquad0\le\tau\le T.\label{eq:untied-truncation-true-moment-bound}
\end{equation}
Moreover, admissibility and compactness of the time interval imply
\begin{equation}
d_T:=\min_{\substack{0\le\tau\le T\\k\in\mathcal K}}
\left(\bar t_k^A(\tau)+\bar t_k^B(\tau)\right)>0.
\label{eq:untied-truncation-D-lower}
\end{equation}
The covariance trajectory
\begin{equation}
\mathcal K_q:=\left\{\bar q(\tau):0\le\tau\le T\right\}
\label{eq:untied-truncation-q-compact}
\end{equation}
is a compact subset of \(\mathcal Q\cap\mathfrak Q\). Since \(\mathcal Q\) is open, there exists
\(\delta_T>0\) such that
\begin{equation}
\left\{q'\in\mathfrak Q:\operatorname{dist}(q',\mathcal K_q)\le2\delta_T\right\}\Subset\mathcal Q.
\label{eq:untied-truncation-q-tube}
\end{equation}
Because \(\mathscr P\) is non-expansive, whenever
$\|q-\bar q(\tau)\|_F<\delta_T$ for some \(\tau\in[0,T]\), we also have
\begin{equation}
\|\mathscr P(q)-\bar q(\tau)\|_F<\delta_T,
\end{equation}
and hence \(\mathscr P(q)\in\mathcal Q\).

Choose any \(s>R\). By \eqref{eq:untied-truncation-true-moment-bound}, the trajectory \(\bar\mu([0,T])\) is bounded in \(X_s^0\).
We may therefore choose an open neighborhood
\(\mathcal N_T\) of this trajectory such that, for every \(\mu\in\mathcal N_T\),
\begin{equation}
\mathscr P(q(\mu))\in\mathcal Q,\qquad\min_{k\in\mathcal K}\left(t_k^A(\mu)+t_k^B(\mu)\right)\ge\frac{d_T}{2}.
\label{eq:untied-truncation-neighborhood}
\end{equation}
On \(\mathcal N_T\), the maps
\begin{equation}
\mu\longmapsto
m_{\mathcal K}^\star(\mathscr P(q(\mu))),
\qquad
\mu\longmapsto a_{kl}^{(r)}(\mu),
\qquad
\mu\longmapsto\beta_k^{\rm ext}(\mu)
\label{eq:untied-truncation-coefficient-maps}
\end{equation}
are uniformly bounded and locally Lipschitz. In particular, the denominator in \(\beta_k^{\rm ext}\) is uniformly bounded below by \(d_T/2\).

The same estimates used in Lemma \ref{lemma:local-untied-well-posed} therefore imply that for every bounded subset of \(\mathcal N_T\) and every \(0<s<s'\), there exists \(L_T<\infty\) such that
\begin{equation}
\|V_{\rm ext}(\mu)-V_{\rm ext}(\nu)\|_{s'}
\le\frac{L_T}{s'-s}\|\mu-\nu\|_s.
\label{eq:untied-truncation-Cauchy-estimate}
\end{equation}

The Ovsyannikov theorem, together with
\eqref{eq:untied-truncation-Cauchy-estimate}, gives local existence, uniqueness, and locally Lipschitz dependence on the initial condition in the Banach scale. Since \(\bar\mu([0,T])\) is compact and remains a positive distance from the boundary specified in \eqref{eq:untied-truncation-neighborhood}, the interval \([0,T]\) can be covered by finitely many such local existence intervals. Composing the corresponding local
stability estimates yields constants \(\varepsilon_T>0\) and \(C_T<\infty\) such that
\eqref{eq:untied-truncation-close-initial} implies existence on the whole interval \([0,T]\) and
\eqref{eq:untied-truncation-stability-bound}.

\eqref{eq:untied-truncation-stability-bound} guarantees that the perturbed trajectory remains inside \(\mathcal N_T\). Hence it cannot leave the domain of \(V_{\rm ext}\), which closes the continuation argument.
\end{proof}

We now define the degree-\(M\) truncated slow dynamics by
\begin{equation}
\frac{d}{d\tau}\mu_\alpha^{(M)}(\tau)=[V_{\rm ext}(\mu^{(M)}(\tau))]_\alpha,
\qquad|\alpha|\le M,
\label{eq:truncated-untied}
\end{equation}
with the boundary condition
\begin{equation}
\mu_\alpha^{(M)}(\tau)\equiv0,\qquad|\alpha|>M,
\label{eq:truncated-untied-boundary}
\end{equation}
and initial condition
\begin{equation}
\mu_\alpha^{(M)}(0)=
\begin{cases}
\bar\mu_\alpha(0),&|\alpha|\le M,\\
0,&|\alpha|>M.
\end{cases}
\label{eq:truncated-untied-initial}
\end{equation}
Notice that \(Y_M\) is invariant under \(V_{\rm ext}\). Hence \eqref{eq:truncated-untied}--\eqref{eq:truncated-untied-initial}
is exactly the restriction of the extended infinite-dimensional equation to the invariant finite-dimensional subspace \(Y_M\). The following corollary proves that the truncated system can approximate the original system arbitrarily well.

\begin{corollary}
\label{cor:truncation-untied}
Under the conditions of Theorem \ref{theo:untied-SGD}, fix \(T>0\). Let \(R=C_*^2\), with \(C_*\) the constant in Lemma \ref{lemma:spectral_bounds_untied}. There exist constants $R<s_-<s_+<\infty$, $C_T<\infty$, and \(M_0(T)<\infty\), all independent of \(M\), such that for every \(M\ge M_0(T)\), the truncated system
\eqref{eq:truncated-untied}--\eqref{eq:truncated-untied-initial} has a unique solution on \([0,T]\). Moreover, for every finite word
\(\alpha\),
\begin{equation}
\sup_{0\le\tau\le T}\left|\mu_\alpha^{(M)}(\tau)-\bar\mu_\alpha(\tau)\right|\le
C_T\left(\frac{R}{s_-}\right)^{M+1}s_+^{|\alpha|}.
\label{eq:untied-truncation-coordinate-bound}
\end{equation}
\end{corollary}

\begin{proof}
For \(M\ge2\), the projection \(\Pi_M\) leaves all first- and second-order moments unchanged:
$q(\Pi_M\bar\mu(0))=q(\bar\mu(0))$ and
$t_k^A(\Pi_M\bar\mu(0))+t_k^B(\Pi_M\bar\mu(0))=\bar t_k^A(0)+\bar t_k^B(0)>0$. Hence \(\Pi_M\bar\mu(0)\) belongs to $\mathcal{U}_s$.

By
\eqref{eq:untied-truncation-true-moment-bound},
\begin{align}
\left\|\Pi_M\bar\mu(0)-\bar\mu(0)\right\|_{s_-}=\sup_{w>M}\sup_{|\alpha|=w}\frac{|\bar\mu_\alpha(0)|}{s_-^w}\le\sup_{w>M}\left(\frac{R}{s_-}\right)^w=\left(\frac{R}{s_-}\right)^{M+1}.
\label{eq:untied-truncation-initial-error}
\end{align}
Since \(R<s_-\), the right-hand side converges exponentially to zero. Therefore, for all sufficiently large \(M\),
\begin{equation}
\left\|\Pi_M\bar\mu(0)-\bar\mu(0)\right\|_{s_-}\le\varepsilon_T,
\end{equation}
where \(\varepsilon_T\) is defined in Lemma \ref{lemma:untied-truncation-stability}.

Apply Lemma \ref{lemma:untied-truncation-stability} with $\nu_0=\Pi_M\bar\mu(0)$.
The resulting solution of the extended ODE exists uniquely on \([0,T]\), and
\begin{equation}
\sup_{0\le\tau\le T}\|\nu(\tau)-\bar\mu(\tau)\|_{s_+}\le C_T\left(\frac{R}{s_-}\right)^{M+1}.
\label{eq:untied-truncation-Xs-bound}
\end{equation}
Since \(\nu_0\in Y_M\) and \(Y_M\) is invariant, we have $\nu(\tau)\in Y_M$ for $0\le\tau\le T$. Consequently, \(\nu\) is precisely the solution
\(\mu^{(M)}\) of the truncated system
\eqref{eq:truncated-untied}--\eqref{eq:truncated-untied-initial}.

Finally, for every word \(\alpha\),
\begin{align}
\left|\mu_\alpha^{(M)}(\tau)-\bar\mu_\alpha(\tau)
\right|\le s_+^{|\alpha|}
\left\|\mu^{(M)}(\tau)-\bar\mu(\tau)\right\|_{s_+}.
\end{align}
Combining this with
\eqref{eq:untied-truncation-Xs-bound} proves
\eqref{eq:untied-truncation-coordinate-bound}.
\end{proof}

The slaved mean variables can be reconstructed from the truncated moments by
\begin{equation}
m_{\mathcal K}^{(M)}(\tau):=m_{\mathcal K}^\star
\left(\mathscr P(q(\mu^{(M)}(\tau)))\right),
\qquad
m_{\setminus\mathcal K}^{(M)}(\tau)=m_{\setminus\mathcal K}(0).
\label{eq:untied-truncation-slaved-mean}
\end{equation}
Since \(m_{\mathcal K}^\star\) is Lipschitz on the compact covariance neighborhood in \eqref{eq:untied-truncation-q-tube}, and since \(\mathscr P\) is non-expansive, \eqref{eq:untied-truncation-Xs-bound} also gives
\begin{equation}
\sup_{0\le\tau\le T}\left\|m_{\mathcal K}^{(M)}(\tau)-\bar m_{\mathcal K}(\tau)\right\|\le
C_T'\left(\frac{R}{s_-}\right)^{M+1}
\label{eq:untied-truncation-m-bound}
\end{equation}
for some \(C_T'<\infty\).

\subsection{Proof of Corollary \ref{cor:untied_weak_recovery}}
\label{app:proof_untied_weak_recovery}
We first prove part (i). By Theorem \ref{theo:untied-SGD}, the fast dynamics move the mean variable to $m_1^\star(\bar q^{(1)}(0),\bar q^{(2)}(0))$, while the normalized moments remain asymptotically frozen.
The slow trajectory therefore starts from 
$\bar{\mathbf q}_{12}(0)=0$ with the mean variables \(\bar m^\star\).

Set $D_1(0):=\bar t_1^A(0)+\bar t_1^B(0)>0$ and write
\begin{equation}
a_{12}^{(r)}:=\left[\nabla_{q^{(r)}}\Phi\left(\bar m^\star,\bar q^{(1)}(0),\bar q^{(2)}(0)\right)\right]_{12},
\qquad r=1,2.
\end{equation}
Thus $g_{12}^\star=(a_{12}^{(1)},a_{12}^{(2)})^T$.

Consider first $\bar q_{12}^{(1)}=\bar\mu_{(M_1,M_2^T)}$.
Since only the first index is trainable, the slow dynamics
\eqref{eq:slow_tau_mu} gives
\begin{align}
\frac{d}{d\tau}\bar q_{12}^{(1)}=-\bar\beta_1\bar\mu_{(A_1+B_1,M_2^T)}-2\sum_{l=1}^2\sum_{r=1}^2a_{1l}^{(r)}\bar\mu_{\left(M_l^{(r)}B_1+A_1M_l^{(r)},M_2^T\right)}-2\gamma\bar q_{12}^{(1)}.
\label{eq:untied-q12-first-dynamics}
\end{align}
At \(\tau=0\), Assumption \ref{assum:untied_initialization}
implies freeness of the student and teacher algebras. Moreover, $\lim_{d\to\infty}\frac1d\Tr M_2=0$. Hence
\begin{equation}
\bar\mu_{(A_1+B_1,M_2^T)}(0)=0,
\end{equation}
and all the \(l=1\) terms in
\eqref{eq:untied-q12-first-dynamics} vanish.

For the \(l=2\) terms, asymptotic freeness gives
\begin{align}
\bar\mu_{(M_2B_1+A_1M_2,M_2^T)}(0)&=D_1(0)\bar q_{22}^{(1)},
\\
\bar\mu_{(M_2^TB_1+A_1M_2^T,M_2^T)}(0)&=D_1(0)\bar q_{22}^{(2)}.
\end{align}
Therefore
\begin{equation}
\left.\frac{d}{d\tau}\bar q_{12}^{(1)}\right|_{\tau=0}=-2D_1(0)\left(a_{12}^{(1)}\bar q_{22}^{(1)}+a_{12}^{(2)}\bar q_{22}^{(2)}\right).
\label{eq:untied-q12-first-initial}
\end{equation}
Similarly, applying \eqref{eq:slow_tau_mu} to
$\bar q_{12}^{(2)}=\bar\mu_{(M_1,M_2)}$
gives
\begin{equation}
\left.\frac{d}{d\tau}\bar q_{12}^{(2)}\right|_{\tau=0}=-2D_1(0)\left(a_{12}^{(1)}\bar q_{22}^{(2)}+a_{12}^{(2)}\bar q_{22}^{(1)}\right).
\label{eq:untied-q12-second-initial}
\end{equation}
Combining \eqref{eq:untied-q12-first-initial} and
\eqref{eq:untied-q12-second-initial},
\begin{equation}
\left.\frac{d}{d\tau}\bar{\mathbf q}_{12}(\tau)
\right|_{\tau=0}=-2D_1(0)\mathsf Q_2g_{12}^\star.
\label{eq:untied-q12-vector-initial}
\end{equation}
By \eqref{eq:untied-weak-recovery-condition}, the right-hand side is nonzero. Define
\begin{equation}
v_0:=\left\|\left.\frac{d}{d\tau}\bar{\mathbf q}_{12}(\tau)\right|_{\tau=0}\right\|_\infty>0.
\end{equation}
The limiting slow trajectory is continuously differentiable. Hence, there exists an index \(r_0\in\{1,2\}\) and $T_0>0$ such that
\begin{equation}
\left|\frac{d}{d\tau}\bar q_{12}^{(r_0)}(\tau)\right|\ge\frac{v_0}{2},\qquad0\le\tau\le T_0.
\end{equation}
In particular, by continuity \(\frac{d}{d\tau}\bar q_{12}^{(r_0)}(\tau)\) has the same sign as \(\frac{d}{d\tau}\bar q_{12}^{(r_0)}(0)\) throughout \([0,T_0]\). Since \(\bar q_{12}^{(r_0)}(0)=0\),
\begin{equation}
\left\|\bar{\mathbf q}_{12}(T_0)\right\|_\infty\ge\frac{v_0T_0}{2}.
\end{equation}
Set $c:=\frac{v_0T_0}{4}$ and $T_+:=T_0$. Then
\begin{equation}
\left\|\bar{\mathbf q}_{12}(T_+)\right\|_\infty\ge2c.
\label{eq:untied-recovery-upper-limit}
\end{equation}
On the other hand, continuity and
\(\bar{\mathbf q}_{12}(0)=0\) imply that there exists
\(T_-\in(0,T_+)\) such that
\begin{equation}
\sup_{0\le\tau\le T_-}\left\|\bar{\mathbf q}_{12}(\tau)\right\|_\infty<\frac c2.
\label{eq:untied-recovery-lower-limit}
\end{equation}
Let
\(\widetilde{\mathbf q}_{12}^{(d)}(\tau)\)
denote the interpolation of the empirical overlap.
Theorem \ref{theo:untied-SGD} gives
\begin{equation}
\sup_{\tau\in[0,T_+]}
\left\|\widetilde{\mathbf q}_{12}^{(d)}(\tau)-\bar{\mathbf q}_{12}(\tau)\right\|_\infty\xrightarrow{\mathbb P}0.
\label{eq:untied-q12-uniform-convergence}
\end{equation}
Using
\eqref{eq:untied-recovery-upper-limit} and
\eqref{eq:untied-recovery-lower-limit}, with probability tending to one,
\begin{equation}
\left\lfloor\frac{T_-d}{\alpha_d}\right\rfloor<N_{\mathrm{wr}}^{(d)}(c)\le\left\lceil\frac{T_+d}{\alpha_d}\right\rceil.
\end{equation}
This proves part (i).

We next prove part (ii).
Let \(\bar\mu(0)\in\mathcal M_{\mathrm{free}}\) be the limiting initial moment sequence. Consider any state
\(\mu\in\mathcal M_{\mathrm{free}}\).
In the slow vector field
\eqref{eq:slow_tau_mu}, the only substitutions which explicitly insert a teacher letter into the evolution of a trainable student letter are $\mathcal T_{i,2,r}^{\mathrm{diff}}$. Their coefficients are
$\left[\nabla_{q^{(r)}}\Phi\left(m^\star(q),q^{(1)},q^{(2)}\right)\right]_{12}$. By assumption \eqref{eq:untied-free-trap-condition}, these coefficients vanish on \(\mathcal M_{\mathrm{free}}\).
The same observation applies to the \(l=2\) contribution to \(\psi_1\), and hence to the effective bias \(\beta_1\). Consequently, when the state belongs to \(\mathcal M_{\mathrm{free}}\), the limiting slow dynamics do not generate mixed student–teacher moments. Thus, \(\mathcal M_{\mathrm{free}}\) is invariant under the limiting slow flow. By uniqueness of the admissible solution in Theorem \ref{theo:untied-SGD}, $\bar\mu(\tau)\in\mathcal M_{\mathrm{free}}$ for $0\le\tau\le T$. In particular, $\bar q_{12}^{(1)}(\tau)=\bar q_{12}^{(2)}(\tau)=0$ for $0\le\tau\le T$.

Applying again the uniform convergence in Theorem \ref{theo:untied-SGD}, for every fixed \(T>0\),
\begin{equation}
\sup_{\tau\in[0,T]}\left\|\widetilde{\mathbf q}_{12}^{(d)}(\tau)\right\|_\infty\xrightarrow{\mathbb P}0.
\end{equation}
Hence, for every fixed \(c>0\),
\begin{equation}
\mathbb P\left(N_{\mathrm{wr}}^{(d)}(c)\le\left\lfloor\frac{Td}{\alpha_d}\right\rfloor
\right)\longrightarrow0.
\end{equation}
This finishes the proof.

\subsection{Convergence rate of strong recovery}
\label{app:untied-strong}
We next study the convergence rate after weak recovery. Recall the first-order moments $t_k^A:=\frac1d\Tr[A_k]$,
$t_k^B:=\frac1d\Tr[B_k]$ and the covariance parameters
$q_{kl}^{(1)}:=\frac1d\Tr[M_kM_l^T]$, $q_{kl}^{(2)}:=\frac1d\Tr[M_kM_l]$. The limiting regularized population loss is
\begin{equation}
\Psi(m,t^A,t^B,q^{(1)},q^{(2)}):=\Phi(m,q^{(1)},q^{(2)})+\frac{\gamma}{2}\sum_{k\in\mathcal K}\left(t_k^A+t_k^B\right).
\label{eq:untied-regularized-potential-strong}
\end{equation}
Along the slow trajectory of Theorem \ref{theo:untied-SGD}, the mean variables are slaved to the covariance,
\begin{equation}
\bar m_{\mathcal K}(\tau)=m_{\mathcal K}^\star
\left(\bar q^{(1)}(\tau),\bar q^{(2)}(\tau)\right).
\end{equation}
Let \(\mathcal Q_{\rm untied}\) denote the asymptotically realizable domain of $\left(m,t^A,t^B,q^{(1)},q^{(2)}\right)$, namely the closure of all limiting order parameters generated by
sequences of matrices $M_k=U_kV_k^T,A_k=U_kU_k^T,B_k=V_kV_k^T$ subject to the fixed teacher components for \(k\notin\mathcal K\). Define
\begin{equation}
\Psi^\star:=\inf_{(m,t^A,t^B,q^{(1)},q^{(2)})\in\mathcal Q_{\rm untied}}\Psi(m,t^A,t^B,q^{(1)},q^{(2)}).
\label{eq:untied-global-optimum-strong}
\end{equation}
We refer to convergence to the global minimizers of \(\Psi\) as strong recovery.

Let $\bar\vartheta(\tau):=\left(\bar t^A(\tau),\bar t^B(\tau),\bar q^{(1)}(\tau),\bar q^{(2)}(\tau)\right)$
denote the corresponding components of the limiting slow trajectory. As in the tied setting, we impose an effective
Polyak--\L ojasiewicz condition along this trajectory.

\begin{assumption}
\label{assum:untied_effective_pl}
There exist constants \(T_0>0\) and \(c>0\) such that, for every \(\tau\ge T_0\),
\begin{equation}
-\frac{d}{d\tau}\Psi\left(\bar m(\tau),\bar t^A(\tau),\bar t^B(\tau),\bar q^{(1)}(\tau),\bar q^{(2)}(\tau)\right)\ge2c\left[\Psi\left(\bar m(\tau),\bar t^A(\tau),\bar t^B(\tau),\bar q^{(1)}(\tau),\bar q^{(2)}(\tau)\right)-\Psi^\star\right].
\label{eq:untied-effective-pl}
\end{equation}
\end{assumption}

For finite \(d\), define the regularized population loss
\begin{align}
\Psi_d(U,V):=\mathbb E_x\left[\mathcal L\left(\left\{\frac1{\sqrt d}x_i^TU_kV_k^Tx_j\right\}_{i,j,k=1}^{L,L,K}
\right)\right]+\frac{\gamma}{2d}\sum_{k\in\mathcal K}
\left(\|U_k\|_F^2+\|V_k\|_F^2\right).
\label{eq:untied-finite-d-population-risk}
\end{align}

\begin{corollary}
\label{cor:untied-strong-recovery}
Suppose that the conditions of Theorem \ref{theo:untied-SGD} and Assumption \ref{assum:untied_effective_pl} hold. Let $n_0:=\left\lceil\frac{T_0d}{\alpha_d}\right\rceil$.
Then, for every fixed target accuracy \(\epsilon>0\), there exists
\begin{equation}
n_\epsilon=n_0+O\left(\frac{d}{\alpha_d}\log_+\frac1\epsilon\right)
\label{eq:untied-strong-recovery-sample-complexity}
\end{equation}
such that
\begin{equation}
\lim_{d\to\infty}\mathbb P\left(\Psi_d(U^{(n_\epsilon)},V^{(n_\epsilon)})-\Psi^\star>\epsilon\right)=0,
\label{eq:untied-strong-recovery-risk-convergence}
\end{equation}
where $\log_+(x):=\max\{0,\log x\}$. In particular, if \(T_0=\Theta(1)\), reaching an \(\epsilon\)-suboptimal population loss requires at most $O\left(\frac{d}{\alpha_d}\log\frac1\epsilon\right)$ additional samples after entering the strong-recovery regime.
For \(\alpha_d=\Theta((d\log d)^{-1})\), this becomes
$O\left(d^2\log d\,\log\frac1\epsilon\right)$.
\end{corollary}

\begin{proof}
Define the limiting excess regularized population loss
\begin{align}
\Delta(\tau):=\Psi\left(\bar m(\tau),\bar t^A(\tau),\bar t^B(\tau),\bar q^{(1)}(\tau),\bar q^{(2)}(\tau)\right)-\Psi^\star\ge0.
\label{eq:untied-strong-gap}
\end{align}
By Assumption \ref{assum:untied_effective_pl}, for every
\(\tau\ge T_0\),
\begin{equation}
\frac{d}{d\tau}\Delta(\tau)\le-2c\Delta(\tau).
\end{equation}
Hence Gr\"onwall's inequality gives
\begin{equation}
\Delta(\tau)\le\Delta(T_0)e^{-2c(\tau-T_0)},
\qquad
\tau\ge T_0.
\label{eq:untied-deterministic-exponential-decay}
\end{equation}
Set $\Delta_0:=\Delta(T_0)$. If \(\Delta_0=0\), there is nothing to prove. Otherwise, for a fixed
\(\epsilon>0\), define
\begin{equation}
T_\epsilon:=T_0+\frac1{2c}\log_+\left(\frac{2\Delta_0}{\epsilon}\right).
\label{eq:untied-T-epsilon-strong}
\end{equation}
Then
\begin{equation}
\Psi\left(\bar m(T_\epsilon),\bar t^A(T_\epsilon),\bar t^B(T_\epsilon),\bar q^{(1)}(T_\epsilon),\bar q^{(2)}(T_\epsilon)\right)-\Psi^\star\le\frac{\epsilon}{2}.
\label{eq:untied-deterministic-epsilon-gap}
\end{equation}
We next transfer this bound to the finite-dimensional population loss. For \(\tau\ge0\), define $U^{(d)}(\tau):=U^{\left(\left\lfloor\tau d/\alpha_d\right\rfloor
\right)}$ and $V^{(d)}(\tau):=V^{\left(\left\lfloor\tau d/\alpha_d\right\rfloor\right)}$. For every fixed $0<T_1<T_2<\infty$, Theorem \ref{theo:untied-SGD}, together with the uniform spectral bounds in Lemma \ref{lemma:spectral_bounds_untied}, implies
\begin{align}
\sup_{\tau\in[T_1,T_2]}\Big|\Psi_d\left(U^{(d)}(\tau),V^{(d)}(\tau)\right)-\Psi\left(
\bar m(\tau),\bar t^A(\tau),\bar t^B(\tau),\bar q^{(1)}(\tau),\bar q^{(2)}(\tau)\right)\Big|\xrightarrow{\mathbb P}0,
\label{eq:untied-population-risk-uniform-convergence}
\end{align}
where we use Theorem \ref{theo:general}.

Since \(T_0>0\) and \(T_\epsilon<\infty\), we may apply
\eqref{eq:untied-population-risk-uniform-convergence} on
\([T_0,T_\epsilon]\). Consequently,
\begin{align}
\mathbb P\Bigg(\Big|\Psi_d\left(U^{(d)}(T_\epsilon),V^{(d)}(T_\epsilon)\right)-\Psi\left(\bar m(T_\epsilon),\bar t^A(T_\epsilon),\bar t^B(T_\epsilon),\bar q^{(1)}(T_\epsilon),\bar q^{(2)}(T_\epsilon)\right)\Big|>\frac{\epsilon}{2}\Bigg)\longrightarrow0.
\label{eq:untied-risk-transfer-epsilon}
\end{align}
Set $n_\epsilon:=\left\lceil\frac{T_\epsilon d}{\alpha_d}\right\rceil$. Using
\eqref{eq:untied-deterministic-epsilon-gap} and
\eqref{eq:untied-risk-transfer-epsilon}, we obtain
\begin{equation}
\mathbb P\left(\Psi_d(U^{(n_\epsilon)},V^{(n_\epsilon)})-\Psi^\star>\epsilon\right)\longrightarrow0.
\end{equation}
Finally, by
\eqref{eq:untied-T-epsilon-strong},
\begin{equation}
n_\epsilon-n_0=O\left(\frac{d}{\alpha_d}\log_+\frac1\epsilon\right),
\end{equation}
which proves the claimed sample complexity.
\end{proof}

\subsection{Proof of Lemma \ref{lemma:quantitative-gaussian-approximation}}
\label{app:proof-gaussian-approximation}
We concatenate the indices \((i,j,k)\) into a single index
\(a\in\{1,\ldots,P\}\), where \(P=L^2K=\Theta(1)\), and write
$G^{(d)}=(G_a^{(d)})_{a=1}^P$.

As in the proof of Theorem \ref{theo:general}, concatenate the input
vectors into $X=[x_1^T,\ldots,x_L^T]^T
\sim\mathcal N(0,\Sigma)$, where $\Sigma=\mathcal C\otimes I_d$.
For each \(a=(i,j,k)\), there exists a symmetric
\(Ld\times Ld\) matrix \(A_a\) such that $G_a^{(d)}=X^TA_aX$. Writing \(X=\Sigma^{1/2}z\), with \(z\sim\mathcal N(0,I_{Ld})\), and setting
$\widetilde A_a:=\Sigma^{1/2}A_a\Sigma^{1/2}$, we obtain a Gaussian quadratic form
\begin{equation}
G_a^{(d)}=z^T\widetilde A_a z.
\end{equation}

By the same norm estimates as in the proof of Theorem \ref{theo:general}, the assumptions \(\|M_k\|_{\mathrm{op}}\le C_*\) imply, uniformly over \(a\),
\begin{equation}
\|\widetilde A_a\|_{\mathrm{op}}\le\frac{C}{\sqrt d},
\qquad\|\widetilde A_a\|_F\le C,
\label{eq:quantitative-A-bounds}
\end{equation}
where \(C\) depends only on
\(C_*,\mathcal C,L,K\). Moreover, $\max_k\frac1{\sqrt d}|\Tr(M_k)|\le R_*$ implies that the means $\mu_a^{(d)}:=\mathbb E G_a^{(d)}=\Tr(\widetilde A_a)$ are uniformly bounded.

Let
\begin{equation}
W_a^{(d)}:=G_a^{(d)}-\mu_a^{(d)}=z^T\widetilde A_a z-\Tr(\widetilde A_a),
\end{equation}
and let \(Z^{(d)}\) be a centered Gaussian vector having the same
covariance as \(W^{(d)}\). Thus $G_{\mathrm G}^{(d)}=\mu^{(d)}+Z^{(d)}$.

As in Theorem \ref{theo:general}, Wick's formula gives
\begin{equation}
\Xi_{ab}^{(d)}:=\operatorname{Cov}(W_a^{(d)},W_b^{(d)})=2\Tr(\widetilde A_a\widetilde A_b).
\label{eq:quantitative-covariance}
\end{equation}
In particular, \(\Xi^{(d)}\) is uniformly bounded.

We now compare \(W^{(d)}\) with \(Z^{(d)}\) through Gaussian
interpolation. Let \(Z^{(d)}\) be independent of \(z\), and define
\begin{equation}
Y_t:=\mu^{(d)}+\sqrt t\,W^{(d)}+\sqrt{1-t}\,Z^{(d)},
\qquad t\in[0,1],
\end{equation}
and $F(t):=\mathbb E[f(Y_t)]$. Then $F(1)=\mathbb E[f(G^{(d)})]$ and $F(0)=\mathbb E[f(G_{\mathrm G}^{(d)})]$.

For the centered quadratic forms \(W_a^{(d)}\), define $\Gamma_{ab}:=2z^T\widetilde A_a\widetilde A_bz$. The Gaussian integration by parts gives
\begin{equation}
\mathbb E\left[W_a^{(d)}\partial_a f(Y_t)\right]=
\sqrt t\sum_{b=1}^P\mathbb E\left[\Gamma_{ab}\,\partial_{ab}^2f(Y_t)\right].
\label{eq:quadratic-IBP}
\end{equation}
On the other hand, Gaussian integration by parts for \(Z^{(d)}\)
gives
\begin{equation}
\mathbb E\left[Z_a^{(d)}\partial_a f(Y_t)\right]=\sqrt{1-t}
\sum_{b=1}^P\Xi_{ab}^{(d)}\mathbb E\left[\partial_{ab}^2f(Y_t)\right].
\label{eq:Gaussian-IBP-target}
\end{equation}
Differentiating \(F(t)\) and combining
\eqref{eq:quadratic-IBP}--\eqref{eq:Gaussian-IBP-target} yields
\begin{equation}
F'(t)=\frac12\sum_{a,b=1}^P\mathbb E\left[\left(\Gamma_{ab}-\Xi_{ab}^{(d)}\right)\partial_{ab}^2f(Y_t)\right].
\label{eq:Gaussian-interpolation-derivative}
\end{equation}

It remains to control the fluctuation of \(\Gamma_{ab}\).
By \eqref{eq:quantitative-covariance}, $\mathbb E[\Gamma_{ab}]=\Xi_{ab}^{(d)}$.
Since \(\widetilde A_a\) and \(\widetilde A_b\) are symmetric, we have
\begin{align}
\mathbb E\left|
\Gamma_{ab}-\Xi_{ab}^{(d)}
\right|^2&\le C\|\widetilde A_a\widetilde A_b\|_F^2\le C\|\widetilde A_a\|_{\mathrm{op}}^2\|\widetilde A_b\|_F^2\le\frac{C}{d},
\label{eq:Gamma-fluctuation}
\end{align}
where we used \eqref{eq:quantitative-A-bounds} in the last step.

As in the final part of the proof of Theorem
\ref{theo:general}, for every fixed \(p<\infty\),
\begin{equation}
\sup_d\mathbb E\|W^{(d)}\|^p<\infty.
\label{eq:quantitative-W-moments}
\end{equation}
The Gaussian vector \(Z^{(d)}\) has uniformly bounded covariance, so
its moments are uniformly bounded as well. Since the means
\(\mu^{(d)}\) are uniformly bounded, it follows that
\begin{equation}
\sup_{d}\sup_{t\in[0,1]}\mathbb E\|Y_t\|^p<\infty
\label{eq:Y-moment-bound}
\end{equation}
for every fixed \(p\).
Since the second derivatives of \(f\) grow polynomially, \eqref{eq:Y-moment-bound} therefore implies
\begin{equation}
\sup_{d}\sup_{t\in[0,1]}\mathbb E\left|\partial_{ab}^2f(Y_t)\right|^2\le C_f.
\label{eq:f-second-derivative-bound}
\end{equation}
Applying Cauchy--Schwarz to
\eqref{eq:Gaussian-interpolation-derivative}, and using
\eqref{eq:Gamma-fluctuation} and
\eqref{eq:f-second-derivative-bound}, gives
\begin{equation}
|F'(t)|\le\frac{C}{\sqrt d},\qquad t\in[0,1],
\end{equation}
uniformly over all admissible matrix collections.

Integrating from \(0\) to \(1\), we conclude that
\begin{equation}
\left|\mathbb E[f(G^{(d)})]-\mathbb E[f(G_{\mathrm G}^{(d)})]\right|=|F(1)-F(0)|\le\frac{C}{\sqrt d}.
\end{equation}
This proves the lemma.

\section{Declaration of LLM usage} \label{app:LLMusage}
We used LLMs in the following stages of the preparation of this work: editing (e.g., grammar, spelling, word choice), drafting sections of the paper, facilitating or running experiments, visualizing results for submission. All LLM-generated material used in this manuscript was manually verified by the authors to ensure that it performs as intended, in particular with regard to the proof and the numerical experiments.
\end{document}